%% file: main.tex
\documentclass{article}

\input{math_commands.tex}

\usepackage{microtype}
\usepackage{graphicx}
\usepackage{subfigure}
\usepackage{booktabs} % for professional tables
\usepackage{enumitem}
\usepackage{setspace}
\usepackage{siunitx}
\usepackage{graphicx}
\usepackage{wrapfig}
\usepackage{subfigure}
\usepackage{multirow}
\usepackage{bbm}
\usepackage{rotating}
\usepackage{threeparttable}

\usepackage{hyperref}

\usepackage[accepted]{icml2024}

\usepackage{amsmath}
\usepackage{amssymb}
\usepackage{mathtools}
\usepackage{amsthm}
\usepackage{commath}

\usepackage[capitalize]{cleveref}

\theoremstyle{plain}

\newtheorem{theorem}{Theorem}[section]

\newtheorem{corollary}[theorem]{Corollary}
\theoremstyle{definition}

\newtheorem{hypothesis}[theorem]{Hypothesis}
\theoremstyle{remark}
\newtheorem{remark}[theorem]{Remark}
\makeatletter
\renewenvironment{proof}[1][\proofname]{\par
	\vspace{-\topsep}% remove the space after the theorem
	\pushQED{\qed}%
	\normalfont
	\topsep0pt \partopsep0pt % no space before
	\trivlist
	\item[\hskip\labelsep
	\itshape
	#1\@addpunct{.}]\ignorespaces
}{%
	\popQED\endtrivlist\@endpefalse
	\addvspace{6pt plus 6pt} % some space after
}
\makeatother

\usepackage[textsize=tiny]{todonotes}

\allowdisplaybreaks[4]

\icmltitlerunning{How to Leverage Diverse Demonstrations in Offline Imitation Learning}

\begin{document}

\twocolumn[
\icmltitle{How to Leverage Diverse Demonstrations in Offline Imitation Learning}

% It is OKAY to include author information, even for blind
% submissions: the style file will automatically remove it for you
% unless you've provided the [accepted] option to the icml2024
% package.

% List of affiliations: The first argument should be a (short)
% identifier you will use later to specify author affiliations
% Academic affiliations should list Department, University, City, Region, Country
% Industry affiliations should list Company, City, Region, Country

% You can specify symbols, otherwise they are numbered in order.
% Ideally, you should not use this facility. Affiliations will be numbered
% in order of appearance and this is the preferred way.
\icmlsetsymbol{equal}{*}

\begin{icmlauthorlist}
\icmlauthor{Sheng Yue}{thu}
\icmlauthor{Jiani Liu}{thu}
\icmlauthor{Xingyuan Hua}{thu}
\icmlauthor{Ju Ren}{thu,zgc}
\icmlauthor{Sen Lin}{hu}
\icmlauthor{Junshan Zhang}{ucd}
\icmlauthor{Yaoxue Zhang}{thu,zgc}
%\icmlauthor{}{sch}
% \icmlauthor{Firstname8 Lastname8}{sch}
% \icmlauthor{Firstname8 Lastname8}{yyy,comp}
%\icmlauthor{}{sch}
%\icmlauthor{}{sch}
\end{icmlauthorlist}

\icmlaffiliation{thu}{Department of Computer Science and Technology, Tsinghua University, Beijing, China}
\icmlaffiliation{zgc}{Zhongguancun Laboratory, Beijing, China}
\icmlaffiliation{hu}{Department of Computer Science, University of Houston, Texas, US}
\icmlaffiliation{ucd}{Department of Electrical and Computer Engineering, University of California, Davis, US}

\icmlcorrespondingauthor{Ju Ren}{renju@tsinghua.edu.cn}

% You may provide any keywords that you
% find helpful for describing your paper; these are used to populate
% the "keywords" metadata in the PDF but will not be shown in the document
\icmlkeywords{imitation learning, imperfect demonstrations}

\vskip 0.3in
]

% this must go after the closing bracket ] following \twocolumn[ ...

% This command actually creates the footnote in the first column
% listing the affiliations and the copyright notice.
% The command takes one argument, which is text to display at the start of the footnote.
% The \icmlEqualContribution command is standard text for equal contribution.
% Remove it (just {}) if you do not need this facility.

\printAffiliationsAndNotice{}  % leave blank if no need to mention equal contribution
% \printAffiliationsAndNotice{\icmlEqualContribution} % otherwise use the standard text.

\begin{abstract}

Offline Imitation Learning (IL) with imperfect demonstrations has garnered increasing attention owing to the scarcity of expert data in many real-world domains. A fundamental problem in this scenario is \emph{how to extract positive behaviors from noisy data}. In general, current approaches to the problem select data building on state-action similarity to given expert demonstrations, neglecting precious information in (potentially abundant) \textit{diverse} state-actions that deviate from expert ones. In this paper, we introduce a simple yet effective data selection method that identifies positive behaviors based on their \emph{resultant states} -- a more informative criterion enabling explicit utilization of dynamics information and effective extraction of both expert and beneficial diverse behaviors. Further, we devise a lightweight behavior cloning algorithm capable of leveraging the expert and selected data correctly. In the experiments, we evaluate our method on a suite of complex and high-dimensional offline IL benchmarks, including continuous-control and vision-based tasks. The results demonstrate that our method achieves state-of-the-art performance, outperforming existing methods on \textbf{20/21} benchmarks, typically by \textbf{2-5x}, while maintaining a comparable runtime to Behavior Cloning (\texttt{BC}). %Overall, this work significantly improves the utilization of suboptimal data in IL.

\end{abstract}

\section{Introduction}
\label{sec:introduction}

Offline Imitation Learning (IL) is the study of learning from demonstrations with no reinforcement signals or interaction with the environment. It has been deemed as a promising solution for safety-sensitive domains like healthcare and autonomous driving, where manually formulating a reward function is challenging but historical human demonstrations are readily available~\citep{bojarski2016end}. Conventional offline IL methods, such as Behavior Cloning~(\texttt{BC})~\citep{pomerleau1988alvinn}, often necessitate an expert dataset with sufficient coverage over the
state-action space to combat error compounding \citep{rajaraman2020toward}, which is prohibitively expensive for many real-world applications. Instead, a more realistic scenario might allow for a limited expert dataset, coupled with substantial imperfect demonstrations sampled from unknown policies~\citep{wu2019imitation,xu2022discriminator,li2023imitation}. For example, autonomous vehicle companies may possess modest high-quality data from experienced drivers but can amass a wealth of mixed-quality data from ordinary drivers. Effective utilization of the imperfect demonstrations would significantly enhance the robustness and generalization of offline IL.

A fundamental question raised in this scenario is: \emph{how can we extract good behaviors from noisy data?} To address this question, several prior studies have attempted to explore and imitate the imperfect behaviors that resemble expert ones~\citep{sasaki2021behavioral,xu2022discriminator,li2023imitation}. Nevertheless, due to the scarcity of expert data, such approaches are ill-equipped to harness valuable information in (potentially abundant) \emph{diverse} behaviors that deviate from limited expert demonstrations (see \cref{sec:preliminaries} for details). Of course, a natural solution to incorporate these behaviors is inferring a reward function and labeling all imperfect data, subsequently engaging in an offline Reinforcement Learning (RL) process~\citep{zolna2020offline,chang2022mitigating,yue2023clare,zeng2023demonstrations,cideron2023get}. Unfortunately, it is highly challenging to define and learn meaningful reward functions without environmental interaction. As a consequence, current offline reward learning methods typically rely on complex adversarial optimization using a learned world model. They easily suffer from hyperparameter sensitivity, learning instability, and limited scalability in practical and high-dimensional environments.

\begin{figure}[t]
    \centering
    %\vspace{-.9em}
    \includegraphics[width=.95\columnwidth]{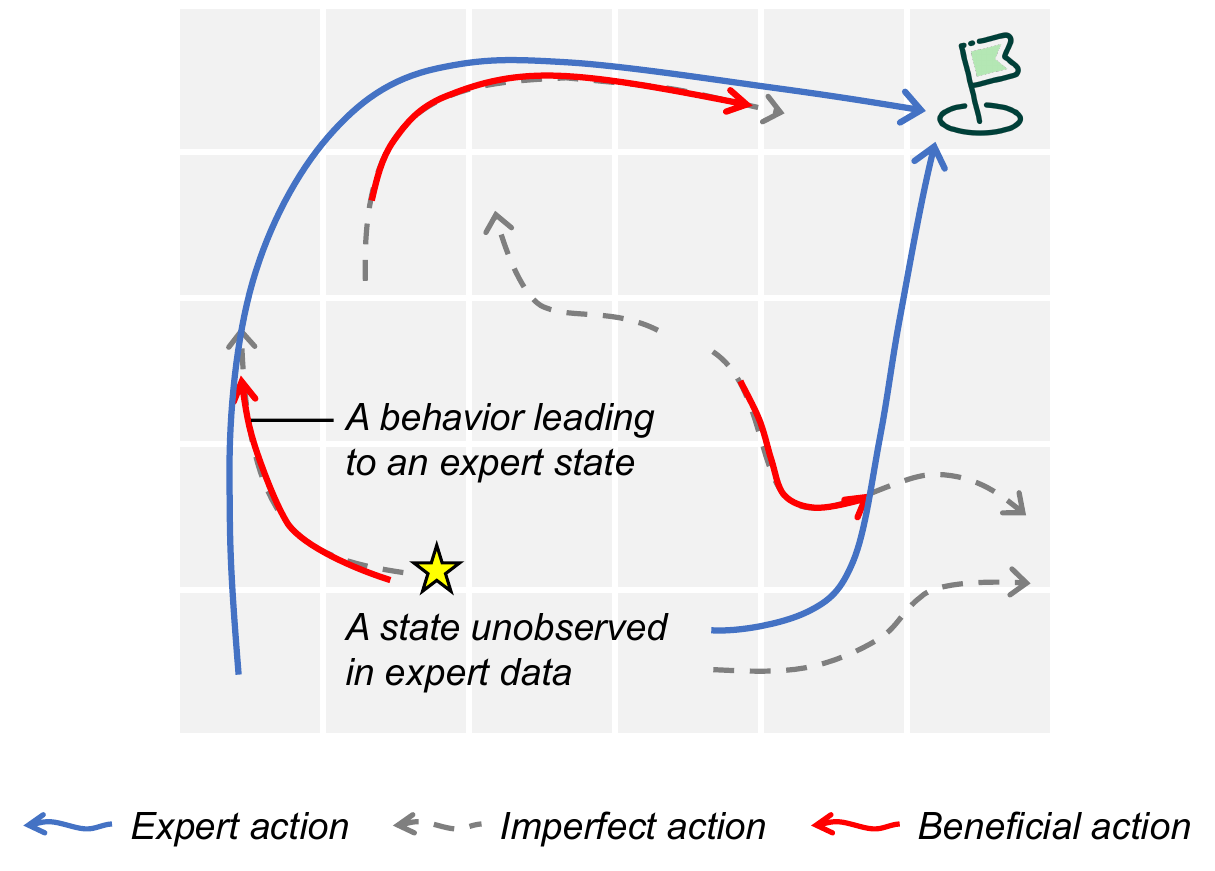}
    % \vspace{-1.0em}
    \vskip -0.1in
    \caption{A cartoon illustration of potential beneficial behaviors in a navigation task, with the goal of reaching the target state (marked by the flag) from arbitrary initial states. With no other prior information, in a state out of the expert observations, a reasonable and often safe choice is to get back to given expert states.}
    \label{fig:motivating}
    % \vskip -0.1in
    % \vspace{-1.5em}
\end{figure}

In this paper, we introduce a simpler data selection principle to fully exploit positive diverse behaviors in imperfect demonstrations without indirect reward learning procedures. Specifically, instead of examining a behavior's similarity to expert demonstrations in and of itself, we assess its value based on whether its \emph{resultant states}, to which environment transitions after performing that behavior, fall within the expert data manifold. In other words, we properly select the state-actions that can lead to expert states, even if they bear no resemblance to expert demonstrations. As depicted in \cref{fig:motivating} and supported by the theoretical results in \cref{sec:data_selection}, the underlying rationale is that when the agent encounters a state unobserved in expert demonstrations, opting to return to the expert states is more sensible than taking a random action. Otherwise, it may persist in making mistakes and remain out-of-expert-distribution for subsequent time steps. Of note, the resultant state is a more informative criterion than the state-action similarity, as it explicitly utilizes the dynamics information, enabling the identification of both expert and beneficial diverse state-actions in noisy data.

Drawing upon this principle, we first train a \emph{state-only} discriminator to distinguish expert and non-expert states in imperfect demonstrations. Leveraging the identified expert states, we appropriately extract their \emph{causal state-actions} and build a complementary training dataset. In light of the suboptimality of the complementary data, we further devise a lightweight weighted behavior cloning algorithm to mitigate the potential interference among behaviors. We term our method \emph{offline Imitation Learning with Imperfect Demonstrations} (\texttt{ILID}) and evaluate it on a suite of offline IL benchmarks, including 14 continuous-control tasks and 7 vision-based tasks. Our method achieves state-of-the-art performance, consistently surpassing existing methods by \textbf{2-5x} while maintaining a comparable runtime to \texttt{BC}. Our main contributions are summarized as follows:
\begin{itemize}[leftmargin=*,topsep=0pt,itemsep=0pt]
    \item We introduce a simple yet effective method that can explicitly exploit the dynamics information and extract beneficial behaviors from imperfect demonstrations;
    \item We devise a lightweight weighted behavior cloning algorithm capable of correctly learning from the extracted behaviors, which can be easily implemented on top of \texttt{BC}; 
    \item We conduct extensive experiments that corroborate the superiority of our method over state-of-the-art baselines in terms of performance and computational cost.
\end{itemize}

\section{Related Work}
\label{sec:related_work}

Offline IL deals with training an agent to mimic the actions of a demonstrator in an entirely offline fashion. The simplest approach to offline IL is \texttt{BC}~\citep{pomerleau1988alvinn} that directly mimics the behavior using supervised learning, but it is prone to covariate shift and inevitably suffers from error compounding, i.e., there is no way for the policy to learn how to recover if it deviates from the expert behavior to a state not seen in the expert demonstrations \citep{rajaraman2020toward}. Considerable research has been devoted to developing new offline IL methods to remedy this problem~\citep{klein2011batch,klein2012inverse,piot2014boosted,herman2016inverse,kostrikov2020imitation,jarrett2020strictly,swamy2021moments,chan2021scalable,garg2021iq,florence2022implicit}. However, since these methods imitate all given demonstrations, they typically require a large amount of clean expert data, which is expensive for many real-world tasks.

% Using the Donsker-Varadhan representation of KL-divergence, \citet{kostrikov2020imitation} develop an off-policy adversarial distribution matching method, namely ValueDICE. 

% several recent works propose dynamics-aware offline IL approaches, e.g., \citet{kostrikov2020imitation,jarrett2020strictly,dadashi2021primal,chang2022mitigating,swaminathan2015batch}. 

Recently, there has been growing interest in exploring how to effectively leverage imperfect data in offline IL~\citep{sasaki2021behavioral,kim2022demodice,xu2022discriminator,yu2022how,li2023imitation}. \citet{sasaki2021behavioral} analyze why the imitation policy trained by \texttt{BC} deteriorates its performance when using noisy demonstrations. They reuse an ensemble of policies learned from the previous iteration as the weight of the original \texttt{BC} objective to extract the expert behaviors. Nevertheless, this requires that expert data occupy the majority proportion of the offline dataset; otherwise, the policy will be misguided to imitate the suboptimal data. \citet{kim2022demodice} retrofit the \texttt{BC} objective with an additional KL-divergence term to regularize the learned policy to stay close to the behavior policy. Albeit with enhanced offline data support, it may fail to achieve satisfactory performance when the imperfect data is highly suboptimal. \citet{xu2022discriminator} cope with this issue by introducing an additional discriminator, the outputs of which serve as the weights of the original \texttt{BC} loss, to imitate demonstrations selectively. Analogously, \citet{li2023imitation} weight the \texttt{BC} objective by the density ratio of empirical expert data and union offline data, implicitly extracting the imperfect behaviors resembling expert ones. Unfortunately, the criterion of state-action similarity neglects the dynamics information and does not suffice to leverage the diverse behaviors in imperfect demonstrations. In offline RL, \citet{yu2022how} propose to utilize unlabeled data by applying zero rewards, but this method necessitates massive labeled offline data. In contrast, this paper focuses on the setting with no access to reward signals.

%, which can take into account the temporal structure and inform what the expert wishes to achieve, rather than simply what they are reacting to. It enables agents to understand and generalize these ``intentions'' when encountering the environments unseen in expert demonstrations and thus makes the approaches more robust \citep{lee2019truly}. 
%learn from both expert demonstrations and a set of diverse data. 
%In addition, it easily lead to the \emph{false negative} issue (classifying successful states or actions as negative ones) when the states in diverse data have a large difference from those in expert data \cite{xu2022discriminator}. 

Offline Inverse Reinforcement Learning (IRL) explicitly learns a reward function from offline datasets, aiming to comprehend and generalize the underlying intentions behind expert actions \citep{lee2019truly}. \citet{zolna2020offline} propose \texttt{ORIL} that constructs a reward function that discriminates expert and exploratory trajectories, followed by an offline RL progress. \citet{chan2021scalable} use a variational method to jointly learn an approximate posterior distribution over the reward and policy. \citet{garg2021iq} propose to learn a soft $Q$-function that implicitly represents both the reward function and policy. \citet{watson2023coherent} develop \texttt{CSIL} that exploits a \texttt{BC} policy to define an estimate of a shaped reward function that can then be used to finetune the policy using online interactions. However, the heteroscedastic parametric reward functions have undefined values beyond the offline data manifold and easily collapse to the reward limits due to the tanh transformation and network extrapolation. The reward extrapolation errors may cause the learned reward functions to incorrectly explain the task and misguide the agent in unseen environments~\citep{yue2023clare,yue2024federated}. To tackle the issue, \citet{chang2022mitigating} introduce a model-based offline IRL algorithm that uses a model inaccuracy estimate to penalize the learned reward function on out-of-distribution state-actions. \citet{yue2023clare} propose to compute a conservative element-wise weight function that implicitly penalizes out-of-distribution behaviors. \citet{zeng2022maximum} propose \texttt{MLIRL} that can recover the reward function, whose corresponding optimal policy maximizes the likelihood of observed expert demonstrations under a learned conservative world model. However, the model-based approaches struggle to scale in high-dimensional environments, and their min-max optimization usually renders training unstable and inefficient.

\section{Background and Challenge}
\label{sec:preliminaries}

In this section, we first provide the necessary preliminaries and then elaborate on the challenges of our problem.

\textbf{Episodic Markov decision process.} Episodic MDP can be specified by $M\doteq\langle\mathcal{S},\mathcal{A},T,{R},H,\mu\rangle$, with state space $\mathcal{S}$, action space $\mathcal{A}$, transition dynamics $T:\mathcal{S}\times\mathcal{A}\rightarrow\mathcal{P}(\mathcal{S})$, reward function ${R}:\mathcal{S}\times\mathcal{A}\rightarrow[0,1]$, horizon $H$, and initial state distribution $\mu:\mathcal{S}\rightarrow\mathcal{P}(\mathcal{S})$, where $\mathcal{P}(\mathcal{S})$ represents the set of distributions over $\mathcal{S}$. A stationary stochastic policy maps states to distributions over actions, $\pi:\mathcal{S}\rightarrow\mathcal{P}(\mathcal{A})$. The value function of $\pi$ is defined as the expected cumulative reward, $V^\pi\doteq\mathbb{E}_\pi[\sum^H_{h=1}R(s_h,a_h)]$, with the expectation taken w.r.t. trajectories generated by rolling out $\pi$ with $M$. The average state visitation and state-action visitation of $\pi$ are denoted as $\rho^{\pi}(s)\doteq\frac{1}{H}\sum^{H}_{h=1}\Pr(s_h=s\mid\pi)$ and $\rho^{\pi}(s,a)\doteq\rho^{\pi}(s)\pi(a|s)$ respectively, where $\Pr(s_h=s\mid\pi)$ represents the probability of visiting state $s$ at step $h$. The objective of RL can be expressed as $\max_{\pi}V^\pi$. 

\textbf{Offline imitation learning}. Offline IL is the setting where the algorithm is neither allowed to interact with the environment nor provided ground-truth rewards. Rather, it has access to an expert dataset and a mix-quality imperfect dataset, collected from unknown expert policy $\pi_e$ and (potentially highly suboptimal) behavior policy $\pi_b$, respectively. We represent the expert and imperfect datasets as $\mathcal{D}_e\doteq\{\tau_i\}^{n_e}_{i=1}$ and $\mathcal{D}_b\doteq\{\tau_i\}^{n_b}_{i=1}$, where $\tau_i\doteq(s_{i,1},a_{i,1},\dots,s_{i,H},a_{i,H})$ denotes a trajectory. 

% Offline IL aims at extracting the best policy from \todo{Change $D_s$ to $D_b$ and define $D_u$}$\mathcal{D}_e$ and $\mathcal{D}_b$ to maximize $V^\pi$, without access to a queryable expert or the environment.

% further querying the expert or interacting with environments.  

\textbf{Behavior cloning.} \texttt{BC} is a classical offline IL algorithm, which seeks to learn an imitation policy using supervised learning~\cite{pomerleau1988alvinn}. The standard objective of \texttt{BC} is to maximize the log-likelihood over expert demonstrations:
%minimize the Kullback-Leibler (KL) divergence between the empirical expert policy and leaned policy:
\begin{align}
    \label{eq:classical_bc}
    \max_{\pi} \mathbb{E}_{(s,a) \sim \mathcal{D}_e}\big[\log(\pi(a|s))\big].
\end{align}
Recent studies consider a more generalized objective~\citep{xu2022discriminator,li2023imitation}, incorporating additional yet imperfect demonstrations:
\begin{align}
    \min_\pi \mathbb{E}_{(s,a)\sim\mathcal{D}_u}\left[f(s,a)\log\pi(a|s)\right]
\end{align}
where $\mathcal{D}_u \doteq \mathcal{D}_e\cup\mathcal{D}_b$ represents the union offline dataset comprised of both expert and imperfect demonstrations, and $f:\mathcal{S}\times\mathcal{A}\rightarrow[0,1]$ is a weighting function aiming to discard low-quality behaviors and only imitate the beneficial ones. For example, \texttt{DWBC}~\citep{xu2022discriminator} pick $f$ as
\begin{align}
    \label{eq:dwbc_weight}
    f(s,a)=
    \begin{cases}
        \alpha - \frac{\eta}{d_\pi(s,a)(1-d_\pi(s,a))},&(s,a)\in\mathcal{D}_e\\
        \frac{1}{1-d_\pi(s,a)},&(s,a)\in\mathcal{D}_b
    \end{cases}
\end{align}
where $\alpha,\eta>0$ are hyperparameters. $d_\pi(s,a)$ is the output of a discriminator that is jointly trained with $\pi$ to distinguish the expert and diverse state-actions:
\begin{align}
    \label{eq:dwbc_discriminator}
    \max_{d_\pi}\;& \mathbb{E}_{\mathcal{D}_e}\left[\log{d}_\pi(s,a)\right] 
    + \frac{1}{\eta}\mathbb{E}_{\mathcal{D}_b}\left[\log (1-{d}_\pi(s,a))\right] \nonumber\\
    &- \mathbb{E}_{\mathcal{D}_e}\left[\log (1-{d}_\pi(s,a))\right].
\end{align}
% where the last term borrows from positive-unlabeled learning to mitigate overfitting. 
\cref{eq:dwbc_discriminator,eq:dwbc_weight} indicate that \texttt{DWBC} assigns high values to $(s,a)\in\mathcal{D}_e$ and low values to $(s,a)\in\mathcal{D}_b\backslash\mathcal{D}_e$. In addition, \texttt{ISWBC}~\citep{li2023imitation} let $f$ denote the importance weight $f(s,a)=\tilde\rho^e(s,a)/\tilde\rho^u(s,a)$ where $\tilde\rho^e$ and $\tilde\rho^u$ are the empirical distributions of $\mathcal{D}_e$ and $\mathcal{D}_u$, respectively. In the same spirit as \citet{xu2022discriminator}, the weight assigns positive values to $(s,a)\in\mathcal{D}_e$ and close-to-zero values to $(s,a)\in\mathcal{D}_b\backslash\mathcal{D}_e$. 

% However, standard BC does not utilize the information in $\mathcal{D}_s$. Due to the limited state coverage of $\mathcal{D}_e$, the learned policy may suffer from severe compounding errors, i.e., the inability for the policy to get back on track if it encounters a state not seen in the expert demonstrations.

\textbf{Challenge.} The above-mentioned weighting functions can extract $(s,a)\in\mathcal{D}_e$ from $\mathcal{D}_b$, (implicitly) filtering out the state-actions in $\mathcal{D}_b\backslash\mathcal{D}_e$. However, the limited \textit{state} coverage of expert data would render these learned policies still brittle to covariate shift due to their inability to get back on track if encountering a state not observed in the expert demonstrations (see \cref{fig:fourrooms} for an illustrative example). Moreover, considering that offline (forward) RL can learn effective policies from highly diverse behavioral data~\citep{fu2020d4rl,rashidinejad2021bridging}, these methods neglect potentially substantial \textit{beneficial} behaviors in $\mathcal{D}_b\backslash\mathcal{D}_e$ that deviate the expert demonstrations. Thus, there is a clear need for new offline IL methods capable of capitalizing on the diverse behaviors of imperfect demonstrations.

\begin{figure*}[ht]
    \centering
    % \vskip -0.15in

    \subfigure[Expert data]{
    \label{fig:expert_data}\includegraphics[width=0.19\textwidth]{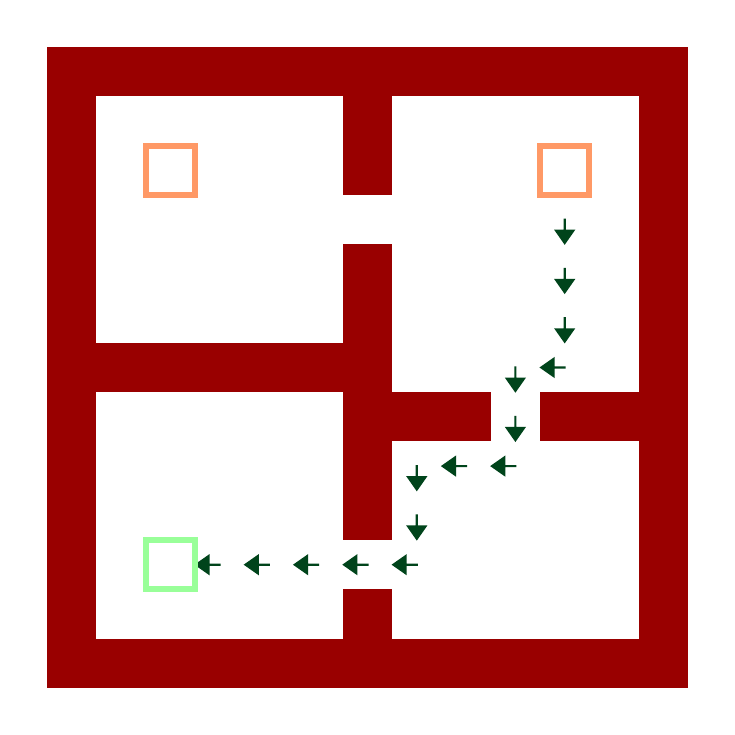}}
    \hspace{-.1in}
    \subfigure[Imperfect data]{
    \label{fig:imperfect_data}\includegraphics[width=0.19\textwidth]{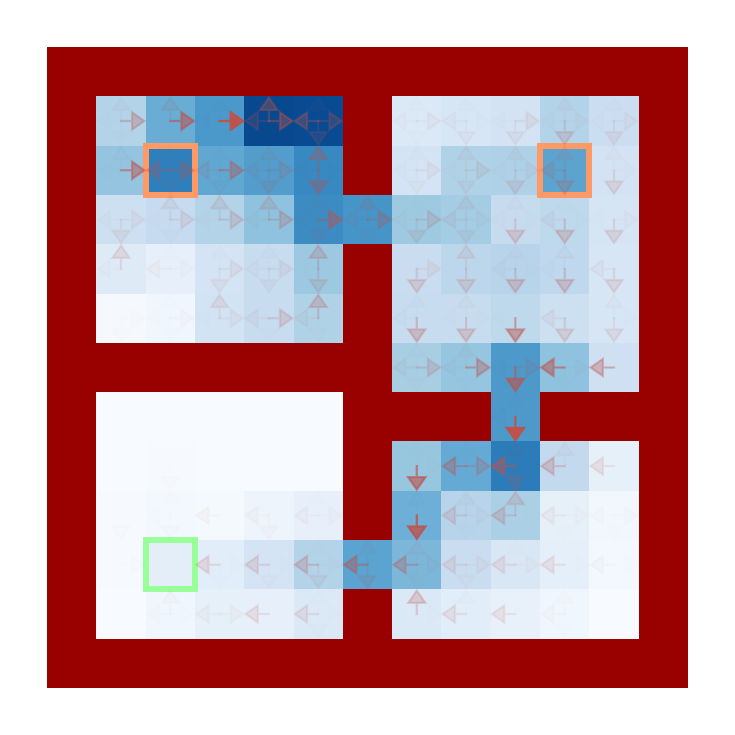}}
    \hspace{-.1in}
    \subfigure[\texttt{ILID} trajectories]{
    \label{fig:ilid_trajs}\includegraphics[width=0.19\textwidth]{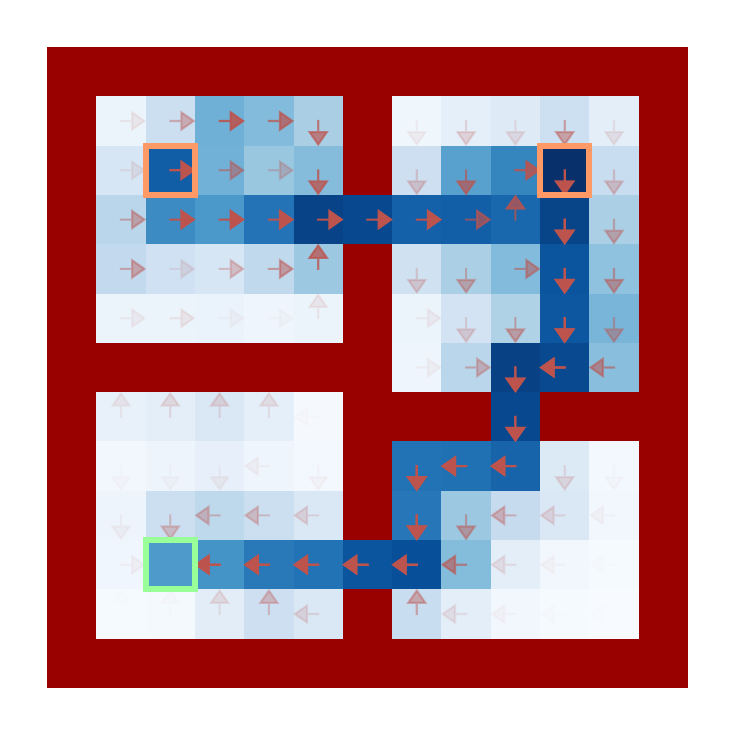}}
    \hspace{-.1in}
    \subfigure[\texttt{DWBC} trajectories]{
    \label{fig:dwbc_trajs}\includegraphics[width=0.19\textwidth]{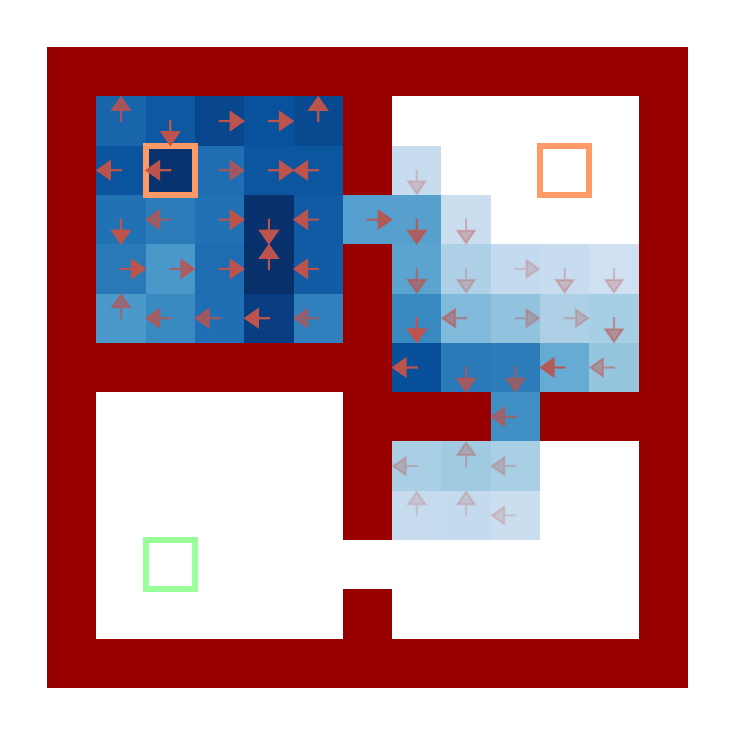}}
    \hspace{-.1in}
    \subfigure[\texttt{ISWBC} trajectories]{
    \label{fig:iswbc_traj}\includegraphics[width=0.19\textwidth]{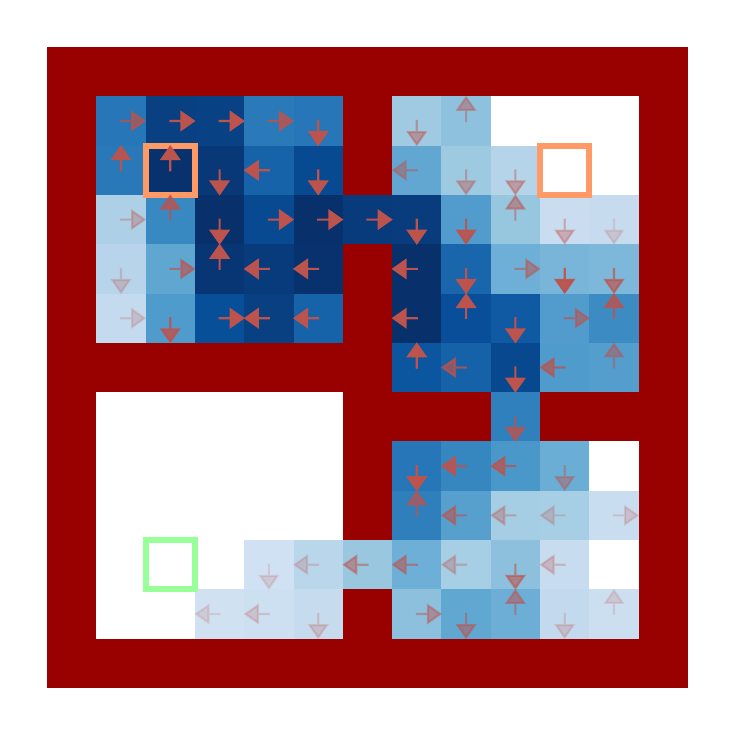}}
    \vskip -0.1in
    \caption{An illustration on the impact of limited expert state coverage in the Four Rooms domains~\citep{sutton1999between,lee2021optidice}. The initial and goal states are represented as orange and green squares, respectively. The maximum trajectory length is 50. (a) depicts the given expert demonstration, which only covers one initial state. (b) shows an imperfect dataset, where the opacity of each square is determined by the empirical state marginal of imperfect data, and the opacity of each arrow represents the empirical action density in a state. (c)-(e) show the empirical trajectory distributions induced by  
    rolling out the policies in the environment from the \textit{left} initial state (beyond expert data). The policies are learned by \texttt{ILID}, \texttt{DWBC}, and \texttt{ISWBC} using both the expert and imperfect data, respectively. In (c)-(e), an arrow denotes the action with the maximum frequency in each state.
    }
    \label{fig:fourrooms}
    % \vskip -0.1in
\end{figure*}

\section{Offline Imitation Learning with Imperfect Demonstrations}
\label{sec:methodology}

This section elaborates on our proposed method. We begin by presenting a hypothesis on behavior selection and providing it with theoretical justification. Building on the hypothesis and theoretical insights, we then delineate our data selection and policy learning methods.

\subsection{Selection of Imperfect Behaviors}
\label{sec:data_selection}

In contrast to the existing works that select data building on state-action resemblance to given expert demonstrations, we propose to access an imperfect behavior by its \textit{resultant states}, to which the environment transitions after performing the behavior. Formally, we present the following hypothesis.
\begin{hypothesis}
	\label{hypo:resultant_state}
    With no other prior knowledge, if a state $s$ lies \textit{beyond} given expert data ($s\notin\mathcal{D}_e$), then, in $s$, taking the action that can transition to a known expert state is more beneficial than selecting actions at random.
\end{hypothesis}
% \vskip -0.1in
To support this hypothesis, we provide the following theoretical results under deterministic dynamics.\footnote{The setting covers many practical environments like MuJoCo.} %Albeit with deterministic dynamics, the problem is not trivial due to the stochasticity in initial states and learning policies.
Represent $\mathcal{D}$ as a demonstration dataset, $\mathcal{S}(\mathcal{D})$ as the set of states in $\mathcal{D}$, and $\mathcal{S}_h(\mathcal{D})$ as the set of $h$-step visited states in $\mathcal{D}$.
%thereby $\mathcal{S}(\mathcal{D})=\bigcup^H_{h=1}\mathcal{S}_h(\mathcal{D})$.
Suppose that $\pi_e$ is optimal and deterministic~\citep{sutton2018reinforcement}, and there exists a supplementary dataset consisting of transitions \textit{from initial states to given expert states}, ${\mathcal{D}_s}\doteq\{(s_i,a_i,s'_i)\mid s_i\sim\mu,T(s_i,a_i)=s'_i,s'_i\in\mathcal{S}_1(\mathcal{D}_e),i=1,\dots,n_s\}$. According to \cref{hypo:resultant_state}, we consider the policy $\tilde{\pi}$ that takes the logging actions in ${\mathcal{D}_s}$ at states $\mathcal{S}_1({\mathcal{D}_s})\backslash\mathcal{S}_1(\mathcal{D}_e)$ and takes the expert actions in $\mathcal{D}_e$ at expert states $\mathcal{S}(\mathcal{D}_e)$:
\begin{align}
    \label{eq:ideal_policy}
    \tilde{\pi}(a|s)\doteq
    \begin{cases}
    	\frac{n((s,a)\in{\mathcal{D}_s})}{n(s\in\mathcal{S}_1({\mathcal{D}_s}))},&\text{if}~s\in\mathcal{S}_1({\mathcal{D}_s})\backslash\mathcal{S}_1(\mathcal{D}_e)\\
        \frac{n((s,a)\in\mathcal{D}_e)}{n(s\in\mathcal{S}(\mathcal{D}_e))},&\text{if}~s\in\mathcal{S}(\mathcal{D}_e)\\
        \frac{1}{\vert\mathcal{A}\vert},&\text{otherwise}
    \end{cases}
\end{align}
% \begin{align}
%     \tilde{\pi}(a|s)\doteq
%     \begin{cases}
%         \frac{\sum_{(\tilde s,\tilde a)\in\mathcal{D}_e}\1((\tilde s,\tilde a)=(s,a))}{\sum_{\tilde s\in\mathcal{S}(\mathcal{D}_e)}\1(\tilde s=s)},&\text{if}~s\in\mathcal{S}(\mathcal{D}_e)\\
%         \frac{\sum_{(\tilde s,\tilde a)\in{\mathcal{D}_s}}\1((\tilde s,\tilde a)=(s,a))}{\sum_{\tilde{s}\in\mathcal{S}_1({\mathcal{D}_s})}\1(\tilde{s} = s)},&\text{if}~s\in\mathcal{S}_1({\mathcal{D}_s})\backslash\mathcal{S}_1(\mathcal{D}_e)\\
%         \frac{1}{|\mathcal{A}|},&\text{otherwise}.
%     \end{cases}
% \end{align}
where $n(s\in\mathcal{D}) = \sum_{s'\in\mathcal{D}}\1(s' = s)$ denotes the number of element $s$ in set $\mathcal{D}$, and $\vert\mathcal{A}\vert$ denotes the cardinality of $\mathcal{A}$.\footnote{Throughout this paper, we use $(s,a,\dots,(s'),(a'))\in\mathcal{D}$ to denote that dataset $\mathcal{D}$ contains sub-trajectory $(s,a,\dots,(s'),(a'))$.} Denote $\delta\doteq\max\{V^{\pi_e}(s_1)-V^{\pi_e}(s_2)\mid\mu(s_1),\mu(s_2)>0\}$ as the maximum return difference among expert trajectories, with $V^\pi(s)\doteq\mathbb{E}_\pi[\sum^H_{h=1}R(s_h,a_h)\mid s_1 = s]$. Next, we characterize the suboptimality of $\tilde{\pi}$ in \cref{thm:det_dyna_gap}.
%\citep{mujoco2023mujoco}
\begin{theorem}
    \label{thm:det_dyna_gap}
    For any finite and episodic MDP with deterministic transition dynamics, the following fact holds:
    \begin{align}
        V^{\pi_e} - \mathbb{E}[V^{\tilde{\pi}}] \le H\epsilon_o +  (\delta + 1) \sqrt{\epsilon_e(1-\epsilon_s)}
    \end{align}
    where $\epsilon_o$, $\epsilon_e$, and $\epsilon_s$ are the missing mass, defined as
    \begin{align}
    	\epsilon_o&\doteq\mathbb{E}_{\mathcal{D}_e,\mathcal{D}_s}\big[\mathbb{E}_{s\sim\mu}\big
    	[\1(s\notin\mathcal{S}_1(\mathcal{D}_e)\cup\mathcal{S}_1(\mathcal{D}_s))\big]\big]\\
    	\epsilon_e&\doteq\mathbb{E}_{\mathcal{D}_e}\big[\mathbb{E}_{s\sim\mu}\big[\1(s\notin\mathcal{S}_1({\mathcal{D}_e}))\big]\big]\\
    	\epsilon_s&\doteq\mathbb{E}_{\mathcal{D}_s}\big[\mathbb{E}_{s\sim\mu}\big[\1(s\notin\mathcal{S}_1({\mathcal{D}_s}))\big]\big].
    \end{align}
    %as the missing mass over the initial distribution w.r.t. $\mathcal{S}_1(\mathcal{D}_e)$ and $\mathcal{S}_1({\mathcal{D}_s}))$. $U(\mathcal{S})$ is the uniform distribution over $\mathcal{S}$.
\end{theorem}

\begin{proof}[Sketch of proof]
    The error stems from the initial states that are not covered by $\mathcal{S}_1(\mathcal{D}_e)$. We bound the errors generated from the states not in $\mathcal{S}_1(\mathcal{D}_e)\cup\mathcal{S}_1({\mathcal{D}_s})$ and from the states in $\mathcal{S}_1({\mathcal{D}_s})\backslash\mathcal{S}_1(\mathcal{D}_e)$ by $H\epsilon_o$ and $(\delta + 1) \sqrt{\epsilon_e(1-\epsilon_s)}$, respectively. Combining these two errors yields the result. For a detailed proof, please refer to \cref{sec:thm_proof}.
\end{proof}
\vskip -0.1in
The missing mass means the probability mass contributed by the states never observed in the corresponding set. Recall that $n_e$ and $n_s$ denote the numbers of trajectories and transitions in $\mathcal{D}_e$ and $\mathcal{D}_s$, respectively. Building on \cref{thm:det_dyna_gap}, we have the following result on sample complexity.

\begin{corollary}
    \label{coro:sample_complexity}
    For any finite and episodic MDP with deterministic transition dynamics, the following fact holds:
    \begin{align*}
    	 V^{\pi_e} - \mathbb{E}[V^{\tilde{\pi}}] \le \frac{\vert\mathcal{S}\vert H}{e(n_e+n_s)} +  (\delta + 1) \cdot\sqrt{\frac{|\mathcal{S}|}{en_e}}
    \end{align*}
    where $e$ denotes the Euler's number. Moreover, with a sufficiently large $n_s$, to obtain an $\varepsilon$-optimal policy, $\tilde{\pi}$ requires at most $\mathcal{O}(\min\{|\mathcal{S}|/\varepsilon^2,|\mathcal{S}|H/\varepsilon\})$ expert trajectories.  
\end{corollary}

\begin{proof}[Sketch of proof]
	The result is concluded via quantifying the missing mass in terms of $n_e$ and $n_s$ (see \cref{sec:coro_proof}).
\end{proof}

\begin{remark}
	It is known that the minimax expected suboptimality of \texttt{BC} is limited to $\mathcal{O}(|S|H/n_e)$ in this setting \citep{rajaraman2020toward,xu2021generalization}, a linear dependency on the episode horizon. This is because $\mu$ may largely differ from $\mathcal{S}(\mathcal{D}_e)$; when the \texttt{BC} policy encounters an initial state far outside $\mathcal{S}(\mathcal{D}_e)$, it will be essentially forced to take an arbitrary action in this state, potentially leading to compounding mistakes over $H$ time steps. 
\end{remark}

\begin{remark}
    As stated in \cref{thm:det_dyna_gap,coro:sample_complexity}, with sufficient $\mathcal{D}_s$, $\tilde\pi$ achieves an expected suboptimality of $\mathcal{O}(\min\{\sqrt{|\mathcal{S}|/n_e},|\mathcal{S}|H/n_e\})$, superior to \texttt{BC} especially with large state spaces, long horizons, and limited expert data.\footnote{Due to following expert behaviors in $\mathcal{S}(\mathcal{D}_e)$, the suboptimality of $\tilde\pi$ is also bounded by $|\mathcal{S}|H/n_e$ (see \cref{sec:coro_proof} for details).} Thanks to the independency of $H$ in the first term, $\tilde\pi$ provably alleviates the error compounding and is robust to initial state perturbations. The underlying rationale is that ${\mathcal{D}_s}$ empowers $\tilde\pi$ to recover from `mistakes': in the states beyond $\mathcal{S}(\mathcal{D}_e)$, albeit without expert guidance, the policy could take actions capable of returning to $\mathcal{S}(\mathcal{D}_e)$ where it exactly knows expert behaviors. In fact, this is very similar to human decision-making: when lost, we always want to get back to familiar roads; when a machine malfunctions, we aim to restore it to normalcy as soon as possible.
\end{remark}

\textbf{Practical behavior selection.} \cref{hypo:resultant_state} implies that resultant states can serve as a criterion for selecting imperfect behaviors -- positive behaviors can be identified according to whether their resultant states fall within the expert state manifold. As an example, if there is an imperfect sub-trajectory $(s_1,a_1,s_2,a_2,s_3)\in\mathcal{D}_b$ such that $s_3\in\mathcal{D}_e$, we can treat $(s_1,a_1)$ and $(s_2,a_2)$ as positive behaviors, even without resemblance to any $(s,a)\in\mathcal{D}_e$. Guided by this, we first train a \textit{state-only} discriminator $d:\mathcal{S}\times\mathcal{A}\rightarrow(0,1)$ to contrast expert and non-expert states in $\mathcal{D}_b$:
\begin{align}
    \label{eq:discriminator}
    \max_{d}\mathbb{E}_{s\sim\mathcal{D}_e}\big[\log {d}(s)\big] + \mathbb{E}_{s\sim\mathcal{D}_u}\big[\log (1-{d}(s))\big]
\end{align}
with $\mathcal{D}_u=\mathcal{D}_e\cup\mathcal{D}_b$. From \citet{goodfello2016generative}, the optimal discriminator ${d}^*$ satisfies 
\begin{align}
    \label{eq:optimal_discriminator}
    d^*(s) = \mathcal{D}_e(s)/(\mathcal{D}_e(s) + \mathcal{D}_u(s))
\end{align}
where we overload notation, denoting $\mathcal{D}_e(s)$ and $\mathcal{D}_u(s)$ as the empirical state marginals in $\mathcal{D}_e$ and $\mathcal{D}_u$, respectively. Building on \cref{eq:optimal_discriminator}, given a small positive threshold $\sigma>0$, if $s\in\mathcal{D}_b$ and $d^*(s)>\sigma$, we identify $s$ as an expert state; otherwise, we treat it as a non-expert one. 

Based on the extracted expert states, we in turn select their \emph{causal state-actions} to construct complementary dataset ${\mathcal{D}_s}$. Recall  $\mathcal{D}_b=\{\tau_i\}^{n_b}_{i=1}$ with $\tau_i=(s_{i,1},a_{i,1},\dots,s_{i,H},a_{i,H})$. If there exist $s_{i,h}\in\mathcal{D}_b$ such that ${d}^*(s_{i,h})\ge \sigma$ for $h>1$ and $i\in\{1,\dots,n_b\}$, we include $K$ causal state-action pairs of $s_{i,h}$ into ${\mathcal{D}_s}$ as follows:
\begin{align}
	\label{eq:data_selection}
	{\mathcal{D}_s}\leftarrow {\mathcal{D}_s}\cup\{(k,s_{i,h-k},a_{i,h-k})\}_{k=1:\min\{h-1,K\}}
\end{align}
where $K\in\{1,2,\dots\}$ is termed as the \emph{rollback step}. We iterate the above process for all identified expert states. For clarity, the process is depicted in \cref{fig:data_selection}. 

\begin{figure}[ht]
    \centering
    % \vskip -0.1in 
    \includegraphics[width=.95\columnwidth]{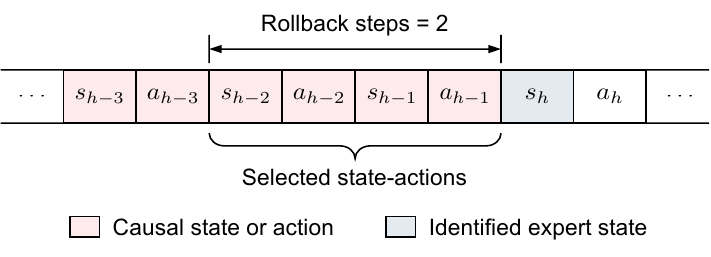}
    \vskip -0.1in
    \caption{An illustration of our behavior selection.}
    \label{fig:data_selection}
    % \vskip -0.1in
\end{figure}

Our behavior selection scheme possesses the following advantages. 1)~The resultant state is informative, capable of effectively identifying both positive diverse behaviors and expert behaviors in $\mathcal{D}_b$. This can be easily seen from the fact that for an expert transition $(s_e,a_e,s'_e)\in\mathcal{D}_e\cap\mathcal{D}_b$, the identification of $s'_e\in\mathcal{S}(\mathcal{D}_e)$ ensures the selection of its causal expert behavior $(s_e,a_e)$. 2)~It explicitly utilizes the dynamics information in $\mathcal{D}_b$, enabling $\mathcal{D}_s$ to cover a relatively large portion of $\mathcal{D}_b$ (with $m$ identified expert states, $\mathcal{D}_s$ can include approximately $mK$ selected state-actions), thus significantly enhancing the utilization of imperfect demonstrations. 3)~The method is easy to implement. Given that the computation in data selection primarily resides in training the discriminator, which is straightforward, it is highly applicable in practical, high-dimensional environments.

%It is evident that ${\mathcal{D}_s}$ comprises both the positive diverse state-actions in $\mathcal{D}_s$ and those similar to $\mathcal{D}_e$ therein. This highlights that using resultant states is a more informative way to extract useful behaviors.

% However, optimizing Problem (\ref{eq:discriminator}) can lead to the problem of \emph{false negative}, where the learned discriminator assigns 1 to all transitions from $\mathcal{D}_e$ and $0$ to all transitions from $\mathcal{D}_s$. This problem is analogous to the positive-unlabeled (PU) classification problem \citep{elkan2008learning}, where both positive (expert) and negative (imperfect) samples exist in the unlabeled data (imperfect demonstrations). Akin to \citet{zolna2020offline,xu2022discriminator}, we adopt the reweighting method from PU learning to address this issue:
% \begin{align}
%     \label{eq:pu_discriminator}
%     {d}^*=\argmax_{d} \eta\cdot\mathbb{E}_{s\sim\mathcal{D}_e}\big[\log {d}(s)\big] \nonumber\\
%     + \mathbb{E}_{s\sim\mathcal{D}_s}\big[\log (1-{d}(s))\big] \nonumber\\
%     - \eta\cdot\mathbb{E}_{s\sim\mathcal{D}_e}\big[\log (1-{d}(s))\big],
% \end{align}
% where $\eta>0$ is a reweighting parameter, corresponding to the proportion of expert states to imperfect states. Intuitively, the third term in \cref{eq:pu_discriminator} could avoid $d^*(s)$ of the states from $\mathcal{S}_1(\mathcal{D}_s)$ but similar to $\mathcal{S}(\mathcal{D}_e)$ becoming 0.

\subsection{Learning from Expert and Selected Behaviors}

After obtaining ${\mathcal{D}_s}$, a natural solution to learn an imitation policy is carrying out \texttt{BC} from the union of $\mathcal{D}_e$ and ${\mathcal{D}_s}$. However, due to the suboptimality of ${\mathcal{D}_s}$, this solution may suffer from potential \textit{interference} among actions. That is, for a selected $(s,a,s')$, if $s,s'\in\mathcal{D}_e$ but $a\neq\pi_e(s)$, action $a$ will affect mimicking the expert behavior in expert state $s$ when learning from the union data (see \cref{subfig:weighted_bc}). Thus, it necessitates exactly following the expert in given expert states (it has been implied by the definition of $\tilde\pi$ in \cref{eq:ideal_policy}).

\begin{algorithm}[t]
    \caption{\texttt{ILID}}
    \label{alg:ilid}
    \begin{algorithmic}[1]
    \REQUIRE Expert data $\mathcal{D}_e$, imperfect data $\mathcal{D}_b$, rollback $K$
    \STATE Initialize policy parameter $\theta$
    \STATE Train discriminators $d^*$ and $D^*$ by \cref{eq:discriminator,eq:discriminator_2}
    \STATE \texttt{\color{gray}// Data selection}
    \STATE Build complementary dataset $\mathcal{D}_s$ by \cref{eq:data_selection}
    \STATE \texttt{\color{gray}// Policy extraction}
    \FOR{$i=1$ {\bfseries to} $n$}
    \STATE $\theta\leftarrow\theta + \eta\tilde\nabla J(\pi_\theta)$
    \ENDFOR
    \end{algorithmic}
\end{algorithm}

To this end, we cast the policy learning as the following weighted behavior cloning problem:
\begin{align*}
	\max_\pi \mathbb{E}_{ \mathcal{D}_e}[\log(\pi(a|s))] + \mathbb{E}_{\mathcal{D}_s}[\1(\mathcal{D}_e(s)=0)\log(\pi(a|s))]
\end{align*}
where the expectation is taken w.r.t. state-action $(s,a)$, and $\mathcal{D}_e(s)$ denotes the empirical state marginals in $\mathcal{D}_e$. In the problem, the first term matches \texttt{BC}, and the second term aims to clone the selected behaviors \textit{outside} the expert state manifold, which essentially discards the suboptimal actions in expert states. Of note, albeit with a Dirichlet function in the second term, based on \cref{eq:optimal_discriminator}, it can be well approximated via the output of $d^*$. In practice, we instantiate the above objective as follows:
\begin{align}
	\label{eq:ilid_objective}
	\max_\pi J(\pi)&\doteq\mathbb{E}_{ \mathcal{D}_u}[\alpha(s,a)\log(\pi(a|s))] \nonumber\\
	&\,+ \mathbb{E}_{\mathcal{D}_s}[\beta(s,a)\log(\pi(a|s))]
\end{align}
with $\mathcal{D}_u=\mathcal{D}_e\cup\mathcal{D}_b$. In \cref{eq:ilid_objective}, we exploit the trick of importance sampling (which is unbiased) to enhance the expert data support, as in \citet{li2023imitation}:
\begin{align}
	\label{eq:alpha}
	\alpha(s,a) \doteq \frac{\mathcal{D}_e(s,a)}{\mathcal{D}_u(s,a)} = \frac{D^*(s,a)}{1 - D^*(s,a)}
\end{align}
where another discriminator $D^*$ is obtained by solving
\begin{align}
	\label{eq:discriminator_2}
	\max_D \mathbb{E}_{\mathcal{D}_e}[\log D(s, a)]+\mathbb{E}_{\mathcal{D}_u}[\log (1-D(s, a))].
\end{align}
In addition, $\beta(s,a)$ approximates the Dirichlet function by
\begin{align}
    \label{eq:beta}
    \beta(s,a)\doteq \1(d^*(s)\le\sigma).
\end{align}
In summary, we term our algorithm \emph{offline Imitation Learning with Imperfect Demonstrations} (\texttt{ILID}) with the pseudocode outlined in \cref{alg:ilid}, which can be easily implemented on top of \texttt{BC} and enjoys fast convergence speed and training stability (see \cref{sec:experiment}).

\section{Experiments}
\label{sec:experiment}

In this section, we carry out extensive experiments to evaluate our proposed method and answer the following key questions: 
\textbf{{1)}}~Can \texttt{ILID} effectively utilize imperfect demonstrations, especially in complex, high-dimensional environments?
%\textbf{{2)}}~What is the convergence property of \texttt{ILID} in terms of speed and stability?
\textbf{{2)}}~How does \texttt{ILID} perform given different numbers of expert demonstrations or varying qualities of imperfect demonstrations?
%\textbf{{3)}}~What is the impact of the rollback steps?
\textbf{{3)}}~What are the effects of components and hyperparameters such as $\alpha(s,a)$, $\beta(s,a)$, and $K$?  Experimental details are elaborated in \cref{sec:expertimental_details}.\footnote{The code is available at \href{https://github.com/HansenHua/ILID-offline-imitation-learning}{https://github.com/HansenHua/ILID-offline-imitation-learning}.}

\textbf{Baselines.} We evaluate our method against six strong baseline methods in offline IL:
\textbf{1)} \texttt{BCE}, the standard \texttt{BC} trained only on expert demonstrations; 
\textbf{2)} \texttt{BCU}, \texttt{BC} trained on union data;
\textbf{3)} \texttt{DWBC} \citep{xu2022discriminator}, an offline IL method that leverages suboptimal demonstrations by jointly training a discriminator to re-weight the \texttt{BC} objective;
\textbf{4)} \texttt{ISWBC} \citep{li2023imitation}, an offline IL method that adopts importance sampling to enhance \texttt{BC};
\textbf{5)} \texttt{CSIL} \citep{watson2023coherent}, a model-free IRL method that learns a shaped reward function using the  \texttt{BC} policy;
\textbf{6)} \texttt{MLIRL} \citep{zeng2023demonstrations}, a model-based offline IRL algorithm based on bi-level optimization.

% \textbf{Baselines.} We evaluate our method against six strong baseline methods in offline IL:
% \emph{\textbf{1)} {B}ehavior {C}loning with {E}xpert Data} (\texttt{BCE}), the standard \texttt{BC} trained only on expert demonstrations; 
% \emph{\textbf{2)} {B}ehavior {C}loning with {U}nion Data} (\texttt{BCU}), \texttt{BC} trained on union data;
% \emph{\textbf{3)} {D}iscriminator-{W}eighted {B}ehavioral {C}loning} (\texttt{DWBC}) \citep{xu2022discriminator}, an offline IL method that leverages suboptimal demonstrations by jointly training a discriminator to re-weight the \texttt{BC} objective;
% \emph{\textbf{4)} {I}mportance-{S}ampling-{W}eighted {B}ehavioral {C}loning}  (\texttt{ISWBC}) \citep{li2023imitation}, an offline IL method that adopts importance sampling to enhance \texttt{BC};
% \emph{\textbf{5)} {C}oherent {S}oft {I}mitation {L}earning} (\texttt{CSIL}) \citep{watson2023coherent}, a model-free IRL method that learns a shaped reward function by entropy-regularized \texttt{BC};
% \emph{\textbf{6)} {M}aximum {L}ikelihood-{I}nverse {R}einforcement {L}earning} (\texttt{MLIRL}) \citep{zeng2023demonstrations}, a recent model-based offline IRL algorithm based on bi-level optimization.

\begin{figure}[htbp]
    % \vskip -0.1in
    \includegraphics[width=\columnwidth]{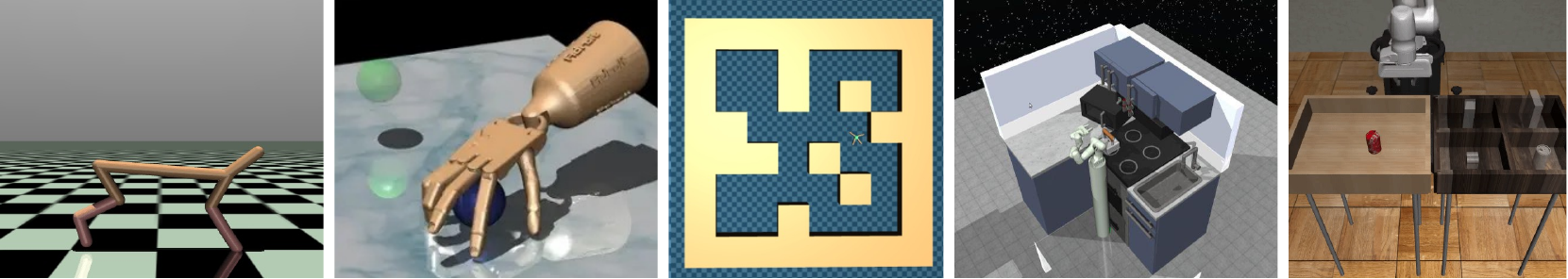}
    \vskip -0.1in
    \caption{Benchmark environments. From left to right: MuJoCo, Adroit, AntMaze, FrankaKitchen, and vision-based Robomimic. We also consider vision-based MuJoCo with image observations.}
    \label{fig:tasks}
    % \vskip -0.1in
\end{figure}

\textbf{Environments and datasets.} We run experiments with 6 domains including 21 tasks: {{1)}}~{{AntMaze}} (\texttt{umaze}, \texttt{medium}, \texttt{large}), {{2)}}~{{Adroit}} (\texttt{pen}, \texttt{hammer}, \texttt{door}, \texttt{relocate}), {{3)}}~{MuJoCo} (\texttt{ant}, \texttt{hopper},
\texttt{halfcheetah}, \texttt{walker2d}), {{4)}}~{FrankaKitchen} (\texttt{complete}, \texttt{partial}, \texttt{undirect}), {{5)}}~{vision-based Robomimic} (\texttt{lift}, \texttt{can}, \texttt{square}), and {{6)}}~{vision-based MuJoCo}. We employ the \texttt{D4RL} datasets~\citep{fu2020d4rl} for AntMaze, MuJoCo, Adroit, and FrankaKitchen and use the \texttt{robomimic}~\citep{robomimic2021} datasets for vision-based Robomimic. In addition, we construct vision-based MuJoCo datasets using the method introduced in \citet{fu2020d4rl}. Details on environments and datasets can be found in \cref{sec:benchmarks,sec:datasets}.

\textbf{Performance measure.} We train a policy using 3 random seeds and evaluate it by running it in the environment for 10 episodes and computing the average undiscounted return of the environment reward. Akin to \citet{fu2020d4rl}, we use the normalized scores in figures and tables, which are measured by $\texttt{score}=100\times\frac{\texttt{score} - \texttt{random\_score}}{\texttt{expert\_score} - \texttt{random\_score}}$.

\begin{table*}[htpb]
    % \vskip -0.1in
    \renewcommand{\arraystretch}{1.05} 
    \centering
    \caption{Normalized performance under limited expert demonstrations and low-quality imperfect data. The number of expert trajectories is 1 for MuJoCo and AntMaze, 10 for Adroit and FrankaKitchen, and 25 for vision-based MuJoCo and Robomimic; and the number of imperfect trajectories is 1000 across tasks. Uncertainty intervals depict standard deviation. The sampling datasets can be found in \cref{table:dataset_comparative}.}
    \vskip 0.1in
    \label{tab:performance}
    \resizebox{1.0\textwidth}{!}    {
        \begin{tabular}{lrrrrrrr}
            %\hline
            \toprule
    Task             & \multicolumn{1}{r}{\texttt{BCE\ \,\,}} & \multicolumn{1}{r}{\texttt{BCU\ \,\,}} & \multicolumn{1}{r}{\texttt{DWBC\ \,}} & \multicolumn{1}{r}{\texttt{CSIL\ \,}} & \multicolumn{1}{r}{\texttt{MLIRL\;\,}} & \multicolumn{1}{r}{\texttt{ISWBC\;\,}} & \multicolumn{1}{r}{\texttt{ILID} (ours)}   \\ \hline
    \texttt{ant}              & $-15.6\pm7.0$           & $31.4\pm0.1$              & \ $23.4\pm7.1$              & $0.2\pm0.0$                & $35.9\pm9.3$                & $27.1\pm6.7$                & $\bm{62.7\pm4.1} $              \\
    \texttt{halfcheetah}      & $0.4\pm1.0$               & $2.3\pm0.0$               & \ $0.9\pm1.3$                & $15.1\pm4.3$               & $21.5\pm0.8$                & $12.6\pm2.4$                & $\bm{32.4\pm2.4}$               \\
    \texttt{hopper}           & $16.7\pm4.3$              & $7.7\pm6.0$               & $\bm{78.3 \pm 10.9}$     & $16.1\pm3.7$               & $55.2\pm 14.6$               & $73.1\pm8.9$                & $68.9\pm4.8$                        \\
    \texttt{walker2d}         & $7.1\pm5.4$               & $0.3\pm0.1$               & \ $46.1\pm9.8$               & $8.9\pm4.2$                & $23.5\pm1.2$                & $39.8\pm1.9$                & $\bm{58.4\pm4.8} $              \\ \hline
    \texttt{hammer}           & $4.5\pm5.3$               & $0.2\pm0.0$               & \ $14.6\pm 12.6$              & $15.3\pm7.1$               & $0.2\pm0.0$                 & $3.8\pm3.0$                 & $\bm{51.0\pm2.4}  $             \\
    \texttt{pen}              & $40.0\pm9.6$                & $2.8\pm7.8$               & \ $36.0\pm 18.9$              & $22.1\pm0.2$               & $17.2\pm3.6$                & $31.8\pm0.0$                & $\bm{75.1\pm5.2}  $             \\
    \texttt{relocate}         & $-0.1\pm0.1$              & $-0.1\pm0.0$              & $-0.1\pm0.0$               & $4.0\pm3.2$                & $0.2\pm0.0$                 & $0.2\pm0.0$                 & $\bm{28.2\pm1.6}$               \\
    \texttt{door}             & $2.9\pm2.1$               & $-0.1\pm0.0$              & $-0.1\pm0.1$               & $16.7\pm7.1$               & $0.2\pm0.0$                 & $0.2\pm0.0$                 & $\bm{25.9\pm1.1} $              \\ \hline
    \texttt{antmaze-umaze}    & $3.6\pm0.0$               & $3.6\pm0.0$               & $22.0\pm2.7$               & $12.0\pm3.2$               & $6.4\pm0.3$                 & $9.9\pm1.1$                 & $\bm{72.3\pm3.8}$               \\
    \texttt{antmaze-medium}   & $0.2\pm0.0$               & $0.2\pm0.0$               & \ $0.2\pm0.0$                & $0.2\pm0.0$                & $0.2\pm0.0$                 & $6.4\pm0.3$                 & $\bm{64.6\pm5.2} $              \\
    \texttt{antmaze-large}    & $0.2\pm0.0$               & $0.2\pm0.0$               & \ $0.2\pm0.0$                & $0.2\pm0.0$                & $0.2\pm0.0$                 & $4.8\pm0.0$                 & $\bm{39.8\pm2.5} $              \\ \hline
    \texttt{undirect} & $0.2\pm0.0$               & $0.2\pm0.0$               & \ $0.2\pm0.0$                & $35.0\pm0.0$               & $0.2\pm0.0$                 & $0.2\pm0.0$                 & $\bm{52.8\pm3.1} $              \\
    \texttt{partial}  & $0.2\pm0.0$               & $0.2\pm0.0$               & \ $0.2\pm0.0$                & $21.7\pm1.4$               & $0.2\pm0.0$                 & $0.2\pm0.0$                 & $\bm{32.5\pm2.6} $              \\
    \texttt{complete} & $0.2\pm0.0$               & $0.2\pm0.0$               & \ $0.2\pm0.0$                & $11.7\pm0.0$               & $0.2\pm0.0$                 & $0.2\pm0.0$                 & $\bm{29.9\pm1.7} $              \\ \hline
    \texttt{ant-img}          & $16.0\pm4.1$             & $15.6\pm2.4$              & $17.6\pm3.2$               & $10.7\pm2.4$               & $0.0\pm0.0$                 & $19.2\pm2.1$                & $\bm{31.5\pm4.0} $              \\
    \texttt{halfcheetah-img}  & $26.6\pm3.2$              & $27.9\pm4.7$              & $18.5\pm6.4$               & $25.3\pm4.8$               & $0.0\pm0.0$                 & $23.5\pm1.5$                & $\bm{41.6\pm3.2}$               \\
    \texttt{hopper-img}       & $12.8\pm4.0$             & $10.9\pm5.2$              & $16.7\pm5.6$               & $11.8\pm4.0$               & $0.0\pm0.0$                 & $15.4\pm6.3$                & $\bm{61.5\pm5.0} $              \\
    \texttt{walker2d-img}     & $8.3\pm2.0$              & $7.7\pm6.3$               & $22.8\pm5.0$               & $7.5\pm5.5$                & $0.0\pm0.0$                 & $27.9\pm3.3$                & $\bm{58.9\pm4.4} $              \\ \hline
    \texttt{can-img}          & $13.7\pm9.6$              & $21.4\pm2.4$              & $21.9\pm1.4$               & $23.3\pm3.2$               & $0.0\pm0.0$                 & $9.8\pm 11.9$                & $\bm{38.8\pm2.7}  $             \\
    \texttt{lift-img}         & $48.5\pm4.9$              & $28.9\pm3.3$              & $46.6\pm5.7$               & $35.9\pm1.7$               & $0.0\pm0.0$                 & $56.9\pm2.4$                & $\bm{90.4\pm2.4} $              \\
    \texttt{square-img}       & $2.0\pm1.6$               & $5.0\pm4.1$              & $11.5\pm2.2$               & $5.0\pm3.3$                & $0.0\pm0.0$                 & $13.2\pm1.4$                & $\bm{37.8\pm3.0}  $             \\ 
            %\hline
            % \midrule
            % \textbf{Average} & \multicolumn{1}{c}{9.0} & \multicolumn{1}{c}{7.9} & \multicolumn{1}{c}{18.0} & \multicolumn{1}{c}{14.2} & \multicolumn{1}{c}{7.7}   & \multicolumn{1}{c}{17.9}  & \multicolumn{1}{c}{\textbf{50.2}} \\ 
            %\hline
            \bottomrule
        \end{tabular}
    }
    % \vskip -0.1in
\end{table*}

\textbf{Reproducibility.} All details of our experiments are provided in the appendices in terms of the tasks, network architectures, hyperparameters, etc. We implement all baselines and environments based on open-source repositories. Of note, our method is robust in hyperparameters -- they are \emph{identical} for all tasks except for the change of neural nets to CNNs in vision-based domains. 

%a diverse range of widely-recognized continuous control benchmark tasks and datasets, as outlined in D4RL \citep{fu2020d4rl}, detailed in \cref{fig:tasks}. Our experimental scope included MuJoCo (Ant, HalfCheetah, Hopper, Walker2d), Adroit (Hammer, Pen, Relocate, Door), AntMaze (umaze, medium, large), and FrankaKitchen tasks (undirect, partial, complete). In addition, we expanded our testing to encompass two image-based domains, specifically the MuJoCo tasks (Ant, HalfCheetah, Hopper, Walker2d), for which we utilized the same dataset construction method as outlined by \citet{fu2020d4rl}, and the Robomimic\citep{robomimic2021} (Can, Lift, Square).

\textbf{Comparative results.} To answer the first question, we evaluate \texttt{ILID}'s performance in each task using limited expert demonstrations and a set of low-quality imperfect data. For example, in the MuJoCo domain, we sample 1 \texttt{expert} trajectory and 1000 \texttt{random} trajectories from \texttt{D4RL} as the expert and imperfect data, respectively (refer to \cref{table:dataset_comparative} for the complete data setup). Comparative results are presented in \cref{tab:performance}, and learning curves are depicted in \cref{fig:convergence_all_1,fig:convergence_all_2}. We find \texttt{ILID} consistently outperforms baselines in \textbf{20/21} tasks often by a significant margin while enjoying fast and stabilized convergence. Due to limited state coverage of expert data and low quality of imperfect data, \texttt{BCE} and \texttt{BCU} fail to fulfill most of the tasks. This reveals \texttt{ILID}'s effectiveness in extracting and leveraging positive behaviors from imperfect demonstrations. \texttt{DWBC} and \texttt{ISWBC} exhibit similar performances, demonstrating relative success in MuJoCo but facing challenges in robotic manipulation and maze domains, which require precise long-horizon manipulation. This is because the similarity-based behavior selection confines their training data to the expert states with narrow coverage, rendering them prone to error compounding. In contrast, \texttt{ILID}, utilizing dynamics information, can \textit{stitch} parts of trajectories and empower the policy to recover from mistakes. In addition, the IRL methods struggle in high-dimensional environments owing to reward extrapolation and world model estimates.

\begin{figure}[t]
    \centering
    % \vskip -0.1in
    \subfigure{\includegraphics[width=0.48\textwidth]{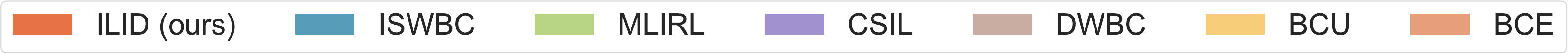}}
     \\[-2pt]  % 调整垂直间距
    \subfigure{\includegraphics[width=0.48\textwidth]{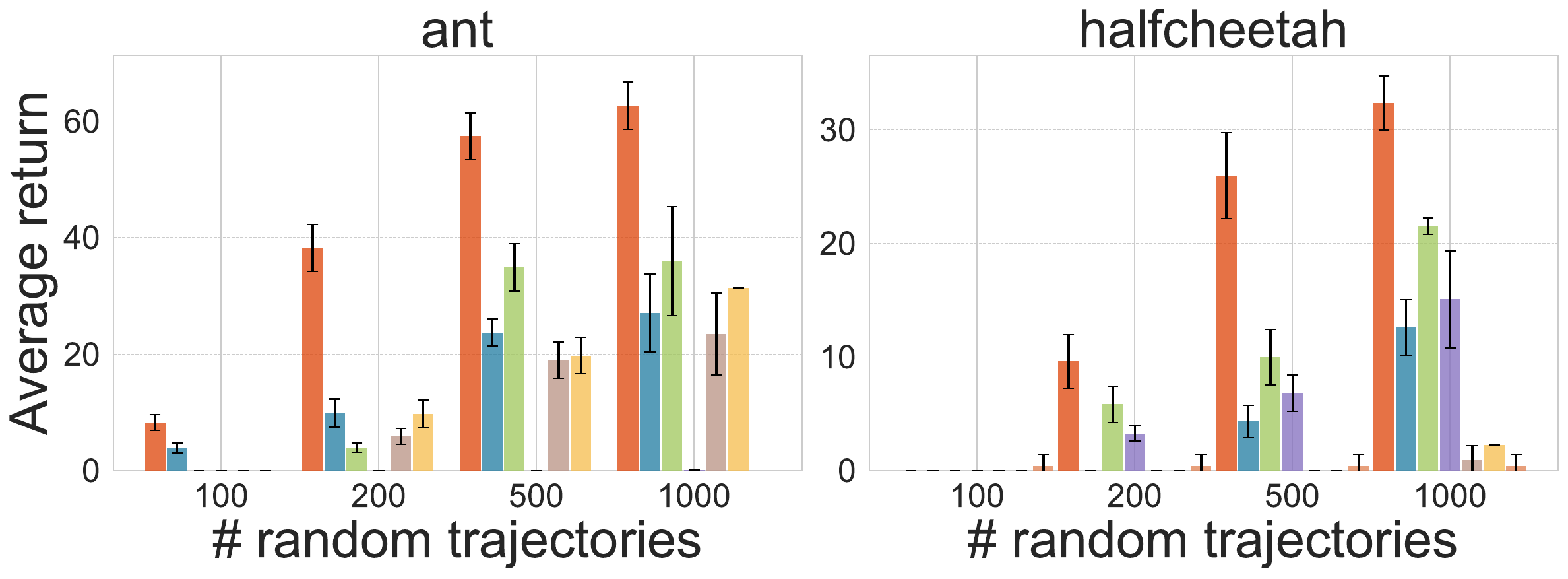}}
    \vskip -0.1in
    \caption{Performance with 1 \texttt{expert} trajectory and varying numbers of \texttt{random} trajectories.} 
    \vskip -0.1in
    \label{fig:random_num}
\end{figure}

\begin{figure*}[t]
    % \vskip -0.2in
    \centering
    
    \subfigure{
    \includegraphics[width=0.16\textwidth]{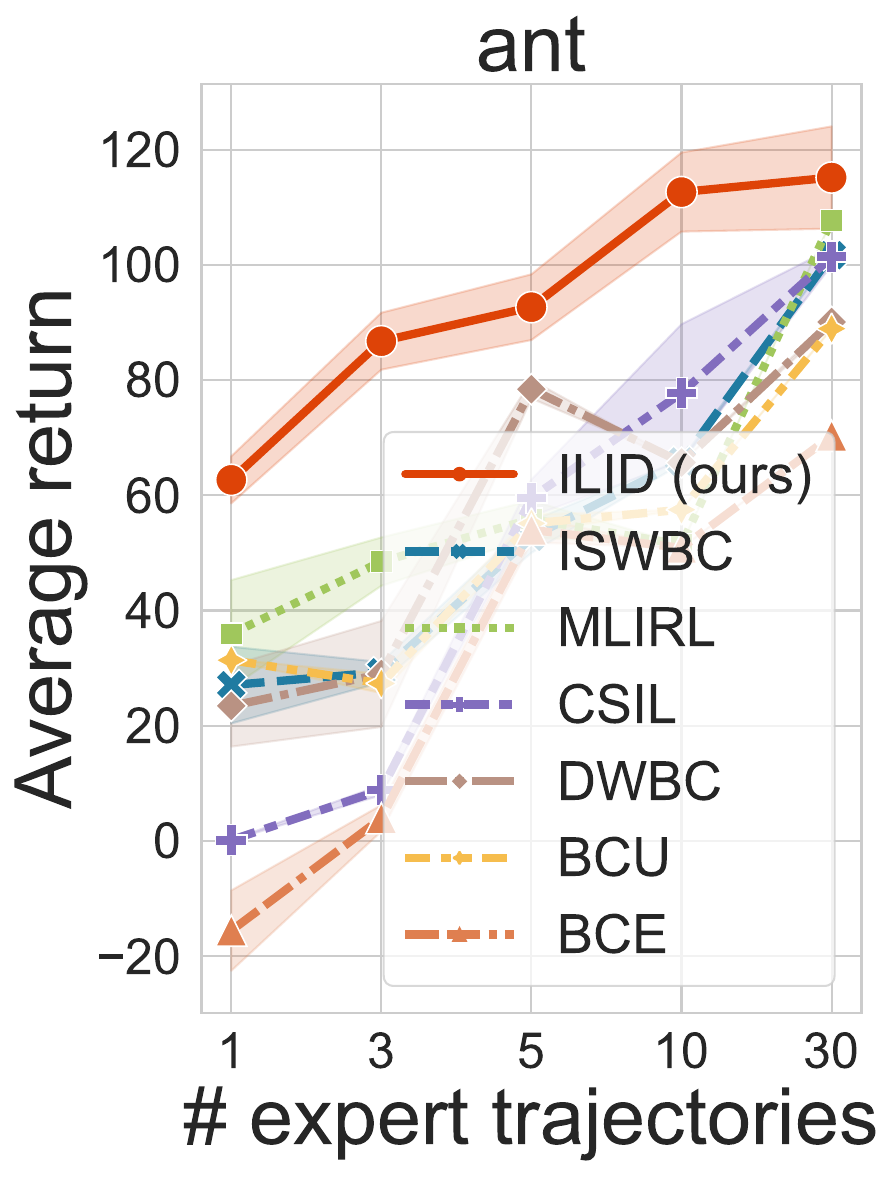}}
    \hspace{-4pt}
    \subfigure{
    \includegraphics[width=0.16\textwidth]{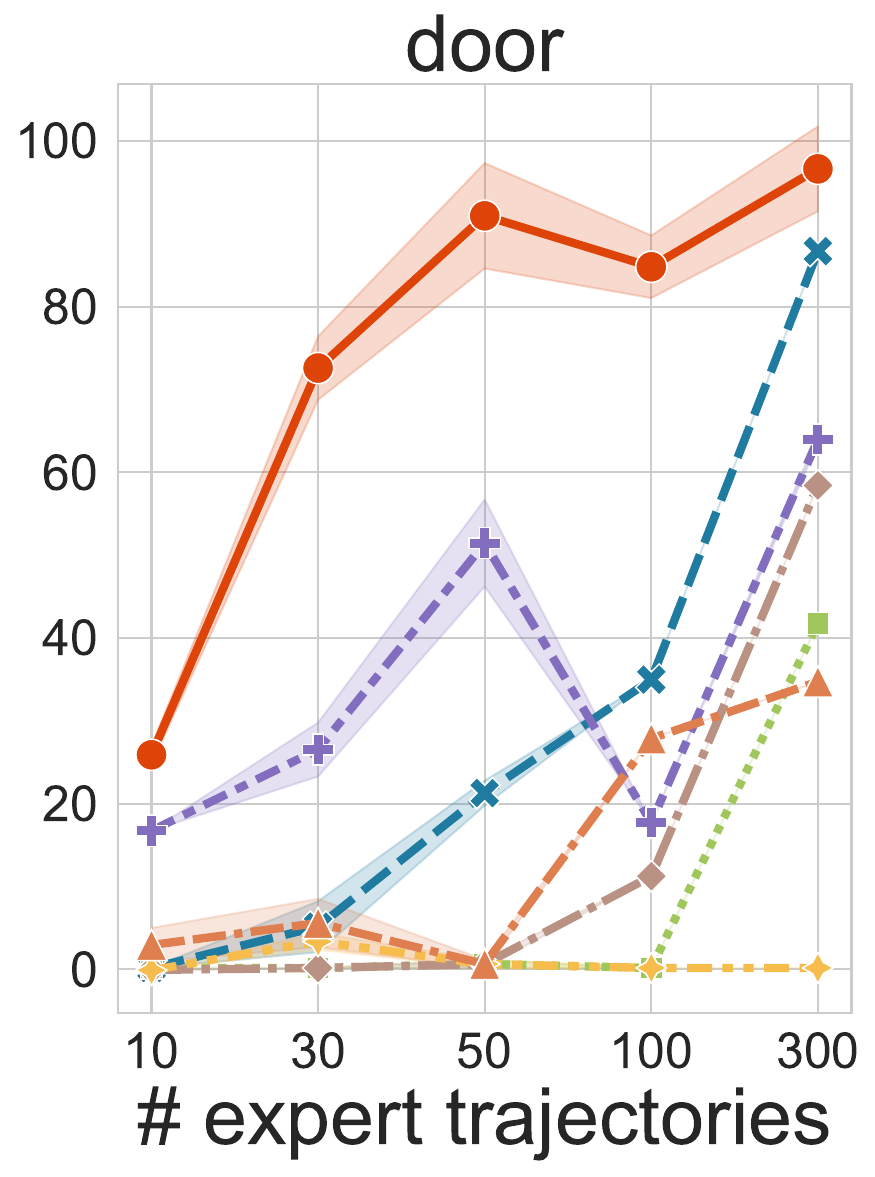}}
    \hspace{-4pt}
    \subfigure{
    \includegraphics[width=0.16\textwidth]{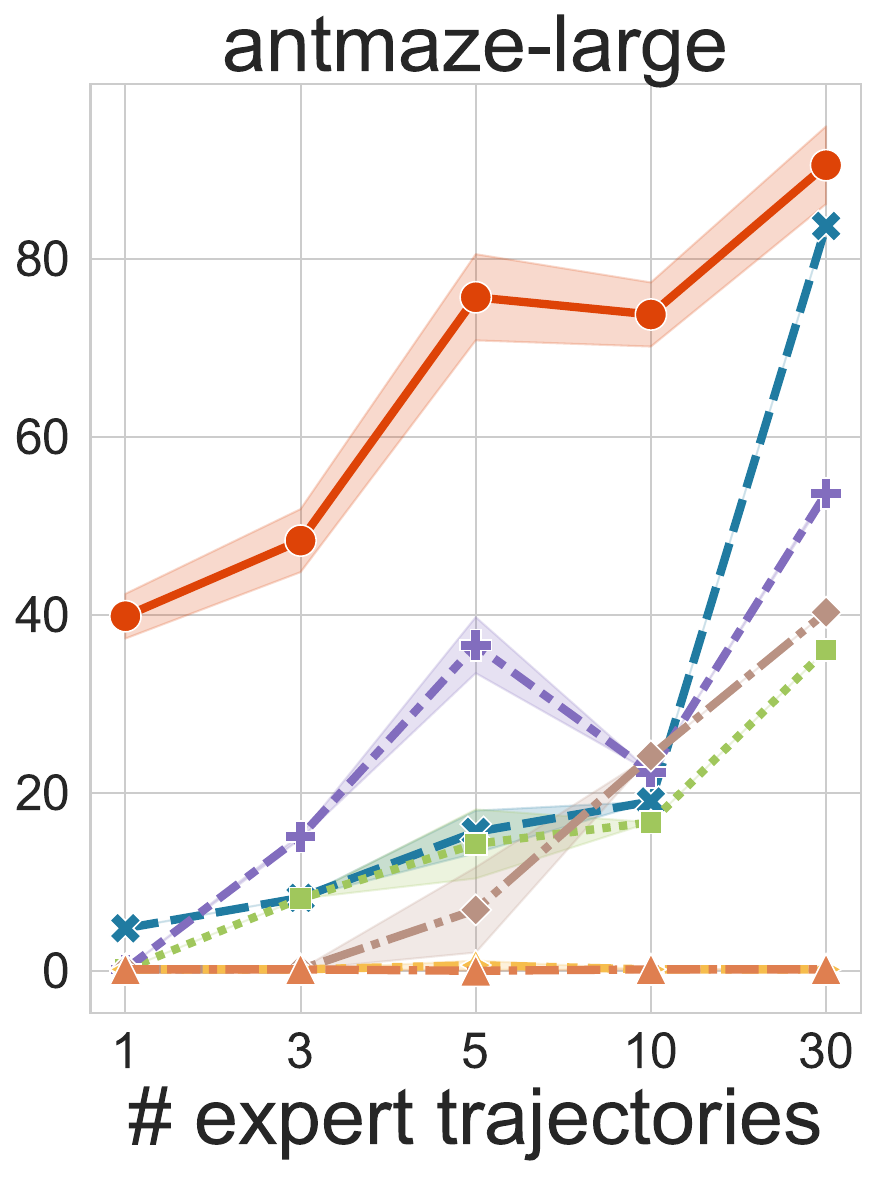}}
    \hspace{-4pt}
    \subfigure{
    \includegraphics[width=0.16\textwidth]{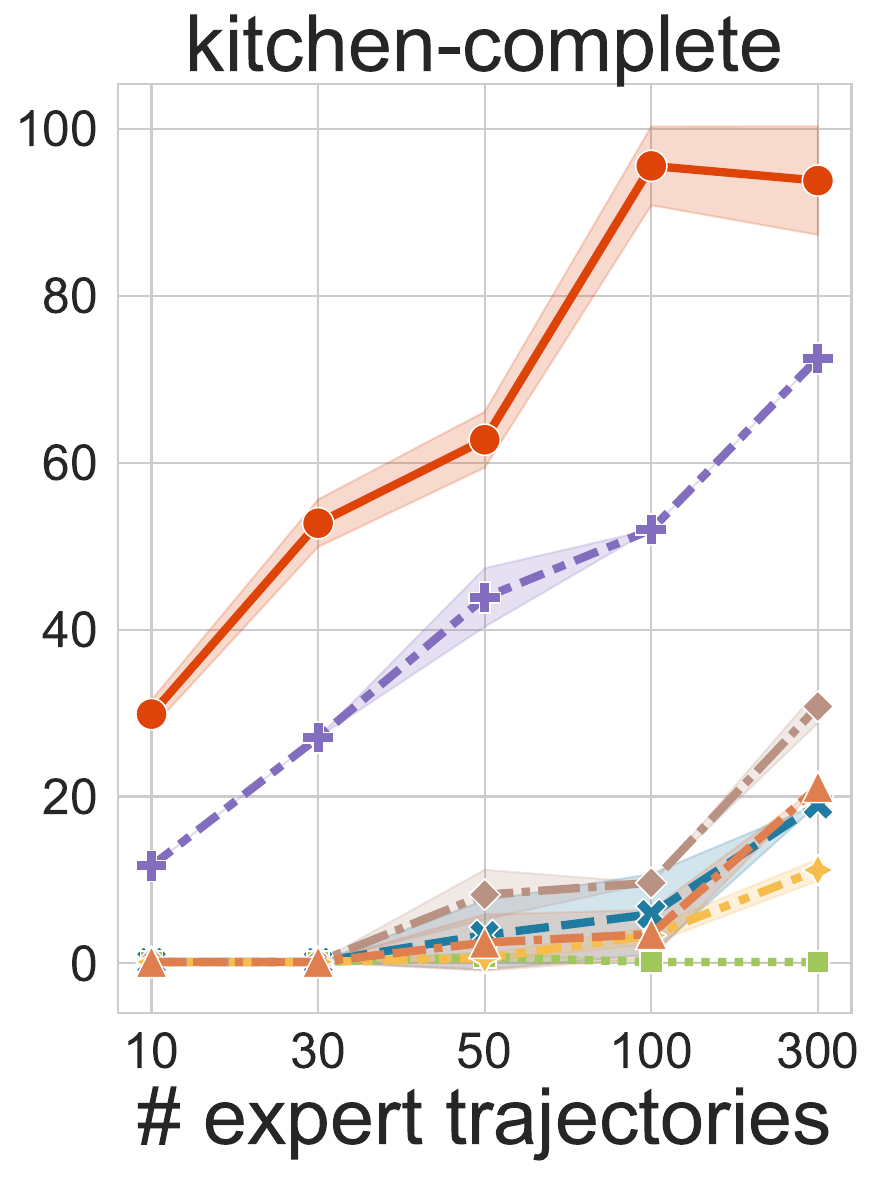}}
    \hspace{-4pt}
    \subfigure{
    \includegraphics[width=0.16\textwidth]{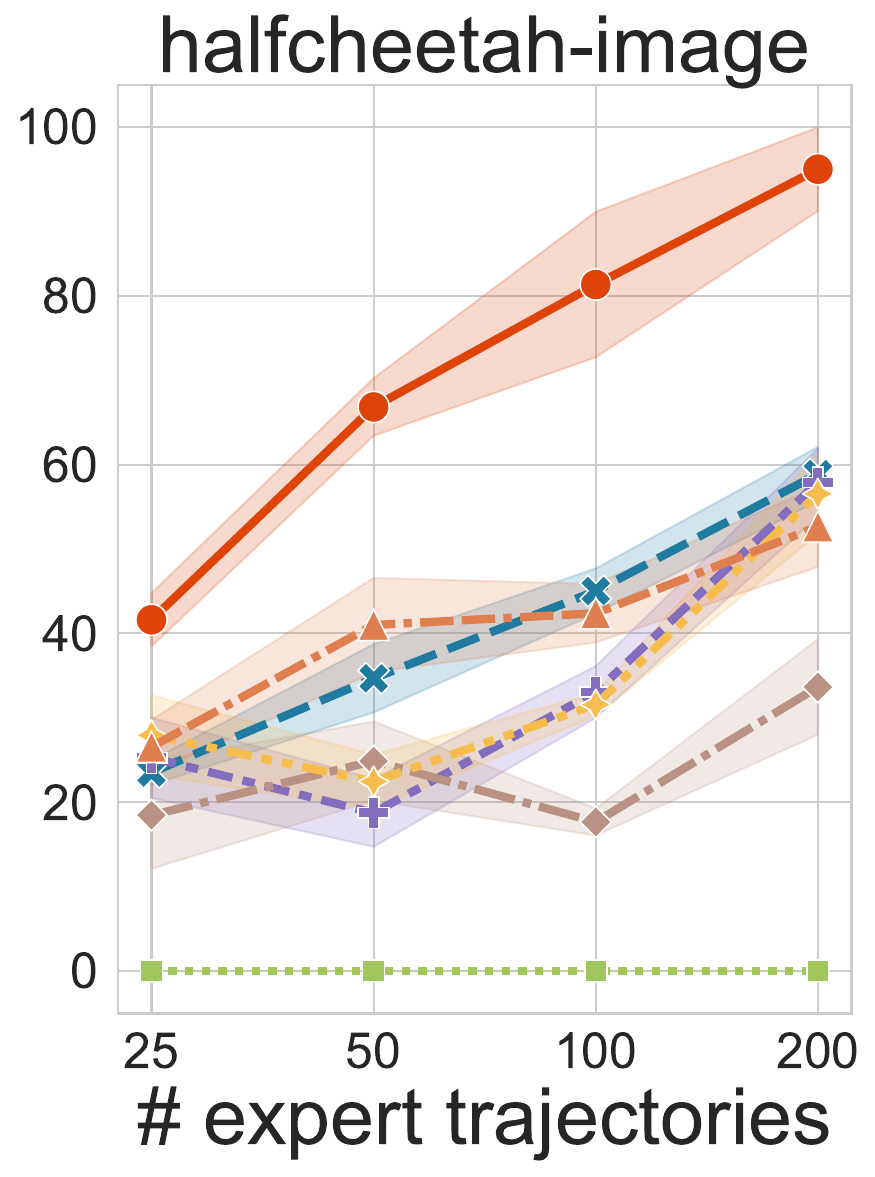}}
    \hspace{-4pt}
    \subfigure{
    \includegraphics[width=0.16\textwidth]{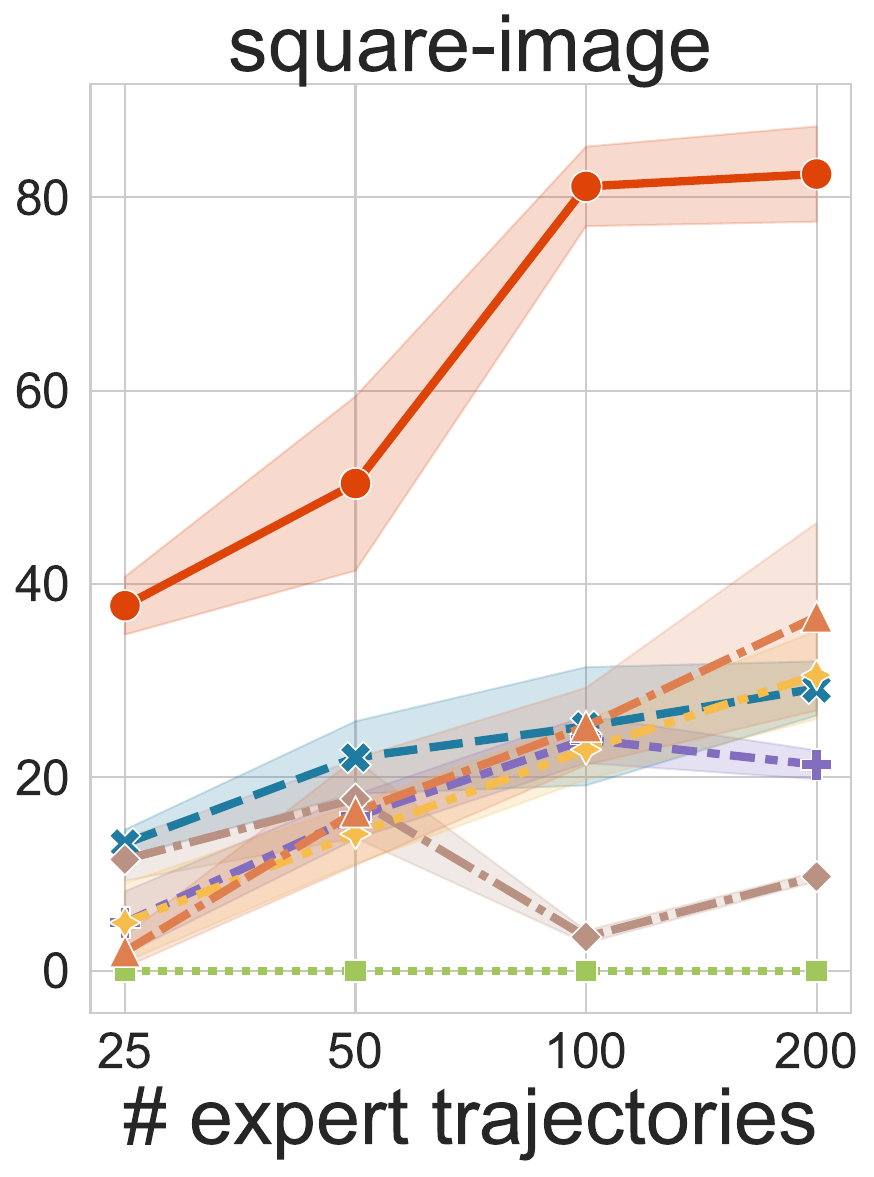}}
    \vskip -0.1in
    \caption{Normalized scores under varying numbers of expert demonstrations.}
    \label{fig:demonstration}
\end{figure*}

\begin{figure*}[t]
    \vskip -0.1in
    \centering
    \subfigure[\texttt{ant}]{
    \includegraphics[width=0.24\textwidth]{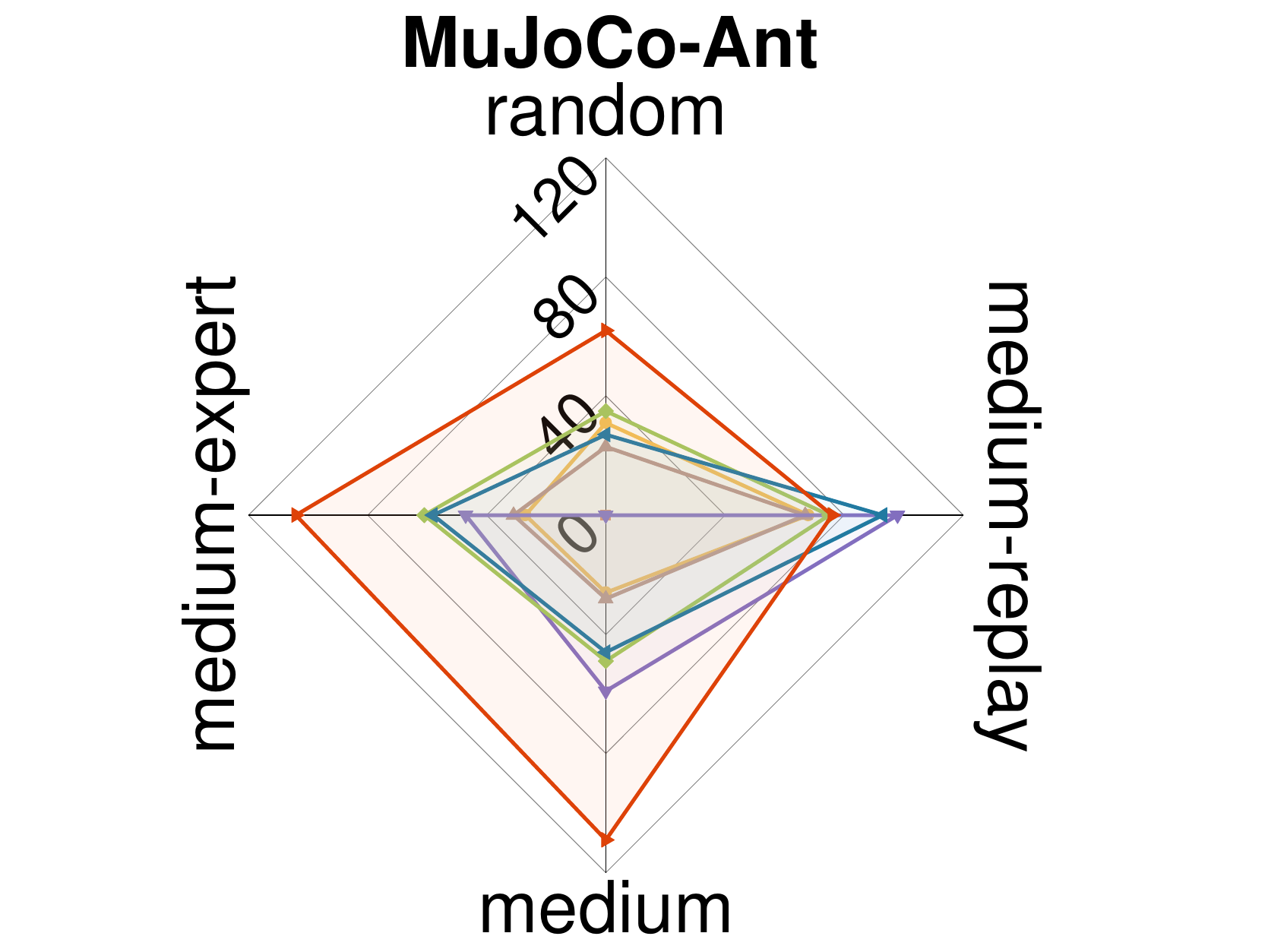}}
    \subfigure[\texttt{walker2d}]{
    \includegraphics[width=0.24\textwidth]{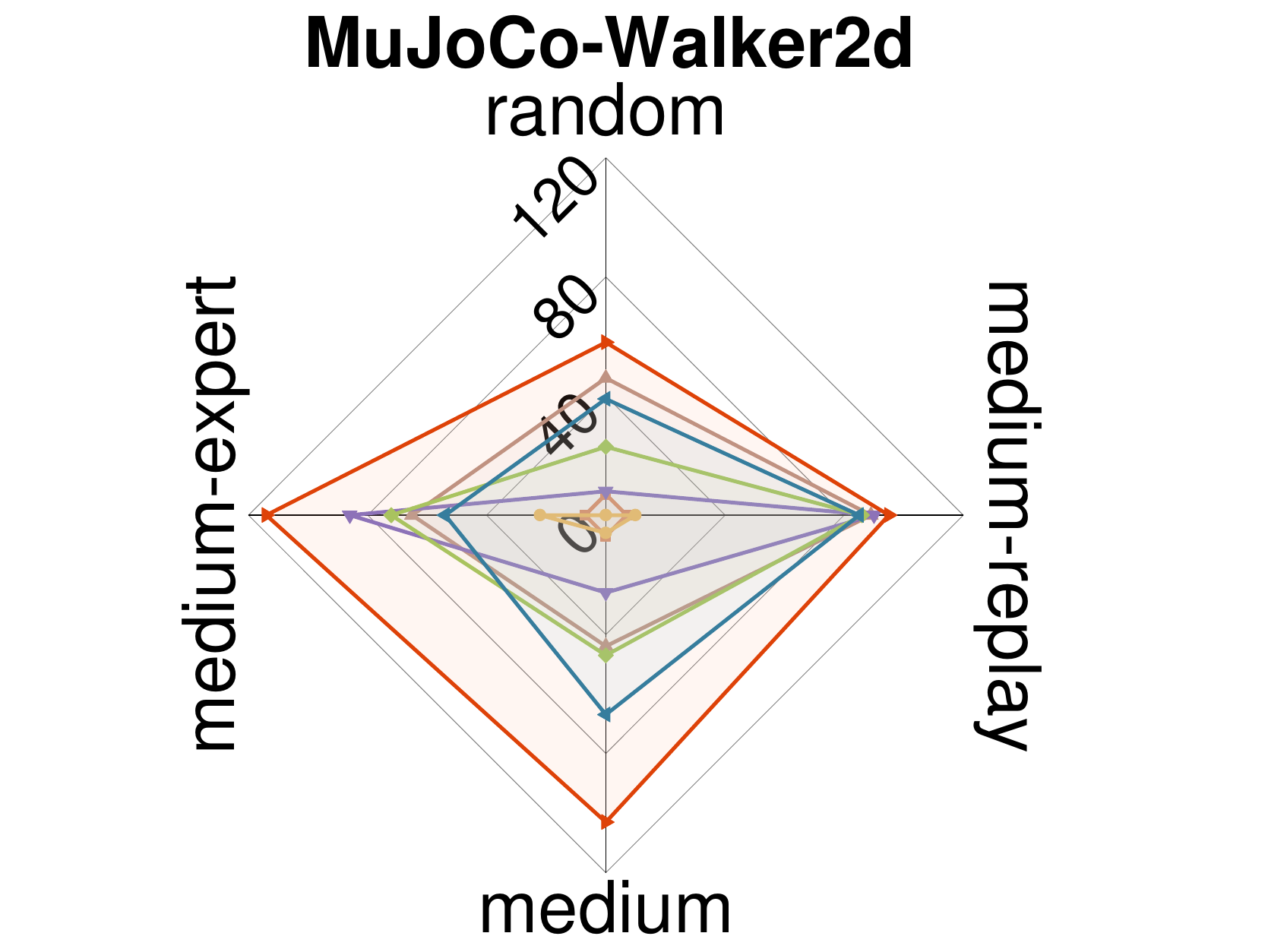}}
    \subfigure[\texttt{hammer} \& \texttt{relocate}]{
    \includegraphics[width=0.24\textwidth]{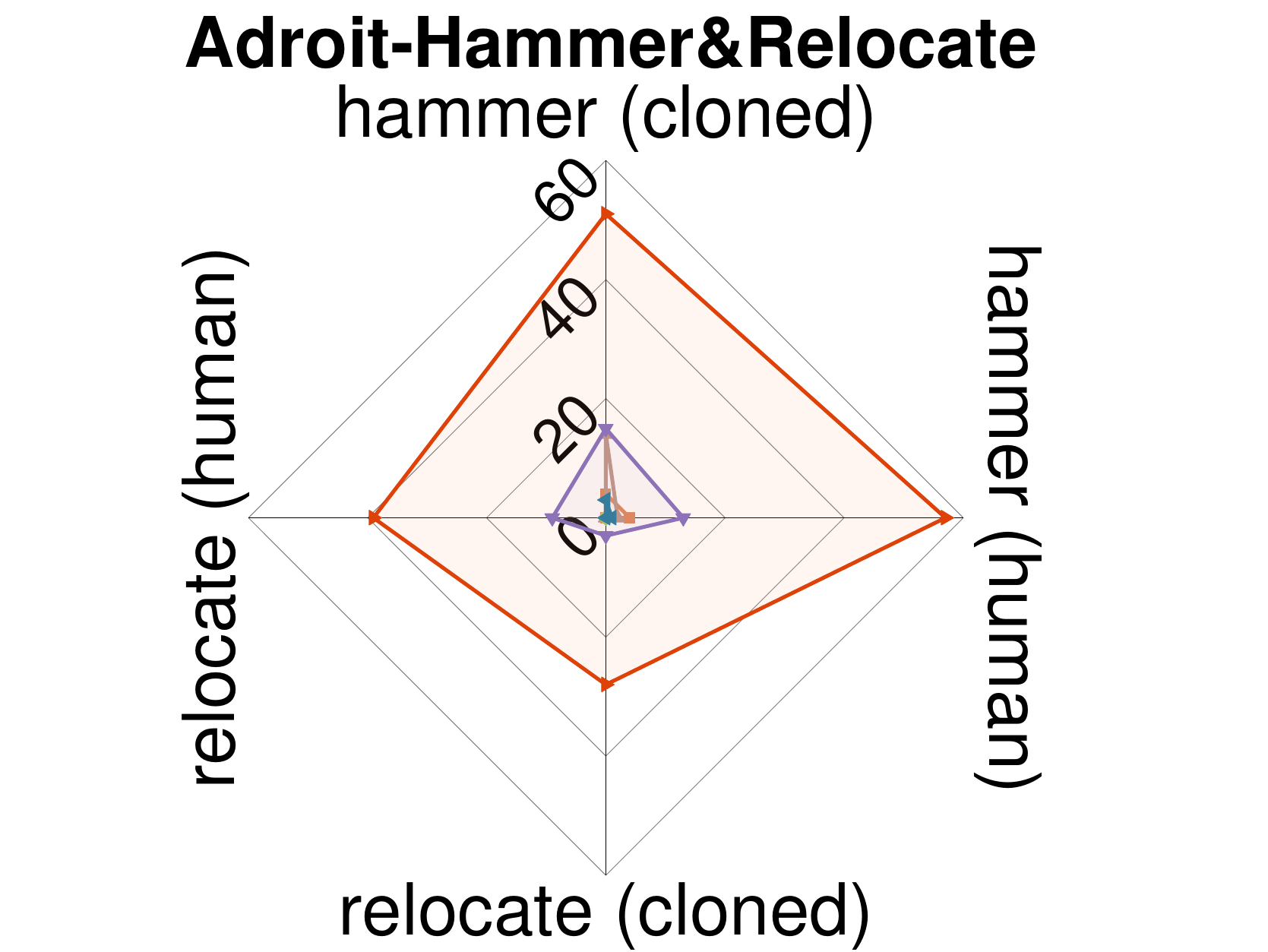}}
    \subfigure[\texttt{pen} \& \texttt{door}]{
    \includegraphics[width=0.24\textwidth]{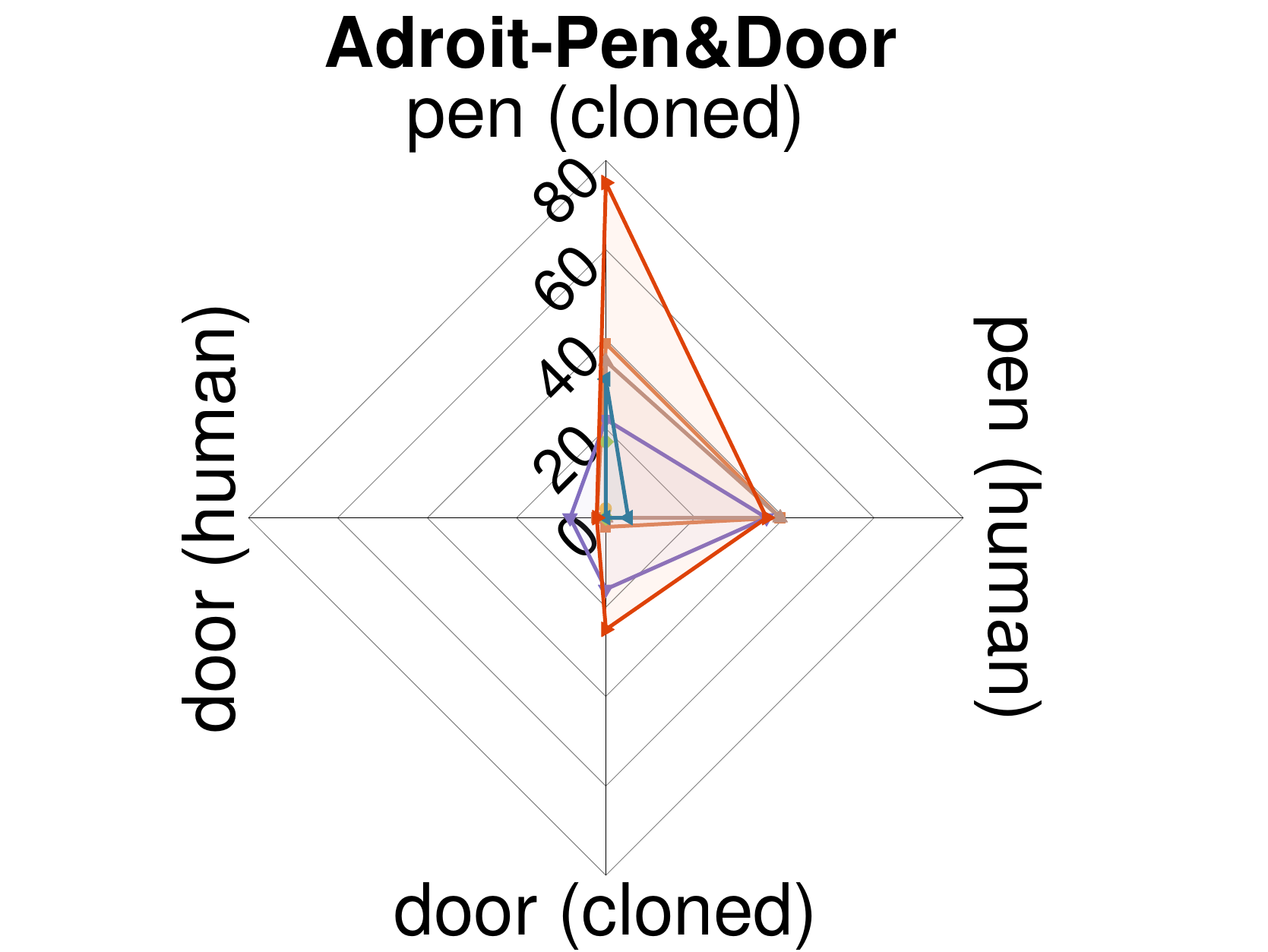}}
    %    \vspace{-4pt}
    %    \subfigure{
    %        \includegraphics[width=0.24\textwidth]{figure/imperfect/walker2d-random-p.pdf}}
    %    \hspace{-4pt}
    %    \subfigure{
    %        \includegraphics[width=0.24\textwidth]{figure/imperfect/walker2d-medium_replay-p.pdf}}
    %    \hspace{-4pt}
    %    \subfigure{
    %        \includegraphics[width=0.24\textwidth]{figure/imperfect/walker2d-medium-p.pdf}}
    %    \hspace{-4pt}
    %    \subfigure{
    %        \includegraphics[width=0.24\textwidth]{figure/imperfect/walker2d-medium_expert-p.pdf}}
    \vskip -0.1in
    \caption{Comparative performance under varying qualities of imperfect demonstrations. Each axis represents a specific data quality, where the values denote normalized scores of methods. The correspondence between methods and colored lines can be found in \cref{fig:demonstration}.}
    % \vskip -0.1in
    \label{fig:imperfect}
\end{figure*}

\begin{figure*}[htpb]
    % \vskip -0.5in
    \centering
    \subfigure[Runtime]{\label{subfig:runtime}\includegraphics[width=0.199\textwidth]{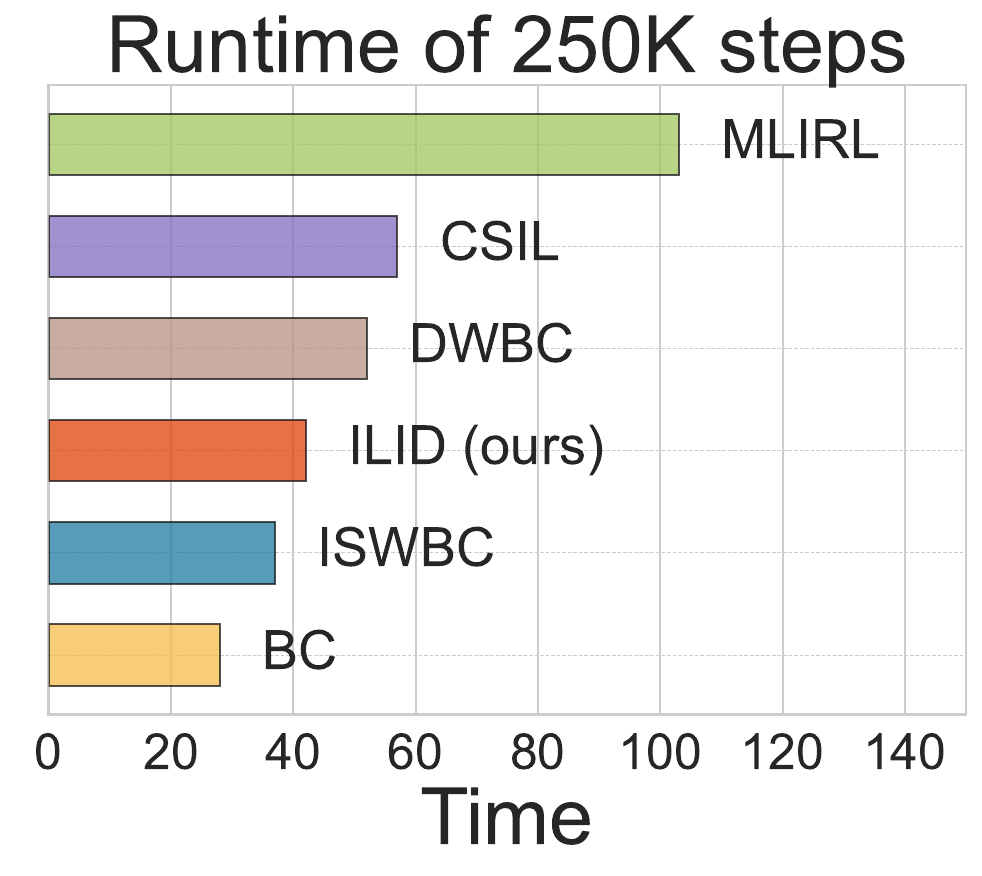}}
    \subfigure[Rollback steps]{\label{subfig:rollback}\includegraphics[width=0.199\textwidth]{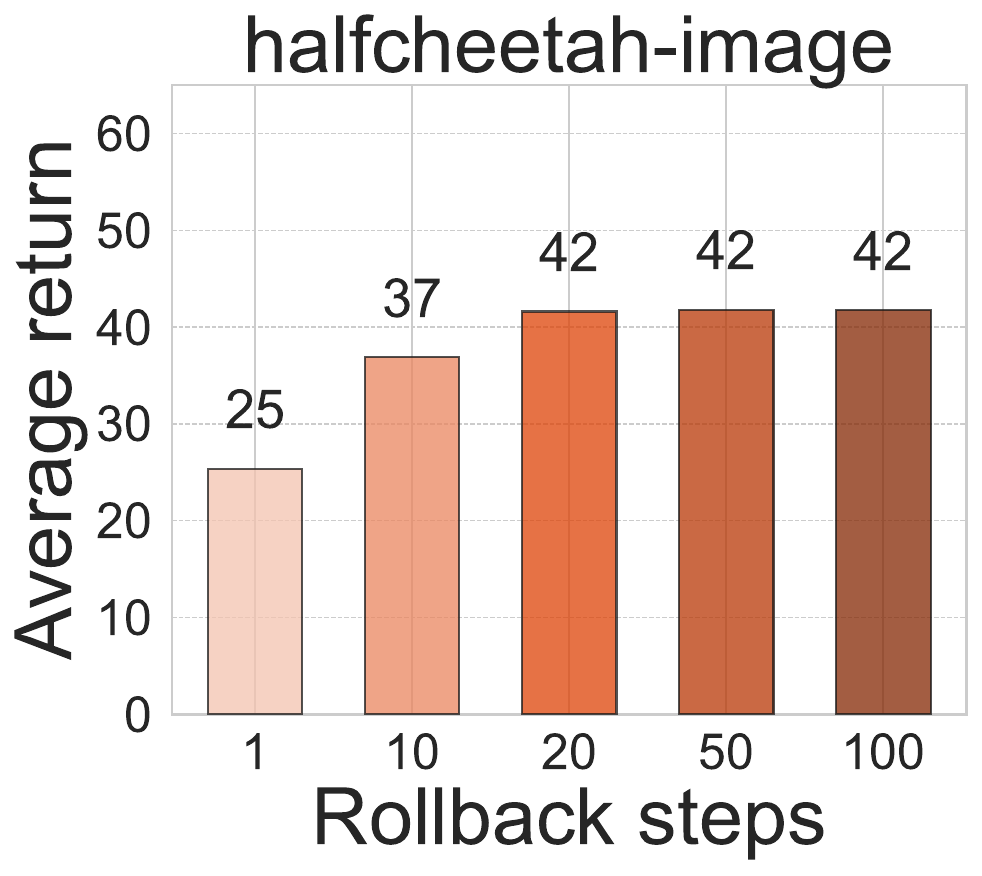}}
    % \hspace{2pt}
    \subfigure[Only ISW]{\label{subfig:dynamic_information}\includegraphics[width=0.19\textwidth]{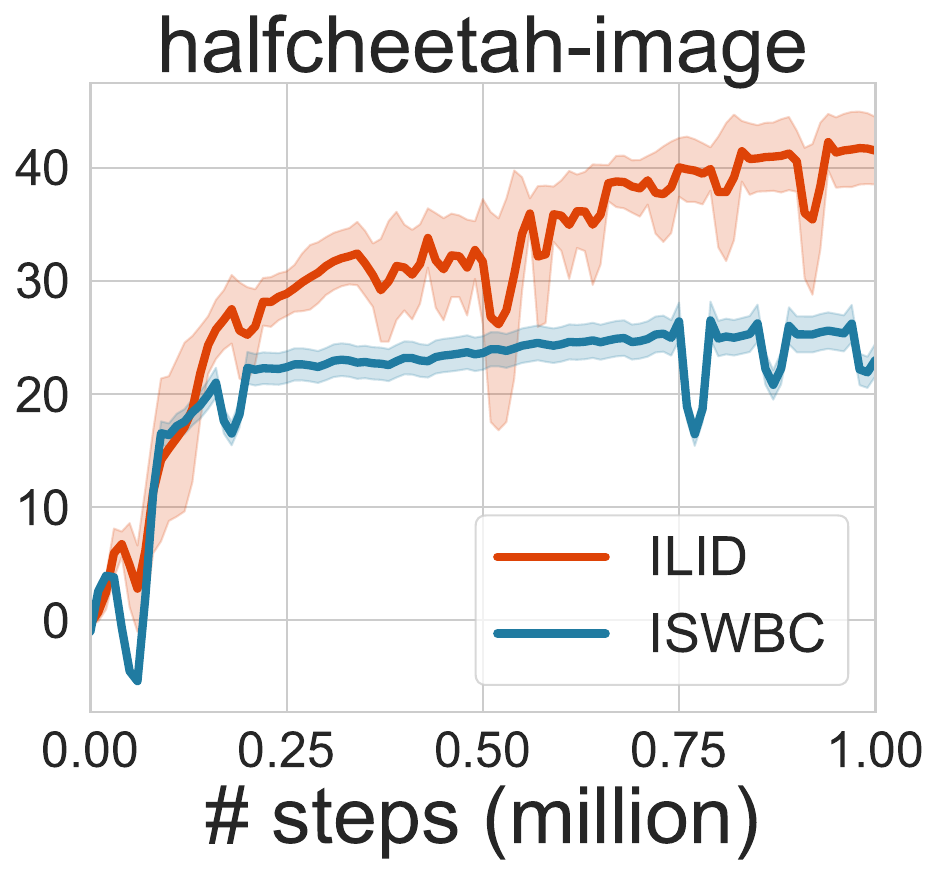}}
    % \hspace{1pt}
    \subfigure[On data selection]{\label{subfig:data_selection}\includegraphics[width=0.19\textwidth]{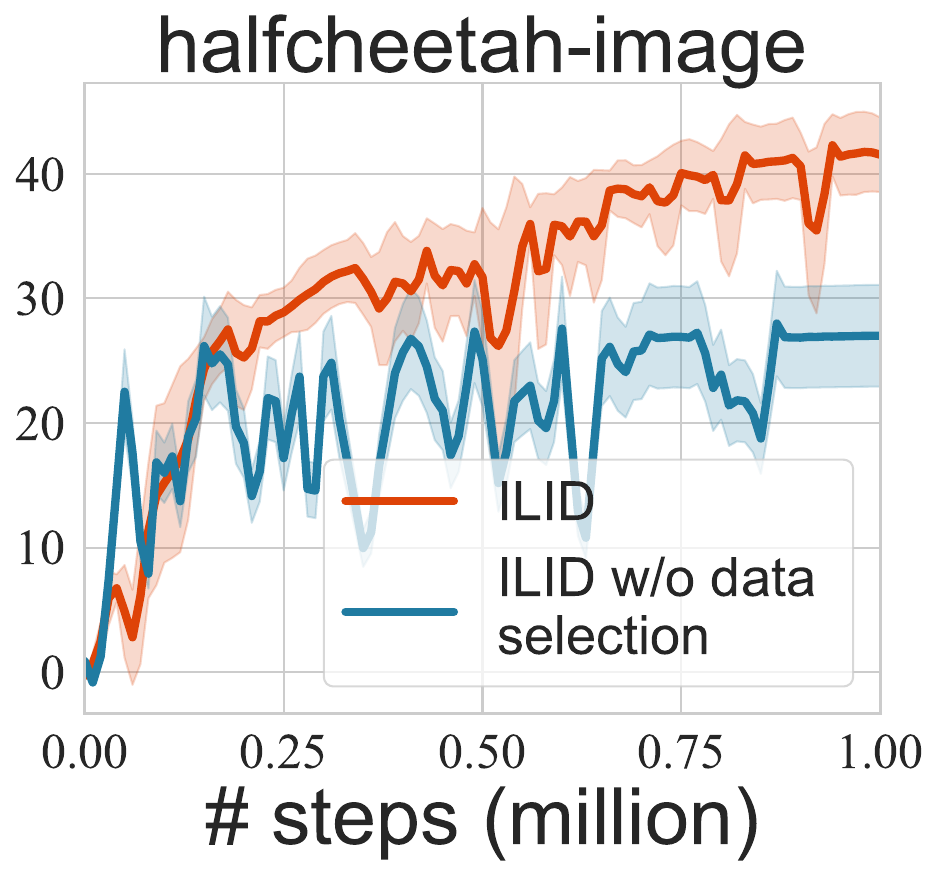}}
    % \hspace{1pt}
    \subfigure[On weighting]{\label{subfig:weighted_bc}\includegraphics[width=0.19\textwidth]{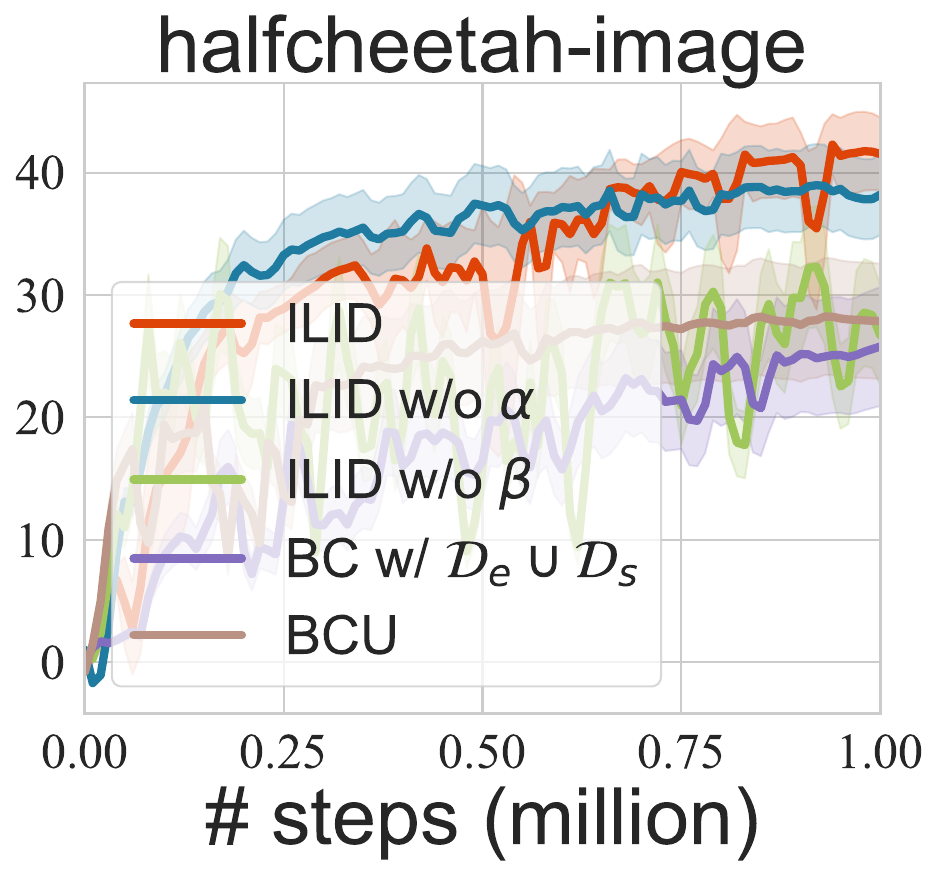}}
    \vskip -0.1in
    \caption{Ablation studies and comparative results of wall-clock runtime in policy learning. } 
    % \vskip -0.2in
    \label{fig:performance_ilid}
\end{figure*}

\textbf{Expert demonstrations.} To answer the second question, we run experiments with varying numbers of expert trajectories (ranging from 1 to 30 in MuJoCo and AntMaze, from 10 to 300 in Adroit and FrankaKitchen, and from 25 to 200 in vision-based domains). The data setup adheres to that of \cref{table:dataset_comparative}. We present selected results in \cref{fig:demonstration} and the full results in \cref{fig:demonstration_all} of \cref{sec:expert_demonstrations}. Our method, consistently requiring much fewer expert trajectories to attain expert performance, demonstrates great demonstration efficiency in comparison with prior methods.

\textbf{Quality and quantity of imperfect data.} For the second question, we also conduct experiments using imperfect demonstrations with varying qualities and quantities to test the robustness of \texttt{ILID}'s performance in behavior selection (the data setup is showcased in \cref{table:dataset_dataqualities}). Selected results are shown in \cref{fig:imperfect,fig:random_num}, with complete results provided in \cref{fig:imperfect_all,fig:random_num_bar} of \cref{sec:quality_quantity}. We find that \texttt{ILID} surpasses the
baselines in \textbf{20/24} settings, corroborating its efficacy
and superiority in the utilization of noisy data. Moreover, \cref{fig:random_num} underscores the importance of leveraging suboptimal data.

\textbf{Rollback steps.}  Regarding the fourth question, we vary $K$ from 1 to 100 and run experiments across all benchmarks. A selected result is shown in \cref{subfig:rollback} and full results are depicted in \cref{fig:rollback_all} of \cref{sec:rollback_steps}. The results clearly indicate that as $K$ increases, there is an initial improvement in performance; once it reaches a sufficiently large value, performance tends to stabilize. Considering that a larger rollback step leads to more selected behaviors capable of reaching expert states, this observation offers support for  \cref{hypo:resultant_state}. Importantly, the performance proves to be robust to a relatively large $K$, rendering \texttt{ILID} forgiving to the hyperparameter.

\textbf{Ablation studies.} We assess the effect of key components by ablating them on \textit{all} benchmarks, under the same setting as that of \cref{table:dataset_comparative} (see \cref{sec:ablation} for complete results). \textit{1) Only importance-sampling weighting.} Without the second term in Problem~(\ref{eq:ilid_objective}), \texttt{ILID} reduces to \texttt{ISWBC}. Yet, as shown in \cref{subfig:dynamic_information} and aforementioned comparative experiments, \texttt{ISWBC} does not suffice satisfactory performance. 
\textit{2) Effect of importance-sampling weighting.} We ablate importance weighting, and accordingly the first term of Problem~(\ref{eq:ilid_objective}) becomes the \texttt{BC} loss. The observed performance degradation in \cref{subfig:weighted_bc} suggests its benefits, which can enhance expert data support, particularly in continuous domains. 
\textit{3) Importance of data selection.} We ablate the data selection and replace $\mathcal{D}_s$ in Problem~(\ref{eq:ilid_objective}) by entire imperfect data of $\mathcal{D}_b$.  \cref{subfig:data_selection} corroborates \cref{hypo:resultant_state} and underscores the importance of our data selection scheme. 
\textit{4) Importance of $\beta(s,a)$.} As demonstrated in \cref{subfig:weighted_bc}, $\beta(s,a)$ assumes a crucial role in imitating selected data. The absence of $\beta(s,a)$ (setting $\beta(s,a)\equiv1$) renders training ineffective and unstable, due to behavior interference.
% \begin{enumerate}[leftmargin=*,topsep=0pt,itemsep=0pt,label=\textit{\arabic*)}]
% 	\item  
% \end{enumerate}

\textbf{Runtime.} We evaluate the runtime of \texttt{ILID} in comparison with baselines. \cref{subfig:runtime} demonstrates \texttt{ILID} remains comparable wall-clock time to \texttt{BC} (see \cref{sec:runtime}).

\section{Conclusion and Future Work}

In this paper, we introduce a simple yet effective data selection method along with a lightweight behavior cloning algorithm, which can explicitly harness the dynamics information in imperfect data, significantly enhancing the utilization of imperfect demonstrations. 
% In contrast to prior methods, we exploit the resultant states to assess the value of behaviors, which can explicitly harness the dynamics information and effectively extract both expert and beneficial diverse behaviors. 
%We provide necessary theoretical justifications for the proposed method and carry out extensive experiments to corroborate its efficacy and superiority over existing counterparts
%In contrast to prior methods, we exploit the resultant states to assess the value of behaviors, which can explicitly harness the dynamics information and effectively extract both expert and beneficial diverse behaviors. 
A limitation of this work is the requirement of \textit{state} overlap/similarity between the expert and imperfect data. While this assumption is weaker than most existing works (which necessitates \textit{state-action} overlap), there might be scenarios where only the expert can reach expert states. In general, it is hard to assess suboptimal behaviors persuasively if none of them bear a state resemblance to the expert's. A potential compromise is to involve prior information like the quality of imperfect data in the problem. This opens up an interesting future direction on offline IL with multi-quality demonstrations. 

% In addition, we plan to establish theoretical guarantees for our method in stochastic MDPs. 
% and explore whether it can benefit offline RL in terms of data selection or policy optimization.

\section*{Acknowledgments}
This research was supported in part by the National Natural Science Foundation of China under Grant No. 62302260, 62341201, 62122095, 62072472 and 62172445, by the National Key R\&D Program of China under Grant No. 2022YFF0604502, by China Postdoctoral Science Foundation under Grant No. 2023M731956, and by a grant from the Guoqiang Institute, Tsinghua University.

\section*{Impact Statement}

This paper presents work whose goal is to advance the field of Reinforcement Learning. There are many potential societal consequences of our work, none which we feel must be specifically highlighted here.

\bibliography{references}
\bibliographystyle{icml2024}

%%%%%%%%%%%%%%%%%%%%%%%%%%%%%%%%%%%%%%%%%%%%%%%%%%%%%%%%%%%%%%%%%%%%%%%%%%%%%%%
%%%%%%%%%%%%%%%%%%%%%%%%%%%%%%%%%%%%%%%%%%%%%%%%%%%%%%%%%%%%%%%%%%%%%%%%%%%%%%%
% APPENDIX
%%%%%%%%%%%%%%%%%%%%%%%%%%%%%%%%%%%%%%%%%%%%%%%%%%%%%%%%%%%%%%%%%%%%%%%%%%%%%%%
%%%%%%%%%%%%%%%%%%%%%%%%%%%%%%%%%%%%%%%%%%%%%%%%%%%%%%%%%%%%%%%%%%%%%%%%%%%%%%%
\clearpage
\appendix
\onecolumn

\section{Experimental Setup}
\label{sec:expertimental_details}

In this section, we present full experimental details for reproducibility. 

\subsection{Benchmarks}
\label{sec:benchmarks}

We evaluate our method on a number of environments (Robomimic, MuJoCo, Adroit, FrankaKitchen, and AntMaze) which are widely used in prior studies~\citep{nakamoto2023cal,watson2023coherent}. We elaborate in what follows.
\begin{itemize}[leftmargin=*,itemsep=0pt,topsep=0pt]
	\item \textbf{Vision-based Robomimic.} The Robomimic tasks (\texttt{lift}, \texttt{can}, \texttt{square}) involve controlling a 7-DoF simulated hand robot~\citep{robomimic2021}, with pixelized observations as shown in \cref{fig:robomimic}. The robot is tasked with lifting objects, picking and placing cans, and picking up a square nut to place it on a rod from random initializations. 
	\begin{figure}[ht]
		%\vspace{-0.5em}
		\centering  \includegraphics[width=0.55\linewidth]{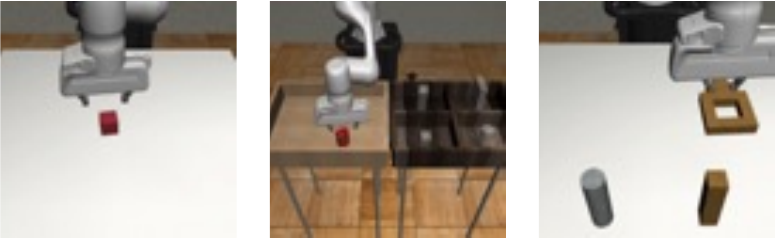}
		\vskip -0.1in
		\caption{Observations of vision-based Robomimic tasks. From left to right: \texttt{lift}, \texttt{can}, \texttt{square}.}
		\vskip -0.1in
		\label{fig:robomimic}
	\end{figure}
	
	\item \textbf{Vision-based MuJoCo.} The MuJoCo locomotion tasks (\texttt{ant}, \texttt{hopper}, \texttt{halfcheetah}, \texttt{walker2d}) are popular benchmarks used in existing work. In addition to the standard setting, we also consider vision-based MuJoCo tasks which uses the image observation as input (see \cref{fig:mujoco_vision}).
	\begin{figure}[ht]
		%\vspace{-0.5em}
		\centering  \includegraphics[width=0.7\linewidth]{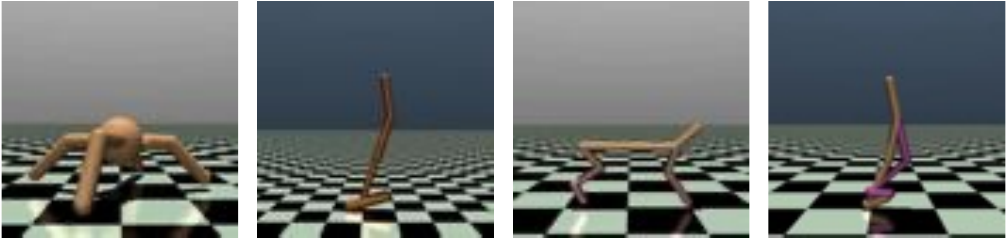}
		\vskip -0.1in
		\caption{Observations of vision-based MuJoCo tasks. From left to right: \texttt{ant}, \texttt{hopper}, \texttt{halfcheetah}, \texttt{walker2d}.}
		\vskip -0.1in
		\label{fig:mujoco_vision}
	\end{figure}
	
	\item \textbf{Adroit.} The Adroit tasks (\texttt{hammer}, \texttt{door}, \texttt{pen}, and \texttt{relocate}) \cite{rajeswaran2017learning} involve controlling a  28-DoF hand with five fingers tasked with hammering a nail, opening a door, twirling a pen, or picking up and moving a ball. 
	%The agent obtains a sparse binary +1/0 reward if it succeeds in solving the task.
	\begin{figure}[ht]
		%\vspace{-0.5em}
		\centering  \includegraphics[width=0.7\linewidth]{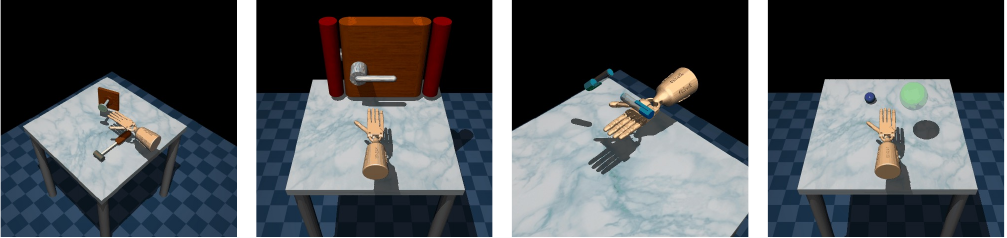}
		\vskip -0.1in
		\caption{Adroit tasks: \texttt{hammer}, \texttt{door}, \texttt{pen}, and \texttt{relocate} (from left to right).}
		\vskip -0.1in
		\label{fig:adroit}
	\end{figure}
	
	\item \textbf{FrankaKitchen.} The FrankaKitchen tasks (\texttt{complete}, \texttt{partial}, \texttt{undirect}), proposed by \citet{gupta2019relay}, involve controlling a 9-DoF Franka robot in a kitchen environment containing several common household items: a microwave, a kettle, an overhead light, cabinets, and an oven. The goal of each task is to interact with the items to reach a desired goal configuration.
	%with only binary 0-1 completion reward forevery object that attains the target configuration. 
	In the \texttt{undirect} task, the robot requires opening the microwave. In the \texttt{partial} task, the robot must first open the microwave and subsequently move the kettle. In the \texttt{complete} task, the robot needs to open the microwave, move the kettle, flip the light switch, and slide open the cabinet door sequentially (see \cref{fig:kitchen}). These tasks are especially challenging since they require composing parts of trajectories, precise long-horizon manipulation, and handling human-provided teleoperation data.
	\begin{figure}[t]
		%\vspace{-0.5em}
		\centering  \includegraphics[width=0.95\linewidth]{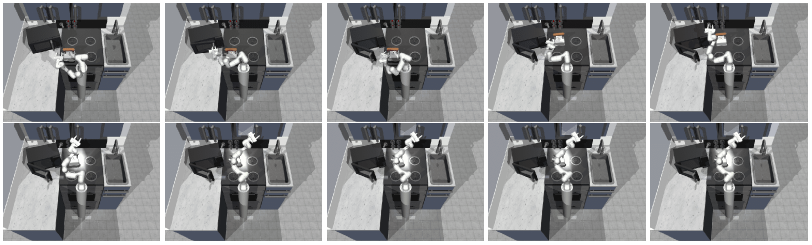}
		\vskip -0.1in
		\caption{Visualized success for opening the microwave, moving the kettle, turning on the light switch, and sliding the slider.}
		%\vspace{-1.0em}
		\label{fig:kitchen}
		%\vskip -0.1in
	\end{figure}
	
	\item \textbf{AntMaze.} The AntMaze tasks require controlling an 8-Degree of Freedom (DoF) quadruped robot to move from a starting point to a fixed goal location~\citep{fu2020d4rl}. 
	%The reward is +1 if the agent reaches within a pre-specified small radius around the goal and 0 otherwise. 
	Three maze layouts (\texttt{umaze}, \texttt{medium}, and \texttt{large}) are provided from small to large.
	\begin{figure}[ht]
		%\vspace{-0.5em}
		\centering  \includegraphics[width=0.5\linewidth]{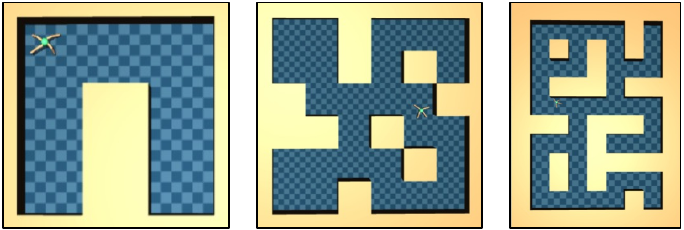}
		\vskip -0.1in
		\caption{AntMaze with three maze layouts, \texttt{umaze}, \texttt{medium}, and \texttt{large} (from left to right).}
		\label{fig:antmaze_layout}
		\vskip -0.1in
	\end{figure}
\end{itemize}

%For high-dimensional benchmarks, we employ domains such as AntMaze (3 tasks), Gym-MuJoCo (4 tasks), Adroit (4 tasks), and Kitchen (3 tasks) from the D4RL dataset \cite{fu2020d4rl}. In the realm of vision-based benchmarks, we integrate domains related to robot manipulation (3 tasks) from the robomimic datasets \cite{robomimic2021}, along with Gym-MuJoCo (4 tasks).

Detailed information about the environments including observation space, action space, and expert performance is provided in \cref{tab:env_performance,tab:env_image_performance}, where expert and random scores are averaged over 1000 episodes.

\begin{table}[ht]
	\vskip -0.2in
	\centering
	\caption{Details of continuous-control tasks.}
	\label{tab:env_performance}
	\vskip 0.1in
	% \resizebox{0.6\textwidth}{!}{
	{\small \begin{threeparttable}
	\begin{tabular}{l|c|c|r|r}
		Task & State dim. & Action dim. & \texttt{random}\tnote{*} & \texttt{expert}\tnote{*} \\ 
		\hline
		\texttt{ant} & $27$ & $8$ & $-325.60$ & $3879.70$ \\
		\texttt{halfcheetah} & $17$ & $6$ & $-280.18$ & $12135.00$ \\
		\texttt{hopper} & $11$ & $3$ & $-20.27$ & $3234.30$ \\
		\texttt{walker2d} & $17$ & $6$ & $1.63$ & $4592.30$ \\
		\texttt{antmaze} & $27$ & $8$ & $0.00$ & $1.00$ \\
		\texttt{door} & $39$ & $28$ & $-56.51$ & $2880.57$ \\
		\texttt{hammer} & $46$ & $26$ & $-274.86$ & $12794.13$ \\
		\texttt{pen} & $45$ & $24$ & $96.26$ & $3076.83$ \\
		\texttt{relocate} & $39$ & $30$ & $-6.43$ & $4233.88$ \\
		\texttt{FrankaKitchen} & $59$ & $9$ & $0.00$ & $1.00 $ \\
		\hline
	\end{tabular}
        {\scriptsize 
	\begin{tablenotes}
		\item[*] Average scores over 1000 trajectories of \texttt{expert} and \texttt{random}.
	\end{tablenotes}}
	\end{threeparttable}}
 %}
	\vskip -0.2in
\end{table}

\begin{table}[ht]
	\vskip -0.15in
	\centering
	\begin{minipage}{.65\textwidth}
	\centering
	\caption{Details of vision-based tasks.}
	\label{tab:env_image_performance}
	\vskip 0.1in
	%\resizebox{0.6\textwidth}{!}{
	{\small\begin{tabular}{l|c|c|r|r}
		Task & \multicolumn{1}{l|}{State dim.} & Action dim. & \texttt{random} & \texttt{expert} \\ 
		\hline
		\texttt{ant} & ($84\times84$) & $8$ & $-325.60$ & $3879.70$ \\
		\texttt{halfcheetah} & ($84\times84$) & $6$ & $-280.18$ & $12135.00$ \\
		\texttt{hopper} & ($84\times84$) & $3$ & $-20.27$ & $3234.30$ \\
		\texttt{walker2d} & ($84\times84$) & $6$ & $1.63$ & $4592.30$ \\
		\texttt{lift} & ($84\times84$) & $7$ & $0.00$ & $1.00$ \\
		\texttt{can} & ($84\times84$) & $7$ & $0.00$ & $1.00$ \\
		\texttt{square} & ($84\times84$) & $7$ & $0.00$ & $1.00$ \\
		\hline
	\end{tabular}}%}
	\end{minipage}
	\begin{minipage}{.3\textwidth}
		\centering
		\caption{Hyperparameters across tasks.} 
		\label{table:parameter}
		\vskip 0.1in
		{\small \begin{tabular}{l|l} 
			%\toprule
			Hyperparameter & Value \\ 
			%\midrule
			\hline
			\# Neural net layers & 2 \\
			Optimizer & \texttt{Adam} \\
			Activation & \texttt{ReLU} \\
			Batchsize & 256 \\
			All learning rates& 1e-5 \\
			Threshold $\sigma$ & 0.2 \\
			Rollback $K$ & 20 \\
			\hline
		\end{tabular}}
	\end{minipage}
	\vskip -0.2in
\end{table}

\subsection{Datasets}
\label{sec:datasets}

We employ \texttt{D4RL}~\citep{fu2020d4rl} for AntMaze, MuJoCo, Adroit, and FrankaKitchen, and use \texttt{robomimic}~\citep{robomimic2021} for Robomimic. \cref{table:dataset_comparative,table:dataset_dataqualities} specify the data setup used for each task across experiments. Of note, we construct vision-based MuJoCo datasets using the same method as \citet{fu2020d4rl}: the expert and imperfect data use video samples from a policy trained to completion with \texttt{SAC}~\citep{haarnoja18soft} and a randomly initialized policy, respectively.

\subsection{Baselines}

We evaluate our method against six strong baseline methods in offline IL:
\emph{\textbf{1)} {B}ehavior {C}loning with {E}xpert Data} (\texttt{BCE}), the standard \texttt{BC} trained only on expert demonstrations; 
\emph{\textbf{2)} {B}ehavior {C}loning with {U}nion Data} (\texttt{BCU}), \texttt{BC} trained on union data;
\emph{\textbf{3)} {D}iscriminator-{W}eighted {B}ehavioral {C}loning} (\texttt{DWBC}) \citep{xu2022discriminator}, an offline IL method that leverages suboptimal demonstrations by jointly training a discriminator to re-weight the \texttt{BC} objective (\url{https://github.com/ryanxhr/DWBC});
\emph{\textbf{4)} {I}mportance-{S}ampling-{W}eighted {B}ehavioral {C}loning}  (\texttt{ISWBC}) \citep{li2023imitation}, an offline IL method that adopts importance sampling to enhance \texttt{BC} (\url{https://github.com/liziniu/ISWBC});
\emph{\textbf{5)} {C}oherent {S}oft {I}mitation {L}earning} (\texttt{CSIL}) \citep{watson2023coherent}, a model-free IRL method that learns a shaped reward function by entropy-regularized \texttt{BC} (\url{https://joemwatson.github.io/csil});
\emph{\textbf{6)} {M}aximum {L}ikelihood-{I}nverse {R}einforcement {L}earning} (\texttt{MLIRL}) \citep{zeng2023demonstrations}, a recent model-based offline IRL algorithm based on bi-level optimization (\url{https://github.com/Cloud0723/Offline-MLIRL}).

We implement and tune baseline methods based on their publicly available implementatinons with the same policy network structures. The tuned codes are included in the supplementary material.

\subsection{Implementation}

Our method is straightforward to implement and robust to hyperparameters (which are consistent across all benchamarks and settings). We represent the policy as a 2-layer feedforward neural network with 256 hidden units, \texttt{ReLU} activation functions, and Tanh Gaussian outputs. Analogously, the discriminators are represented as a 2-layer feedforward neurl net with 256 hidden units, \texttt{ReLU} activations with the output clipped to $[0.1,0.9]$. For vision-based tasks, we change the network architectures to a simple CNN, consisting of two convolutional layers, each with a $3\times3$ convolutional kernel and $2\times2$ max pooling. We adopt \texttt{Adam} as the optimizer. All learning rates and batchsizes are set to 1e-5 and 256, respectively. The thresholds $\sigma$ for identifying expert states is set to $0.2$, and the rollback step $K$ is set to 20. The hyperparameters are summarized in \cref{table:parameter}.

We implement our code using Pytorch 1.8.1, built upon the open-source framework of offline RL algorithms, provided at \url{https://github.com/tinkoff-ai/CORL} (under the Apache-2.0 License) and the implementation of DWBC, provided at \url{https://github.com/ryanxhr/DWBC} (under the MIT License).
All the experiments are run on Ubuntu 20.04.2 LTS with 8 NVIDIA GeForce RTX 4090 GPUs.

\begin{sidewaystable}
	\caption{Data used in the comparative experiment.}
	\label{table:dataset_comparative}
	\vskip 0.1in
	\centering
	\resizebox{\textwidth}{!}{
	{\small\begin{threeparttable}
	\begin{tabular}{l|l|l|l|l|c|l|cl}
		Domain & Dataset & Traj. length & Task & Expert data & \multicolumn{1}{l|}{\# Expert traj.} & Imperfect data & \multicolumn{1}{l}{\# Imperfect traj.} &  \\ 
		\cline{1-8}
		\multirow{4}{*}{MuJoCo} & \multirow{4}{*}{\texttt{D4RL}} & \multirow{4}{*}{$\le$\,1000} & \texttt{ant} & \texttt{ant-expert-v2} & 1 & \texttt{ant-random-v2} & 1000 &  \\
		&  &  & \texttt{hopper} & \texttt{hopper-expert-v2} & 1 & \texttt{hopper-random-v2} & 1000 &  \\
		&  &  & \texttt{halfcheetah} & \texttt{halfcheetah-expert-v2} & 1 & \texttt{halfcheetah-random-v2} & 1000 &  \\
		&  &  & \texttt{walker2d} & \texttt{walker2d-expert-v2} & 1 & \texttt{walker2d-random-v2} & 1000 &  \\ 
		\cline{1-8}
		\multirow{3}{*}{AntMaze} & \multirow{3}{*}{\texttt{D4RL}} & \multirow{3}{*}{$\le$\,100} & \texttt{umaze} & \texttt{antmaze-umaze-v0} & 1 & \texttt{antmaze-umaze-diverse-v0} & 1000 &  \\
		&  &  & \texttt{medium} & \texttt{antmaze-medium-v0} & 1 & \texttt{antmaze-medium-diverse-v0} & 1000 &  \\
		&  &  & \texttt{large} & \texttt{antmaze-large-v0} & 1 & \texttt{antmaze-large-diverse-v0} & 1000 &  \\ 
		\cline{1-8}
		\multirow{4}{*}{Adroit} & \multirow{4}{*}{\texttt{D4RL}} & \multirow{4}{*}{$\le$\,100} & \texttt{pen} & \texttt{pen-expert-v1} & 10 & \texttt{pen-cloned-v1} & 1000 &  \\
		&  &  & \texttt{hammer} & \texttt{hammer-expert-v1} & 10 & \texttt{hammer-cloned-v1} & 1000 &  \\
		&  &  & \texttt{door} & \texttt{door-expert-v1} & 10 & \texttt{door-cloned-v1} & 1000 &  \\
		&  &  & \texttt{relocate} & \texttt{relocate-expert-v1} & 10 & \texttt{relocate-cloned-v1} & 1000 &  \\ 
		\cline{1-8}
		\multirow{3}{*}{FrankaKitchen} & \multirow{3}{*}{\texttt{D4RL}} & \multirow{3}{*}{$\le$\,280} & \texttt{complete} & \texttt{kitchen-complete-v0} & 10 & \texttt{kitchen-mixed-v0} & 1000 &  \\
		&  &  & \texttt{partial} & \texttt{kitchen-partial-v0} & 10 & \texttt{kitchen-mixed-v0} & 1000 &  \\
		&  &  & \texttt{indirect} & \texttt{kitchen-partial-v0} & 10 & \texttt{kitchen-mixed-v0} & 1000 &  \\ 
		\cline{1-8}
		\multirow{3}{*}{Robomimic} & \multirow{3}{*}{\texttt{robomimic}} & \multirow{3}{*}{$\le$\,500} & \texttt{lift} & \texttt{lift-proficient-human} & 25 & \texttt{lift-paired-bad} & 1000 &  \\
		&  &  & \texttt{can-paired-bad} & \texttt{can-proficient-human} & 25 & \texttt{can-paired-bad} & 1000 &  \\
		&  &  & \texttt{square-paired-bad} & \texttt{square-proficient-human} & 25 & \texttt{square-paired-bad} & 1000 &  \\ 
		\cline{1-8}
		\multirow{4}{*}{MuJoCo (vision)} & \multirow{4}{*}{-} & \multirow{4}{*}{$\le$\,1000} & \texttt{ant} & \texttt{ant-expert-vision} & 25 & \texttt{ant-random-vision} & 1000 &  \\
		&  &  & \texttt{hopper} & \texttt{hopper-expert-vision} & 25 & \texttt{hopper-random-vision} & 1000 &  \\
		&  &  & \texttt{halfcheetah} & \texttt{halfcheetah-expert-vision} & 25 & \texttt{halfcheetah-random-vision} & 1000 &  \\
		&  &  & \texttt{walker2d} & \texttt{walker2d-expert-vision} & 25 & \texttt{walker2d-random-vision} & 1000 &  \\ 
		\cline{1-8}
		\end{tabular}
		{\scriptsize \begin{tablenotes}
			\item[1] In vision-based MuJoCo, we collect the \texttt{expert-vision} and \texttt{random-vision} data use video samples from a policy trained to completion with \texttt{SAC} and a randomly \\
			initialized policy, respectively.
		\end{tablenotes}}
		\end{threeparttable}}}
\end{sidewaystable}

\begin{sidewaystable}
	\caption{Data used in the experiment on varying data qualities.}
	\label{table:dataset_dataqualities}
	\vskip 0.1in
	\centering
	\resizebox{.85\textwidth}{!}{
	{\small \begin{tabular}{l|l|c|l|c|l|r}
		Task & Trajectory length & \multicolumn{1}{l|}{\# Expert trjectories} & Expert data & \multicolumn{1}{l|}{\# Imperfect trjectories} & Imperfect data & Score \\ 
		\hline
		\multirow{4}{*}{\texttt{ant}} & \multirow{4}{*}{$\le1000$} & \multirow{4}{*}{1} & \multirow{4}{*}{\texttt{ant-expert-v2}} & \multirow{4}{*}{1000} & \texttt{ant-random-v2} & 9.2 \\
		&  &  &  &  & \texttt{ant-medium-replay-v2} & 19.0 \\
		&  &  &  &  & \texttt{ant-medium-v2} & 80.3 \\
		&  &  &  &  & \texttt{ant-medium-expert-v2} & 90.1 \\ 
		\hline
		\multirow{4}{*}{\texttt{halfcheetah}} & \multirow{4}{*}{$\le1000$} & \multirow{4}{*}{1} & \multirow{4}{*}{\texttt{halfcheetah-expert-v2}} & \multirow{4}{*}{1000} & \texttt{ant-random-v2} & -0.1 \\
		&  &  &  &  & \texttt{ant-medium-replay-v2} & 7.3 \\
		&  &  &  &  & \texttt{ant-medium-v2} & 40.7 \\
		&  &  &  &  & \texttt{ant-medium-expert-v2} & 70.3 \\ 
		\hline
		\multirow{4}{*}{\texttt{hopper}} & \multirow{4}{*}{$\le1000$} & \multirow{4}{*}{1} & \multirow{4}{*}{\texttt{hopper-expert-v2}} & \multirow{4}{*}{1000} & \texttt{ant-random-v2} & 1.2 \\
		&  &  &  &  & \texttt{ant-medium-replay-v2} & 6.8 \\
		&  &  &  &  & \texttt{ant-medium-v2} & 44.1 \\
		&  &  &  &  & \texttt{ant-medium-expert-v2} & 72.0 \\ 
		\hline
		\multirow{4}{*}{\texttt{walker2d}} & \multirow{4}{*}{$\le1000$} & \multirow{4}{*}{1} & \multirow{4}{*}{\texttt{walker2d-expert-v2}} & \multirow{4}{*}{1000} & \texttt{ant-random-v2} & 0.0 \\
		&  &  &  &  & \texttt{ant-medium-replay-v2} & 13.0 \\
		&  &  &  &  & \texttt{ant-medium-v2} & 62.0 \\
		&  &  &  &  & \texttt{ant-medium-expert-v2} & 81.0 \\ 
		\hline
		\multirow{2}{*}{\texttt{hammer}} & \multirow{2}{*}{$\le100$} & \multirow{2}{*}{10} & \multirow{2}{*}{\texttt{hammer-expert-v1}} & \multirow{2}{*}{1000} & \texttt{pen-human-v1} & 2.7 \\
		&  &  &  &  & \texttt{pen-cloned-v1} & 0.5 \\ 
		\hline
		\multirow{2}{*}{\texttt{pen}} & \multirow{2}{*}{$\le100$} & \multirow{2}{*}{10} & \multirow{2}{*}{\texttt{pen-expert-v1}} & \multirow{2}{*}{1000} & \texttt{hammer-human-v1} & 2.1 \\
		&  &  &  &  & \texttt{hammer-cloned-v1} & 59.9 \\ 
		\hline
		\multirow{2}{*}{\texttt{door}} & \multirow{2}{*}{$\le100$} & \multirow{2}{*}{10} & \multirow{2}{*}{\texttt{door-expert-v1}} & \multirow{2}{*}{1000} & \texttt{door-human-v1} & 2.6 \\
		&  &  &  &  & \texttt{door-cloned-v1} & -0.1 \\ 
		\hline
		\multirow{2}{*}{\texttt{relocate}} & \multirow{2}{*}{$\le100$} & \multirow{2}{*}{10} & \multirow{2}{*}{\texttt{relocate-expert-v1}} & \multirow{2}{*}{1000} & \texttt{relocate-human-v1} & 2.3 \\
		&  &  &  &  & \texttt{relocate-cloned-v1} & -0.1 \\
		\hline
	\end{tabular}}}
\end{sidewaystable}

\clearpage

\section{Complete Experimental Results}

This section provides complete experimental results to answer the questions raised in \cref{sec:experiment}.

\subsection{Comparative Experiments}
\label{comparative_results}

To answer the first question, we evaluate \texttt{ILID}'s performance in each task using limited expert demonstrations and a set of low-quality imperfect data. For example, in the MuJoCo domain, we sample 1 \texttt{expert} trajectory and 1000 \texttt{random} trajectories from \texttt{D4RL} as the expert and imperfect data, respectively (refer to \cref{table:dataset_comparative} for the complete data setup). Comparative results are presented in \cref{tab:performance}, and learning curves are depicted in \cref{fig:convergence_all_1,fig:convergence_all_2}. We find \texttt{ILID} consistently outperforms baselines in \textbf{20/21} tasks often by a significant margin while enjoying fast and stabilized convergence. Due to limited state coverage of expert data and low quality of imperfect data, \texttt{BCE} and \texttt{BCU} fail to fulfill most of the tasks. This reveals \texttt{ILID}'s effectiveness in extracting and leveraging positive behaviors from imperfect demonstrations. \texttt{DWBC} and \texttt{ISWBC} exhibit similar performances, demonstrating relative success in MuJoCo but facing challenges in robotic manipulation and maze domains, which require precise long-horizon manipulation. This is because the similarity-based behavior selection confines their training data to the expert states with narrow coverage, rendering them prone to error compounding. In contrast, \texttt{ILID}, utilizing dynamics information, can \textit{stitch} parts of trajectories and empower the policy to recover from mistakes. In addition, the IRL methods struggle in high-dimensional environments owing to reward extrapolation and world model estimates.

\begin{figure}[ht]
    \centering
    \subfigure{
        \includegraphics[width=0.24\textwidth]{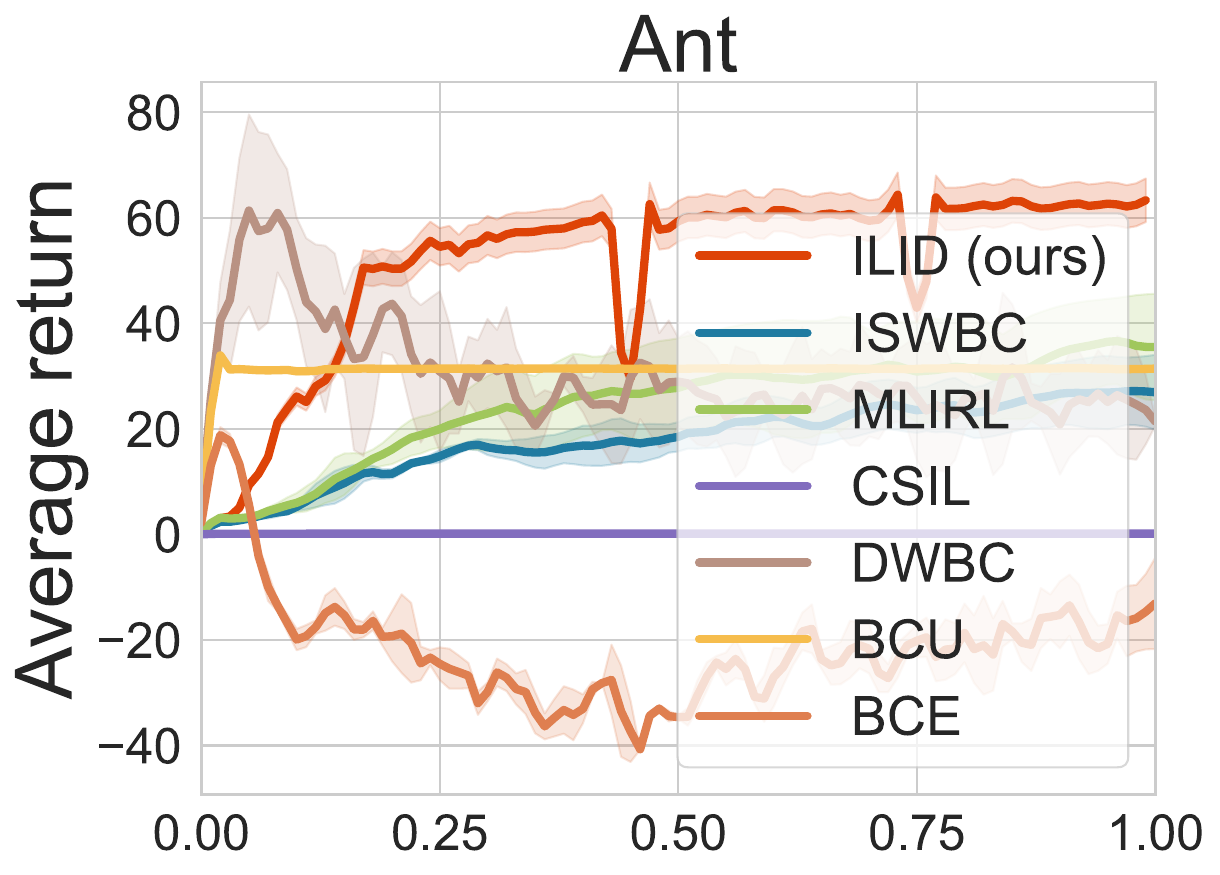}}
    \hspace{-4pt}
    \subfigure{
        \includegraphics[width=0.24\textwidth]{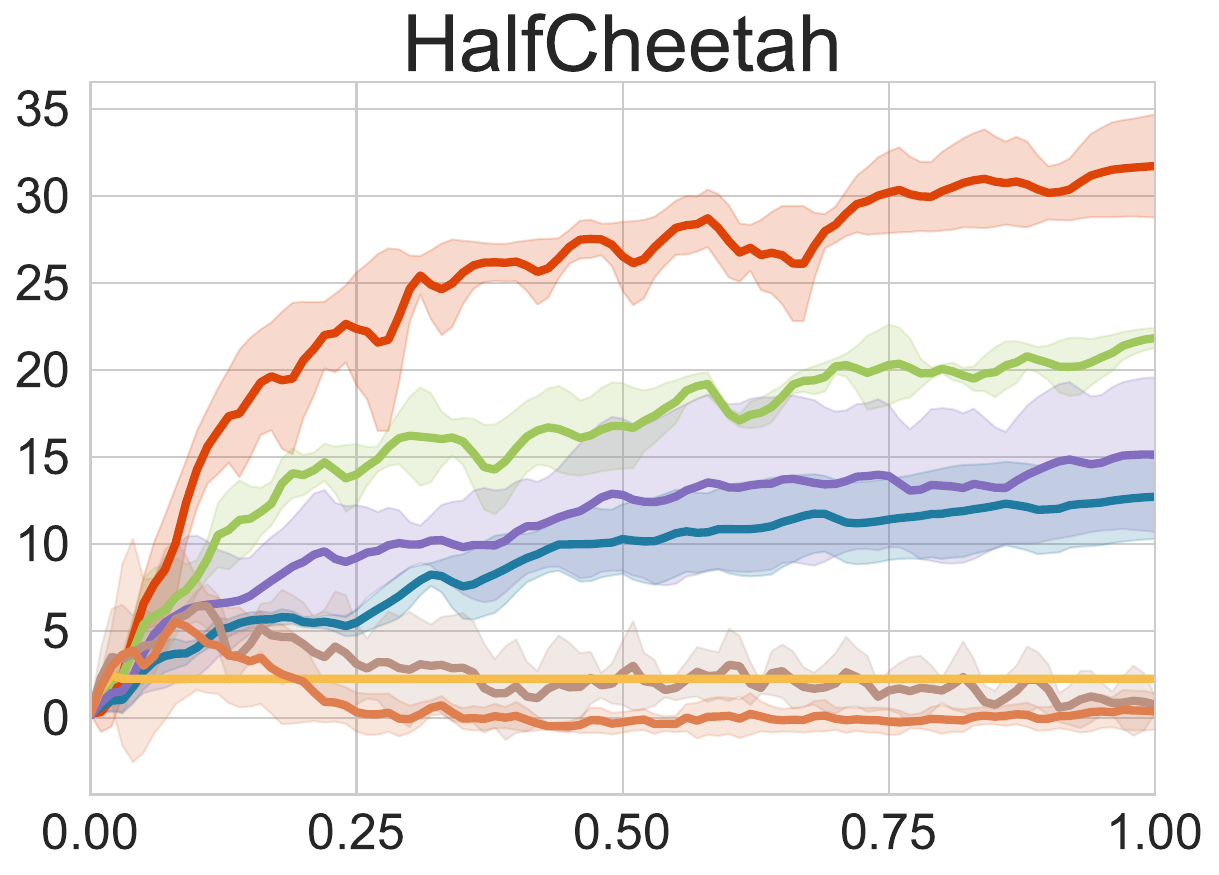}}
    \hspace{-4pt}
    \subfigure{
        \includegraphics[width=0.24\textwidth]{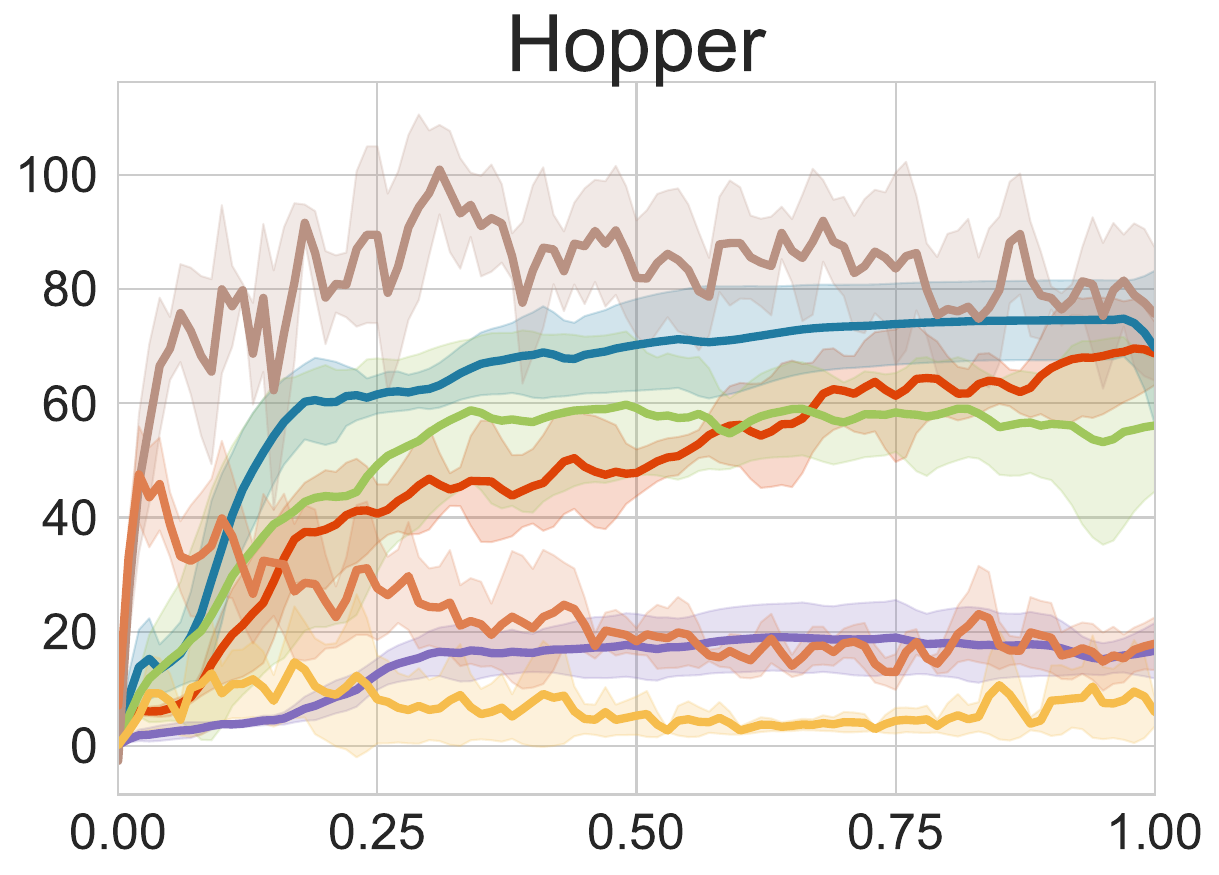}}
    \hspace{-4pt}
    \subfigure{
        \includegraphics[width=0.24\textwidth]{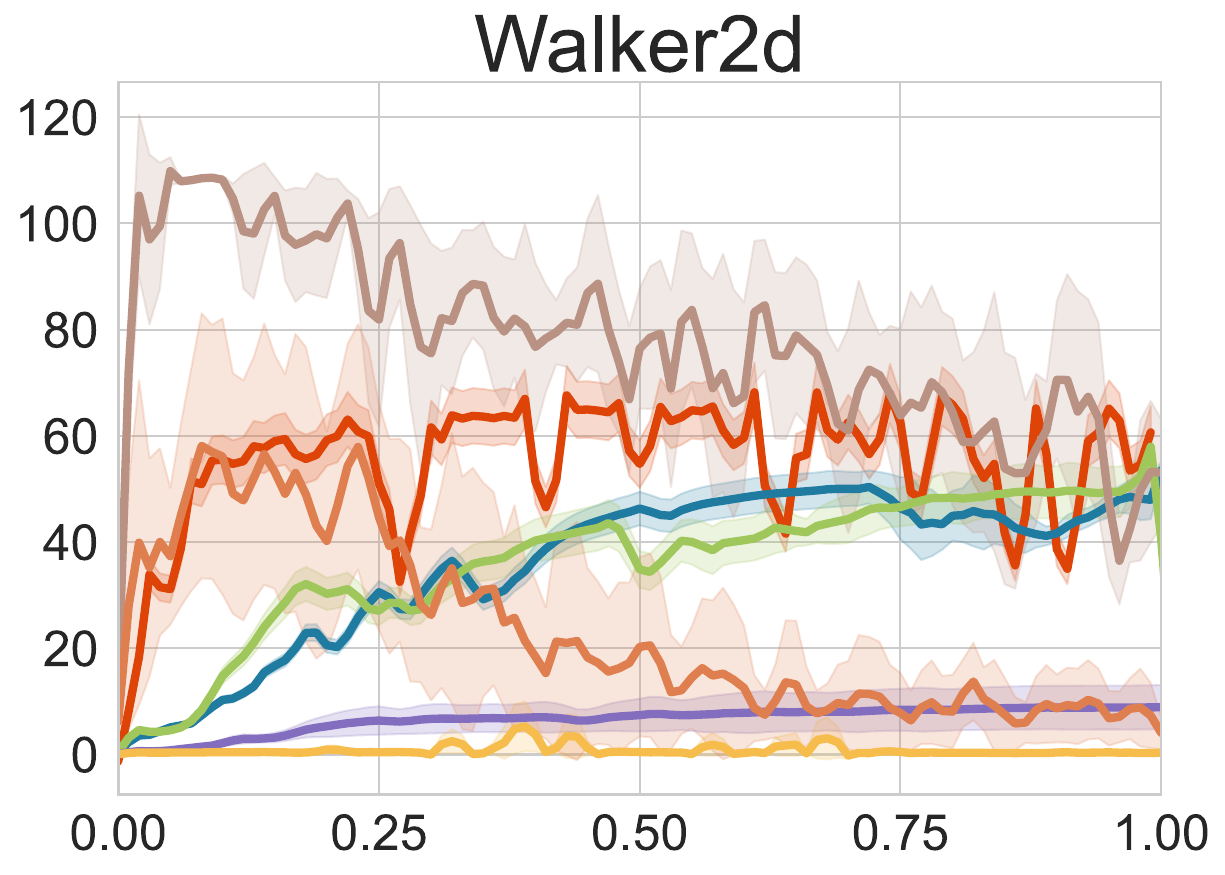}}
    
    \vspace{-10pt}
    \subfigure{
        \includegraphics[width=0.24\textwidth]{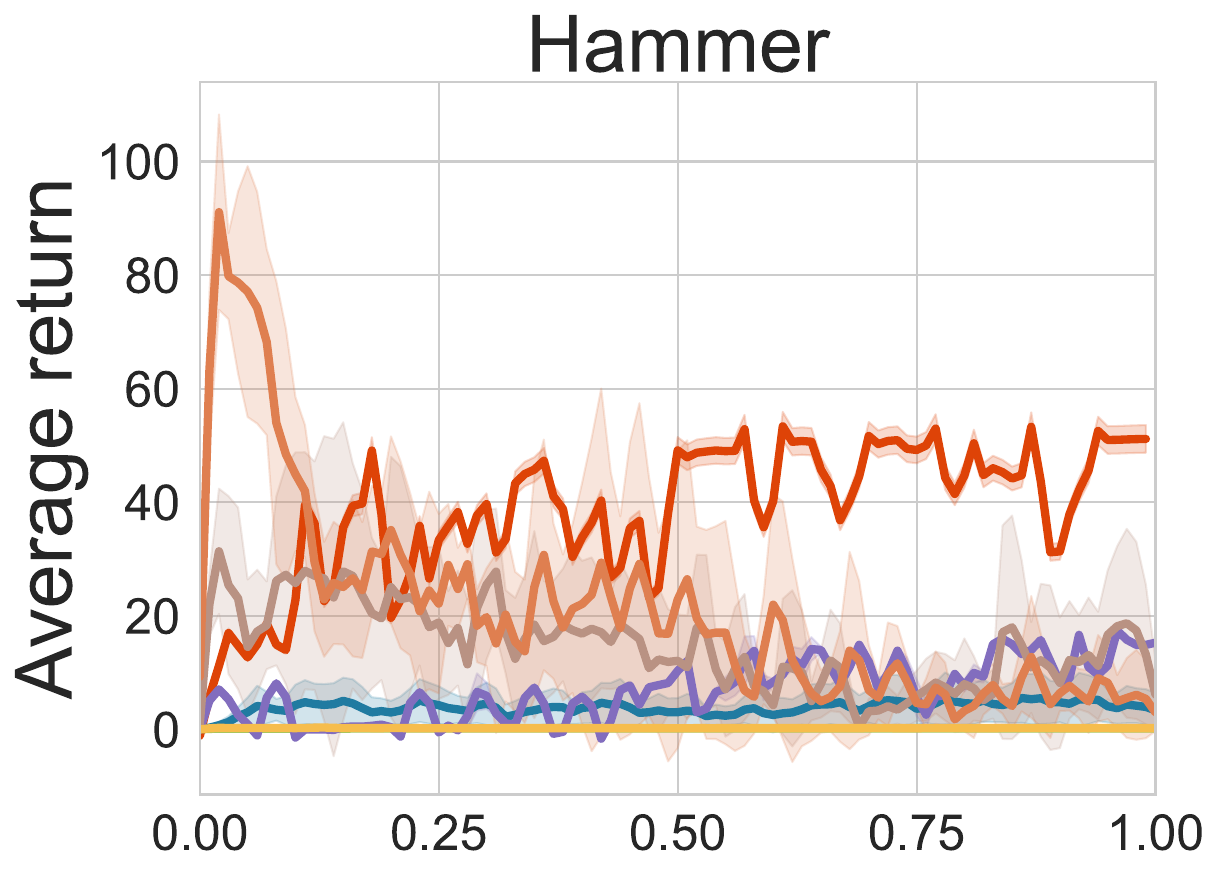}}
    \hspace{-4pt}
    \subfigure{
        \includegraphics[width=0.24\textwidth]{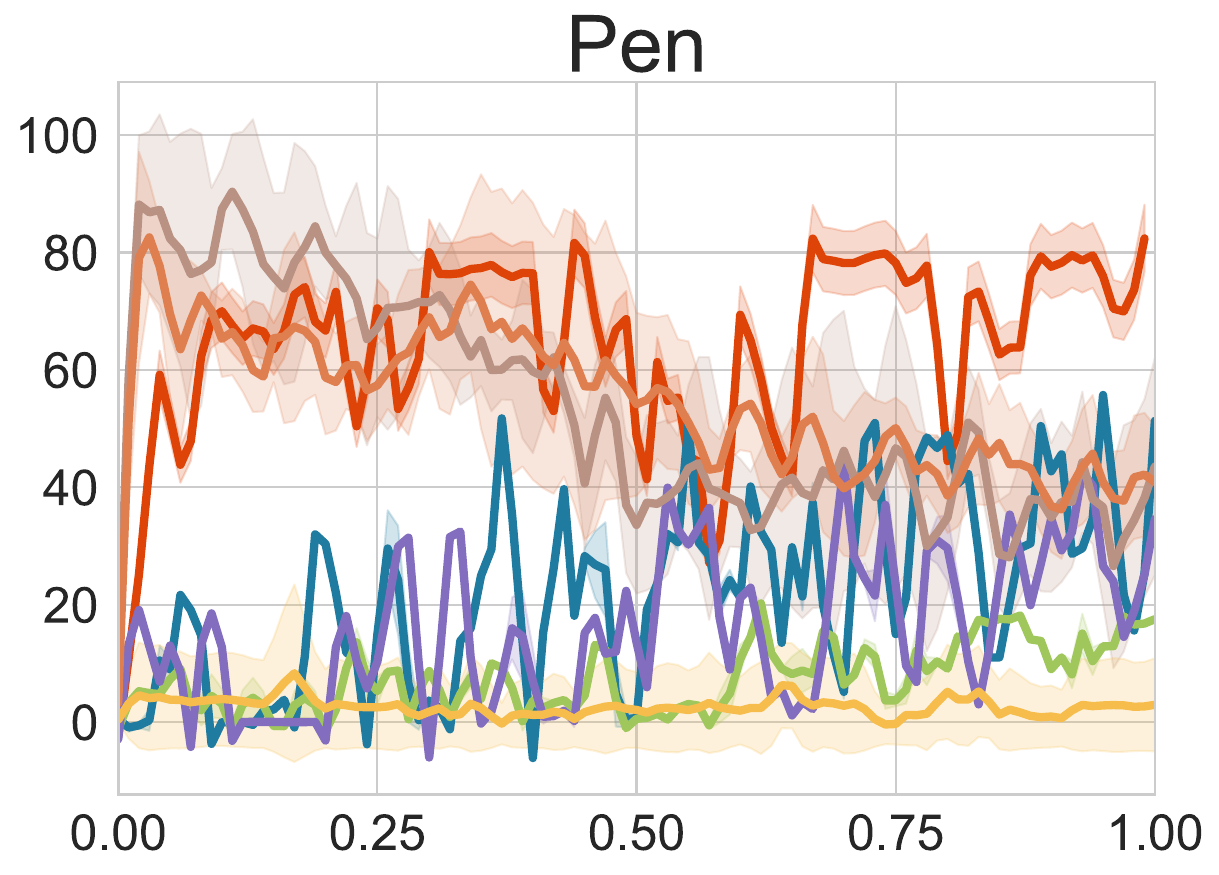}}
    \hspace{-4pt}
    \subfigure{
        \includegraphics[width=0.24\textwidth]{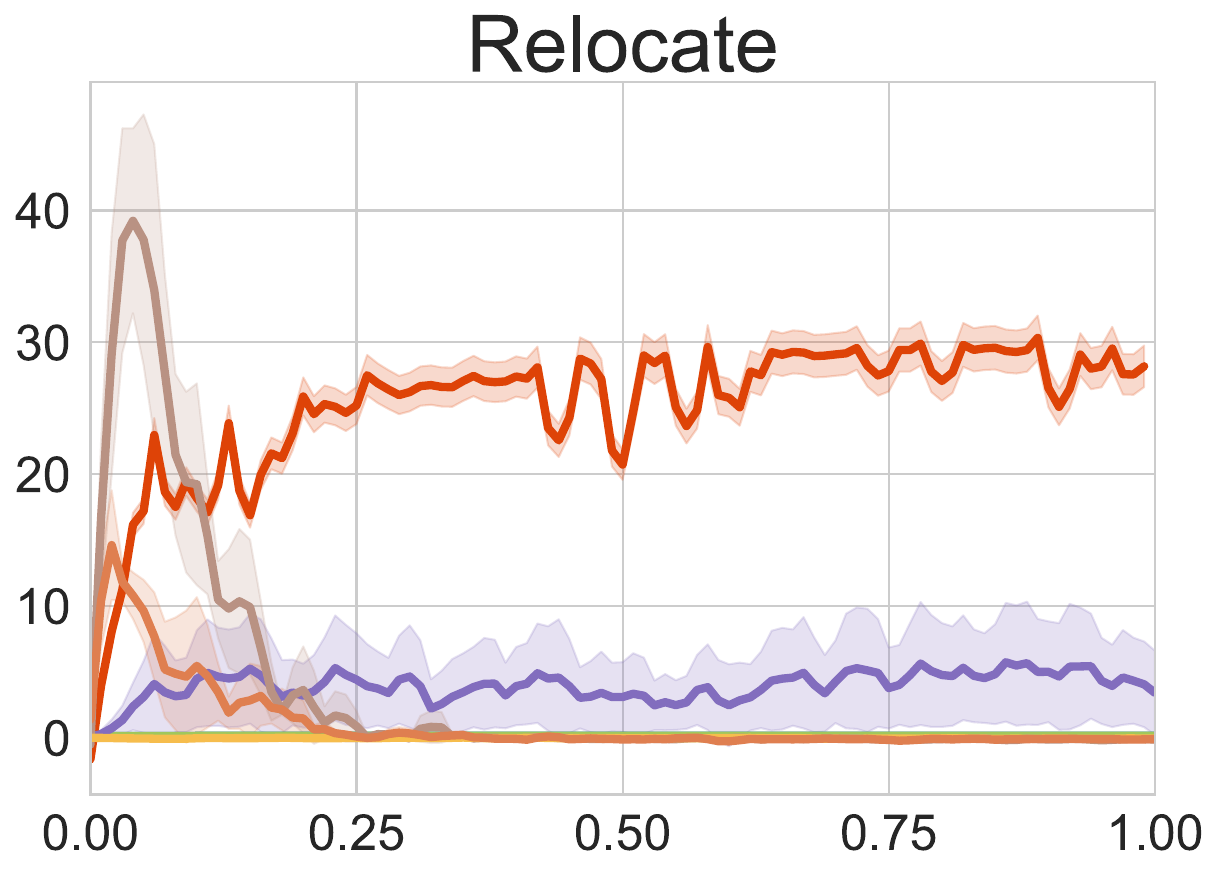}}
    \hspace{-4pt}
    \subfigure{
        \includegraphics[width=0.24\textwidth]{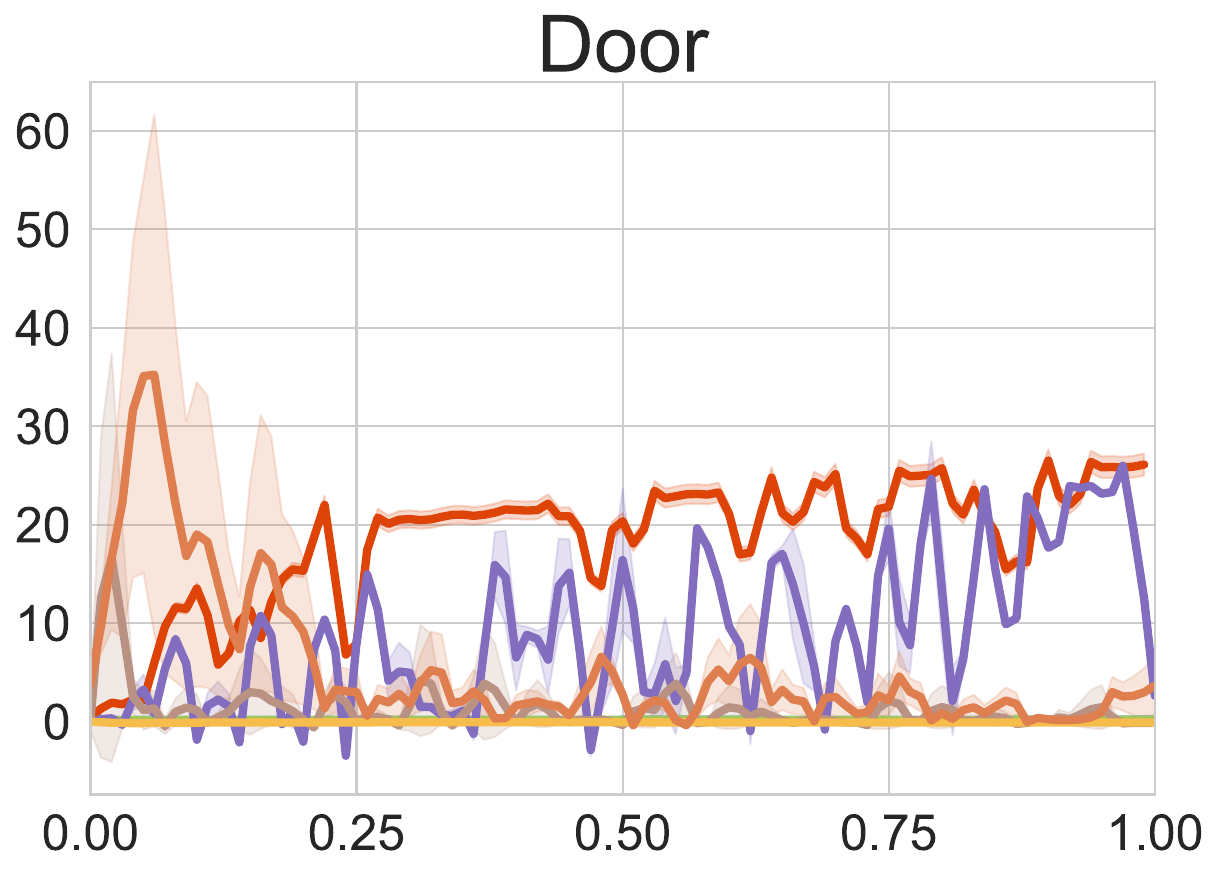}}
    
    \vspace{-10pt}
    \subfigure{
        \includegraphics[width=0.24\textwidth]{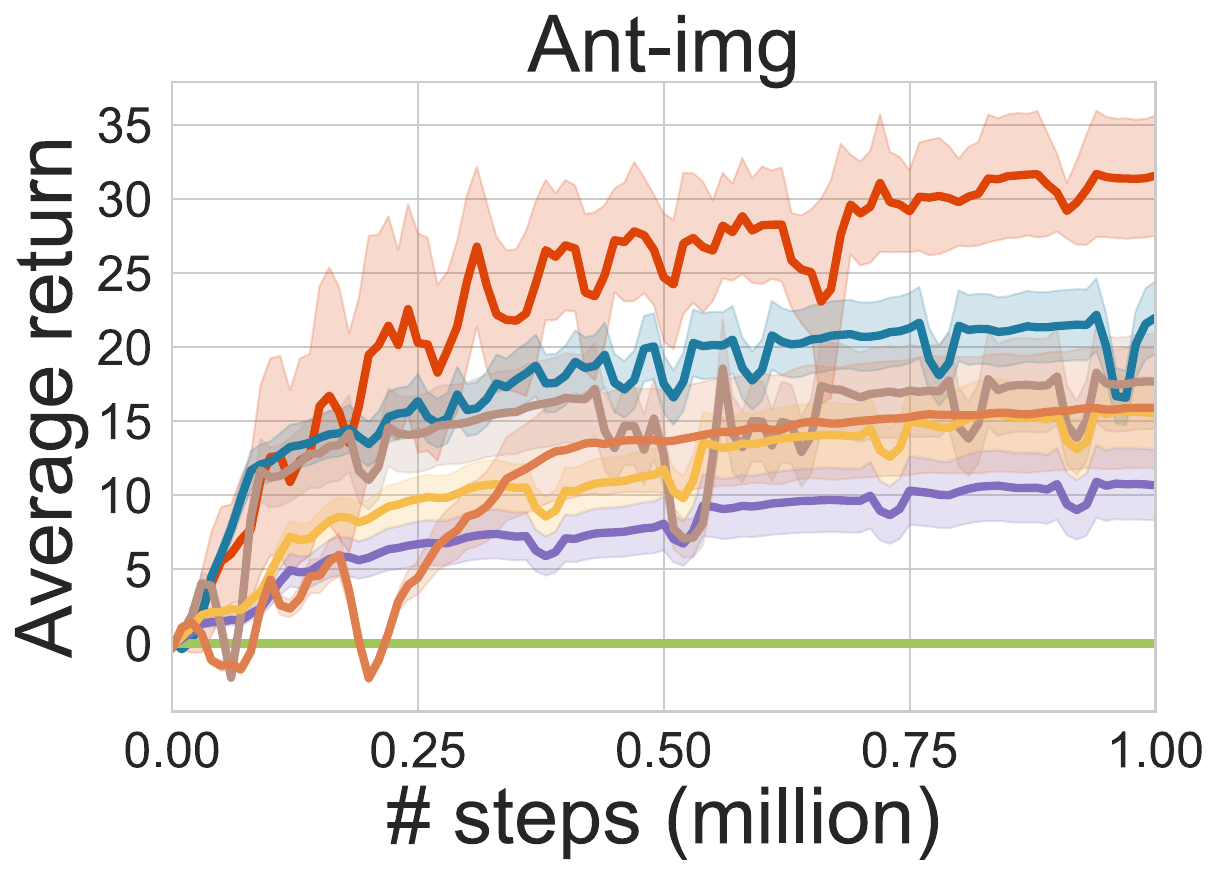}}
    \hspace{-4pt}
    \subfigure{
        \includegraphics[width=0.24\textwidth]{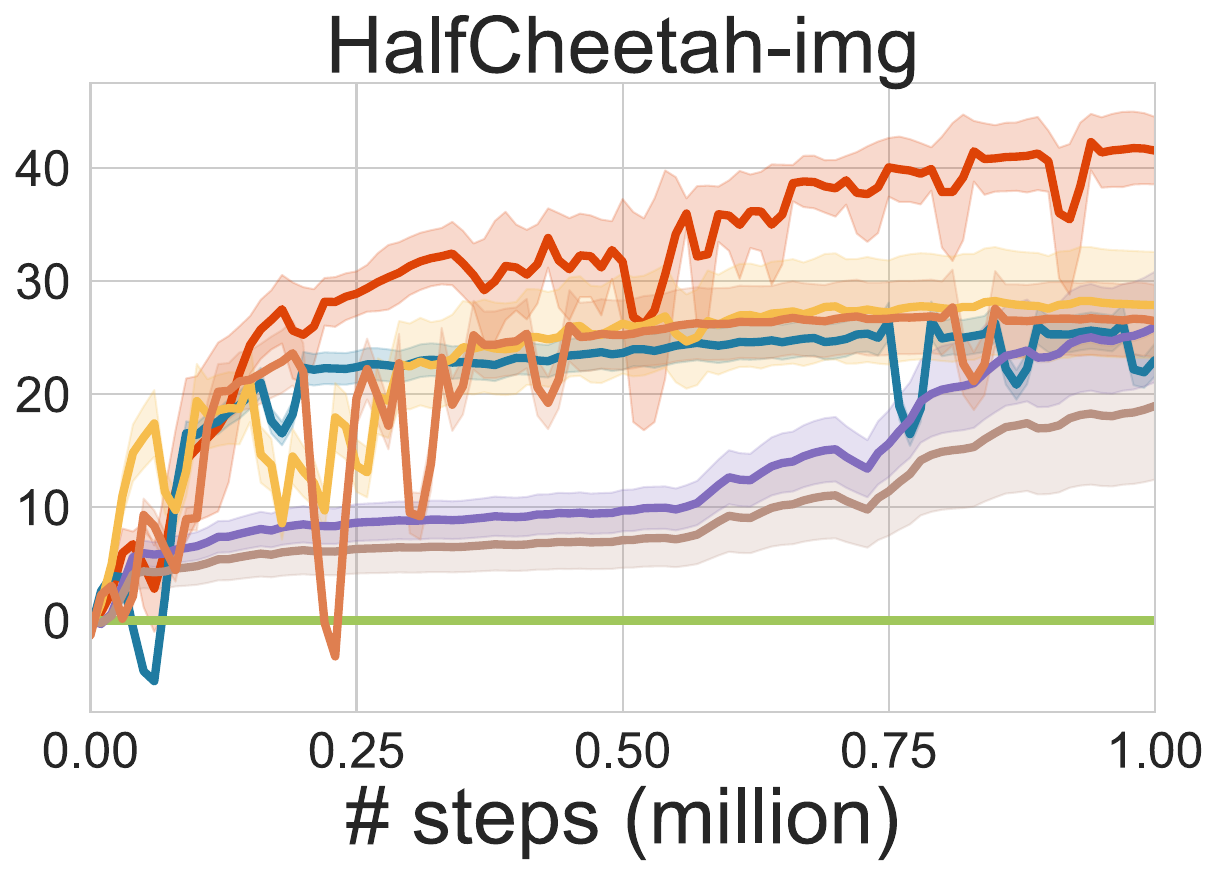}}
    \hspace{-4pt}
    \subfigure{
        \includegraphics[width=0.24\textwidth]{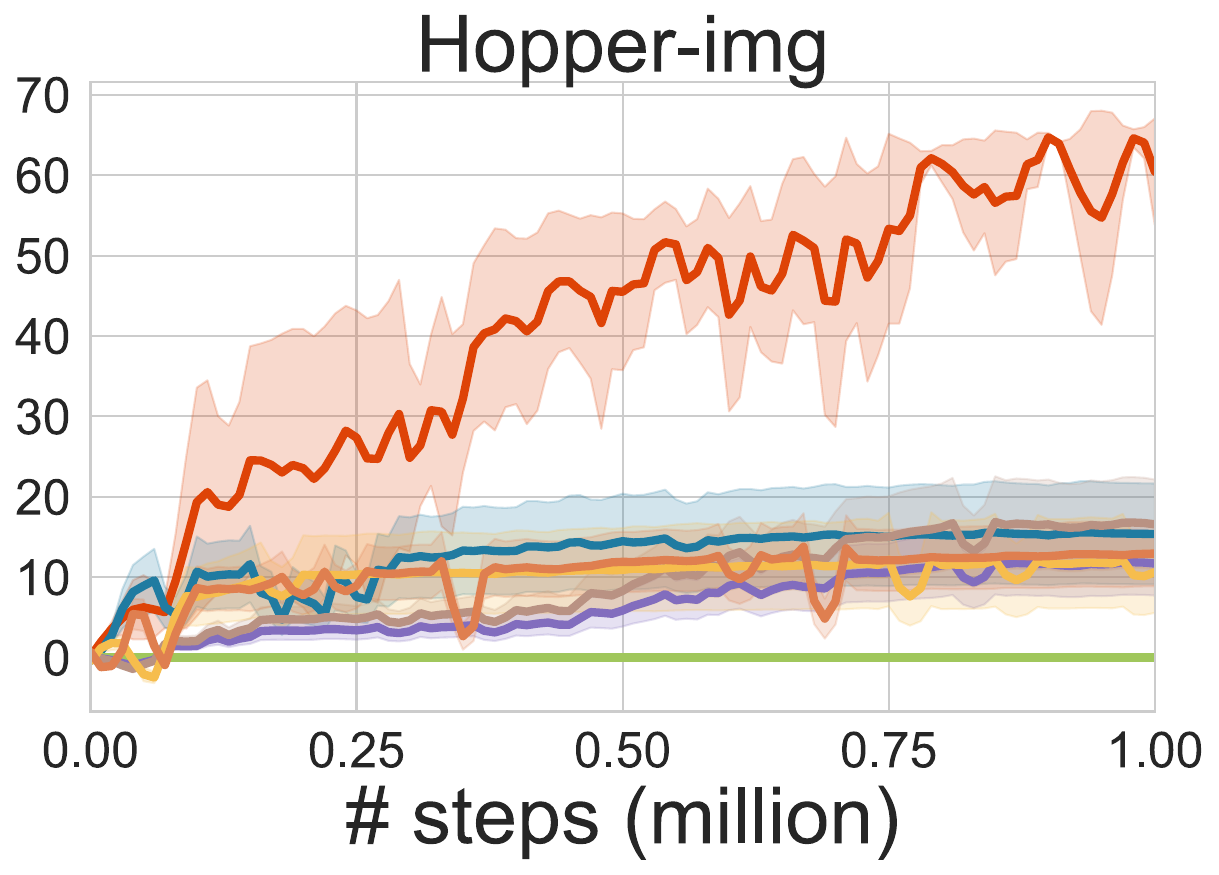}}
    \hspace{-4pt}
    \subfigure{
        \includegraphics[width=0.24\textwidth]{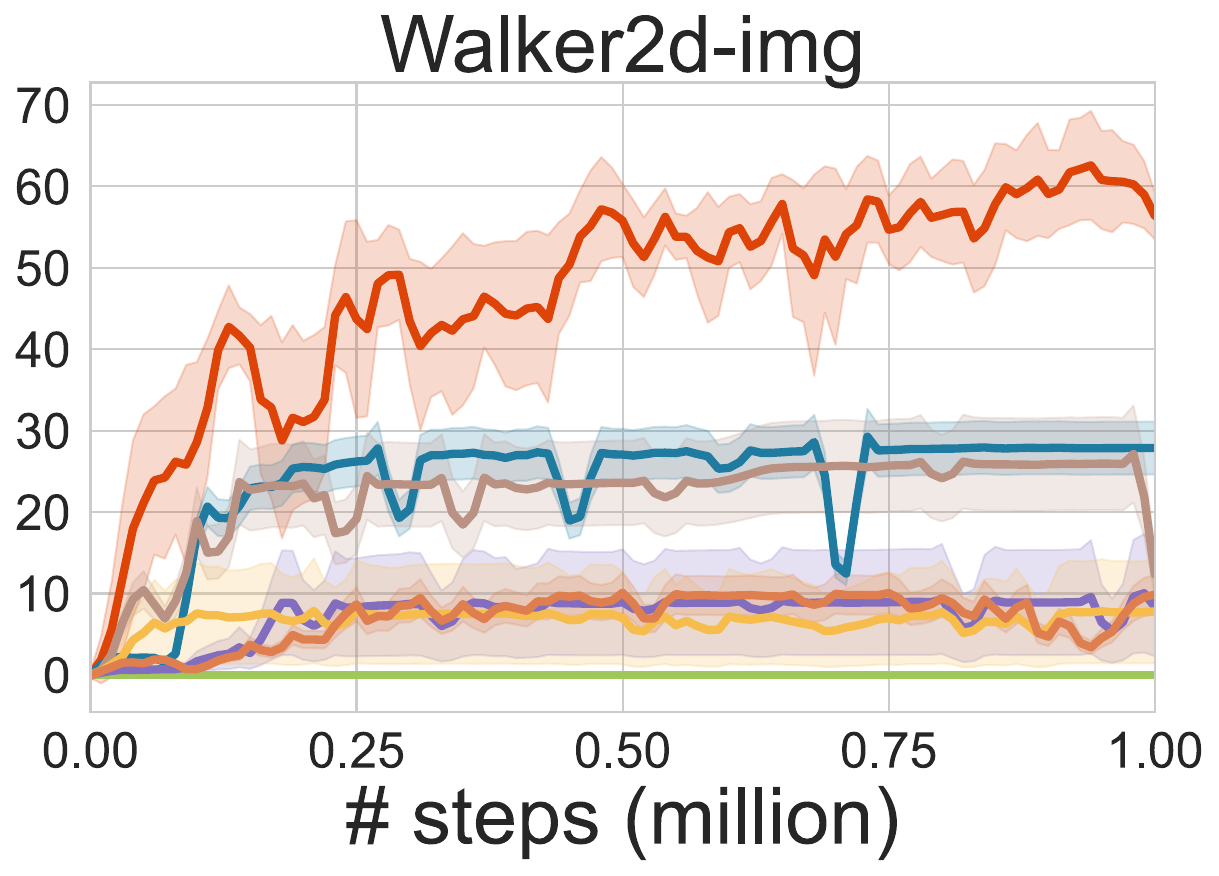}}
    % \vspace{-10pt}
    \caption{Learning curves for \cref{tab:performance}.  `\texttt{-img}' represents vision-based MuJoCo tasks.}
    \label{fig:convergence_all_1}
\end{figure}

\begin{figure}[t]
    \centering
    \subfigure{
        \includegraphics[width=0.24\textwidth]{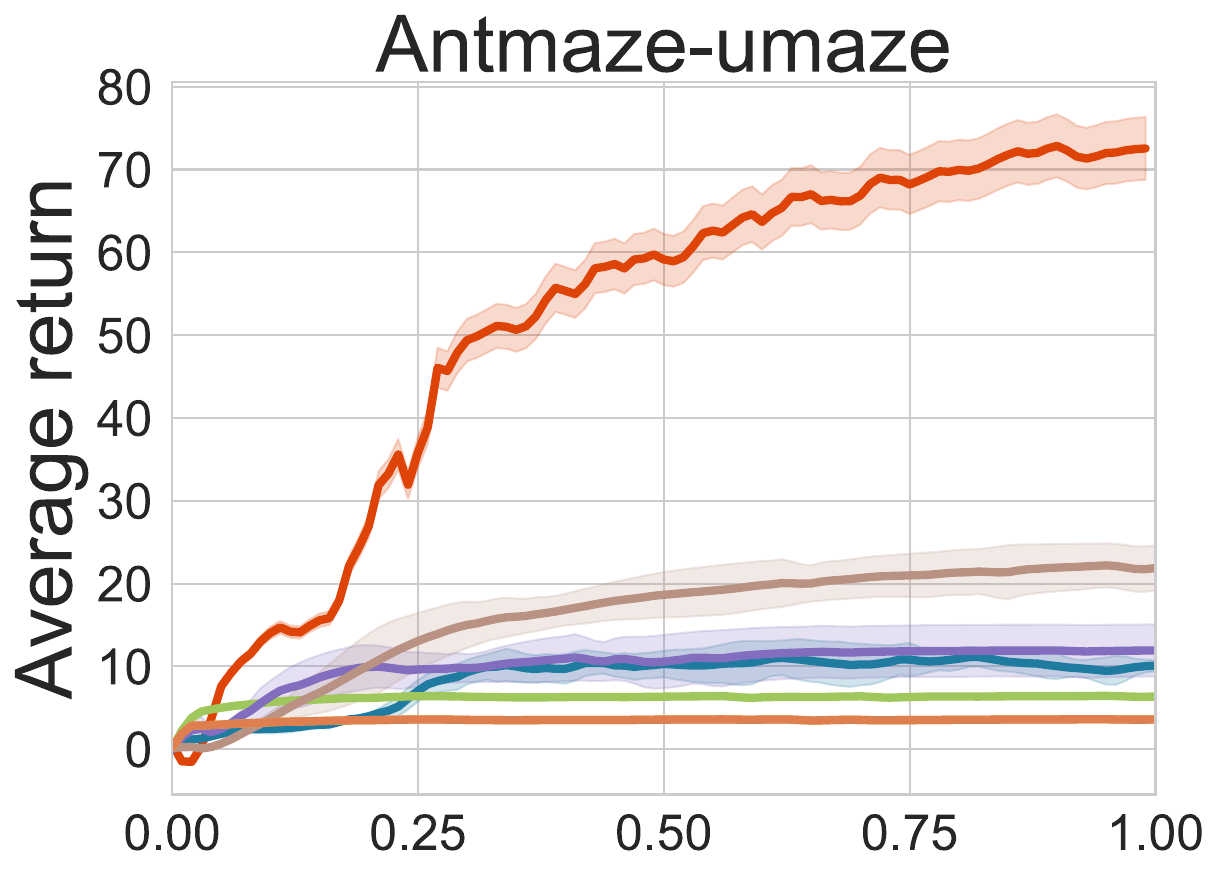}}
    \subfigure{
        \includegraphics[width=0.24\textwidth]{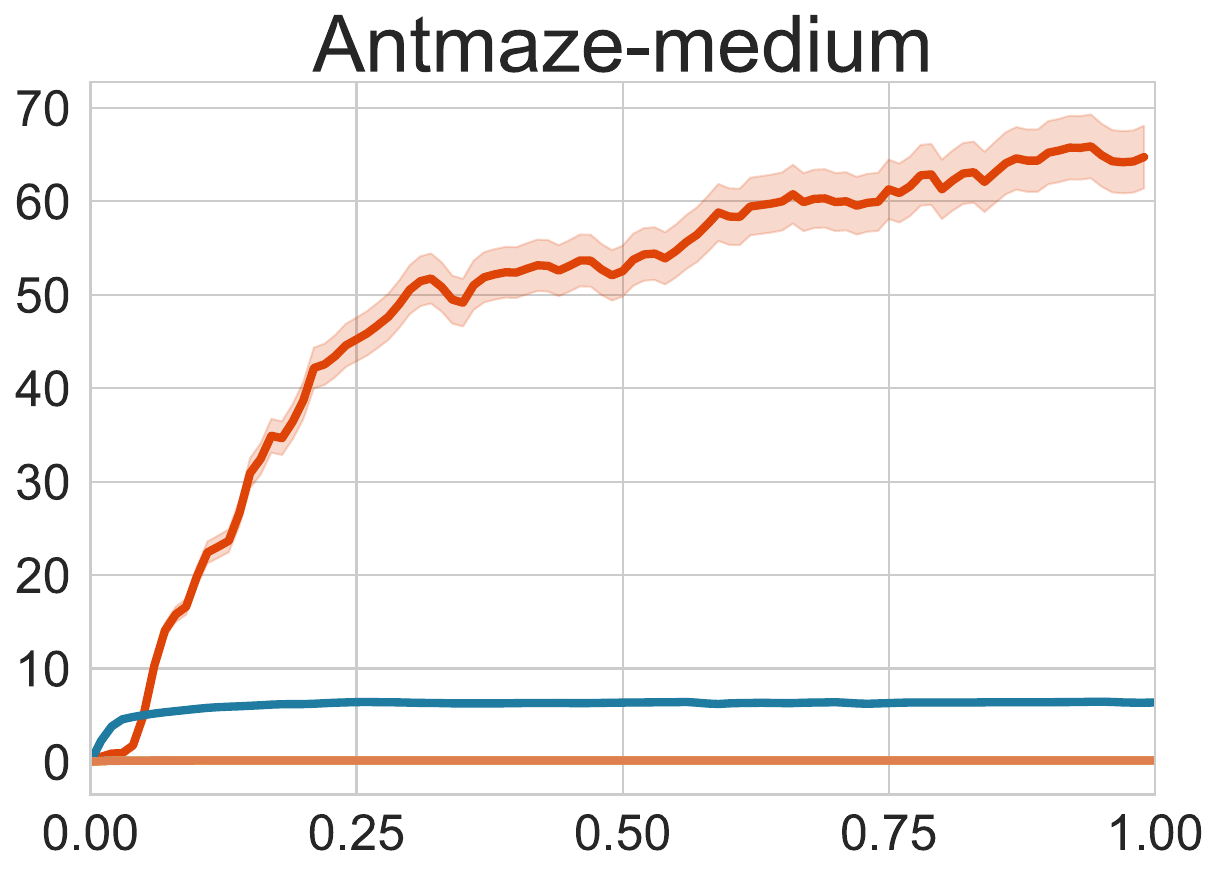}}
    \subfigure{
        \includegraphics[width=0.24\textwidth]{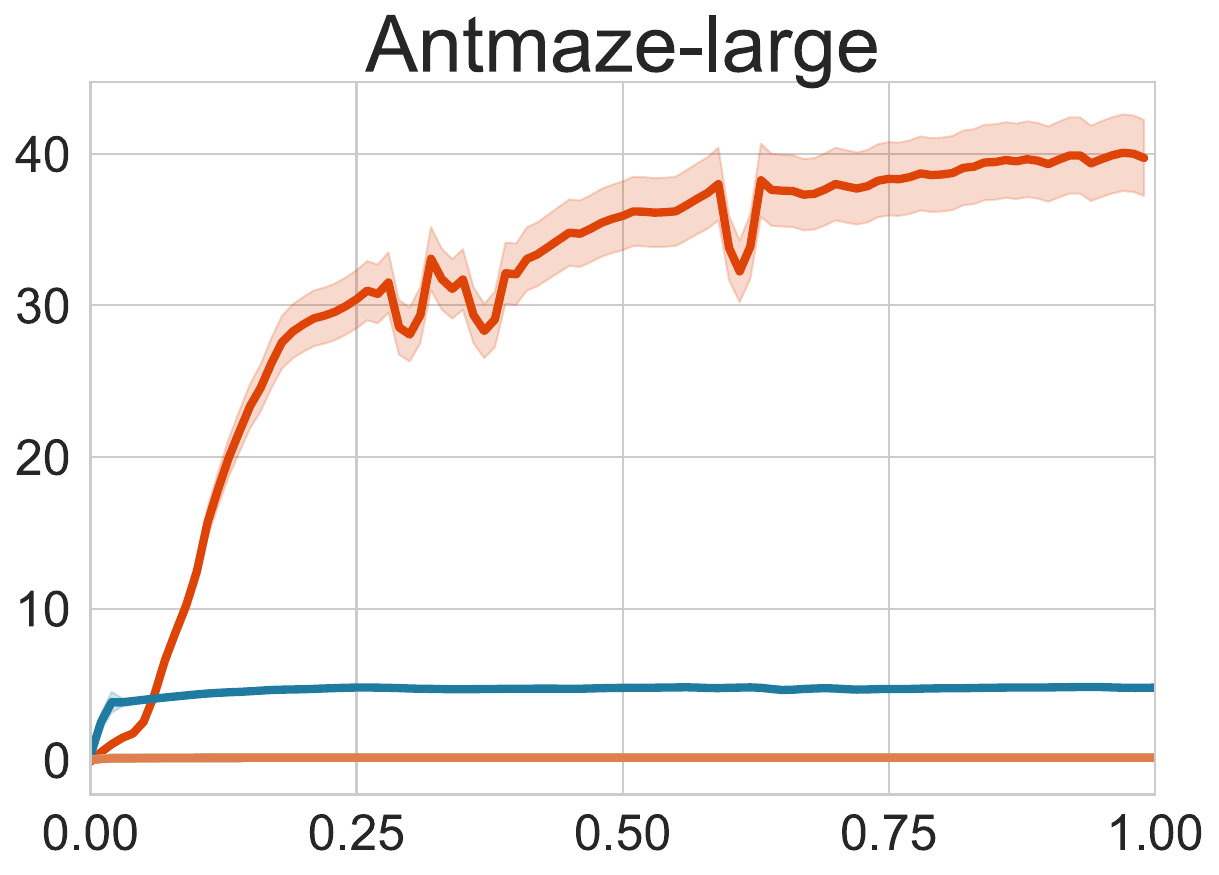}}
    
    \vspace{-10pt}
    \subfigure{
        \includegraphics[width=0.24\textwidth]{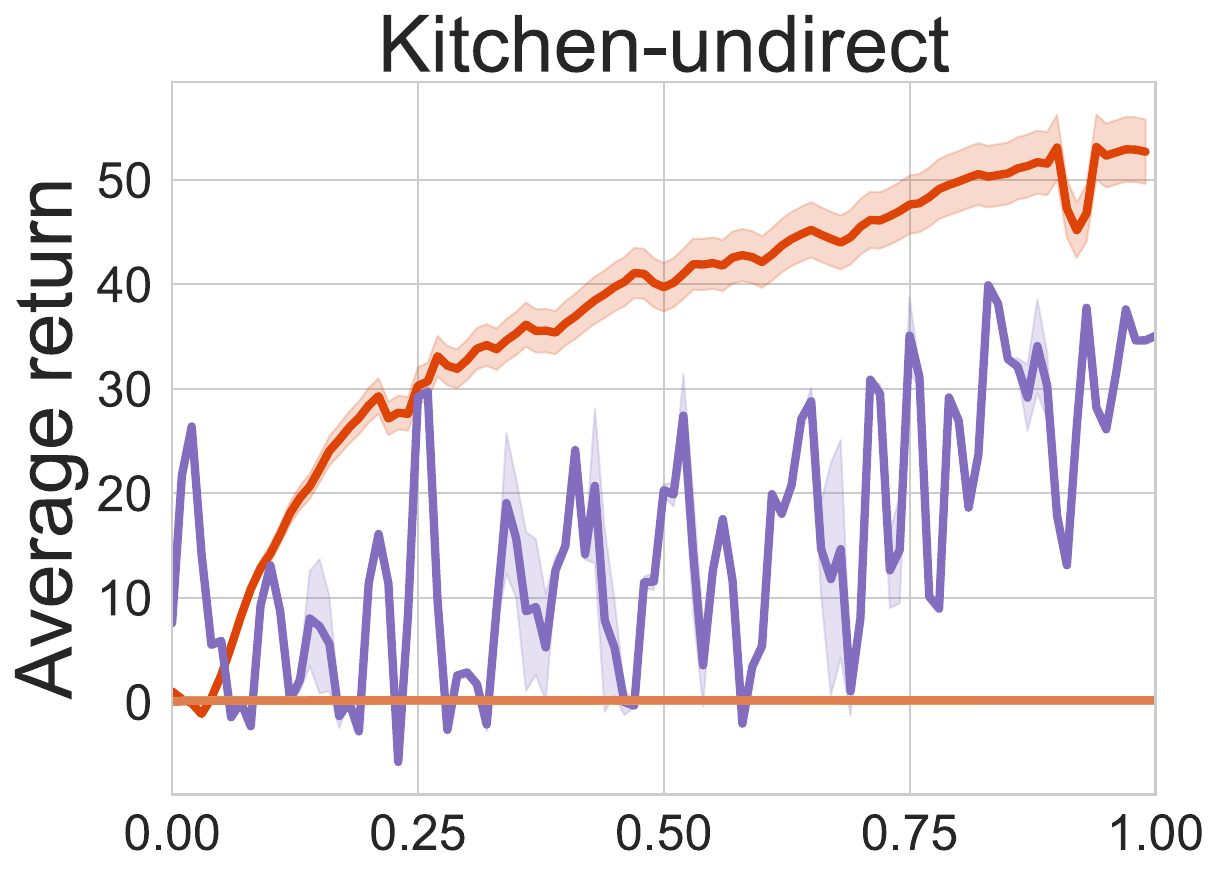}}
    \subfigure{
        \includegraphics[width=0.24\textwidth]{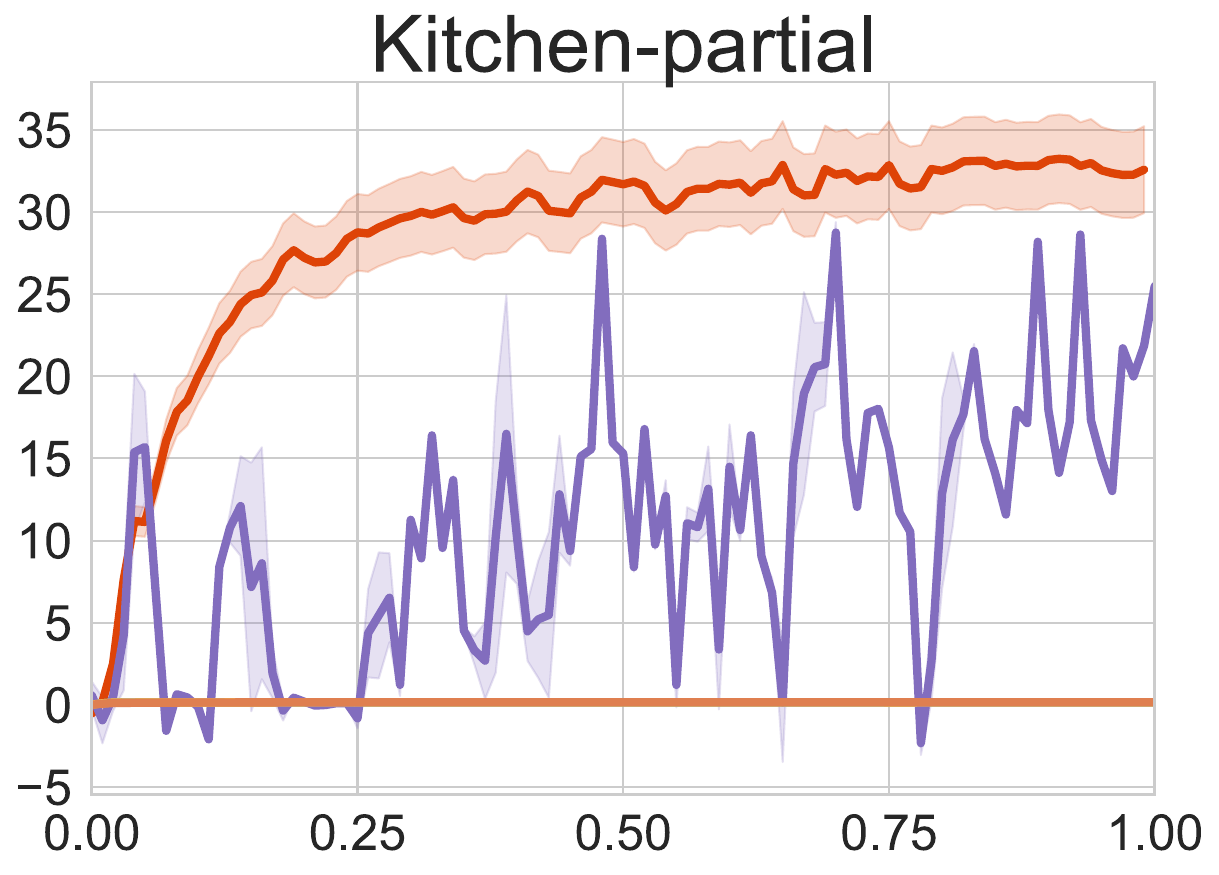}}
    \subfigure{
        \includegraphics[width=0.24\textwidth]{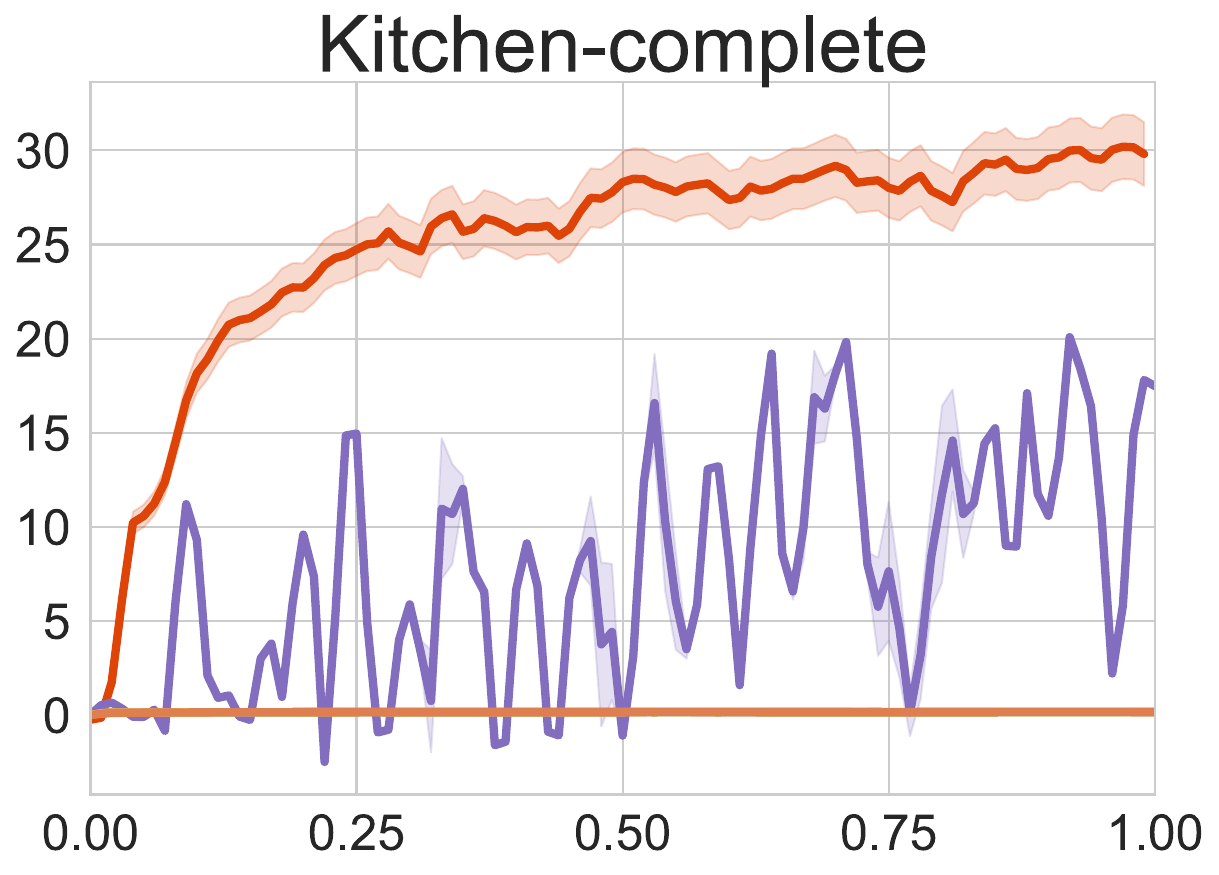}}

    \vspace{-10pt}
    \subfigure{
        \includegraphics[width=0.24\textwidth]{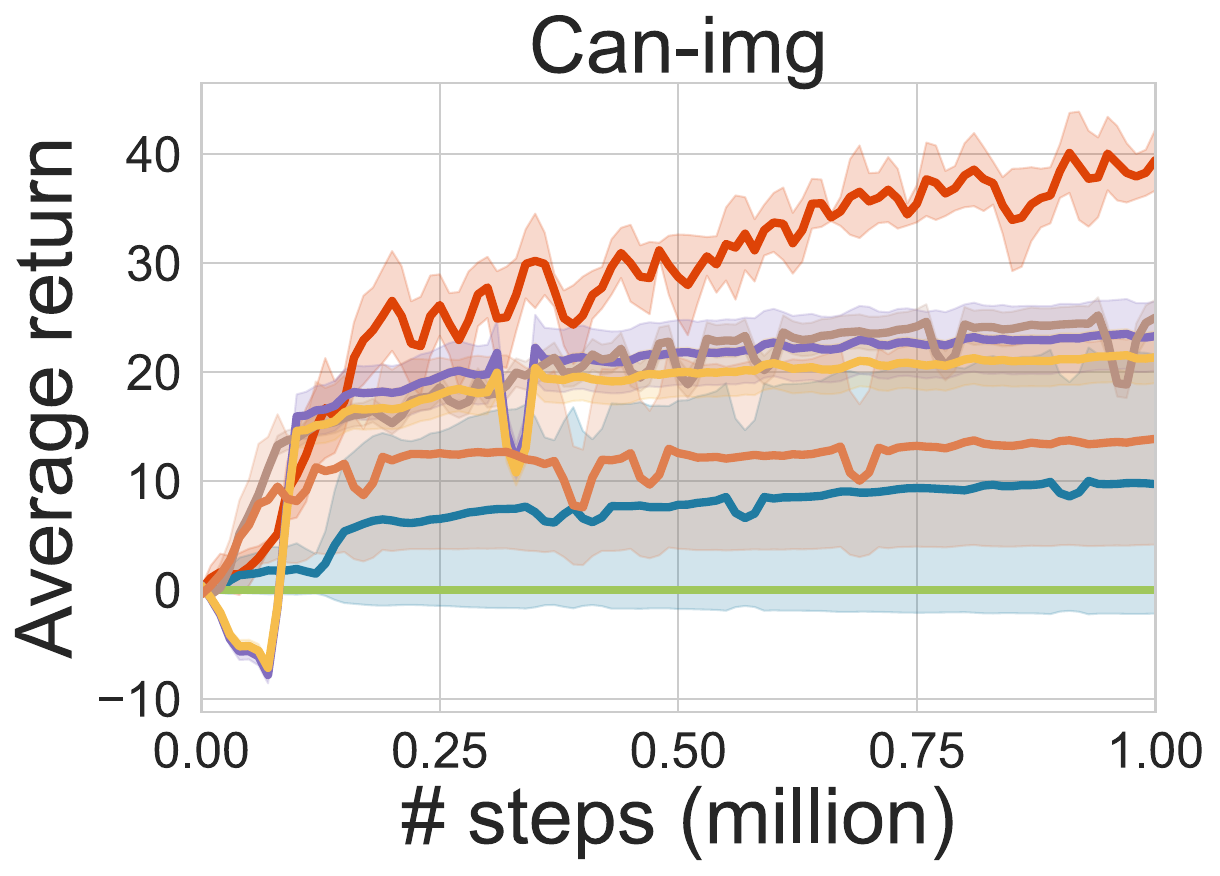}}
    \subfigure{
        \includegraphics[width=0.24\textwidth]{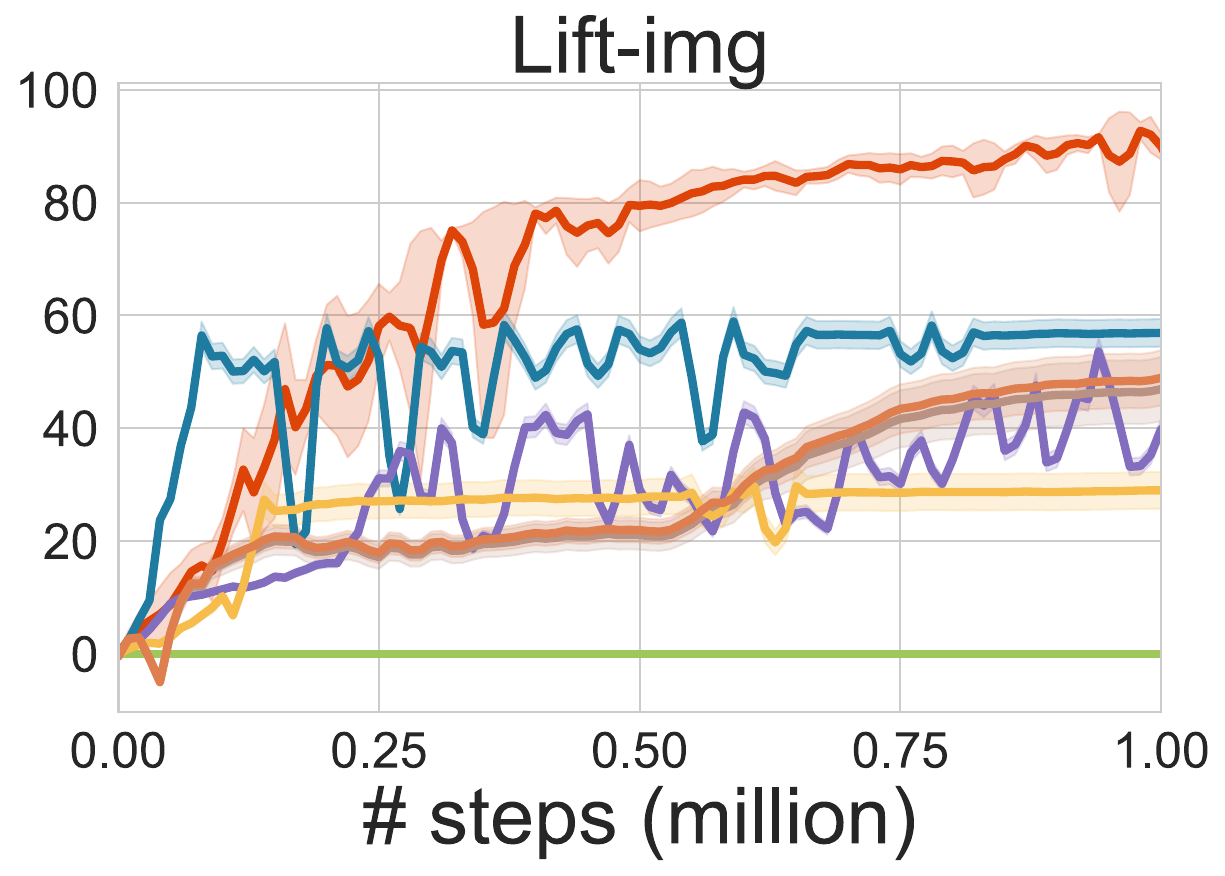}}
    \subfigure{
        \includegraphics[width=0.24\textwidth]{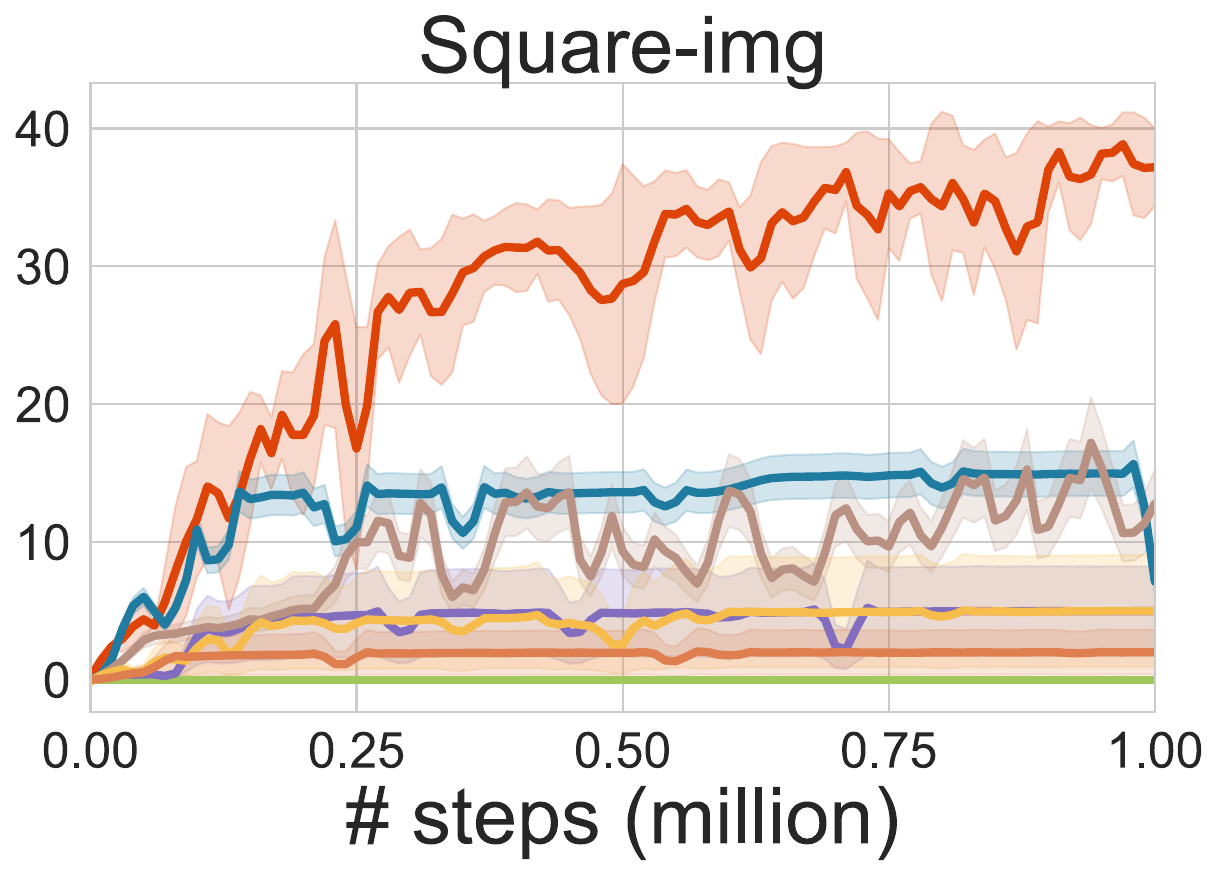}}
    \vskip -0.1in
    \caption{Learning curves for \cref{tab:performance}. Uncertainty intervals depict standard deviation over three seeds.}
    \label{fig:convergence_all_2}
\end{figure}

\clearpage

\subsection{Expert Demonstrations}
\label{sec:expert_demonstrations}

To answer the second question, we run experiments with varying numbers of expert trajectories (ranging from 1 to 30 in MuJoCo and AntMaze, from 10 to 300 in Adroit and FrankaKitchen, and from 25 to 200 in vision-based domains). The data setup adheres to that of \cref{table:dataset_comparative}. As illustrated in \cref{fig:demonstration_all},  our method, consistently requiring much fewer expert trajectories to attain expert performance, demonstrates great demonstration efficiency in comparison with prior methods.

\begin{figure}[ht]
    \centering
    \subfigure{
        \includegraphics[width=0.99\textwidth]{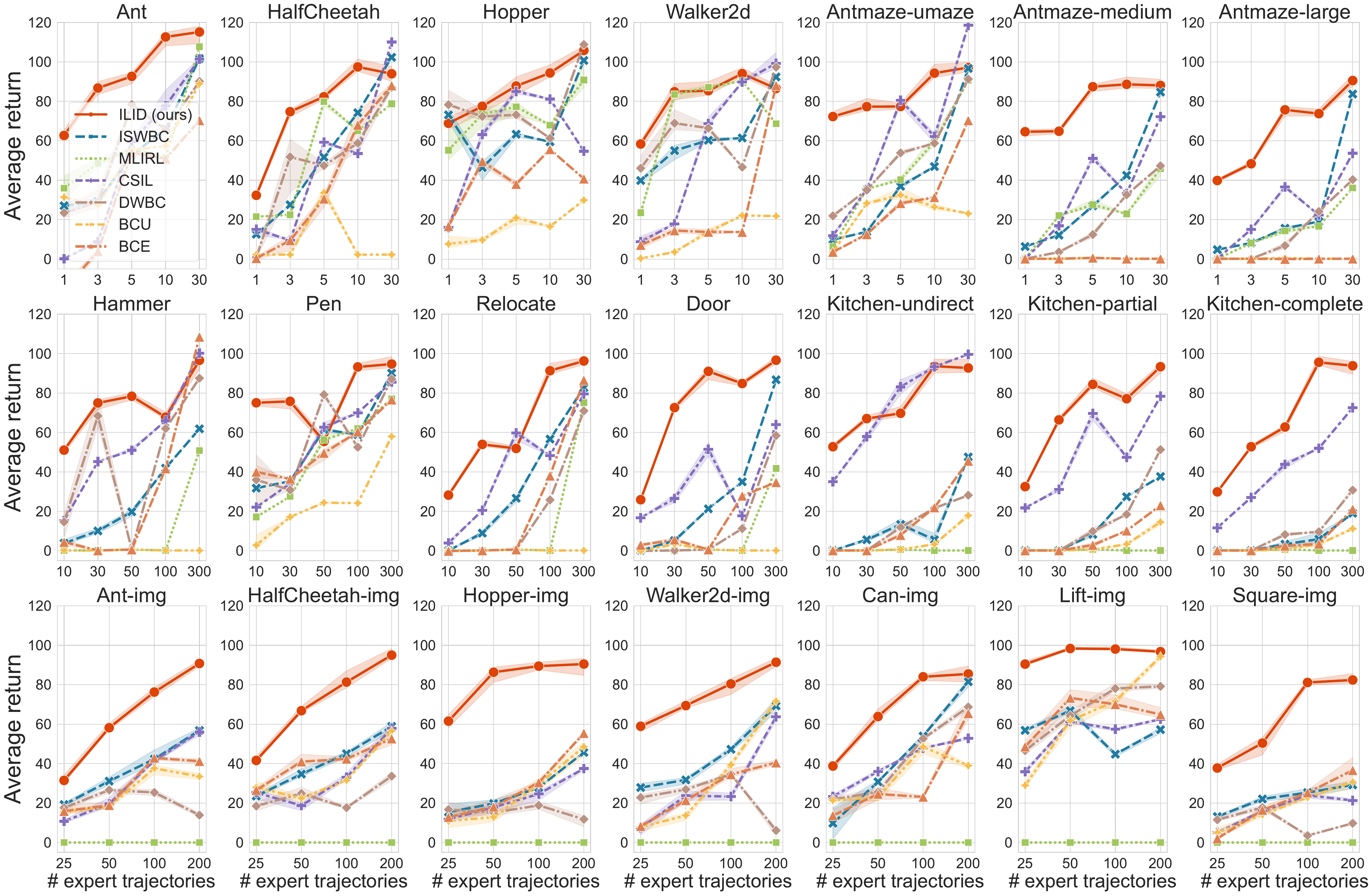}}
    % \vspace{-10pt}
    \caption{Normalized scores under varying numbers of expert demonstrations. }
    \label{fig:demonstration_all}
    % \vspace{-5pt}
\end{figure}

\clearpage

\subsection{Quality and Quantity of Imperfect Data}
\label{sec:quality_quantity}

For the second question, we also conduct experiments using imperfect demonstrations with varying qualities and quantities to test the robustness of \texttt{ILID}'s performance in behavior selection (the data setup is showcased in \cref{table:dataset_dataqualities}). As shown in \cref{fig:imperfect,fig:random_num,fig:imperfect_all,fig:random_num_bar}, we find that \texttt{ILID} surpasses the
baselines in \textbf{20/24} settings, corroborating its efficacy
and superiority in the utilization of noisy data. Moreover, \cref{fig:random_num} underscores the importance of leveraging suboptimal data.

% For the second question, we also conduct experiments using imperfect demonstrations of varying qualities to test the robustness of \texttt{ILID}'s performance in behavior selection (the data setup is showcased in \cref{table:dataset_dataqualities}). As shown in \cref{fig:imperfect,fig:imperfect_all}. We find that \texttt{ILID} surpasses the baselines in 20/24 settings,  corroborating its efficacy and superiority in the utilization of noisy data.

%To answer the third questions, we conduct a comparison with the baselines under \emph{setting II} to explore the impact of imperfect datasets of varying quality on algorithm performance. As shown in \cref{fig:imperfect_all}, \texttt{ILID} is able to select more state-actions that can lead to expert states when the quality of the data continuously improves, thereby demonstrating better performance. Albeit with higher overall scores, the performance of some baseline on the \texttt{medium} data frequently falls short in comparison to its performance on the \texttt{medium-replay} data, revealing the larger state coverage and richer information embedded within the Replay data. In contrast, \texttt{ILID} still shows a leading performance on \texttt{medium} data.

\begin{figure}[H]
    \centering
    % \vspace{-0.2in}
    \subfigure{
        \includegraphics[width=0.24\textwidth]{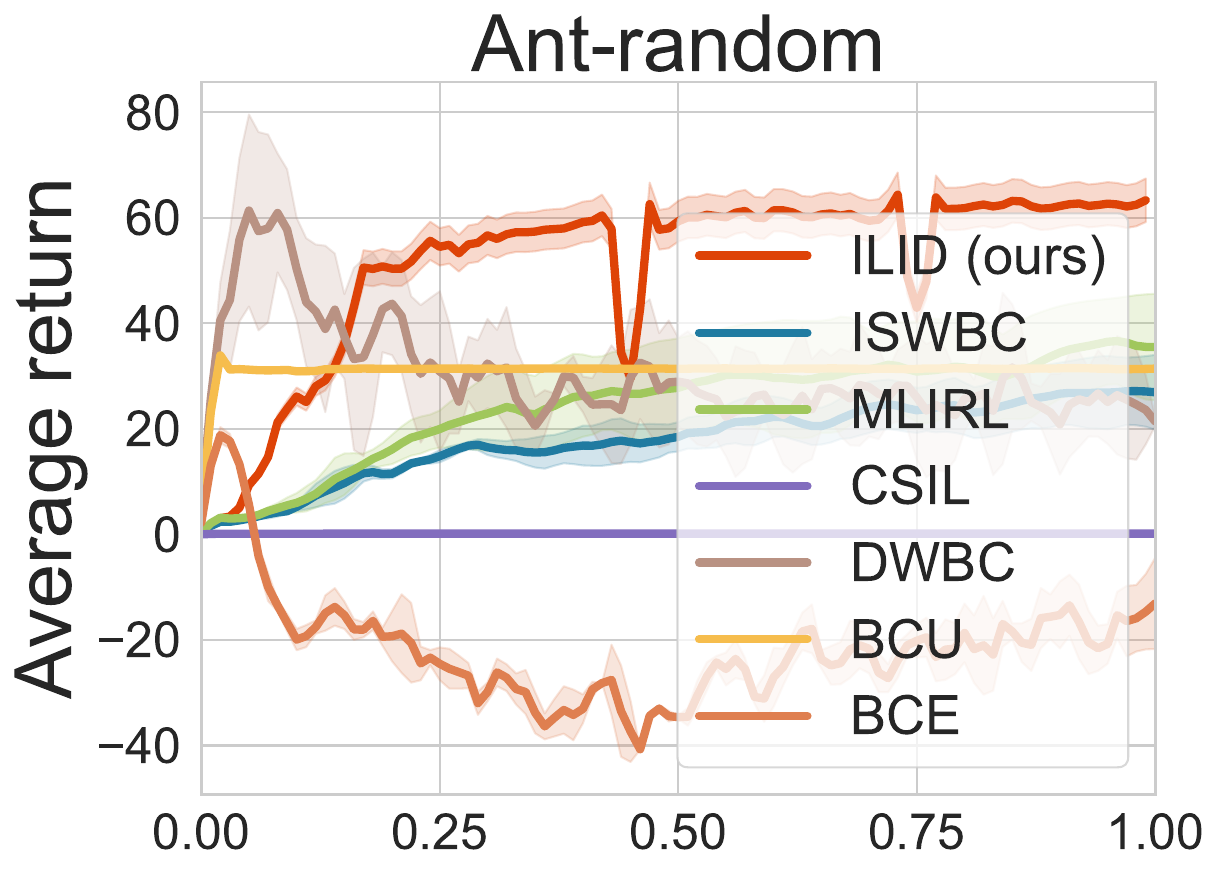}}
    \hspace{-4pt}
    \subfigure{
        \includegraphics[width=0.24\textwidth]{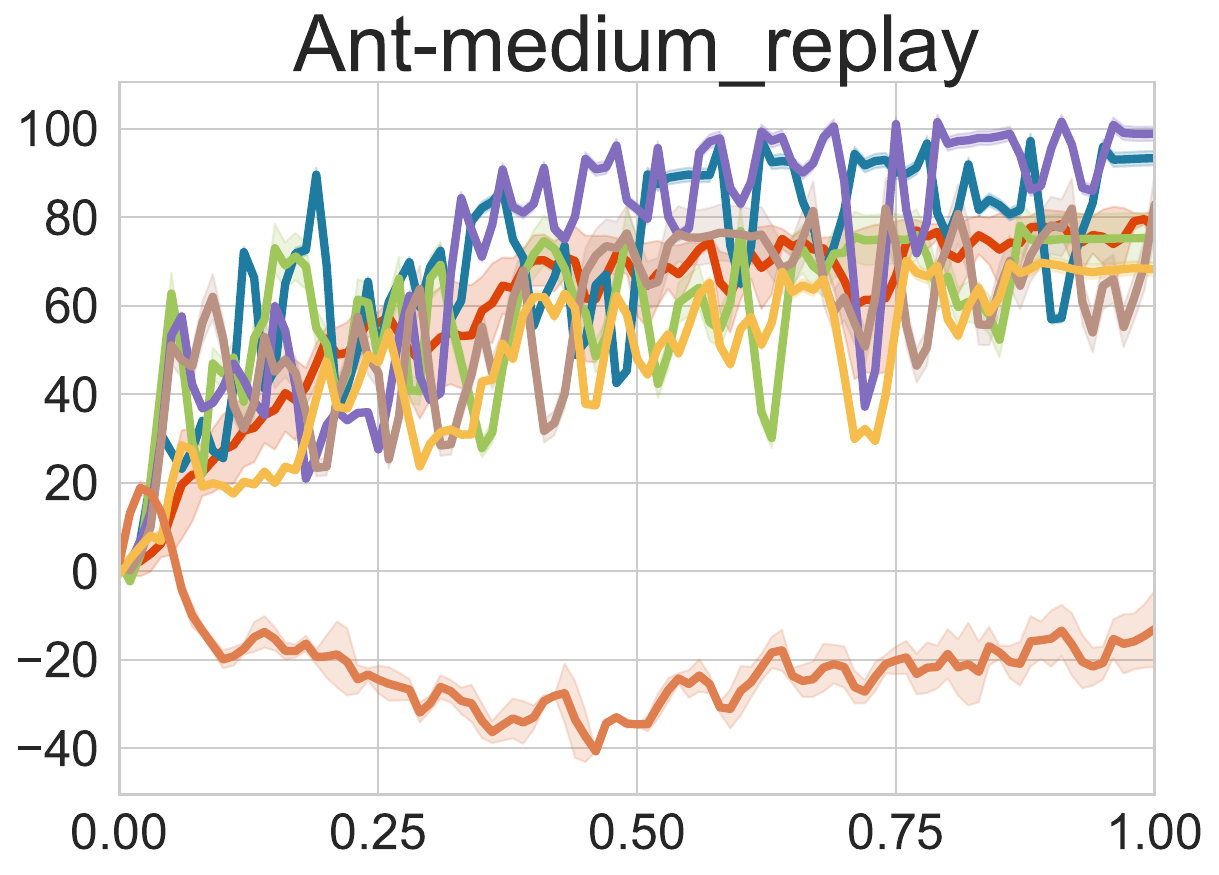}}
    \hspace{-4pt}
    \subfigure{
        \includegraphics[width=0.24\textwidth]{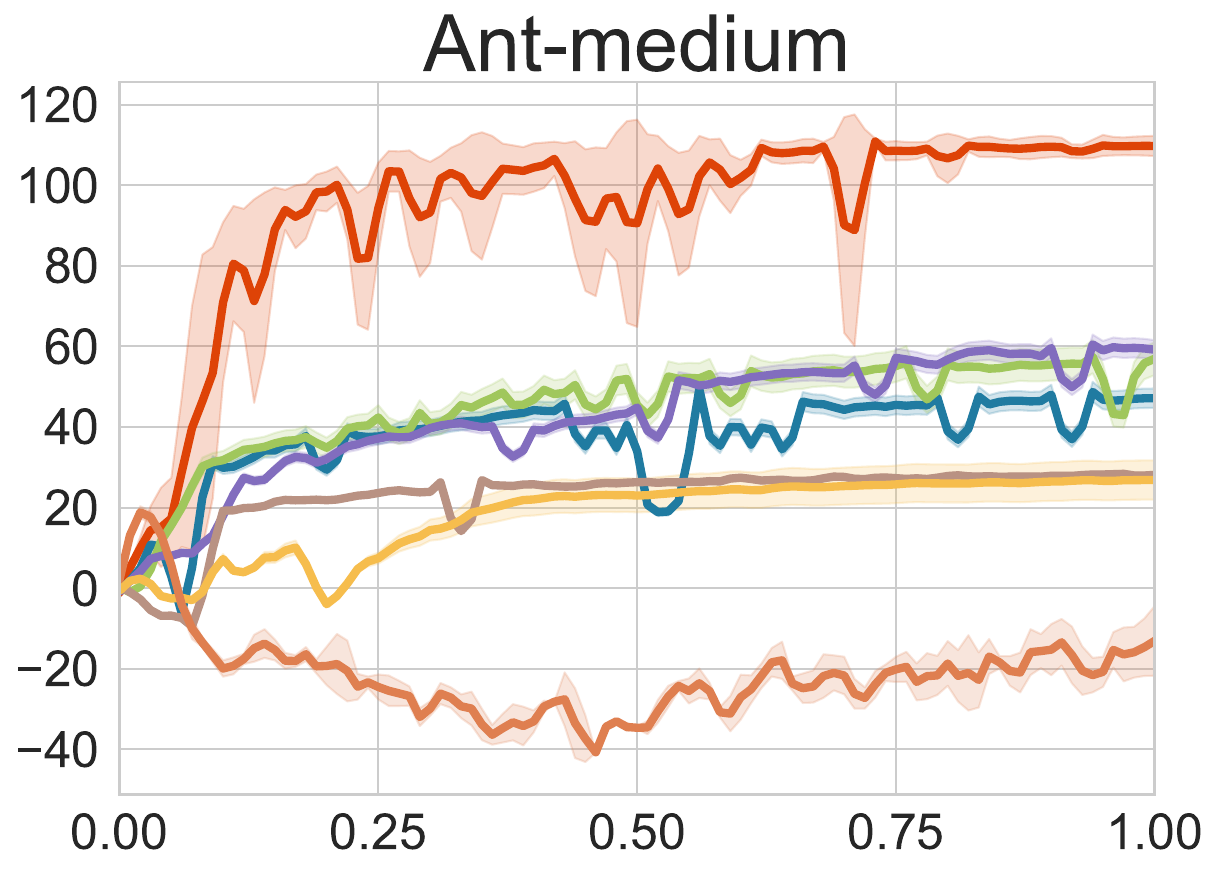}}
    \hspace{-4pt}
    \subfigure{
        \includegraphics[width=0.24\textwidth]{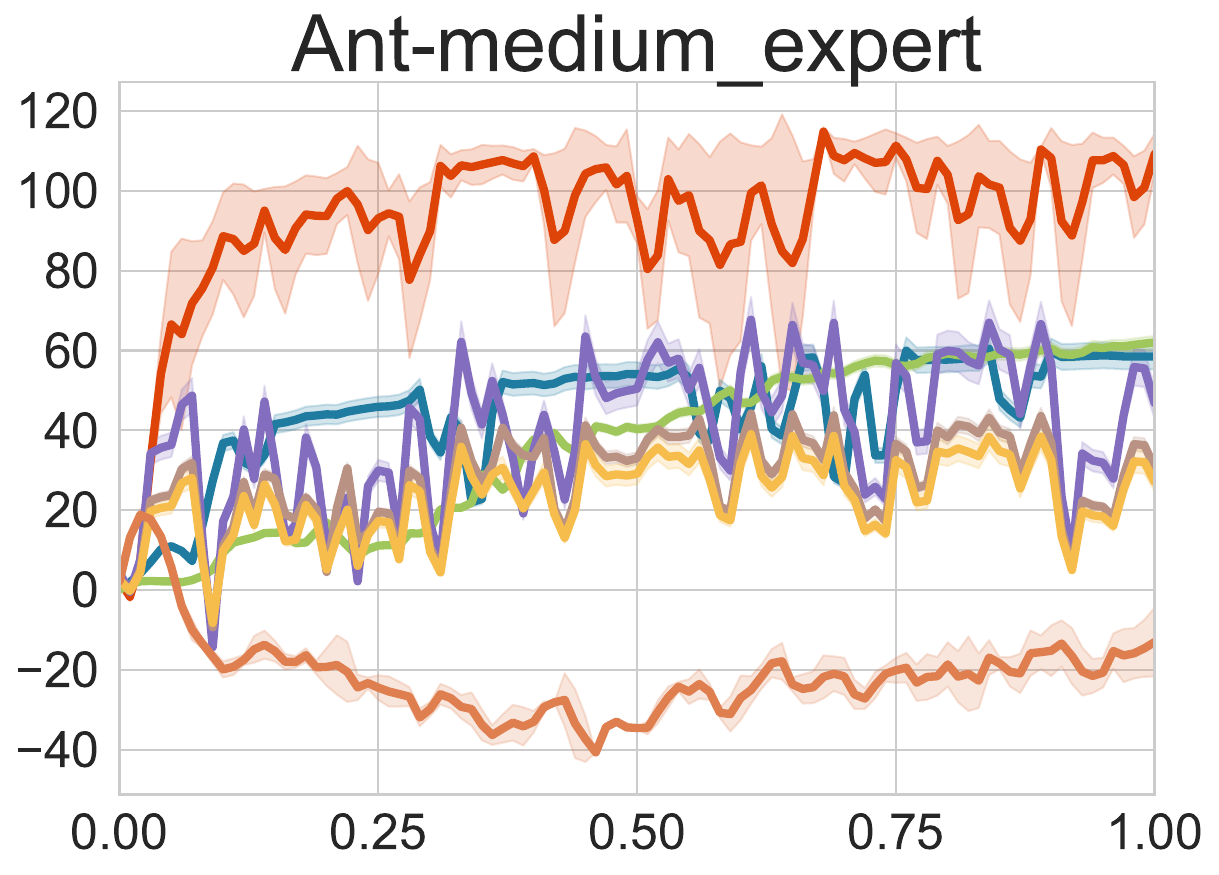}}
    
    \vspace{-10pt}
    \subfigure{
        \includegraphics[width=0.24\textwidth]{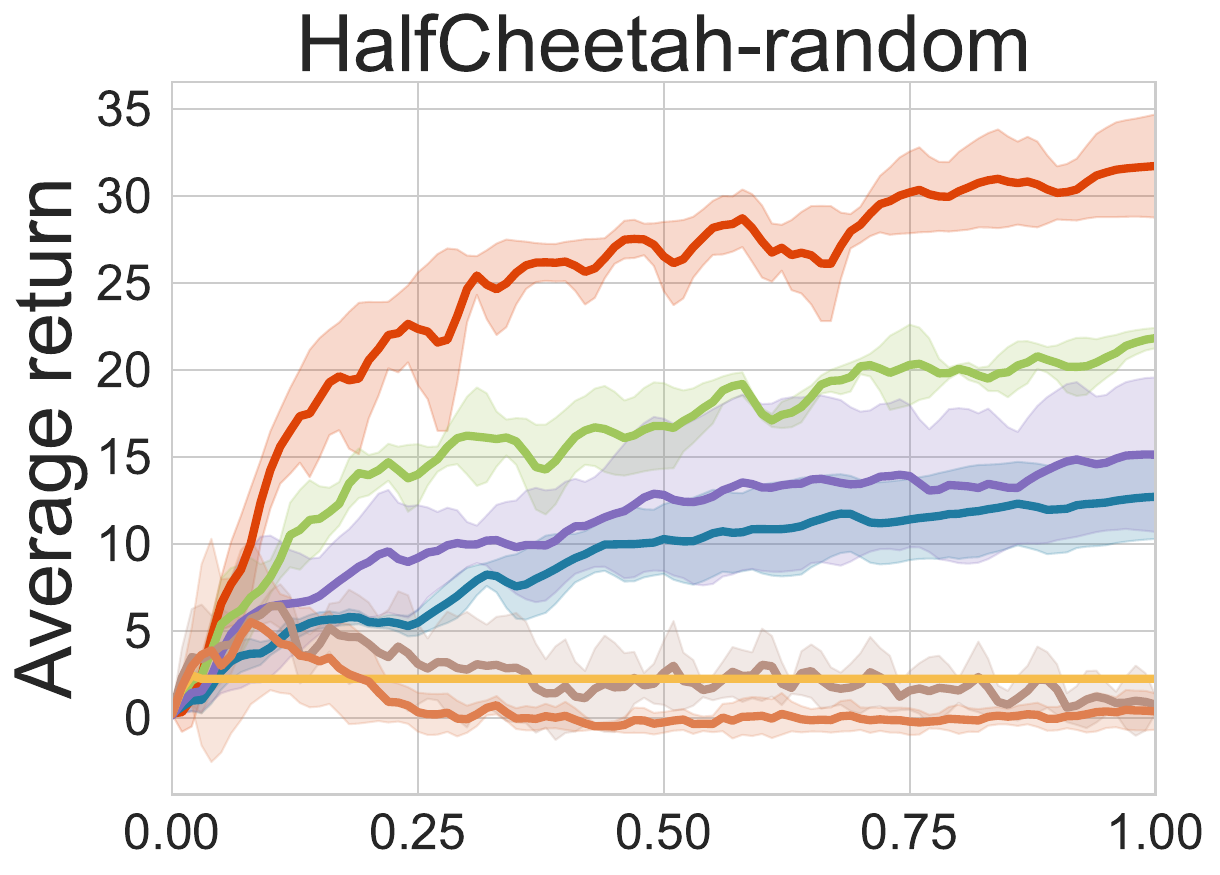}}
    \hspace{-4pt}
    \subfigure{
        \includegraphics[width=0.24\textwidth]{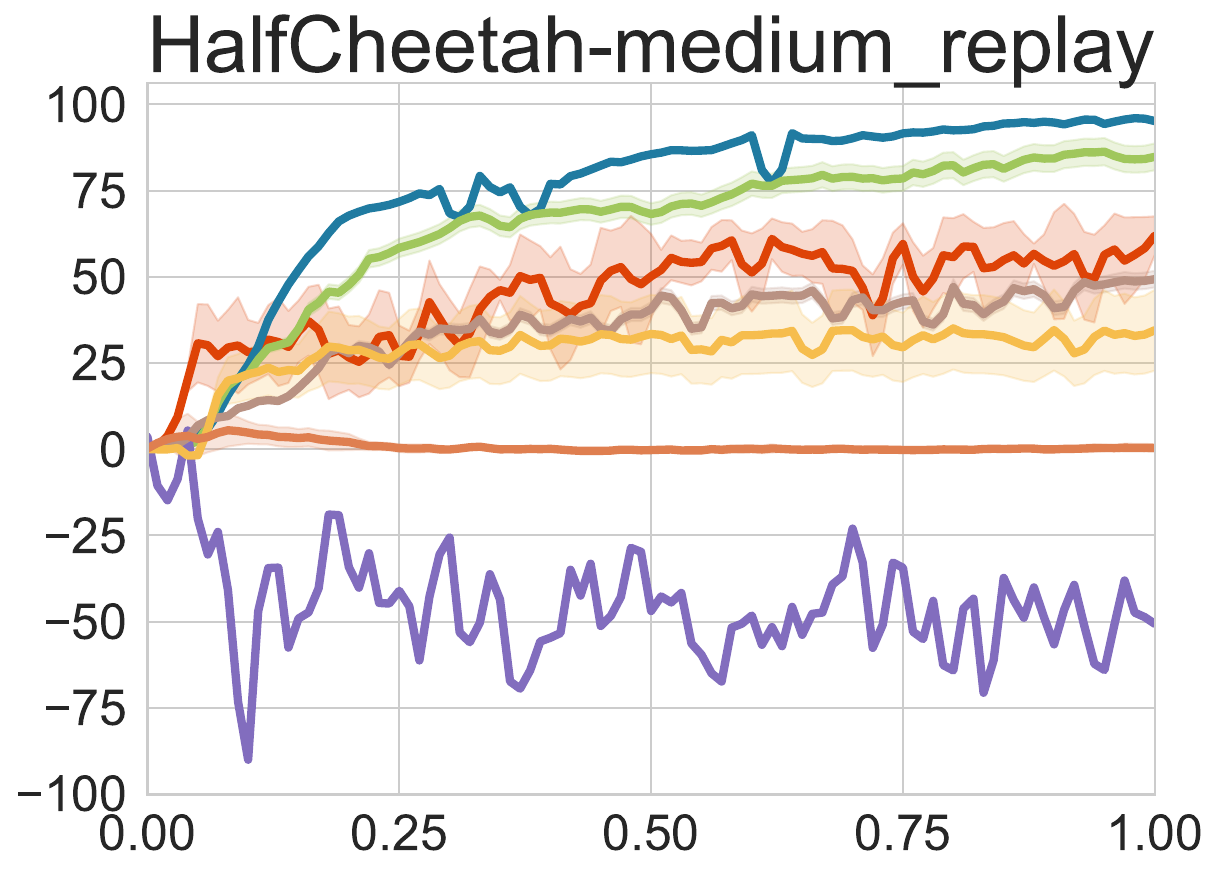}}
    \hspace{-4pt}
    \subfigure{
        \includegraphics[width=0.24\textwidth]{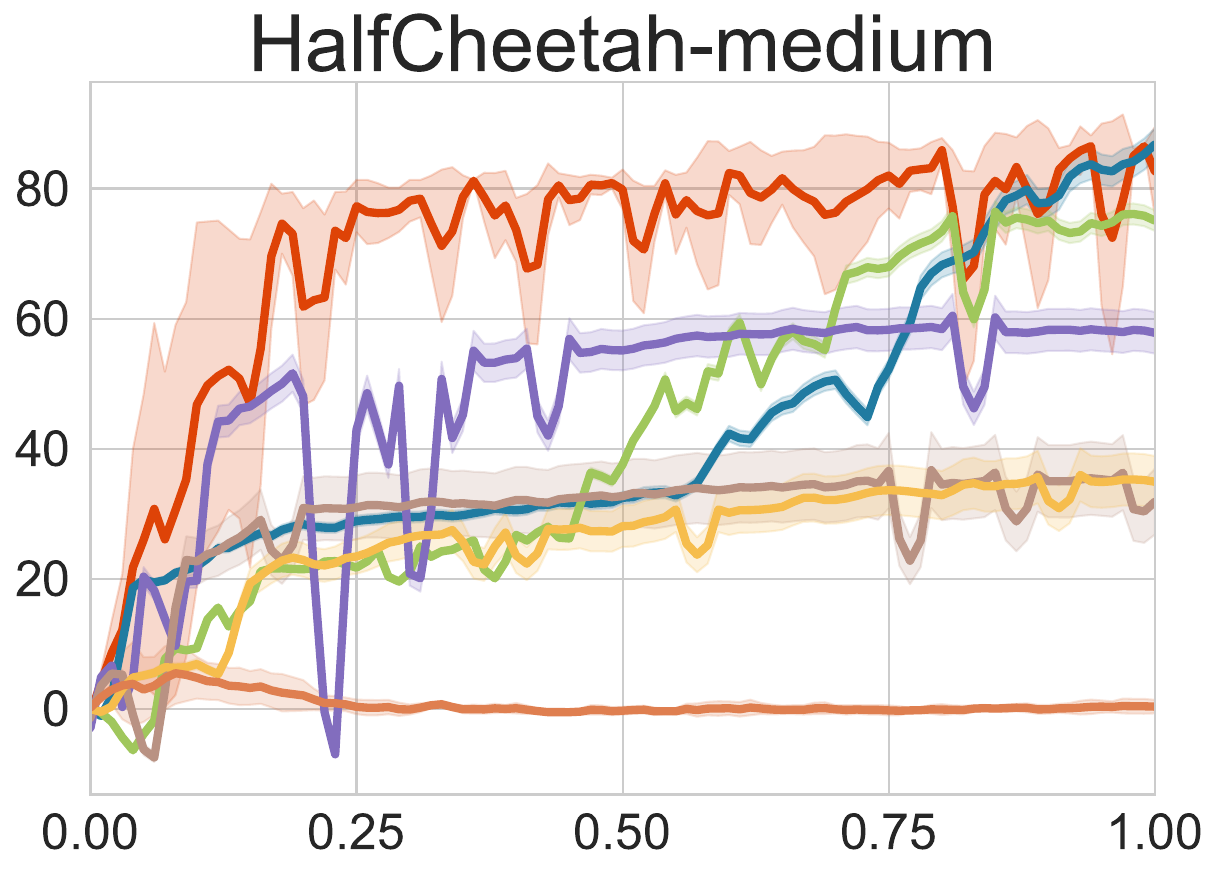}}
    \hspace{-4pt}
    \subfigure{
        \includegraphics[width=0.24\textwidth]{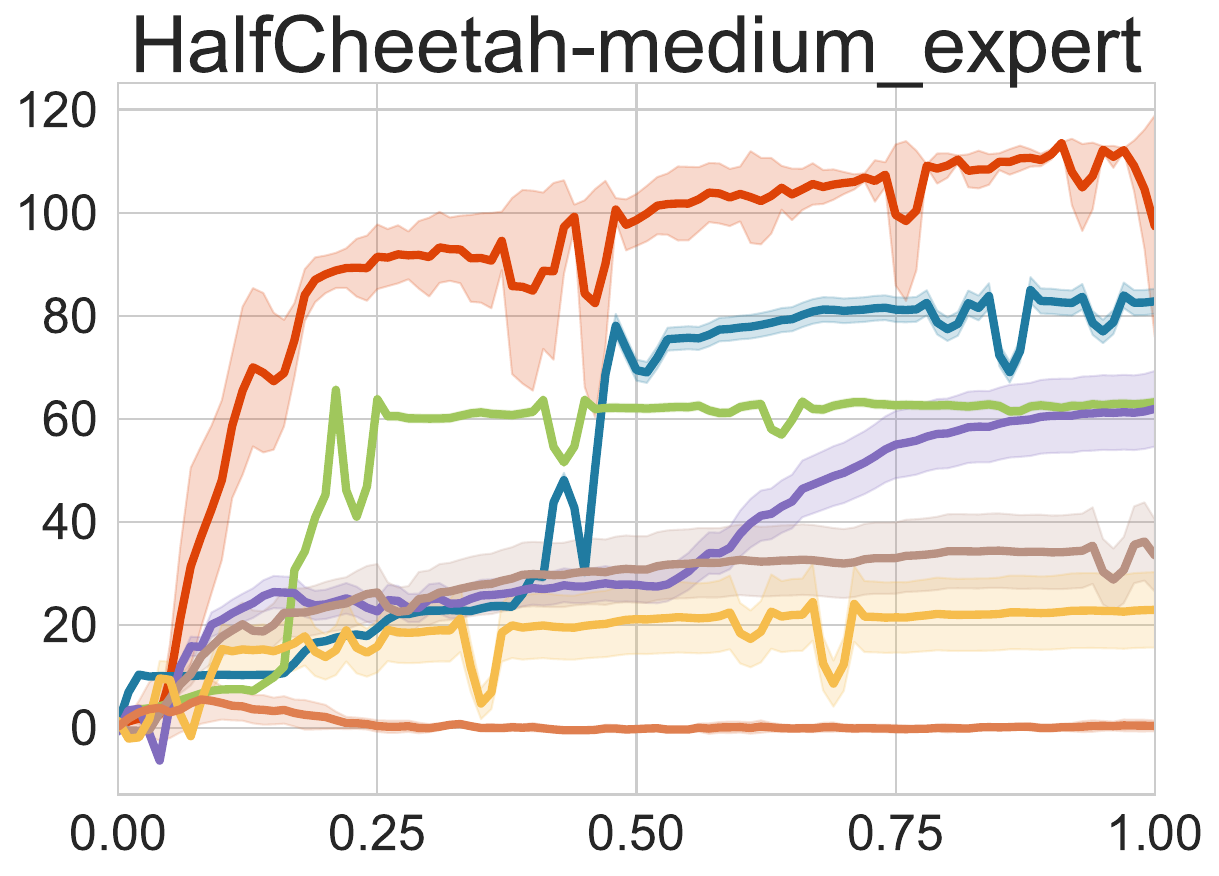}}
    
    \vspace{-10pt}
    \subfigure{
        \includegraphics[width=0.24\textwidth]{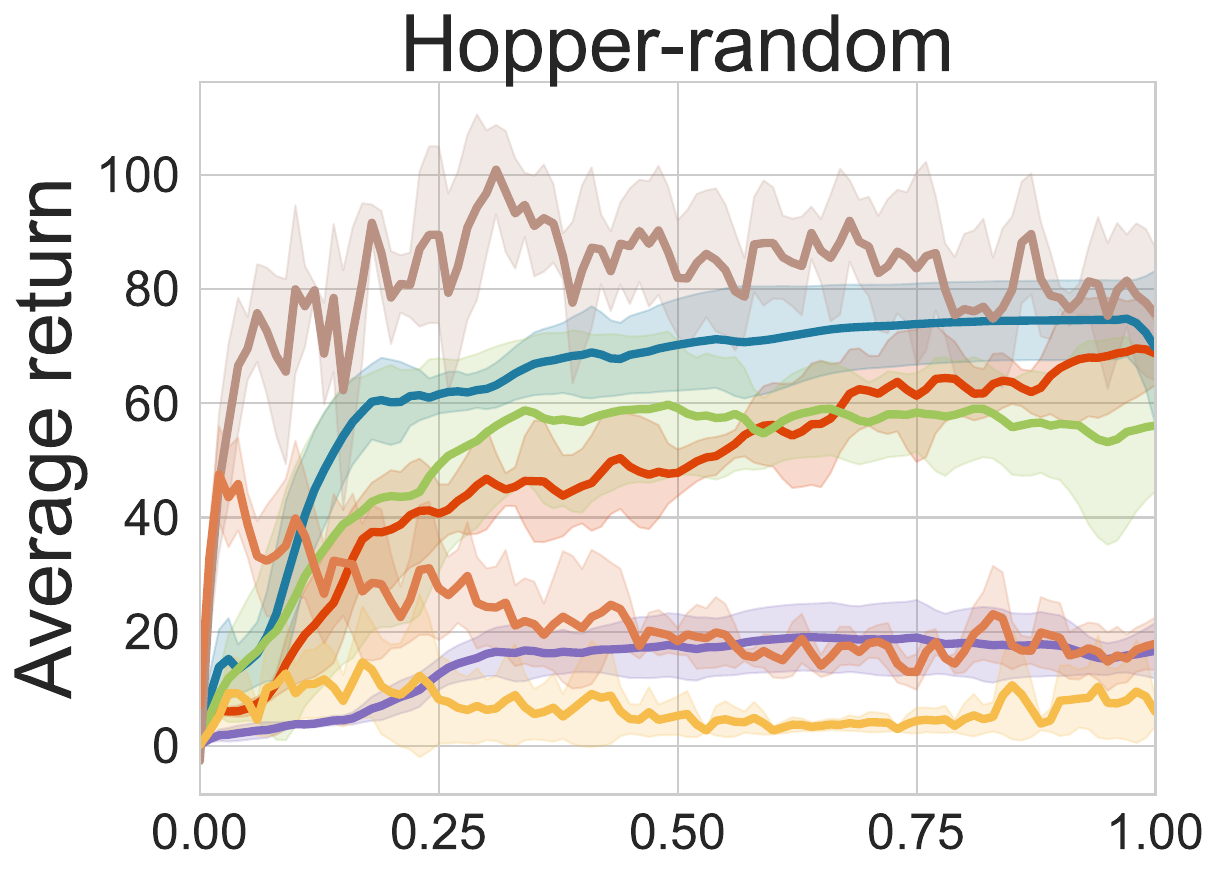}}
    \hspace{-4pt}
    \subfigure{
        \includegraphics[width=0.24\textwidth]{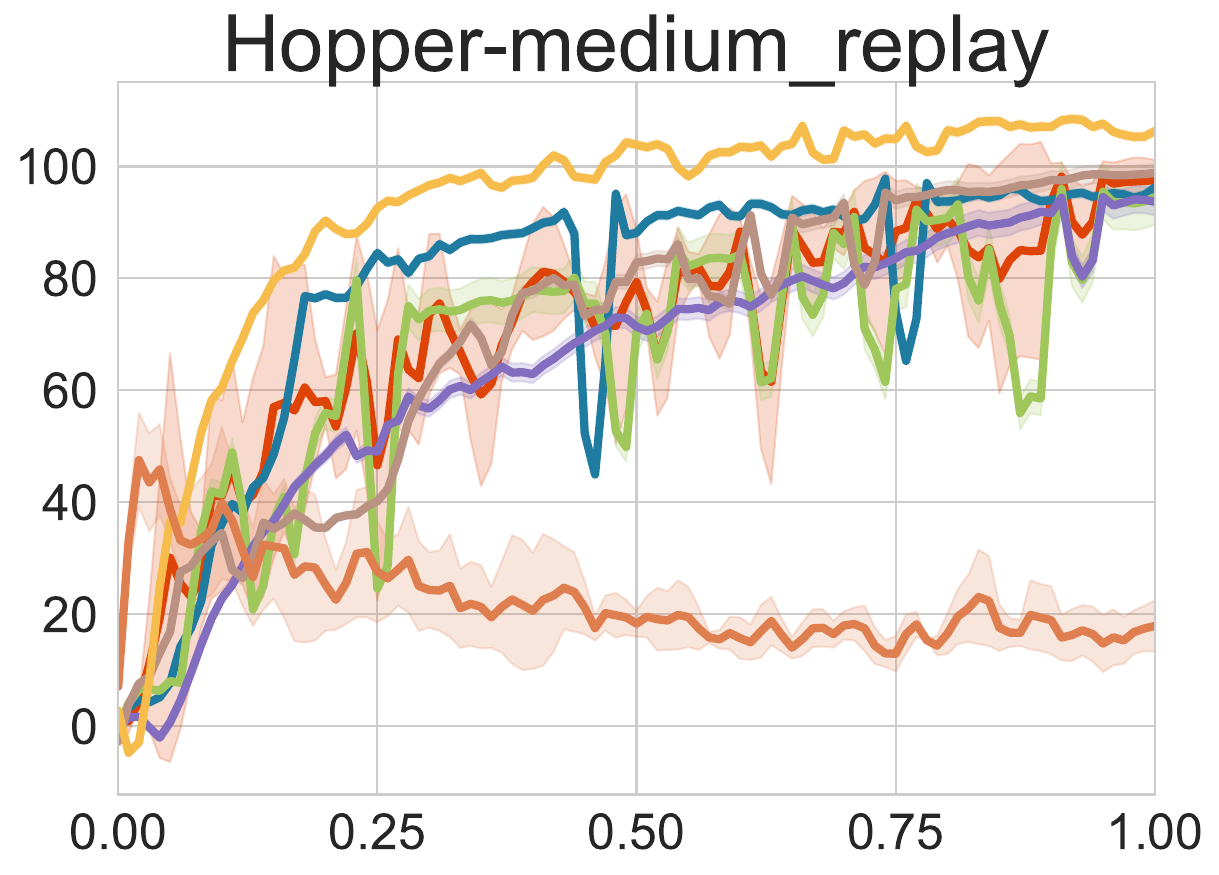}}
    \hspace{-4pt}
    \subfigure{
        \includegraphics[width=0.24\textwidth]{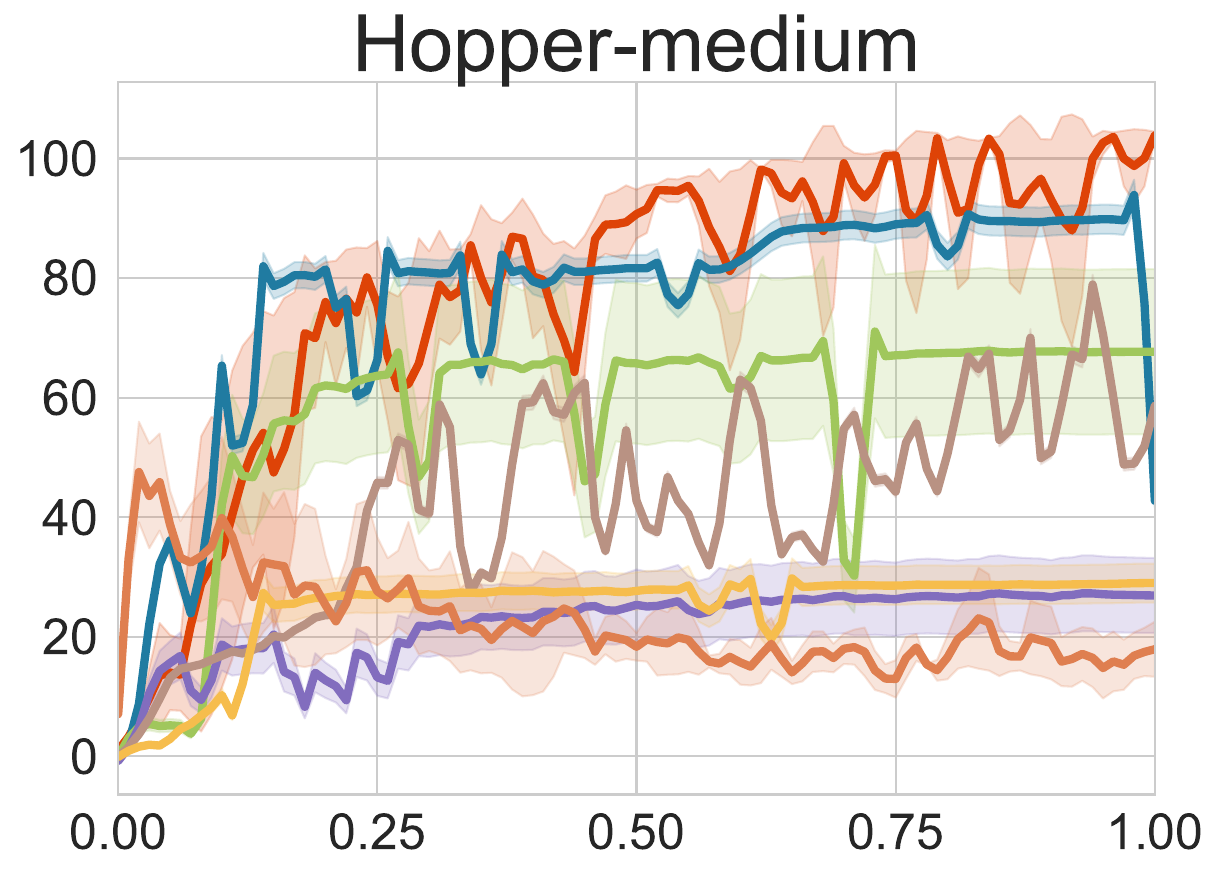}}
    \hspace{-4pt}
    \subfigure{
        \includegraphics[width=0.24\textwidth]{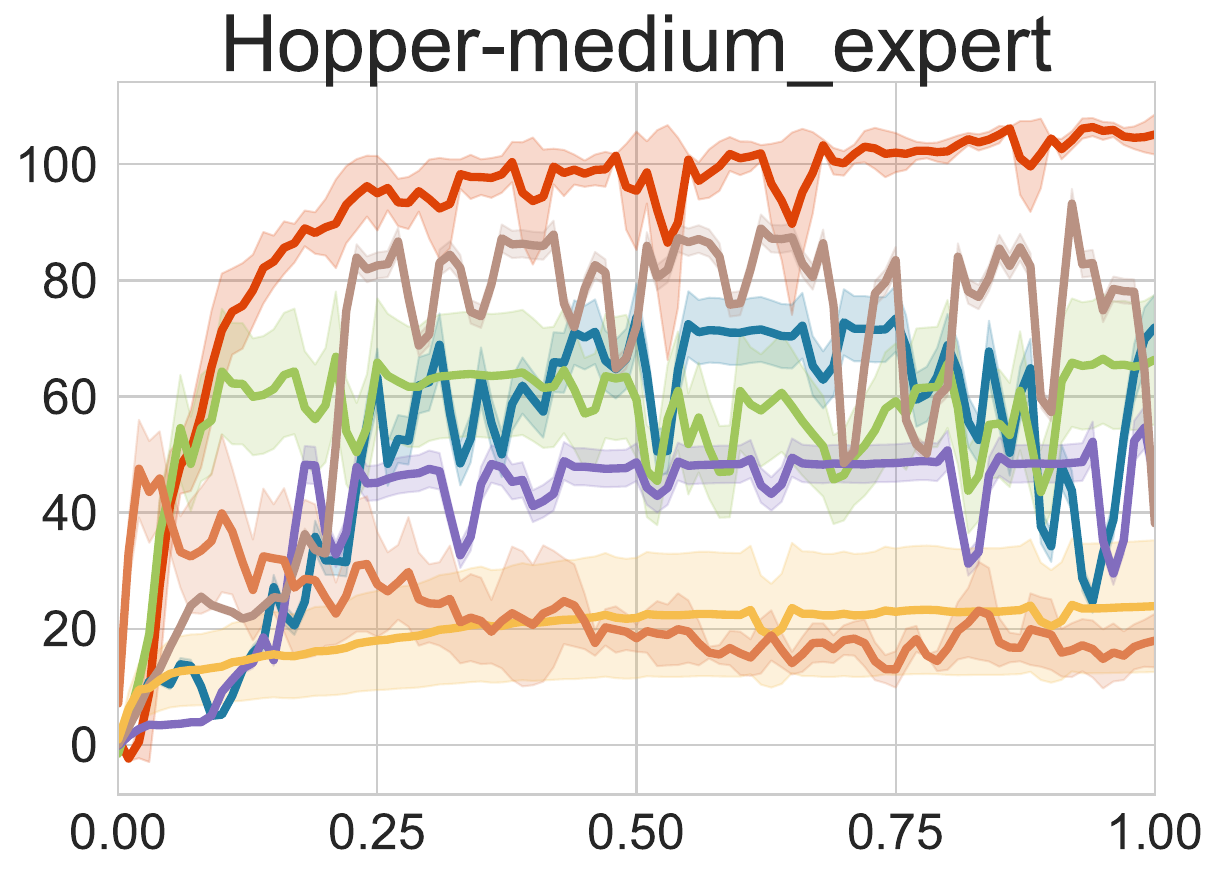}}
    
    \vspace{-10pt}
    \subfigure{
        \includegraphics[width=0.24\textwidth]{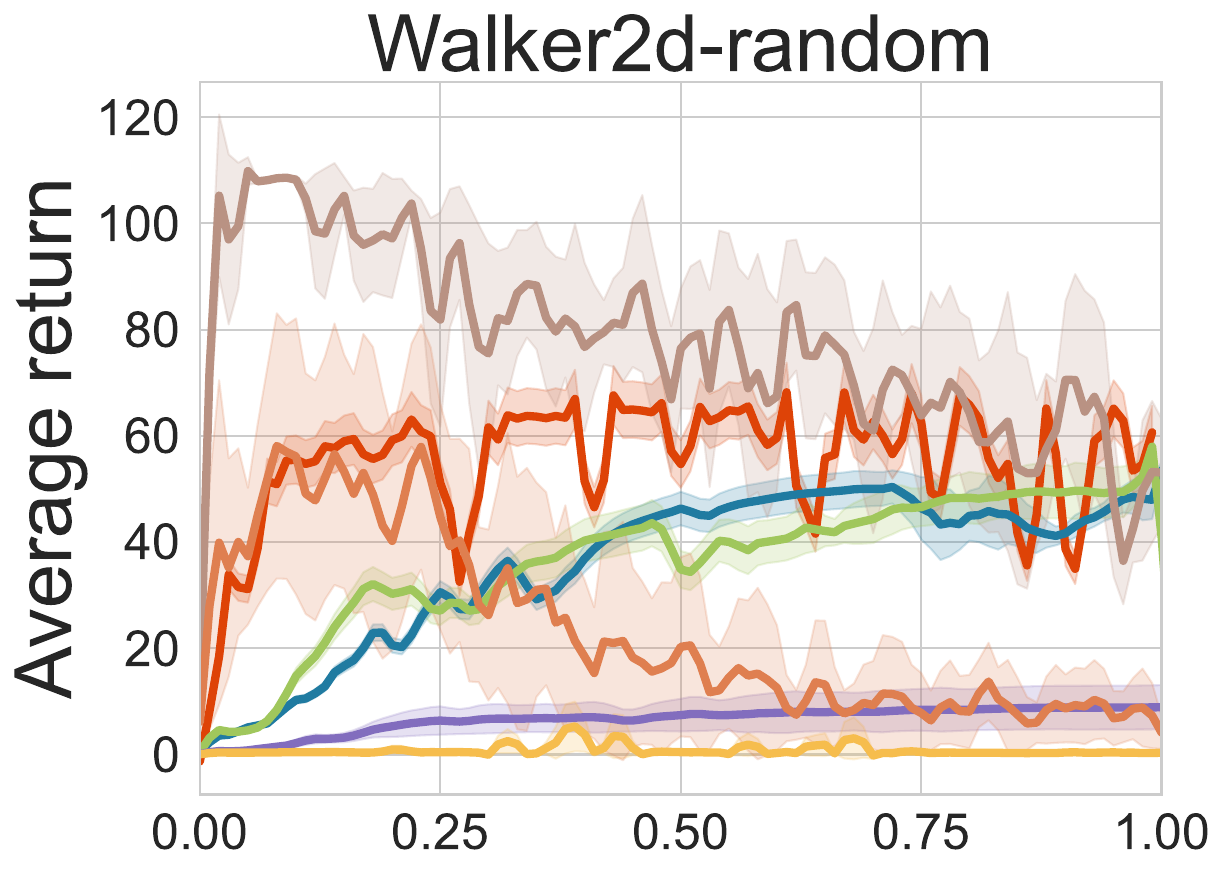}}
    \hspace{-4pt}
    \subfigure{
        \includegraphics[width=0.24\textwidth]{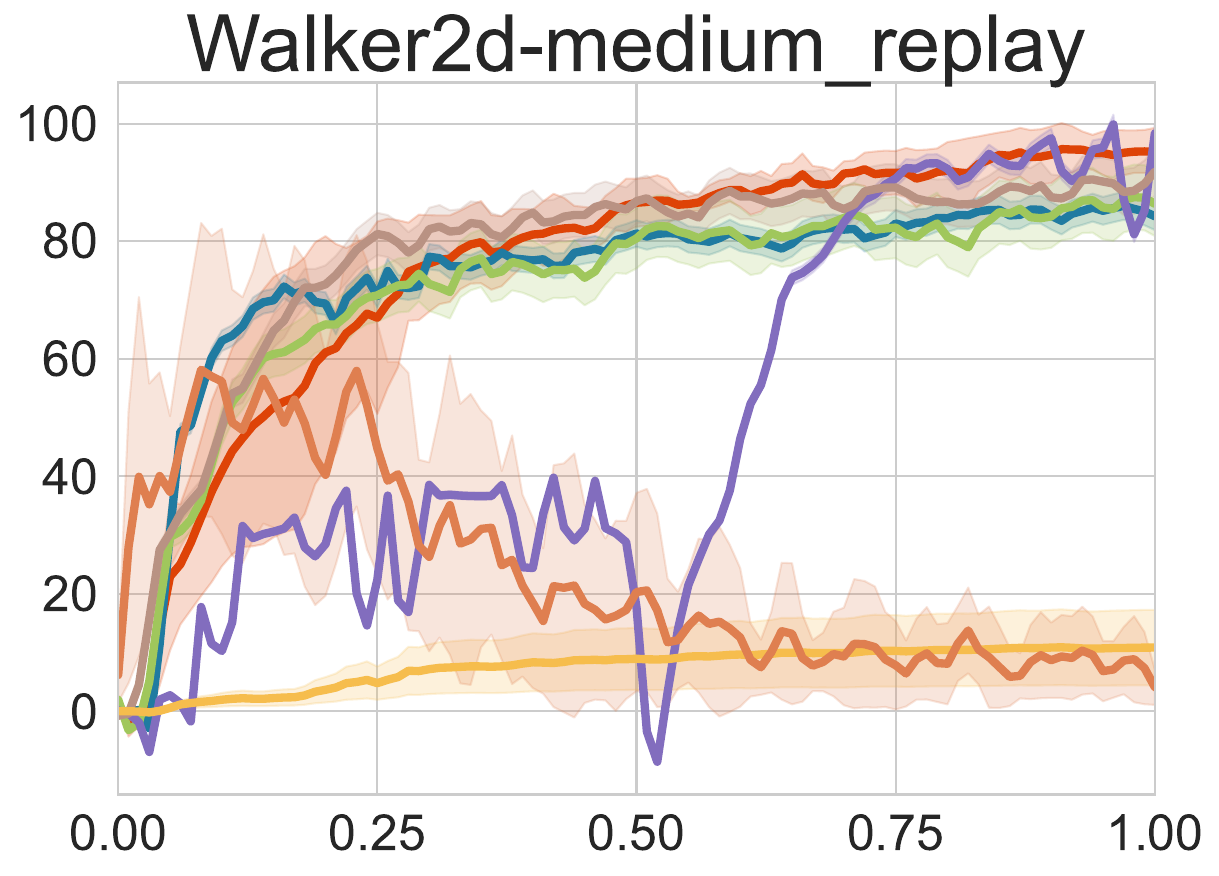}}
    \hspace{-4pt}
    \subfigure{
        \includegraphics[width=0.24\textwidth]{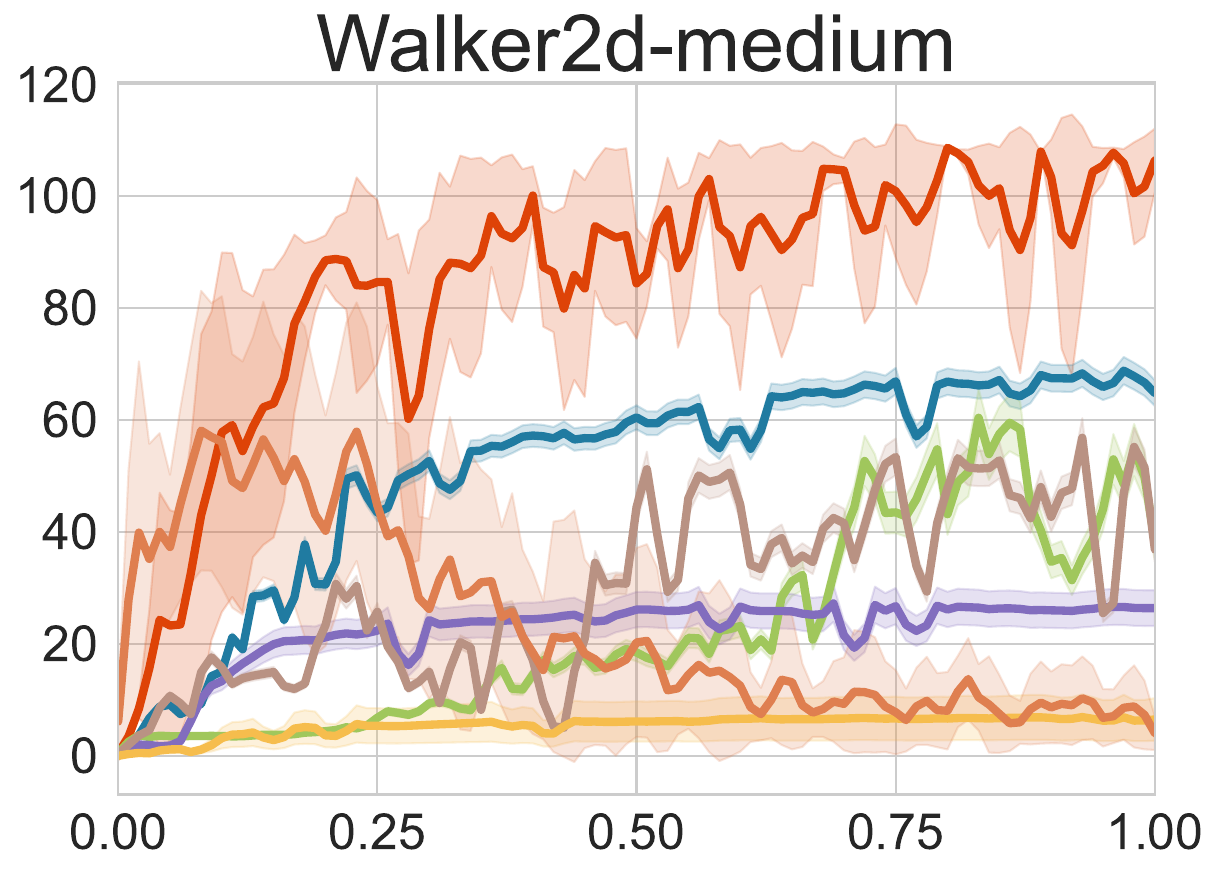}}
    \hspace{-4pt}
    \subfigure{
        \includegraphics[width=0.24\textwidth]{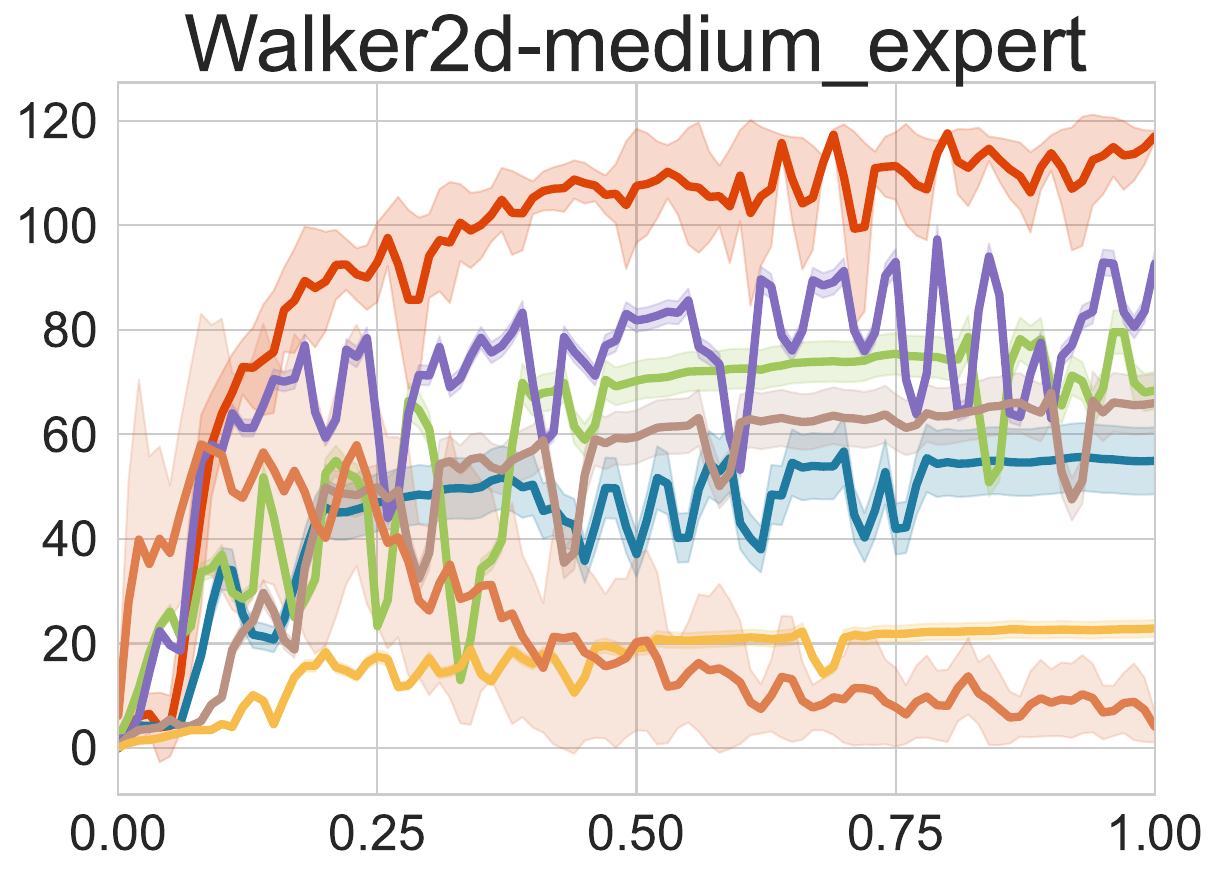}}

    \vspace{-10pt}
    \subfigure{
        \includegraphics[width=0.24\textwidth]{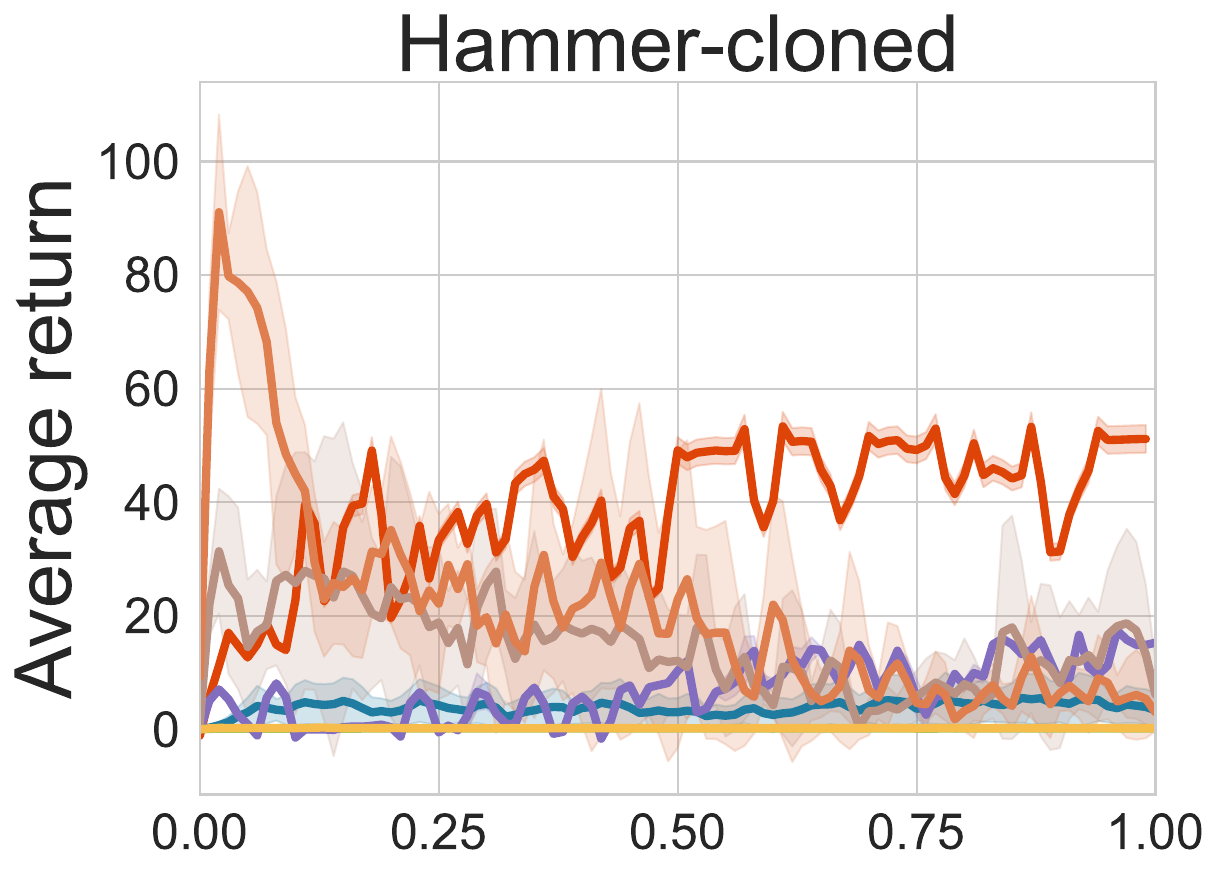}}
    \hspace{-4pt}
    \subfigure{
        \includegraphics[width=0.24\textwidth]{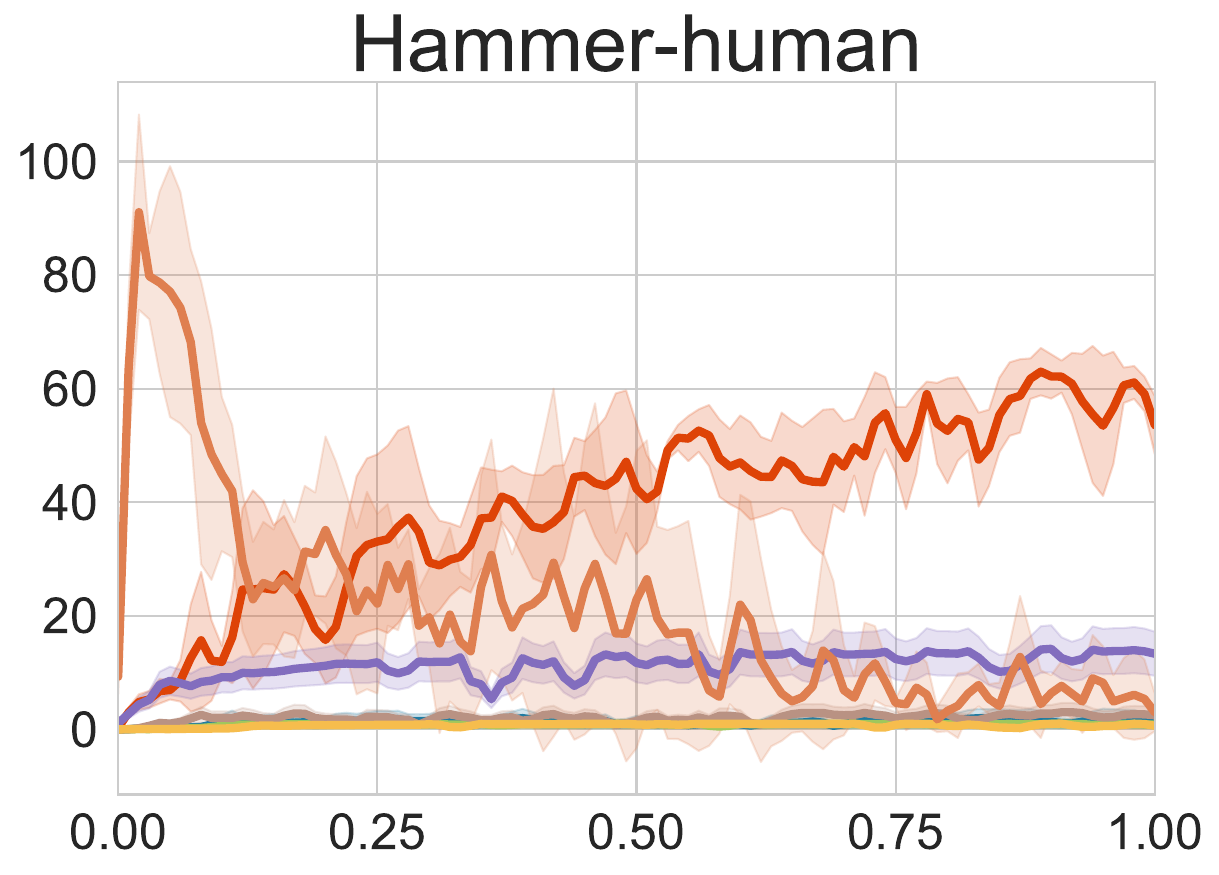}}
    \hspace{-4pt}
    \subfigure{
        \includegraphics[width=0.24\textwidth]{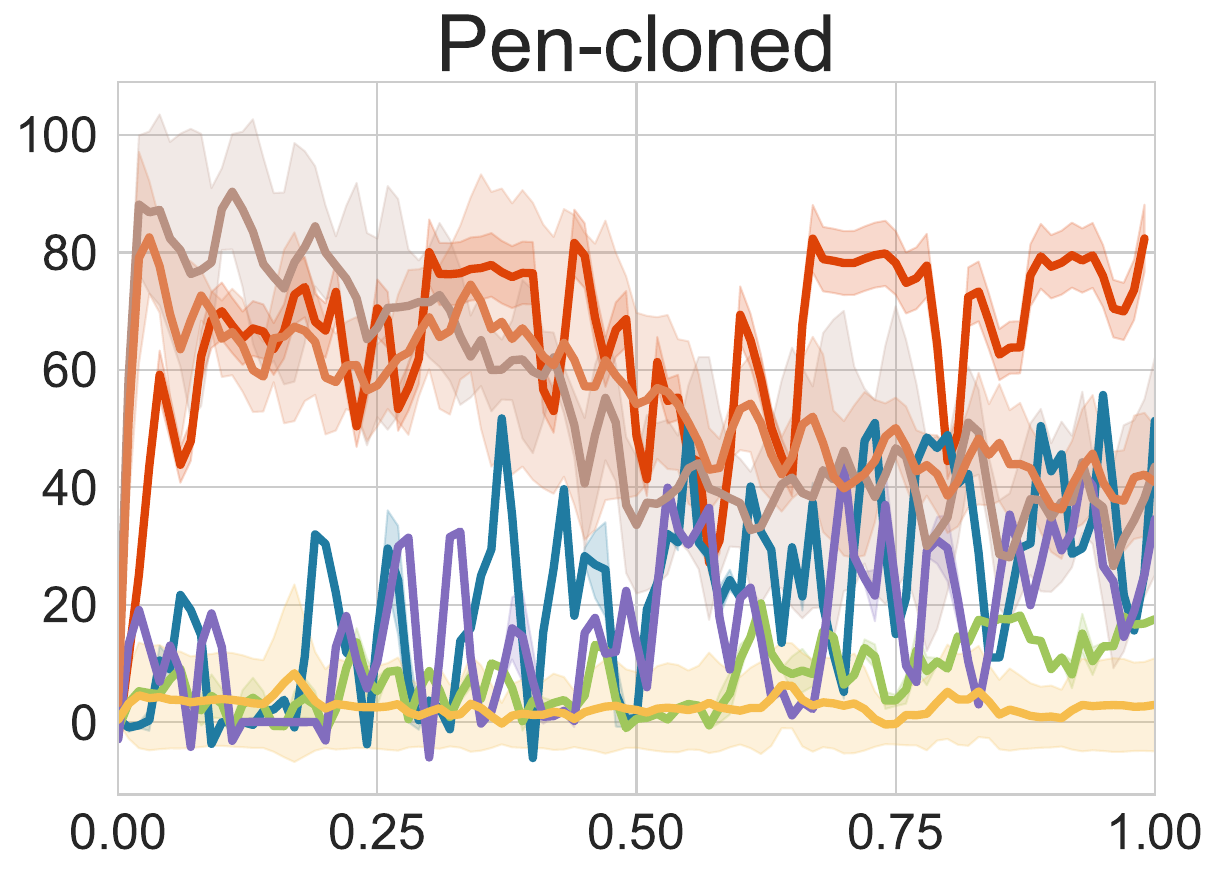}}
    \hspace{-4pt}
    \subfigure{
        \includegraphics[width=0.24\textwidth]{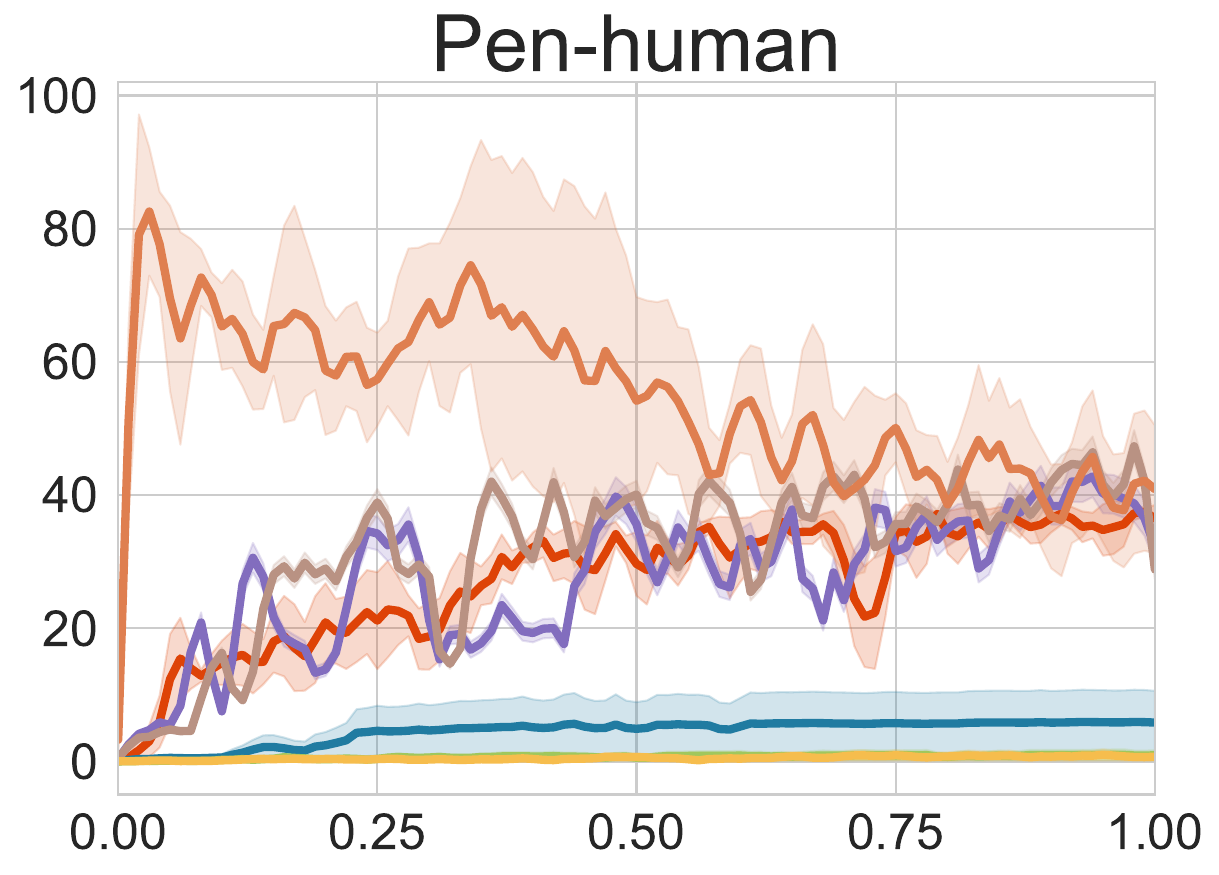}}
        
    \vspace{-10pt}
    \subfigure{
        \includegraphics[width=0.24\textwidth]{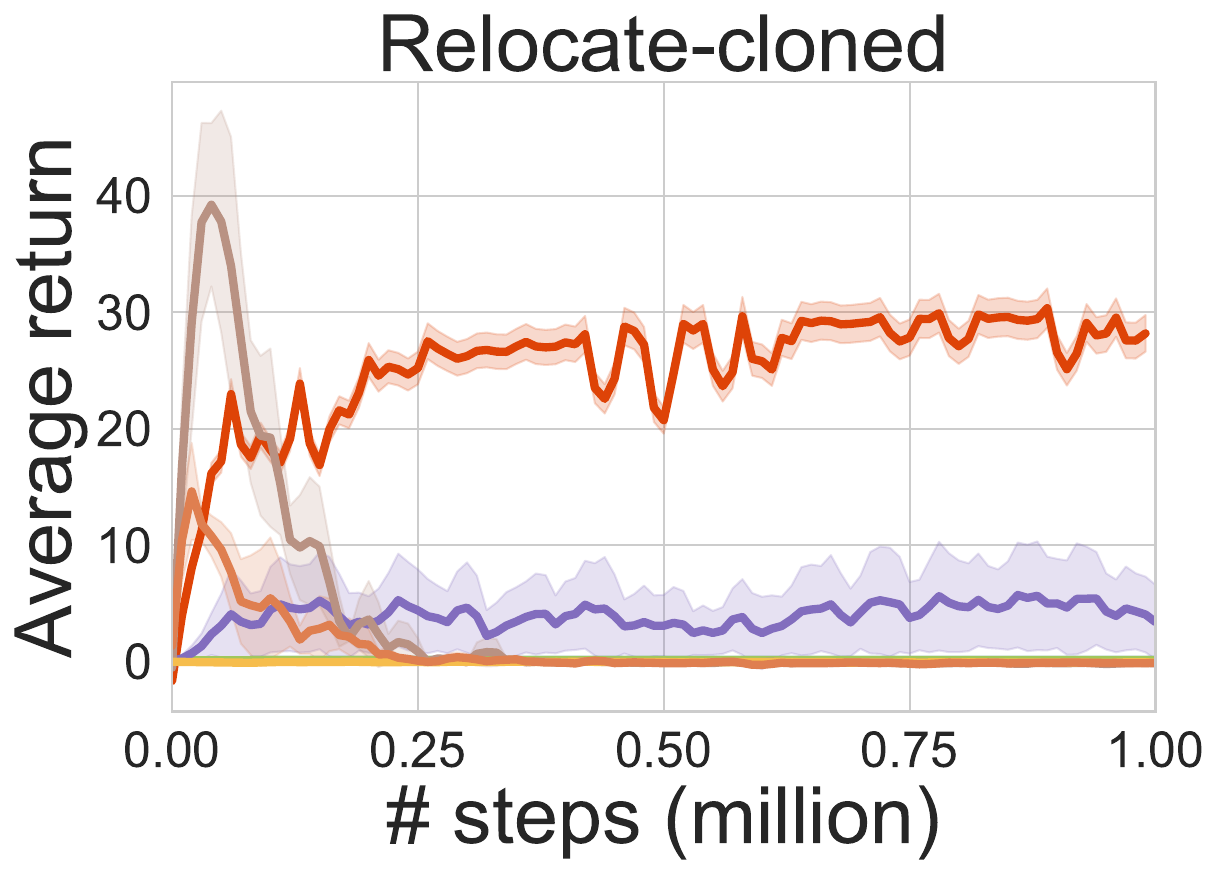}}
    \hspace{-4pt}
    \subfigure{
        \includegraphics[width=0.24\textwidth]{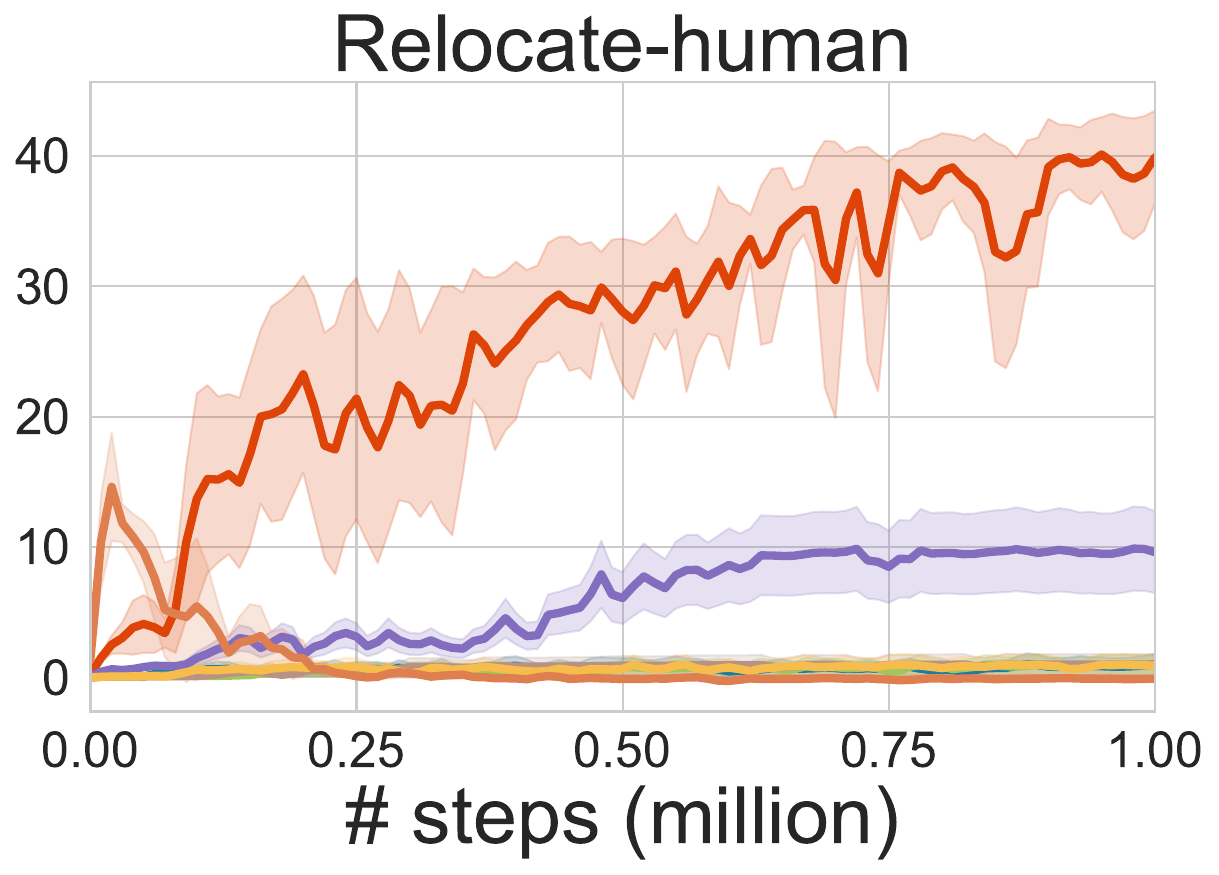}}
    \hspace{-4pt}
    \subfigure{
        \includegraphics[width=0.24\textwidth]{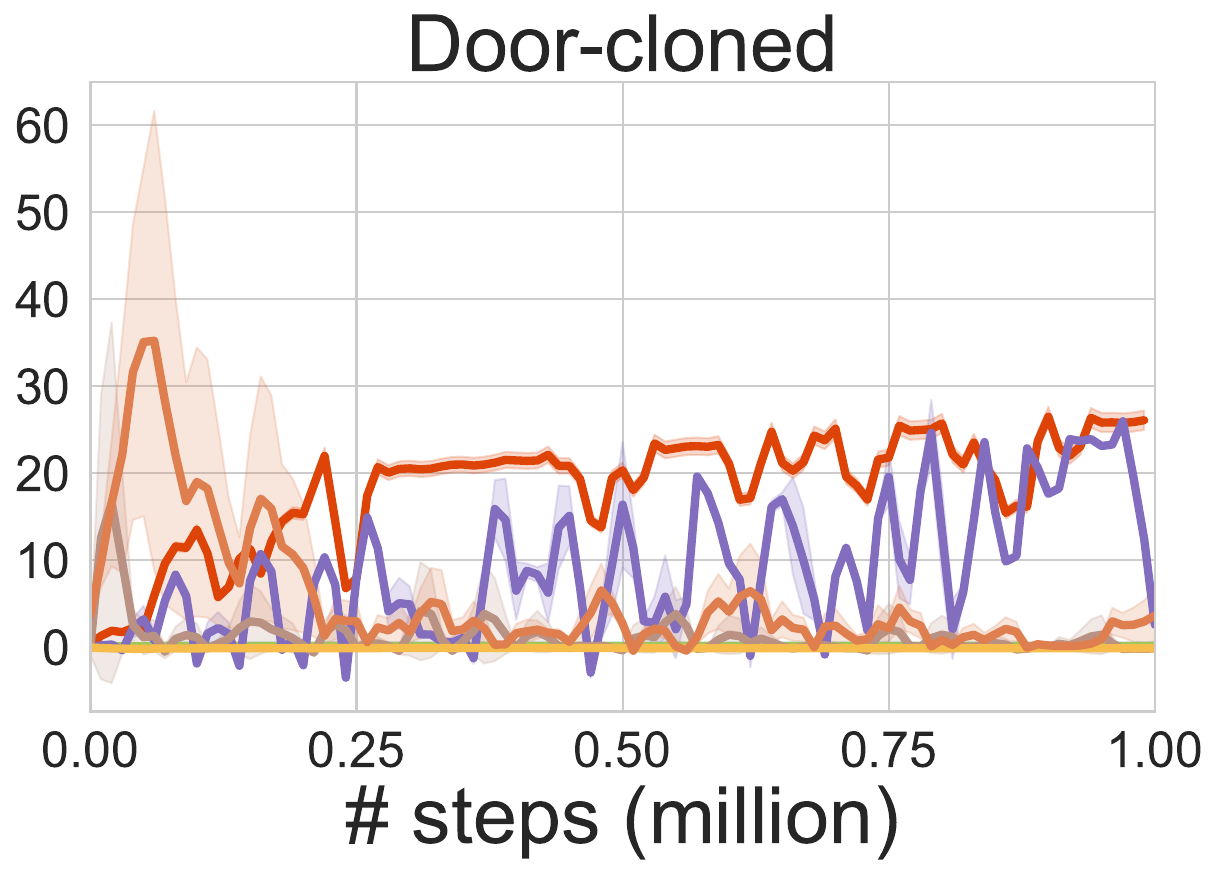}}
    \hspace{-4pt}
    \subfigure{
        \includegraphics[width=0.24\textwidth]{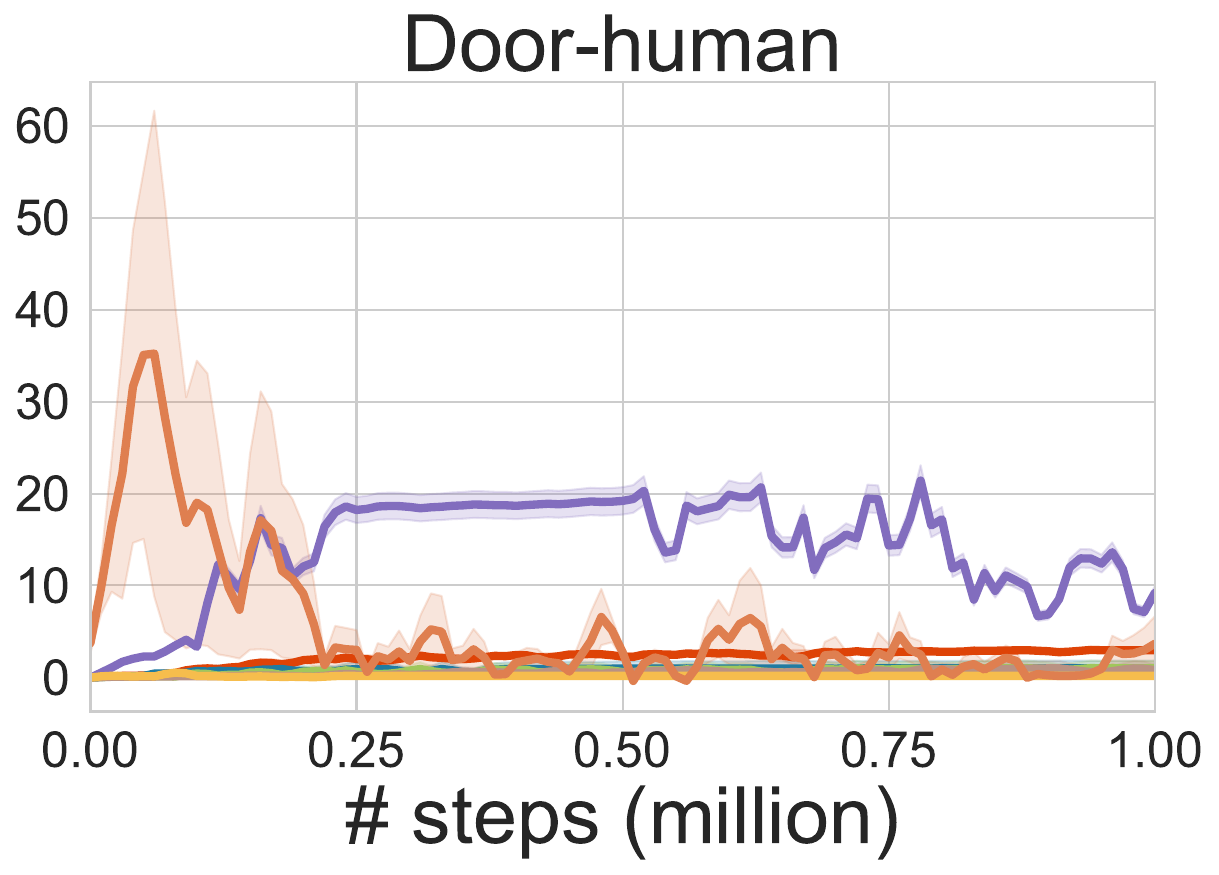}}
    \caption{Normalized performance under varying qualities of imperfect data.}    
    \label{fig:imperfect_all}
\end{figure}

\begin{figure}[ht]
    \centering
    \subfigure{
        \includegraphics[width=0.24\textwidth]{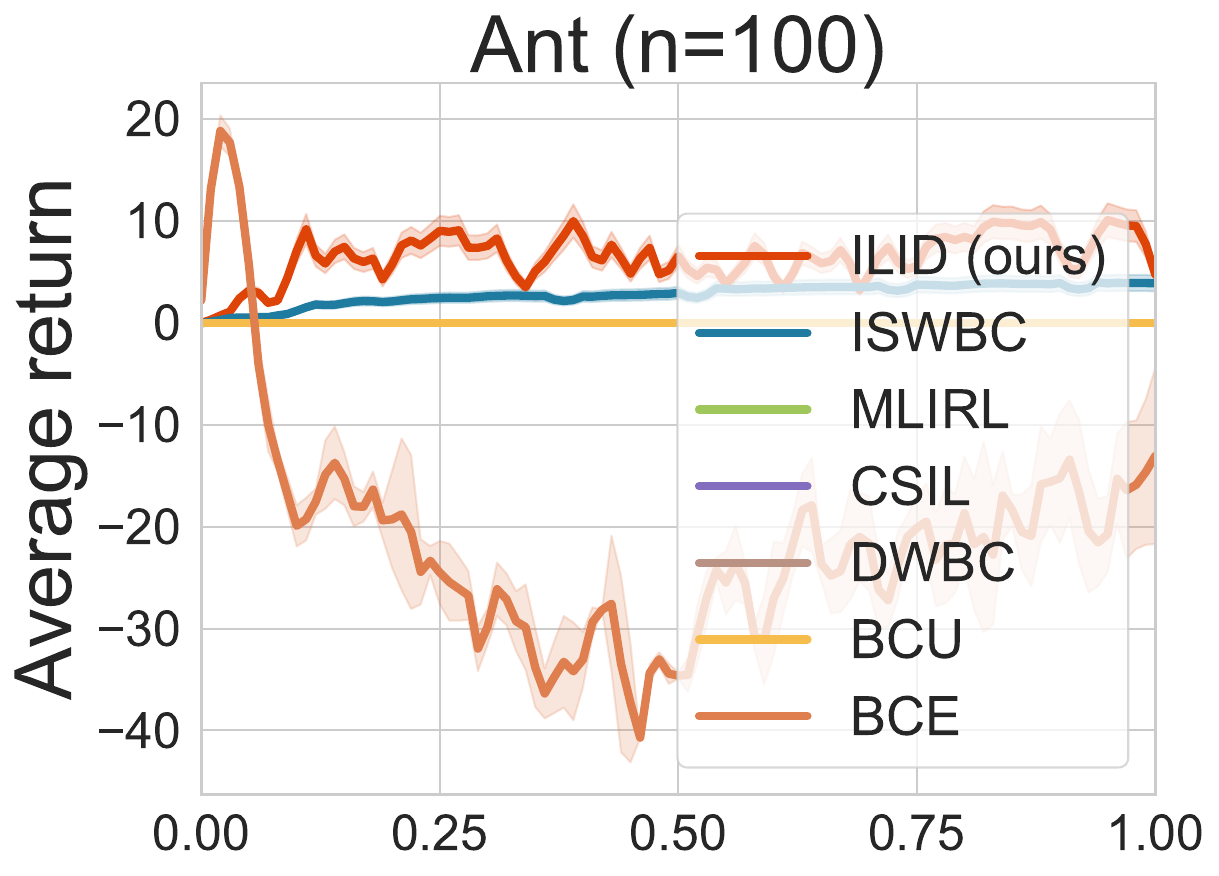}}
    \hspace{-4pt}
    \subfigure{
        \includegraphics[width=0.24\textwidth]{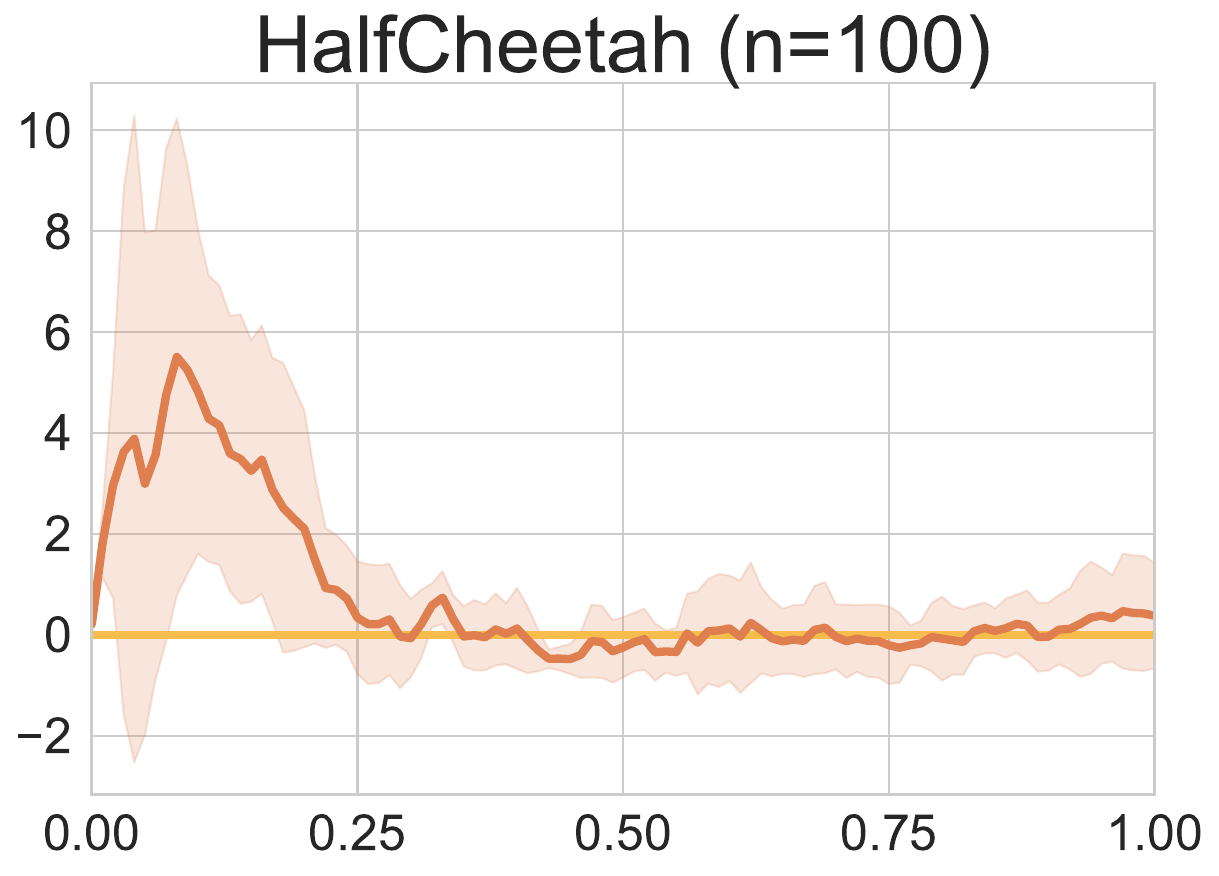}}
    \hspace{-4pt}
    \subfigure{
        \includegraphics[width=0.24\textwidth]{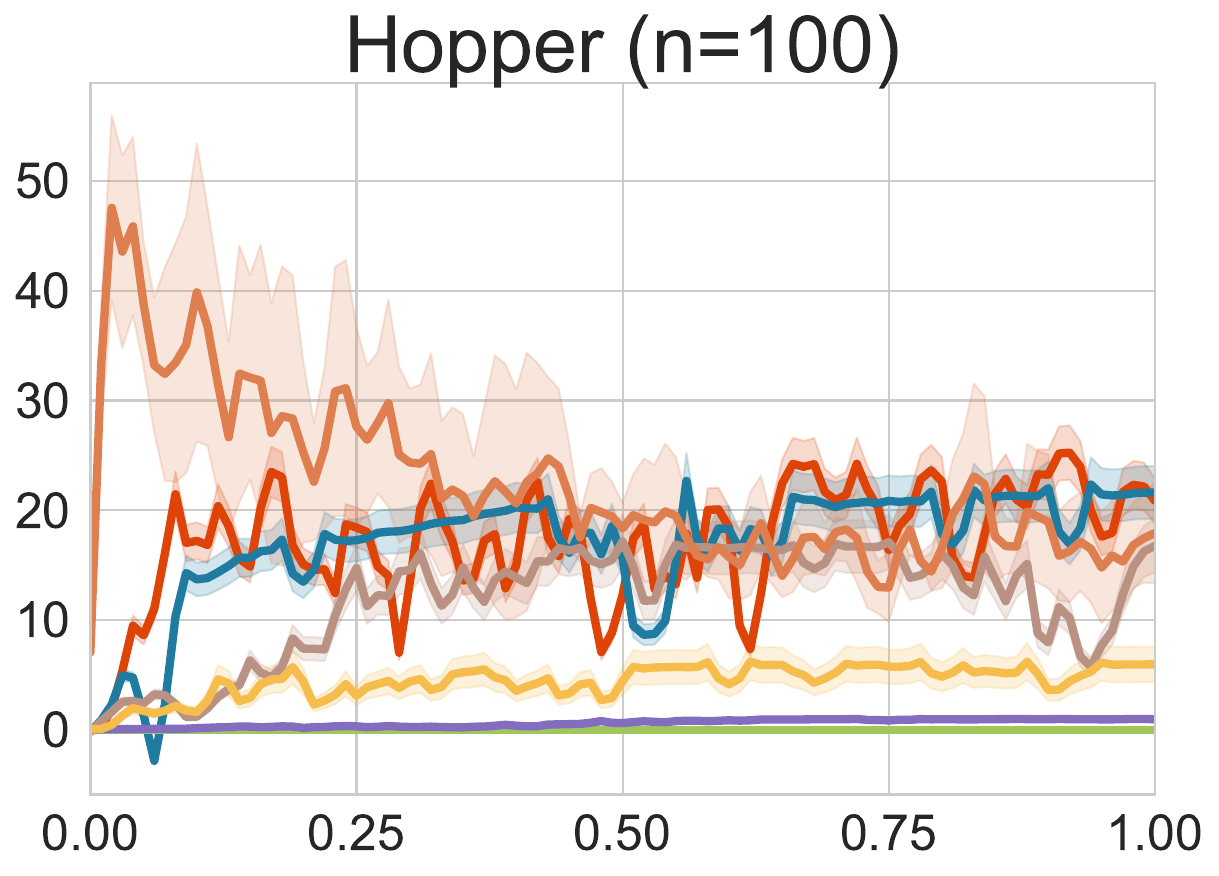}}
    \hspace{-4pt}
    \subfigure{
        \includegraphics[width=0.24\textwidth]{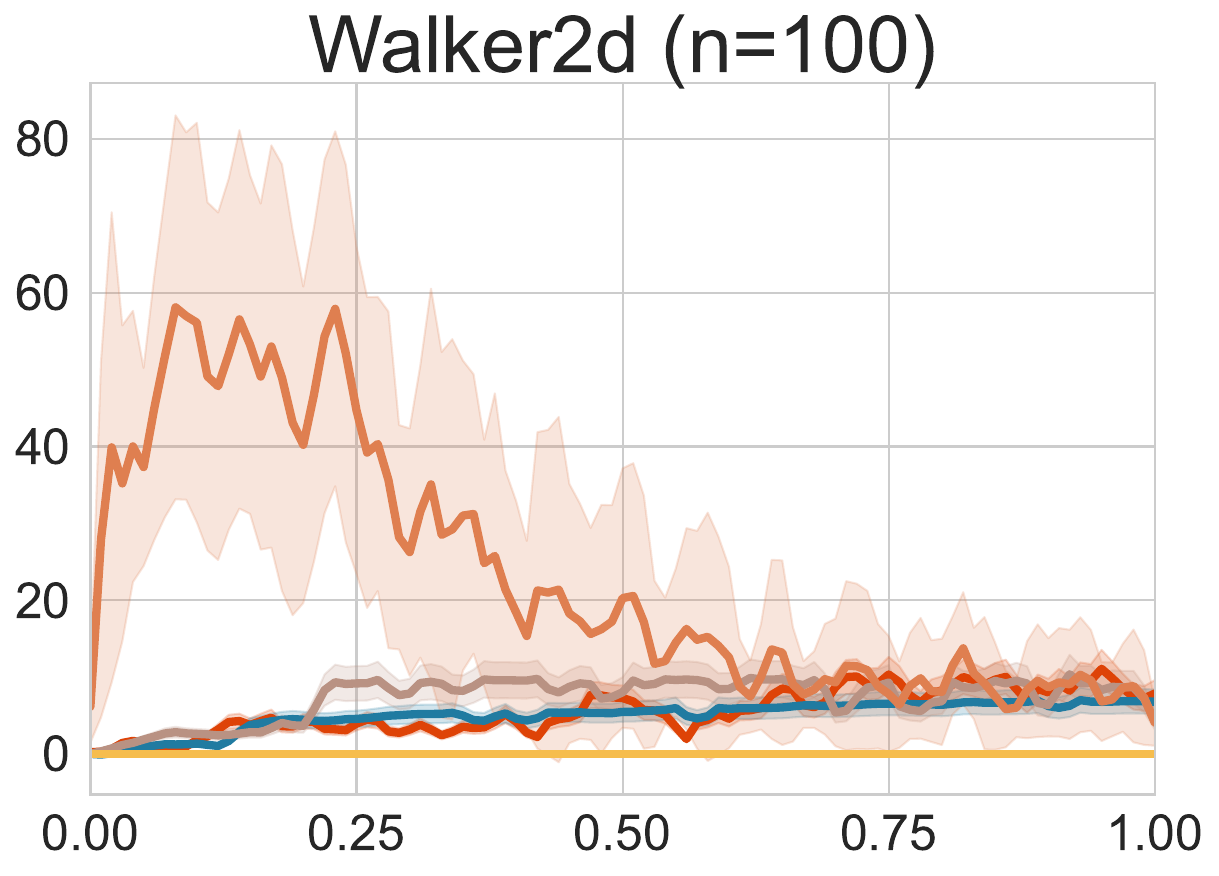}}
    
    \vspace{-10pt}
    \subfigure{
        \includegraphics[width=0.24\textwidth]{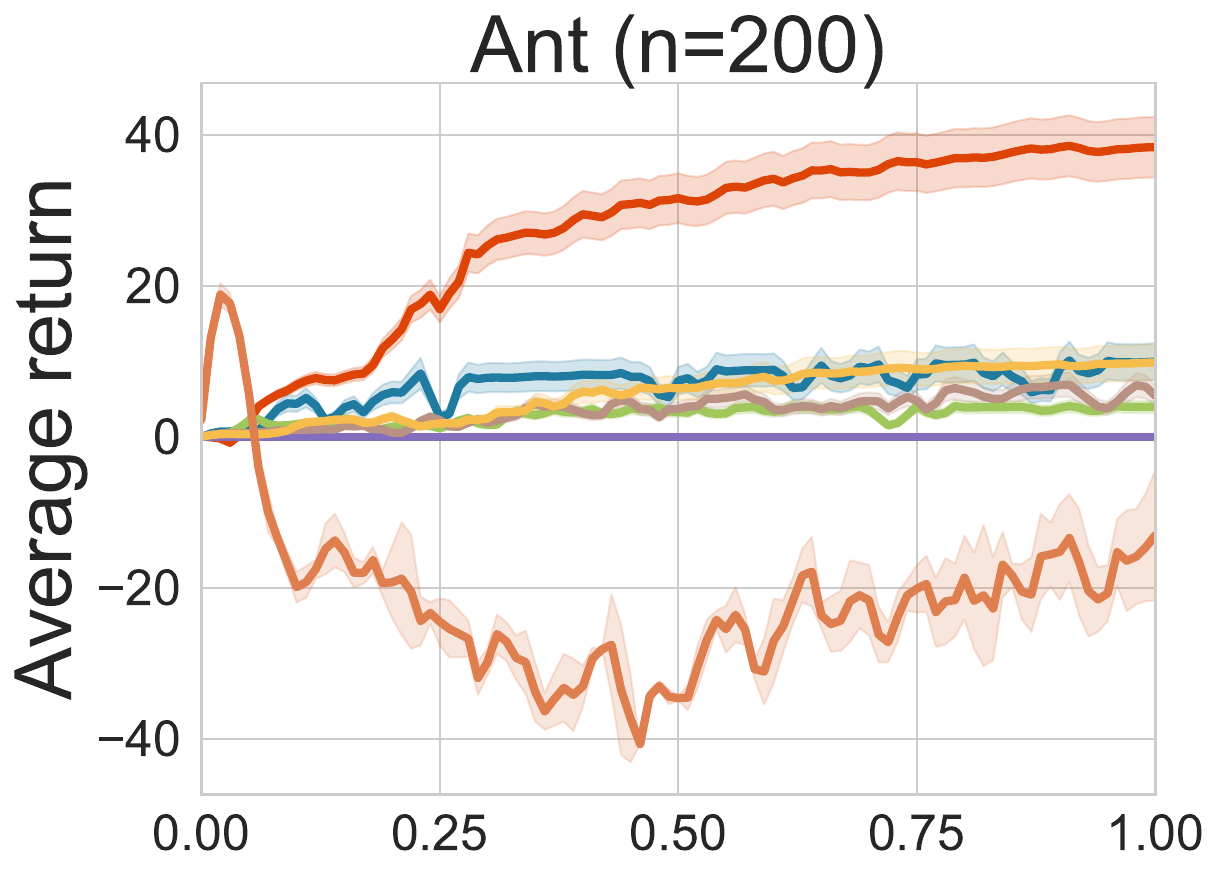}}
    \hspace{-4pt}
    \subfigure{
        \includegraphics[width=0.24\textwidth]{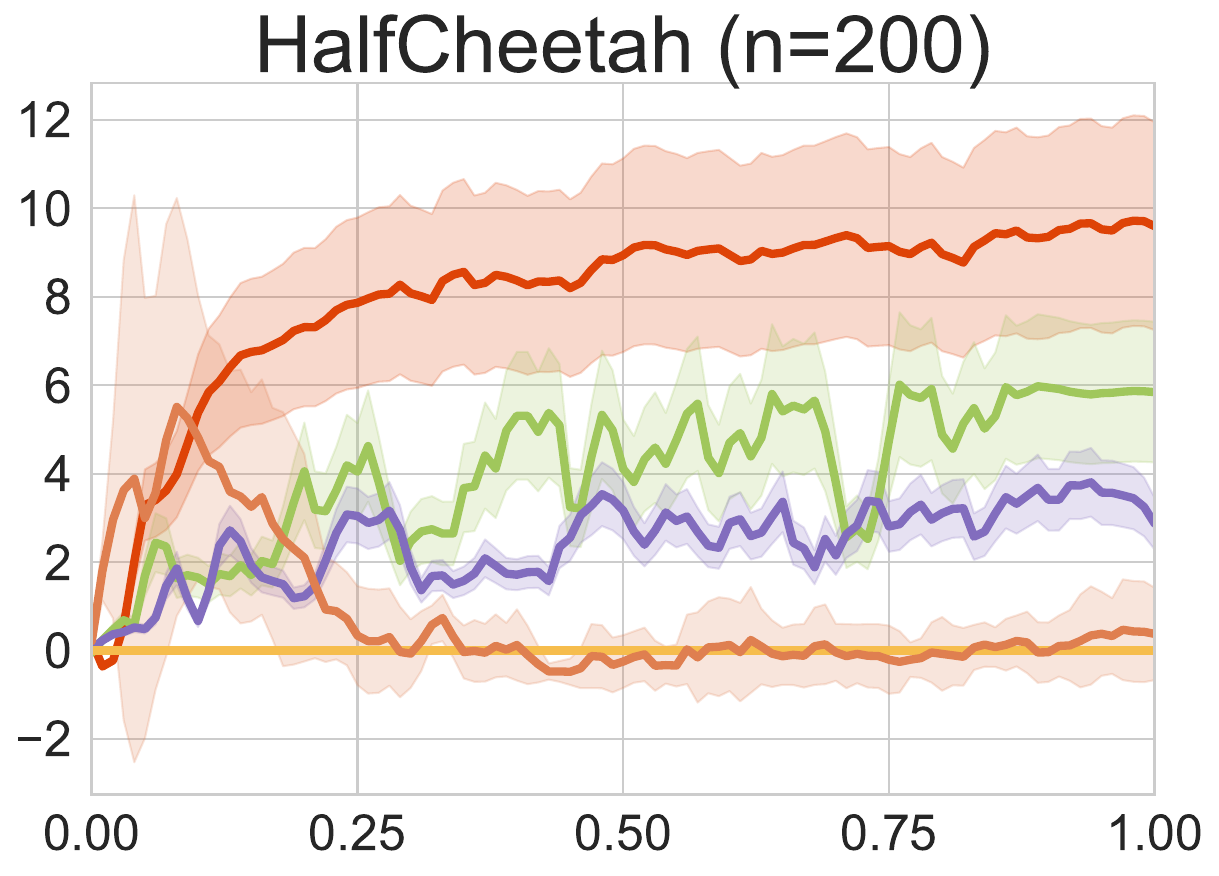}}
    \hspace{-4pt}
    \subfigure{
        \includegraphics[width=0.24\textwidth]{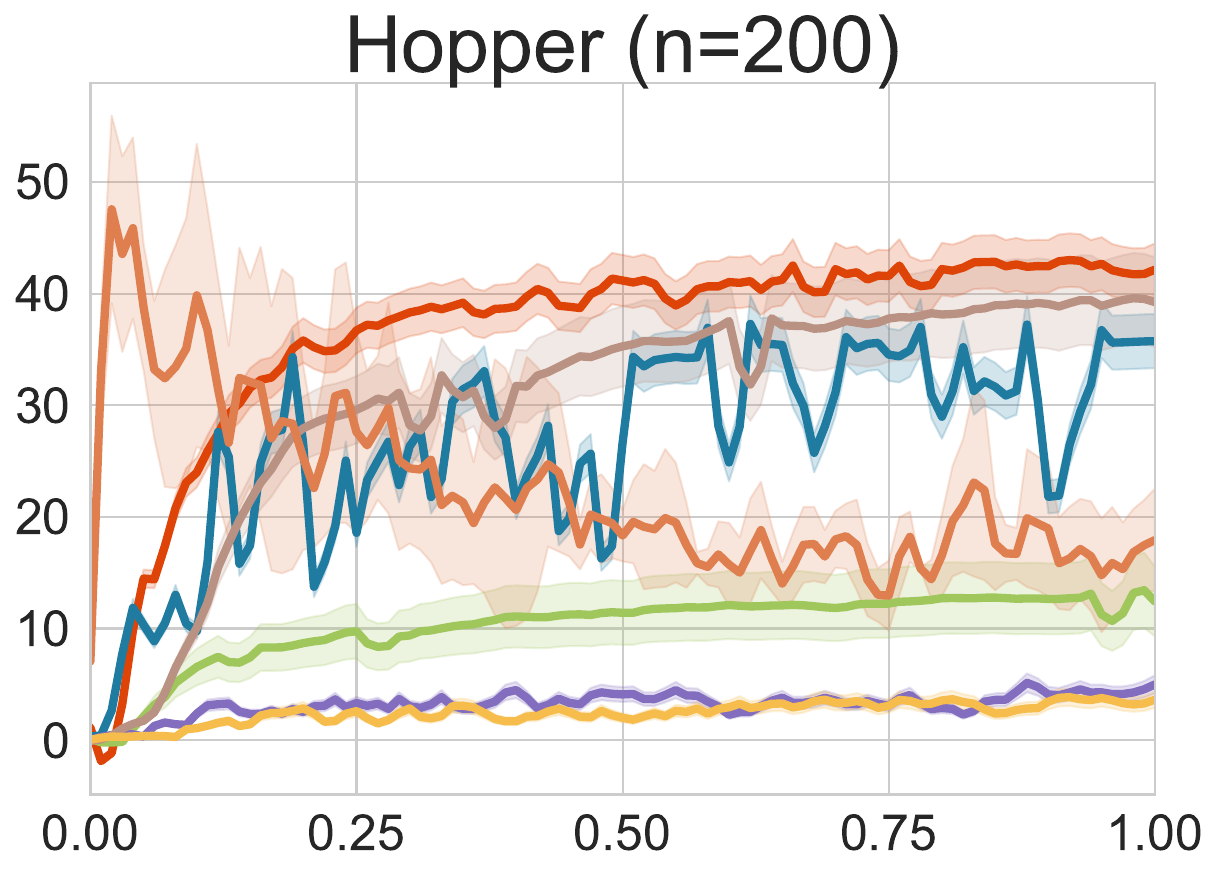}}
    \hspace{-4pt}
    \subfigure{
        \includegraphics[width=0.24\textwidth]{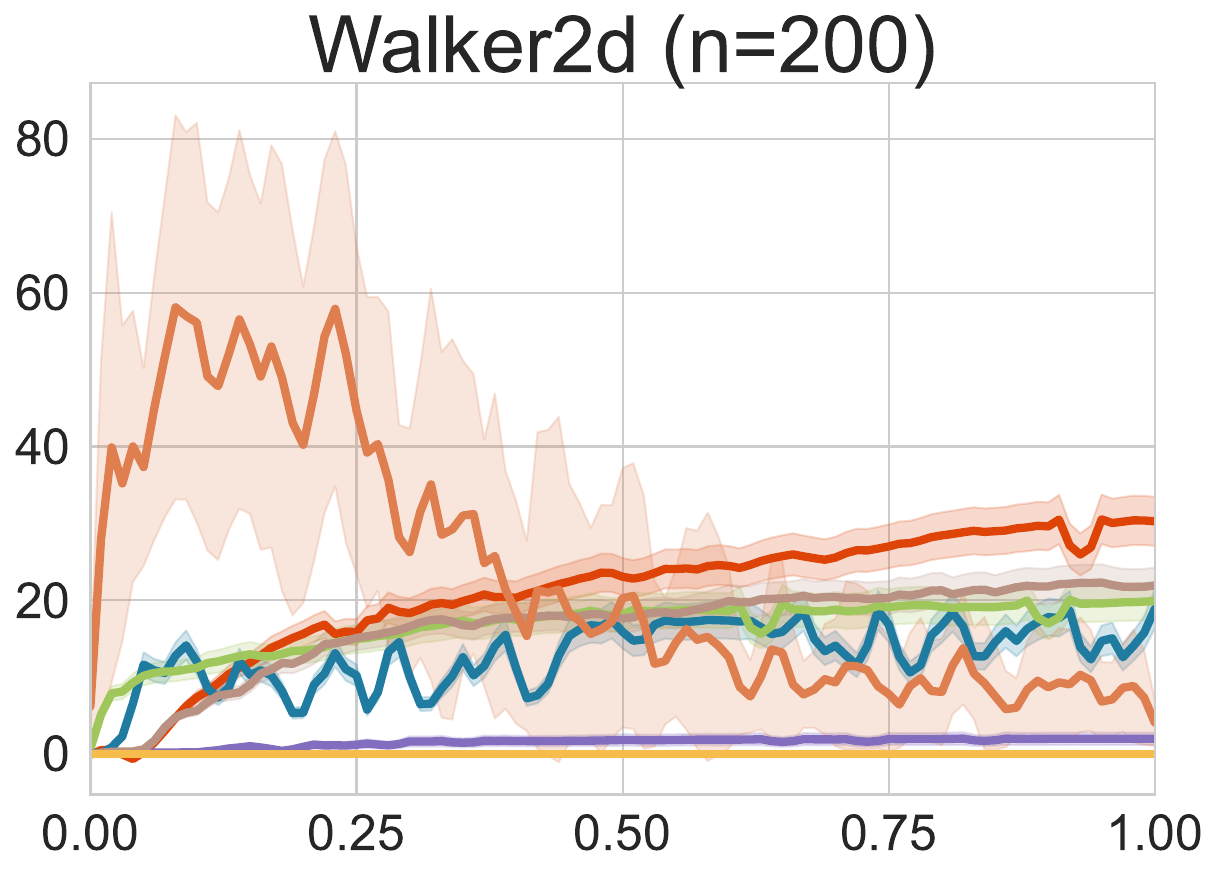}}
    
    \vspace{-10pt}
    \subfigure{
        \includegraphics[width=0.24\textwidth]{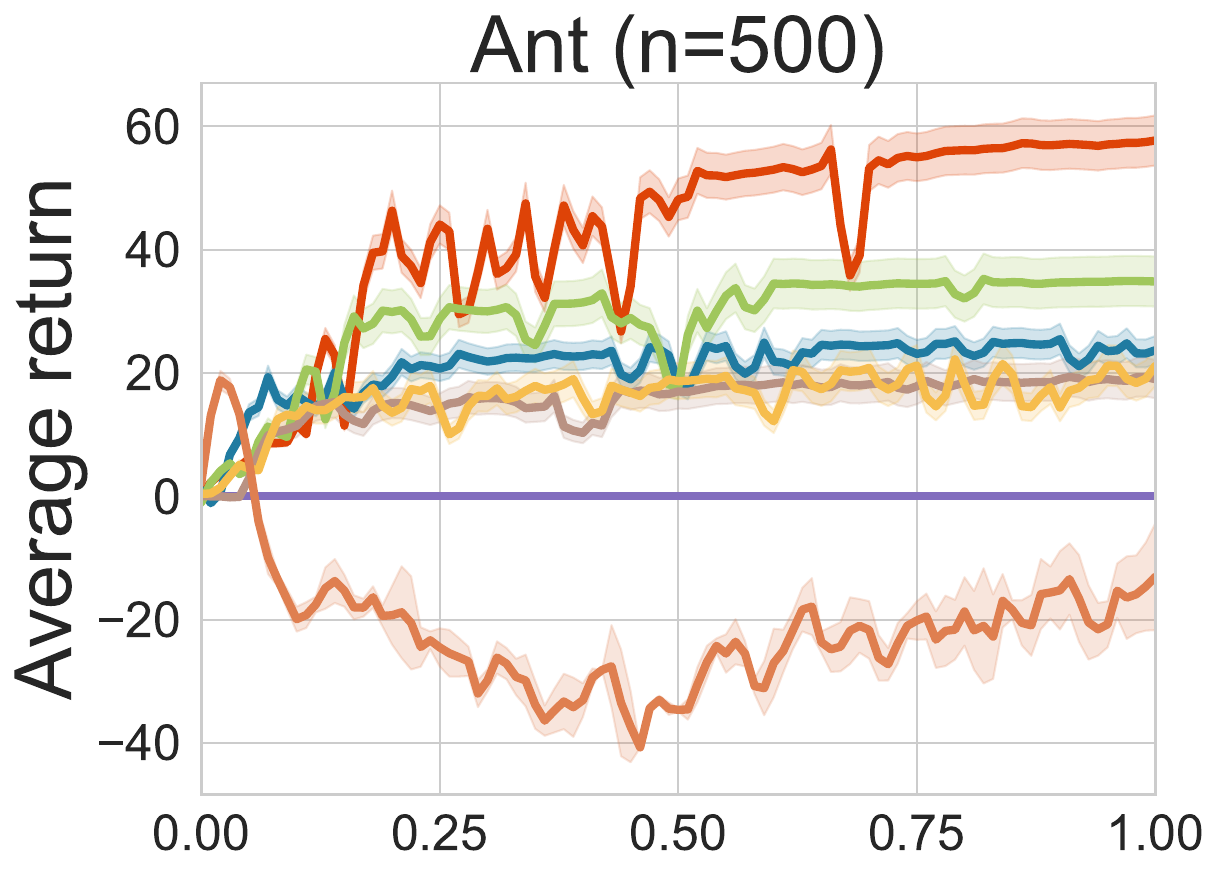}}
    \hspace{-4pt}
    \subfigure{
        \includegraphics[width=0.24\textwidth]{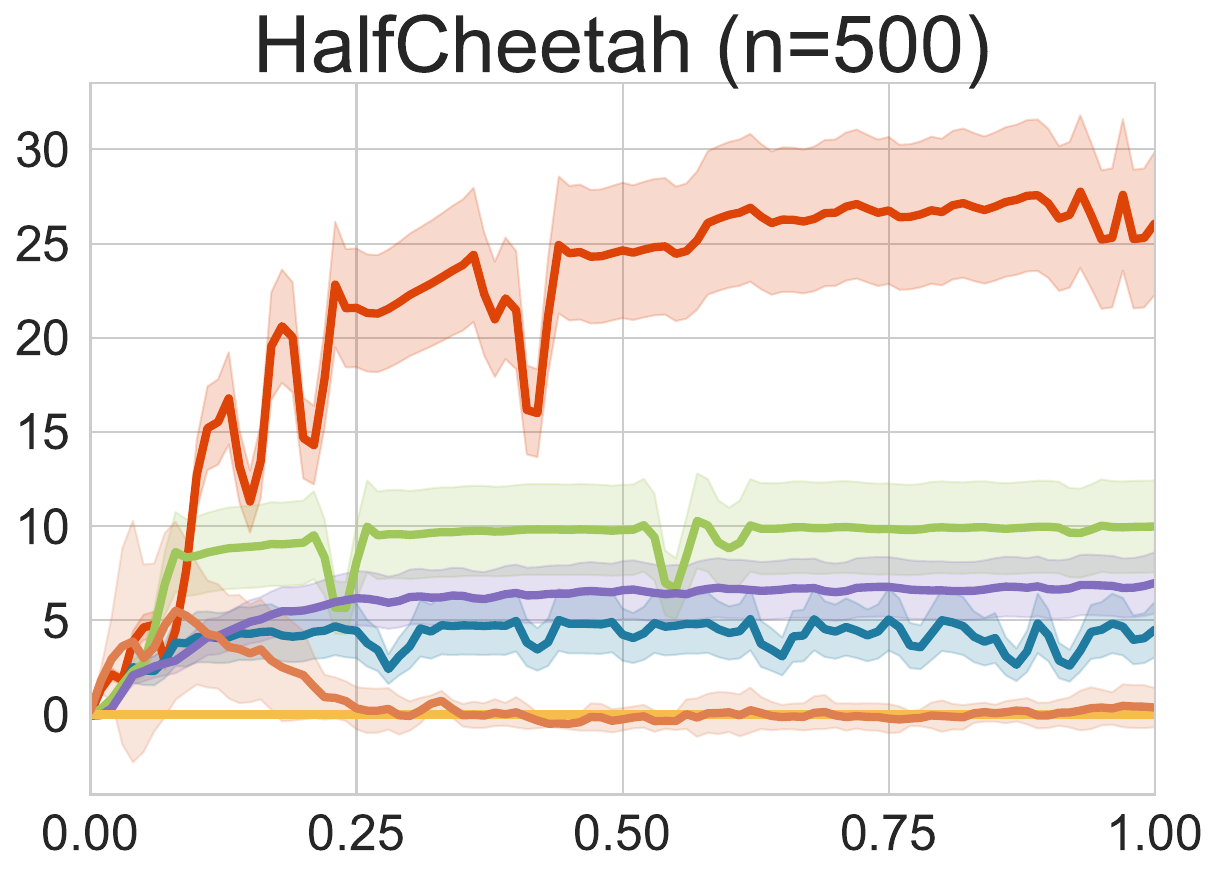}}
    \hspace{-4pt}
    \subfigure{
        \includegraphics[width=0.24\textwidth]{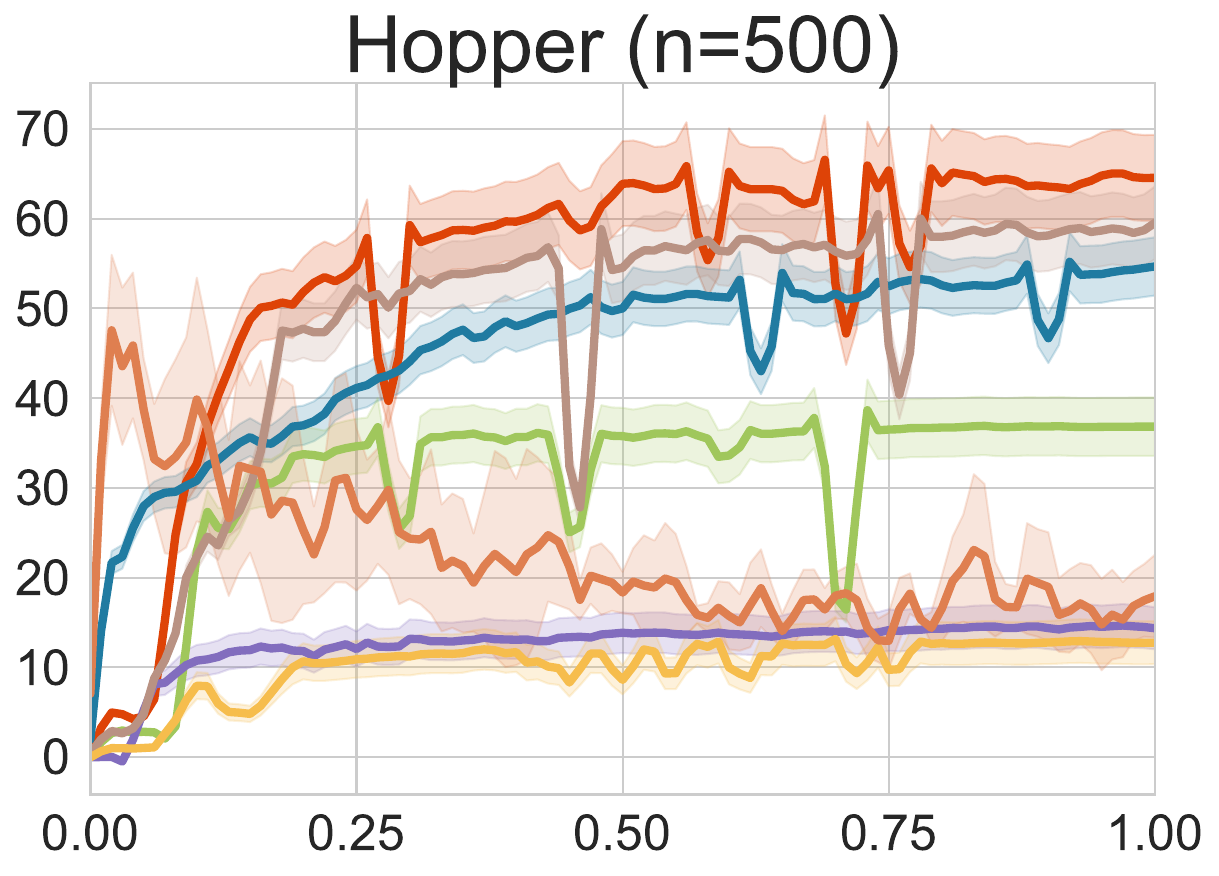}}
    \hspace{-4pt}
    \subfigure{
        \includegraphics[width=0.24\textwidth]{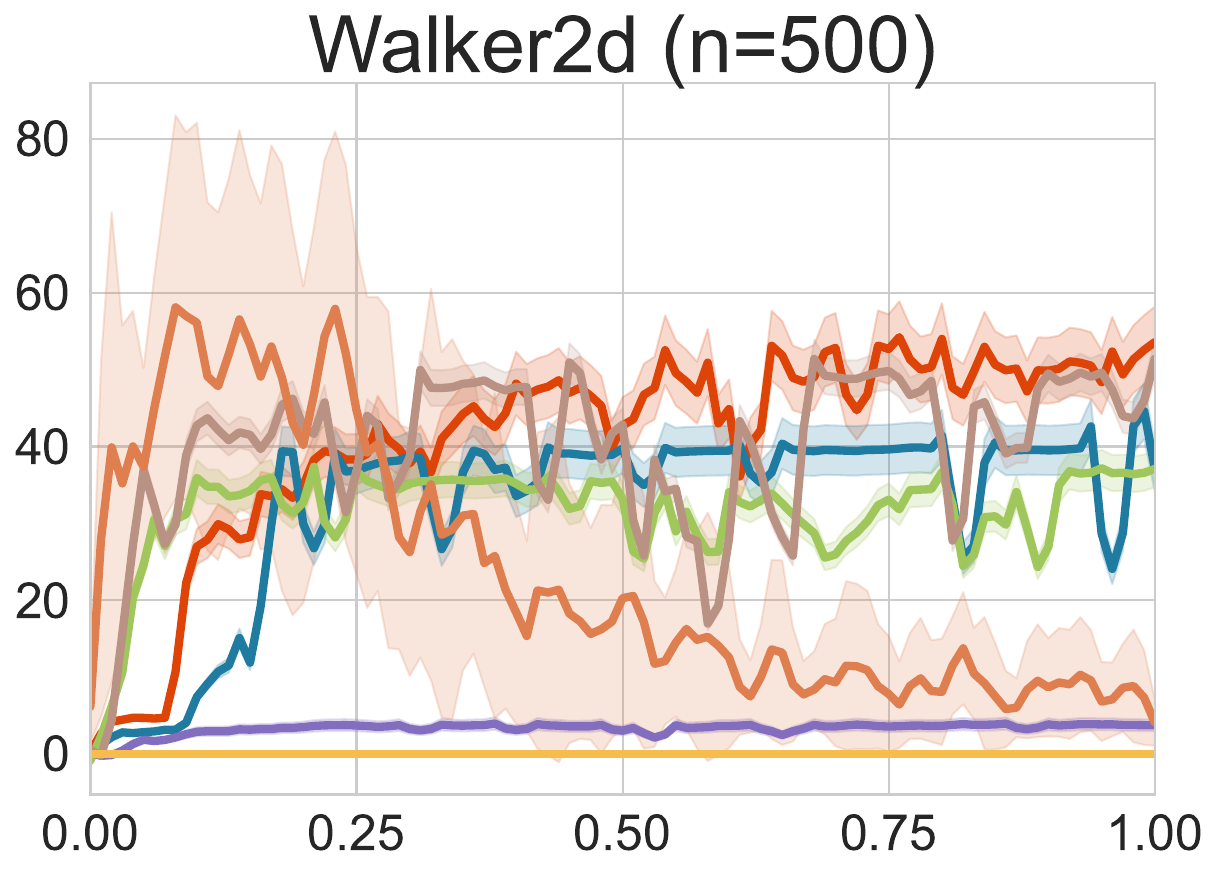}}
    
    \vspace{-10pt}
    \subfigure{
        \includegraphics[width=0.24\textwidth]{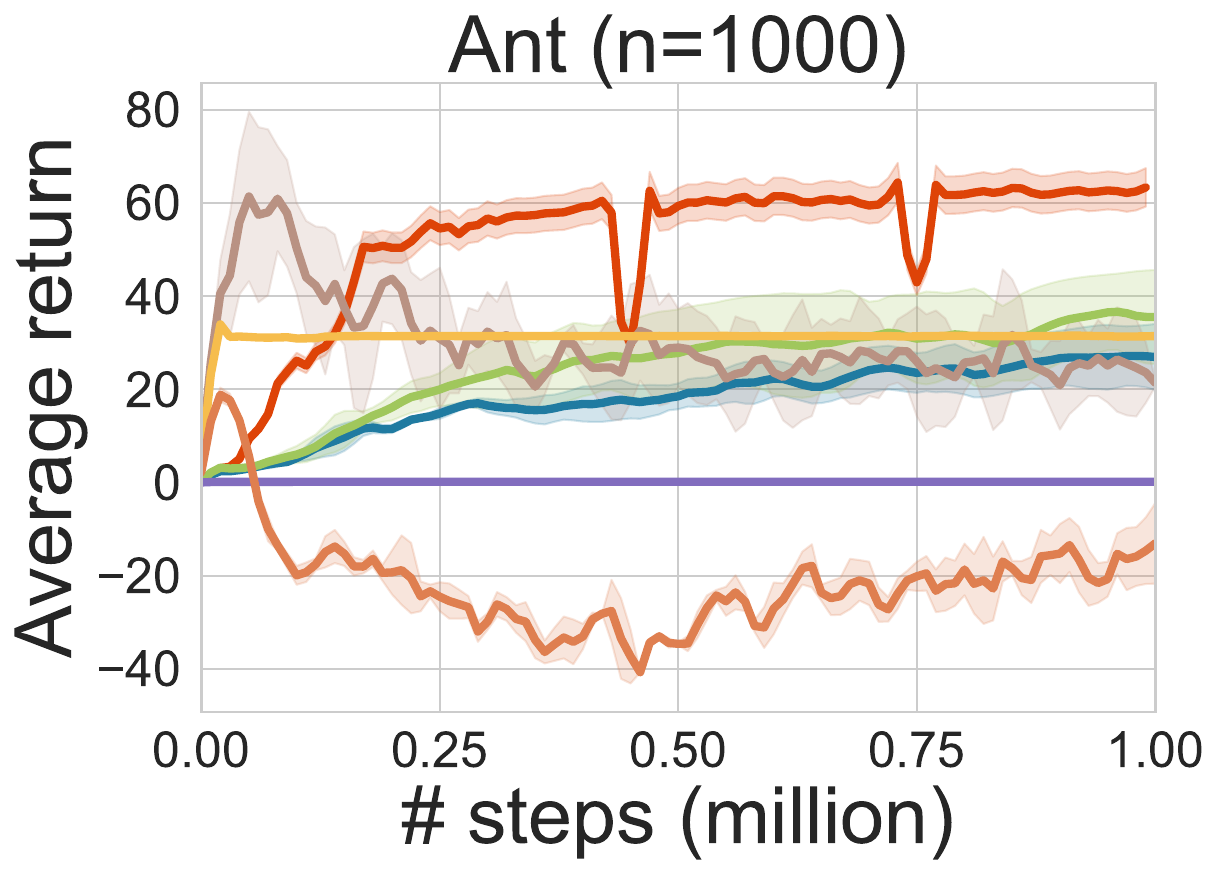}}
    \hspace{-4pt}
    \subfigure{
        \includegraphics[width=0.24\textwidth]{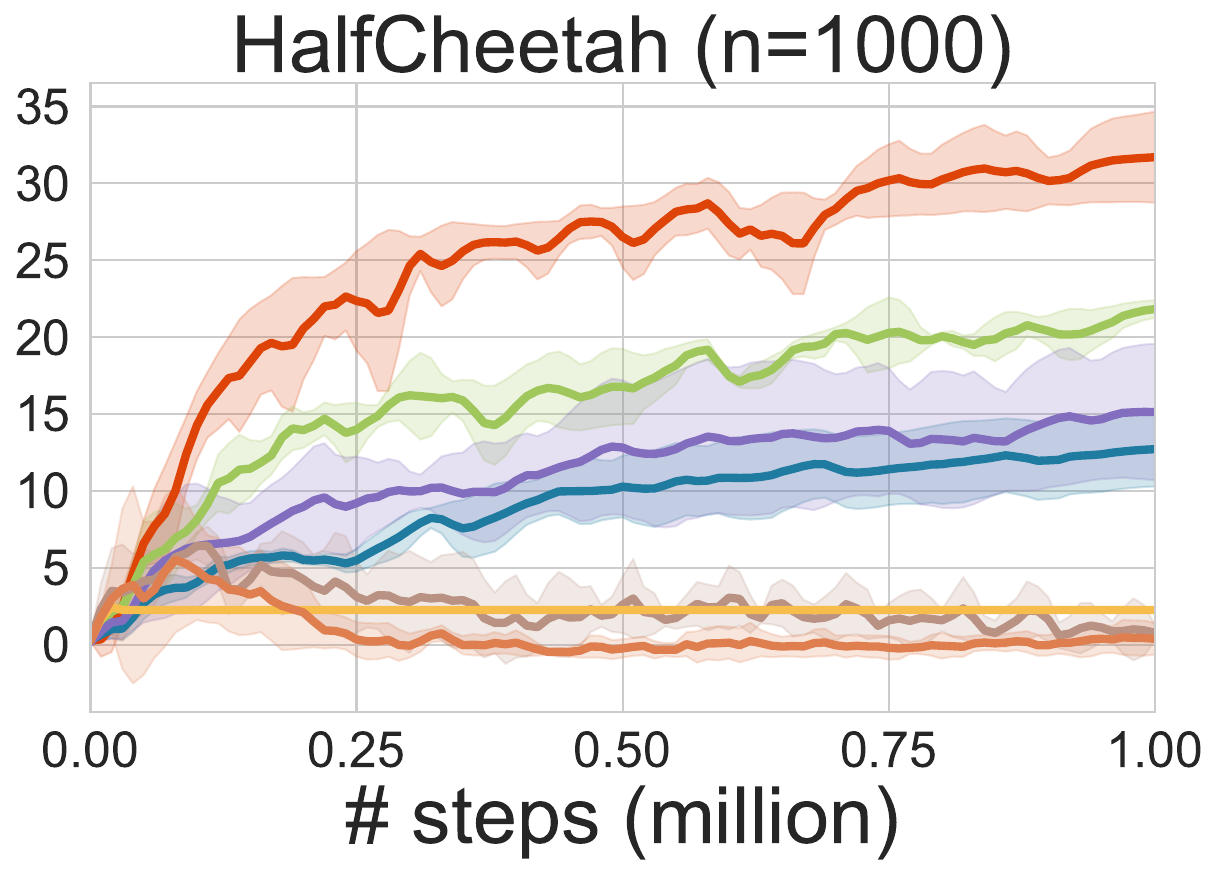}}
    \hspace{-4pt}
    \subfigure{
        \includegraphics[width=0.24\textwidth]{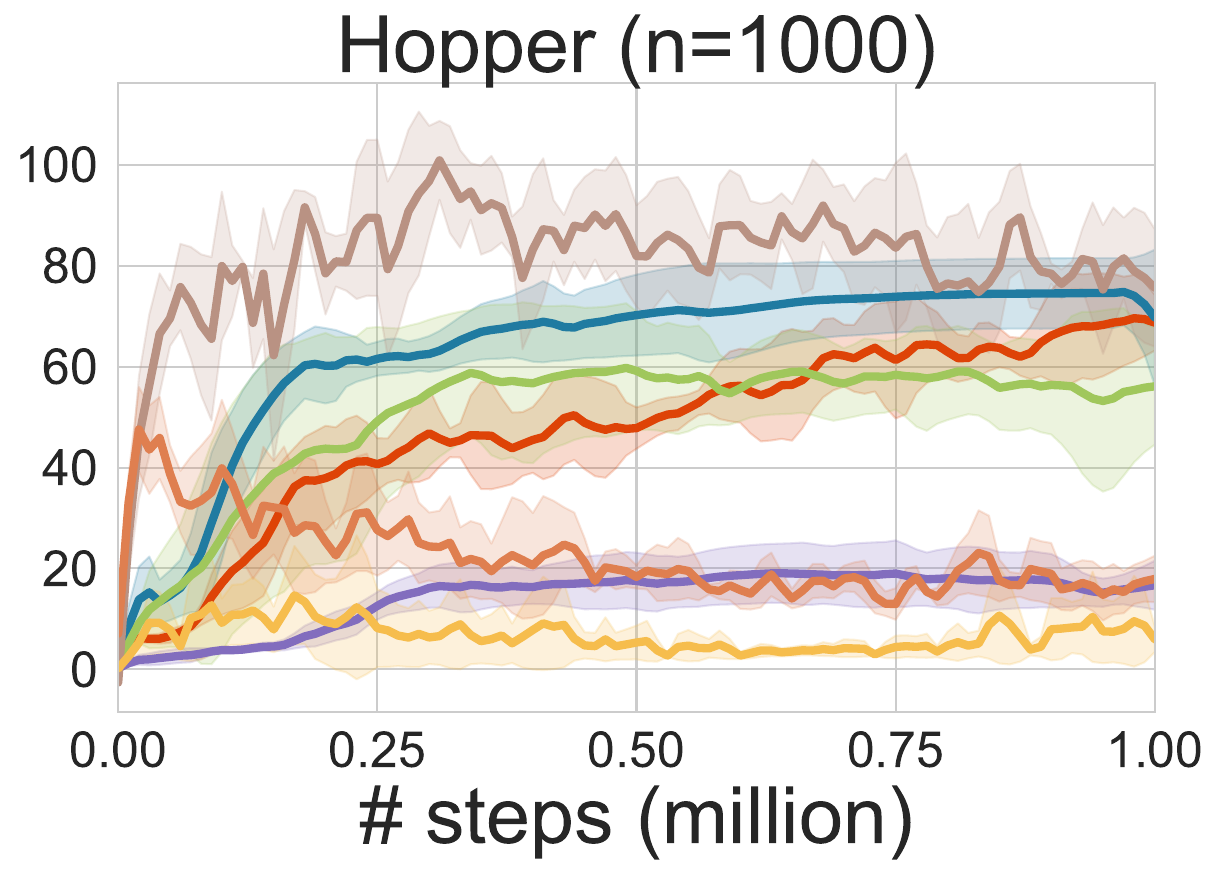}}
    \hspace{-4pt}
    \subfigure{
        \includegraphics[width=0.24\textwidth]{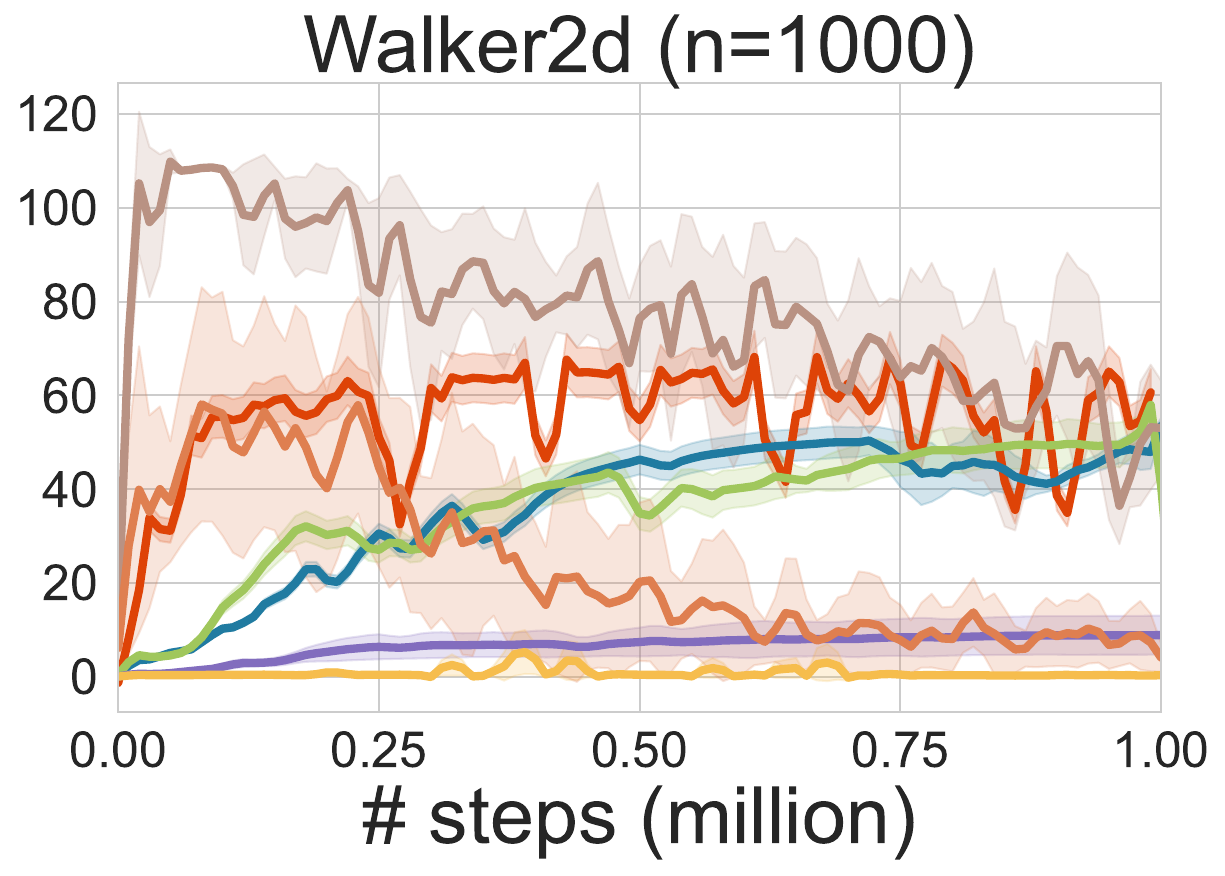}}
    
    \caption{Effect of the quantity of imperfect demonstrations}
    \label{fig:random_num_bar}
\end{figure}

\clearpage

\subsection{Rollback Steps}
\label{sec:rollback_steps}

Regarding the fourth question, we vary $K$ from 1 to 100 and run experiments across all benchmarks. \cref{fig:rollback_all} clearly indicate that as $K$ increases, there is an initial improvement in performance; once it reaches a sufficiently large value, performance tends to stabilize. Considering that a larger rollback step leads to more selected behaviors capable of reaching expert states, this observation offers support for  \cref{hypo:resultant_state}. Importantly, the performance proves to be robust to a relatively large $K$, rendering \texttt{ILID} forgiving to the hyperparameter.

\begin{figure}[ht]
    \centering
    \subfigure{
        \includegraphics[width=0.99\textwidth]{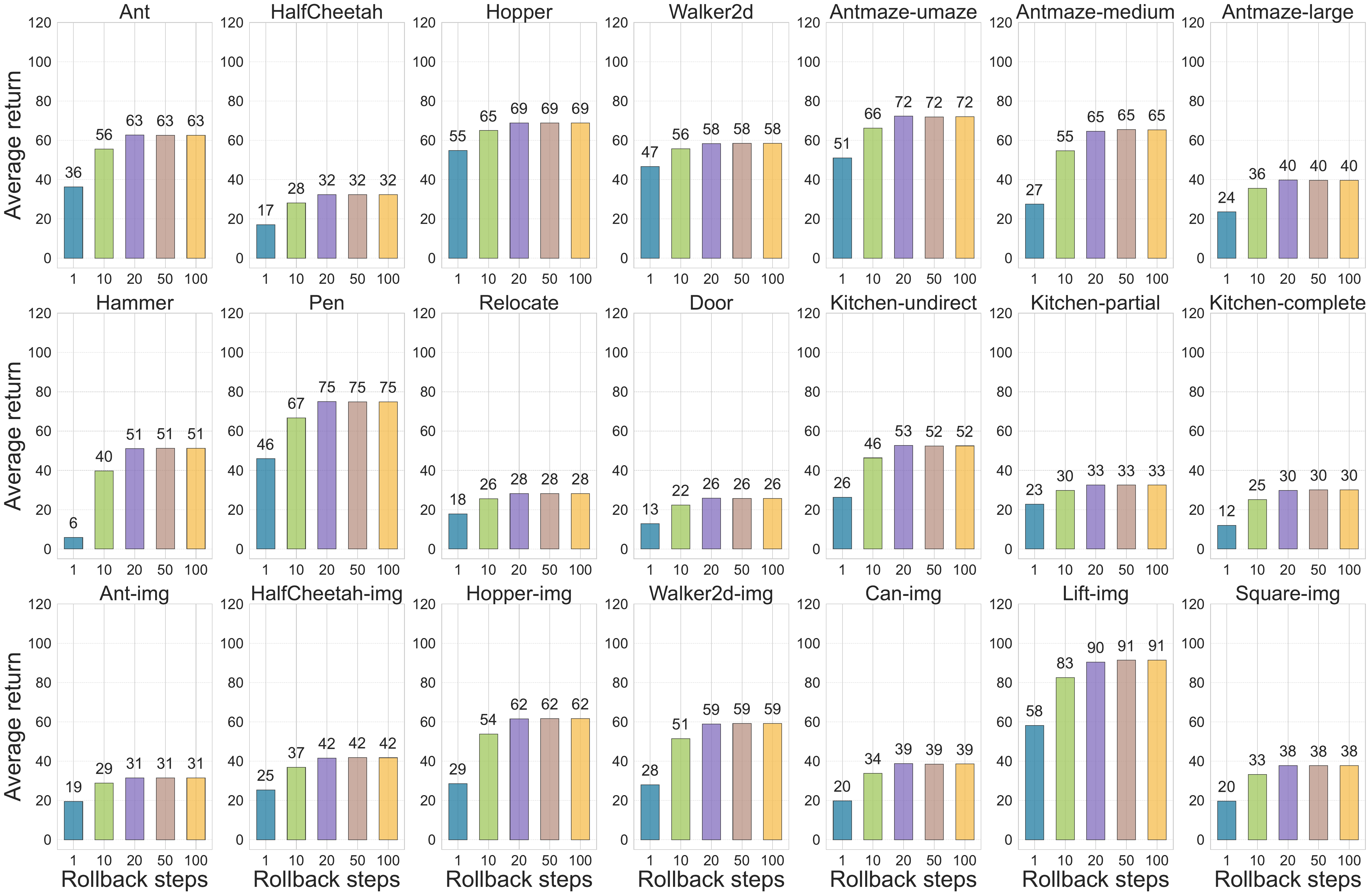}}
    \caption{Performance of \texttt{ILID} under varying numbers of rollback steps}
    \label{fig:rollback_all}
\end{figure}

% \clearpage

% \subsection{Runtime}
% \label{sec:runtime}

\subsection{Runtime}
\label{sec:runtime}

We evaluate the runtime of \texttt{ILID} compared with baseline algorithms for 250,000 training steps, utilizing the same network size and batch size on an NVIDIA 4090 GPU. As illustrated by \cref{subfig:runtime}, the runtime of \texttt{ILID} (around \SI{40}{min}) is slightly longer than \texttt{BC} (around \SI{30}{min}), which substantiates the low computational cost of \texttt{ILID}. 

\clearpage

\subsection{Ablation Studies}
\label{sec:ablation}

In this section, we assess the effect of key components by ablating them, under the same setting as that of \cref{table:dataset_comparative}. 

% We assess the effect of key components by ablating them on \textit{all} benchmarks, under the same setting as \cref{table:dataset_comparative}. 

\subsubsection{Only Importance-Sampling Weighting}

Without the second term in Problem~(\ref{eq:ilid_objective}), \texttt{ILID} reduces to \texttt{ISWBC}. Unsurprisingly, as shown in \cref{fig:dynamic_information_all_1,fig:dynamic_information_all_2}, \texttt{ISWBC} does not suffice satisfactory performance.

% In the absence of the second term in Problem~(\ref{eq:ilid_objective}), \texttt{ILID} reduces to \texttt{ISWBC}. Yet, as shown in \cref{fig:dynamic_information_all_1,fig:dynamic_information_all_2} and aforementioned comparative experiments, \texttt{ISWBC} does not suffice satisfactory performance.

%\textbf{Dynamic information.} To assess whether dynamic information is useful and whether \texttt{ILID} can effectively select appropriate data, our first step is to evaluate the usefulness of dynamic information. We compare the algorithm of \texttt{ILID} with other algorithms that use importance weights in various environments only, and present the results in \cref{fig:dynamic_information_all}. The results show that \texttt{ILID} consistently performs better. Furthermore, in tasks like Antmaze, where trajectory splicing is possible, \texttt{ILID} brings about a significant performance improvement.

\begin{figure}[ht]
    % \vspace{-20pt}
    \centering
    \subfigure{
        \includegraphics[width=0.24\textwidth]{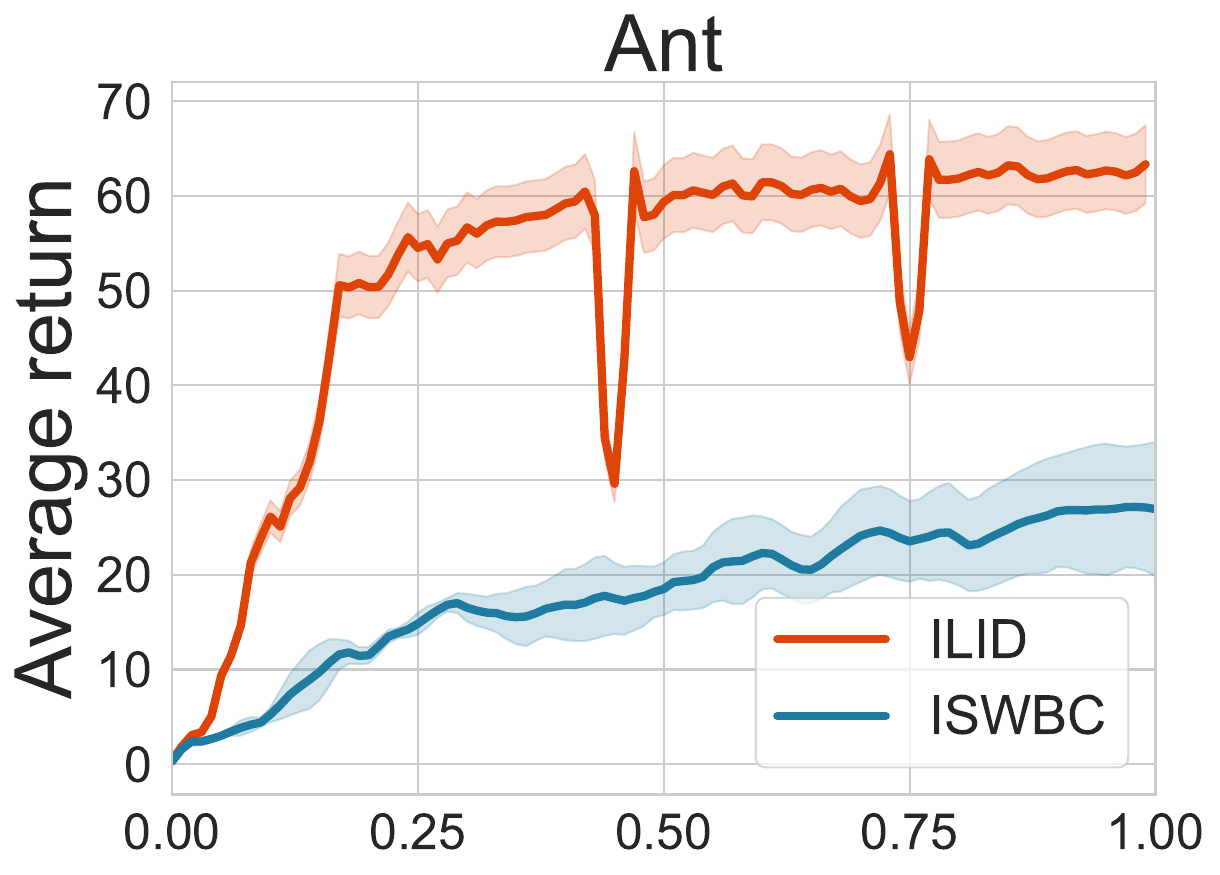}}
    \hspace{-4pt}
    \subfigure{
        \includegraphics[width=0.24\textwidth]{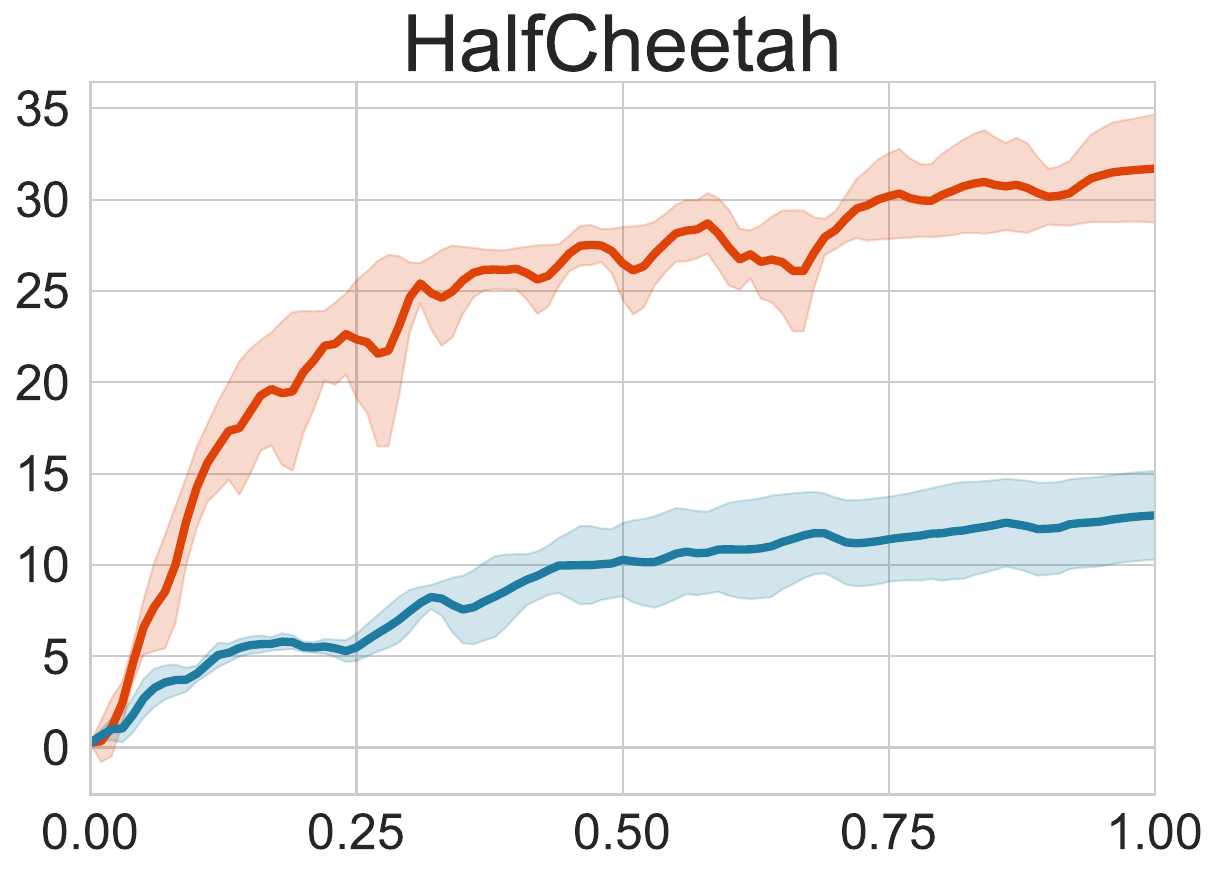}}
    \hspace{-4pt}
    \subfigure{
        \includegraphics[width=0.24\textwidth]{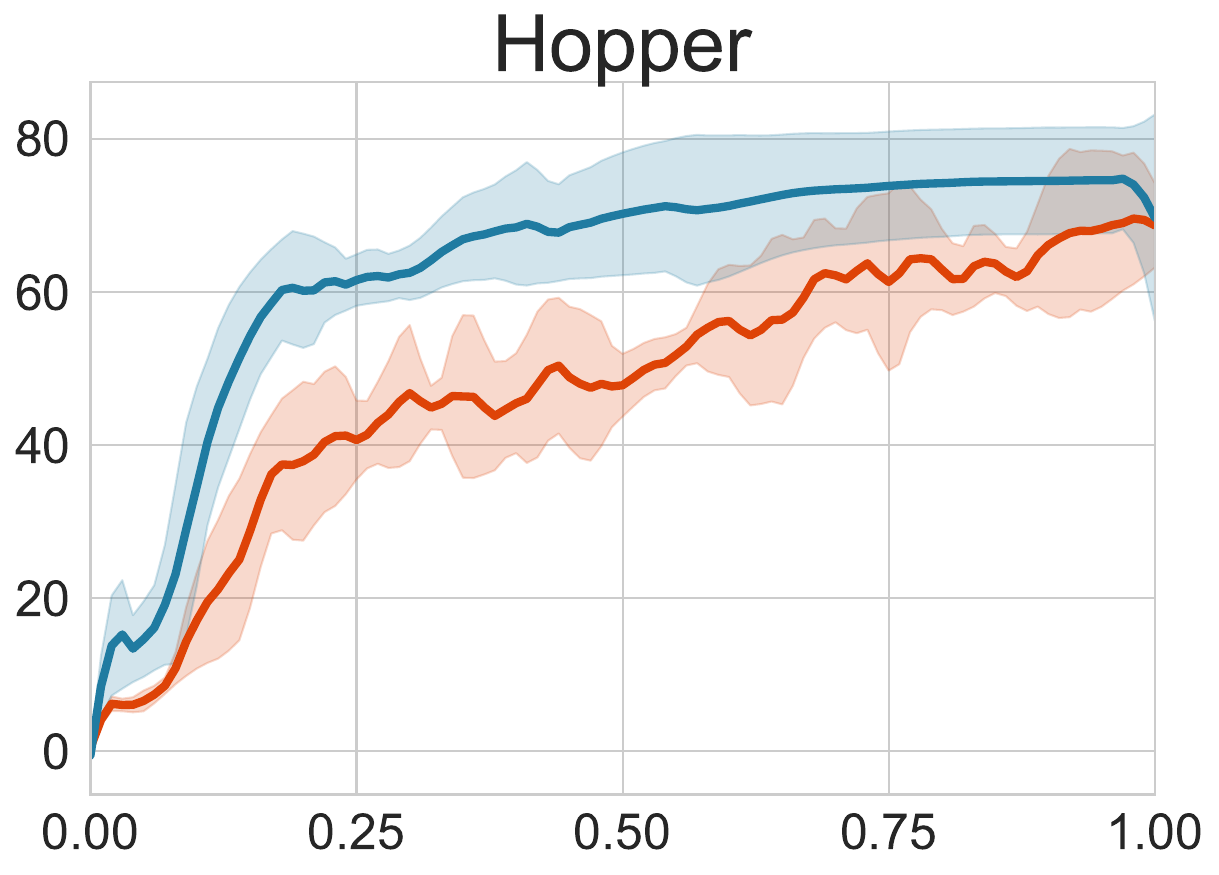}}
    \hspace{-4pt}
    \subfigure{
        \includegraphics[width=0.24\textwidth]{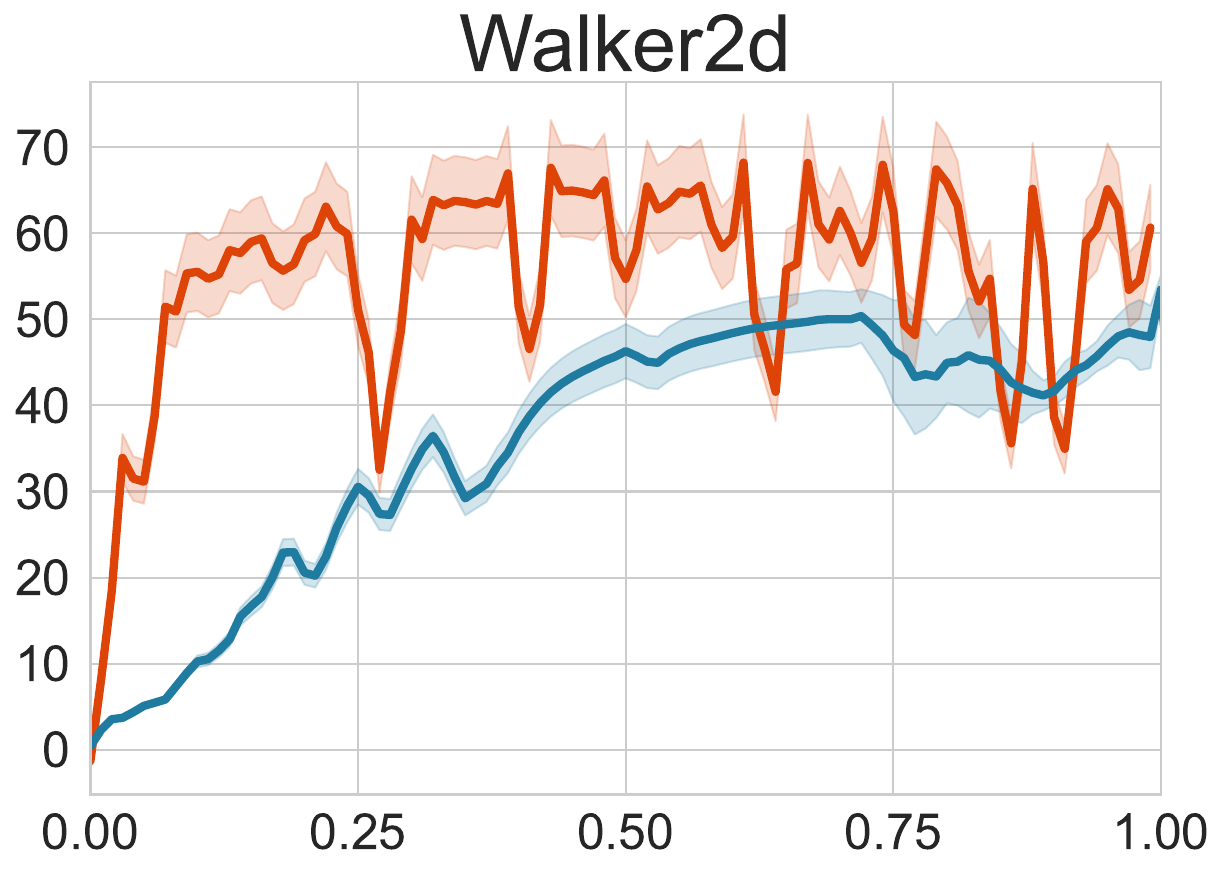}}
    
    \vspace{-10pt}
    \subfigure{
        \includegraphics[width=0.24\textwidth]{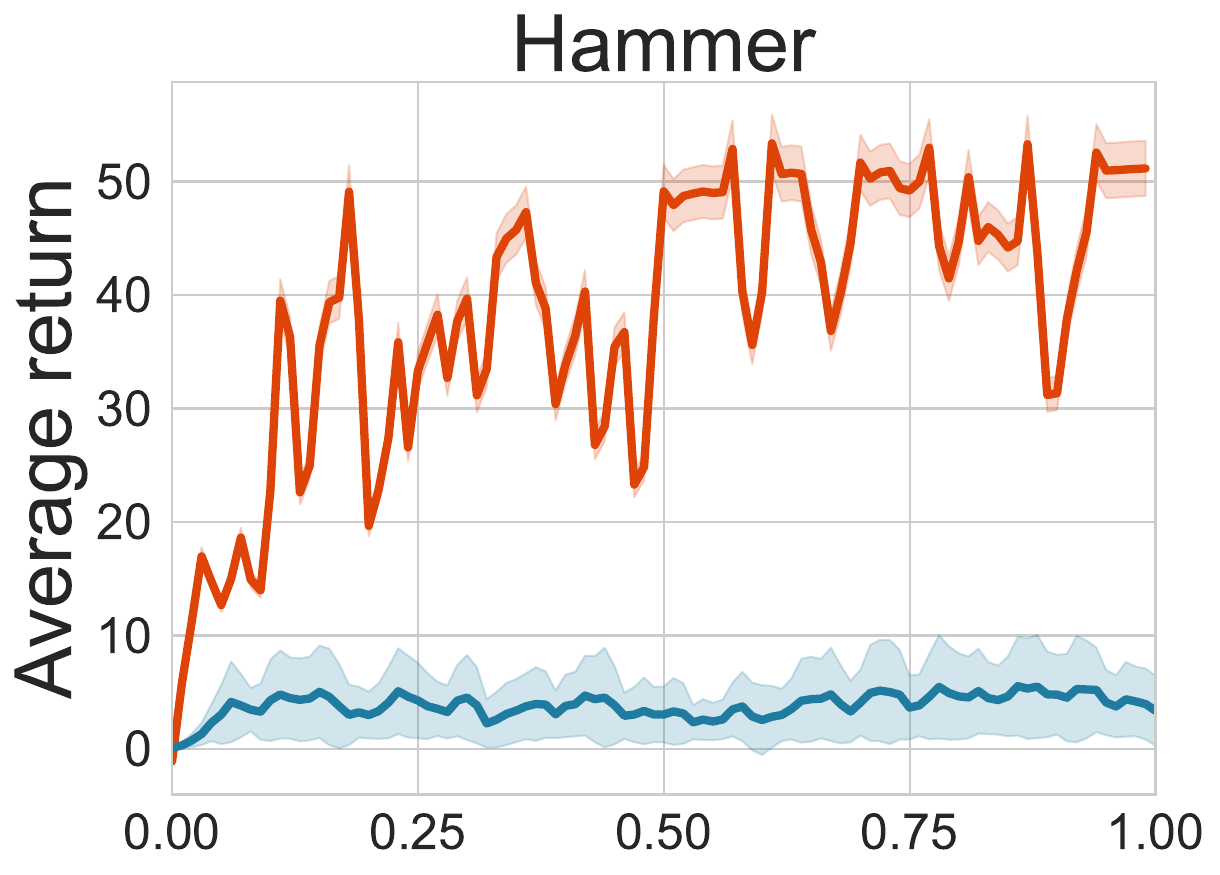}}
    \hspace{-4pt}
    \subfigure{
        \includegraphics[width=0.24\textwidth]{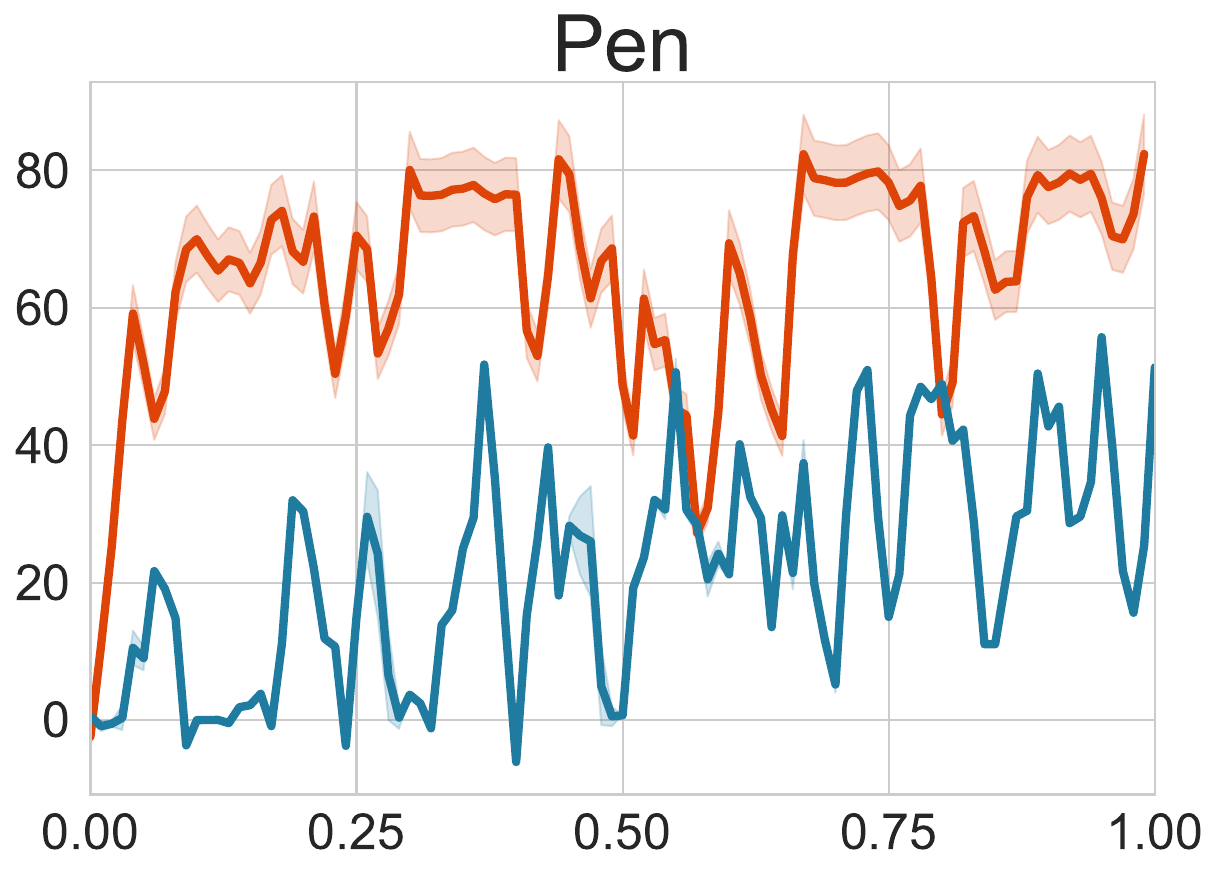}}
    \hspace{-4pt}
    \subfigure{
        \includegraphics[width=0.24\textwidth]{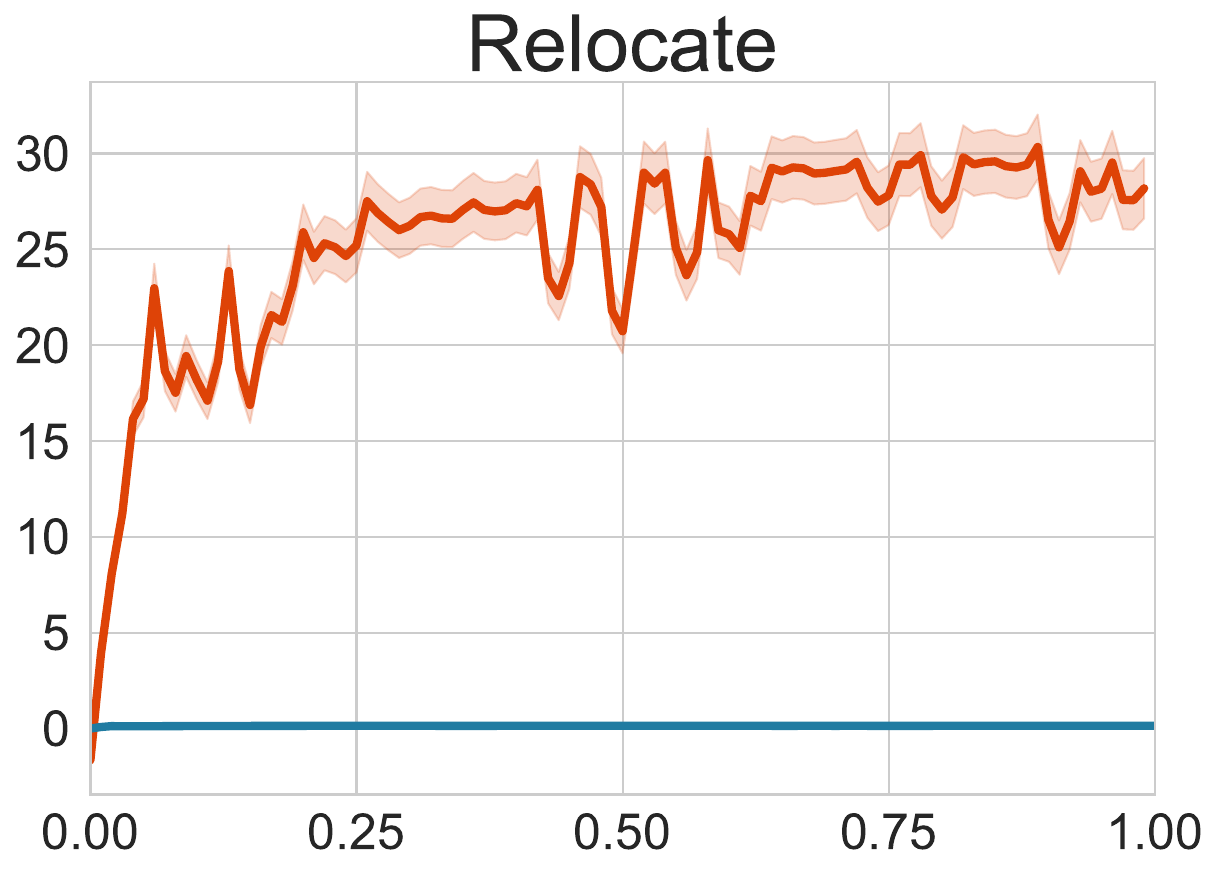}}
    \hspace{-4pt}
    \subfigure{
        \includegraphics[width=0.24\textwidth]{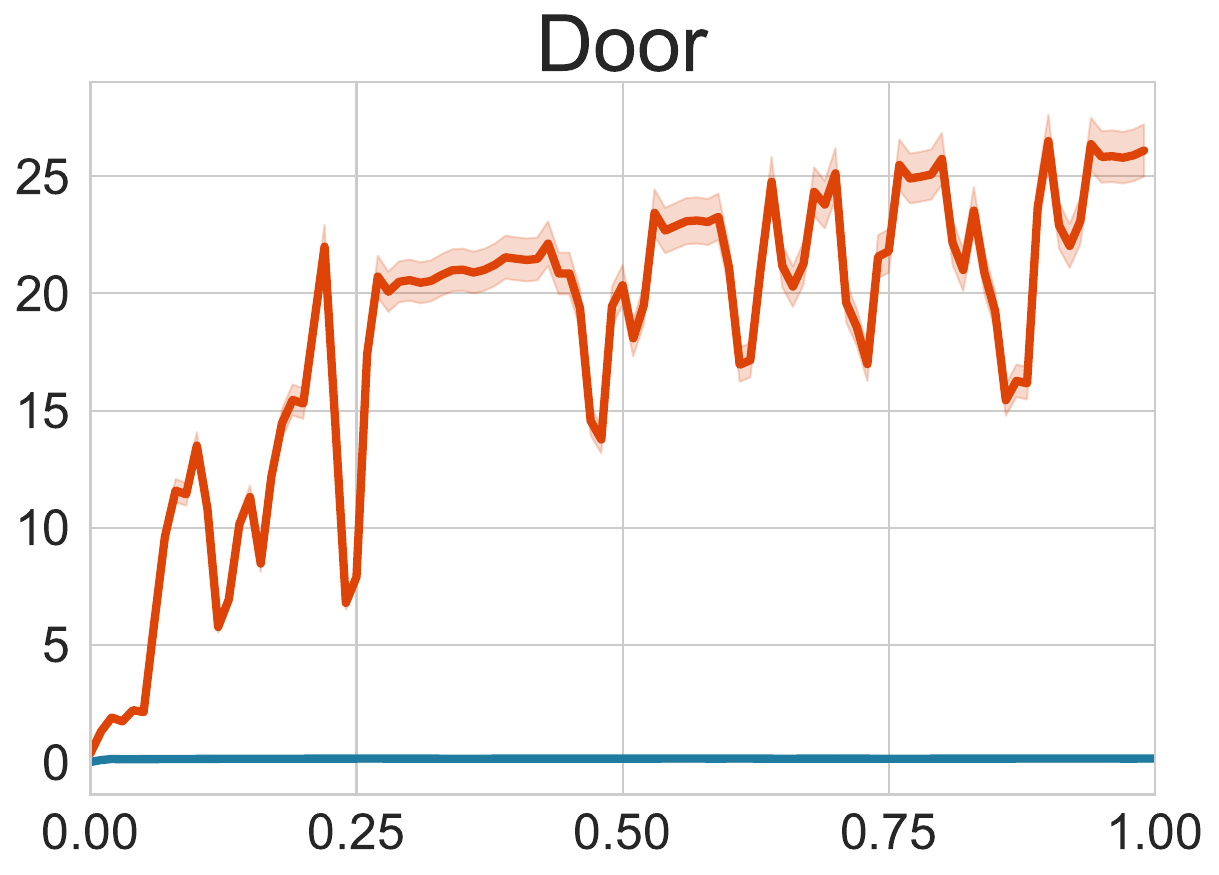}}
    
    \vspace{-10pt}
    \subfigure{
        \includegraphics[width=0.24\textwidth]{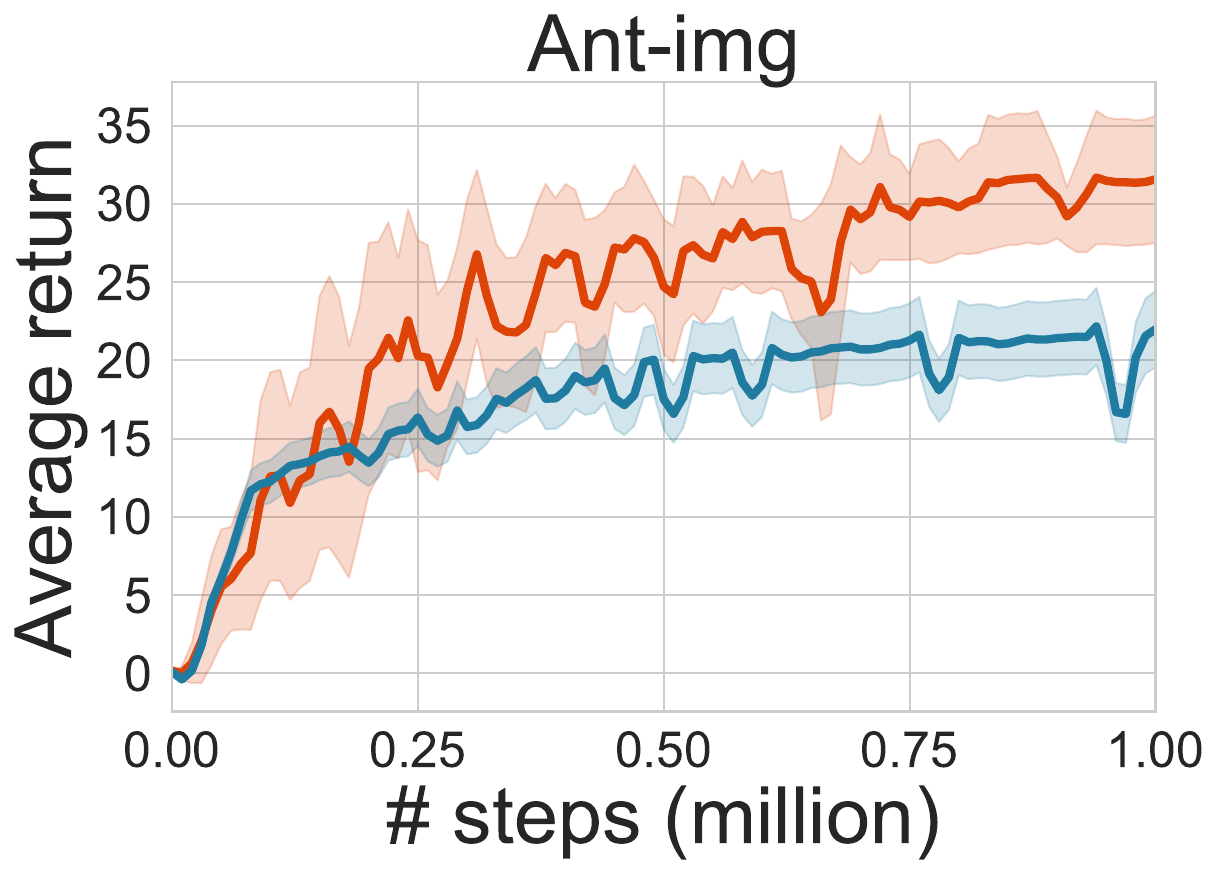}}
    \hspace{-4pt}
    \subfigure{
        \includegraphics[width=0.24\textwidth]{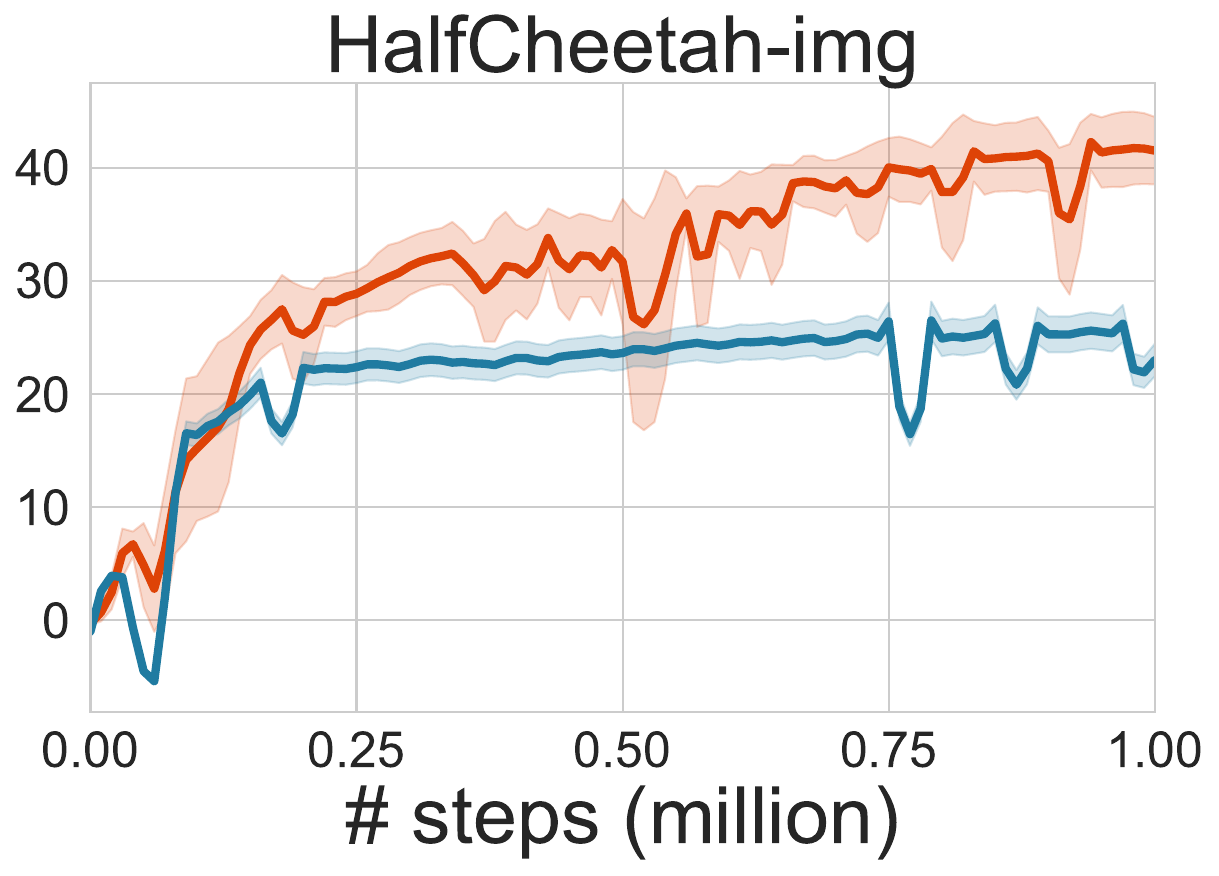}}
    \hspace{-4pt}
    \subfigure{
        \includegraphics[width=0.24\textwidth]{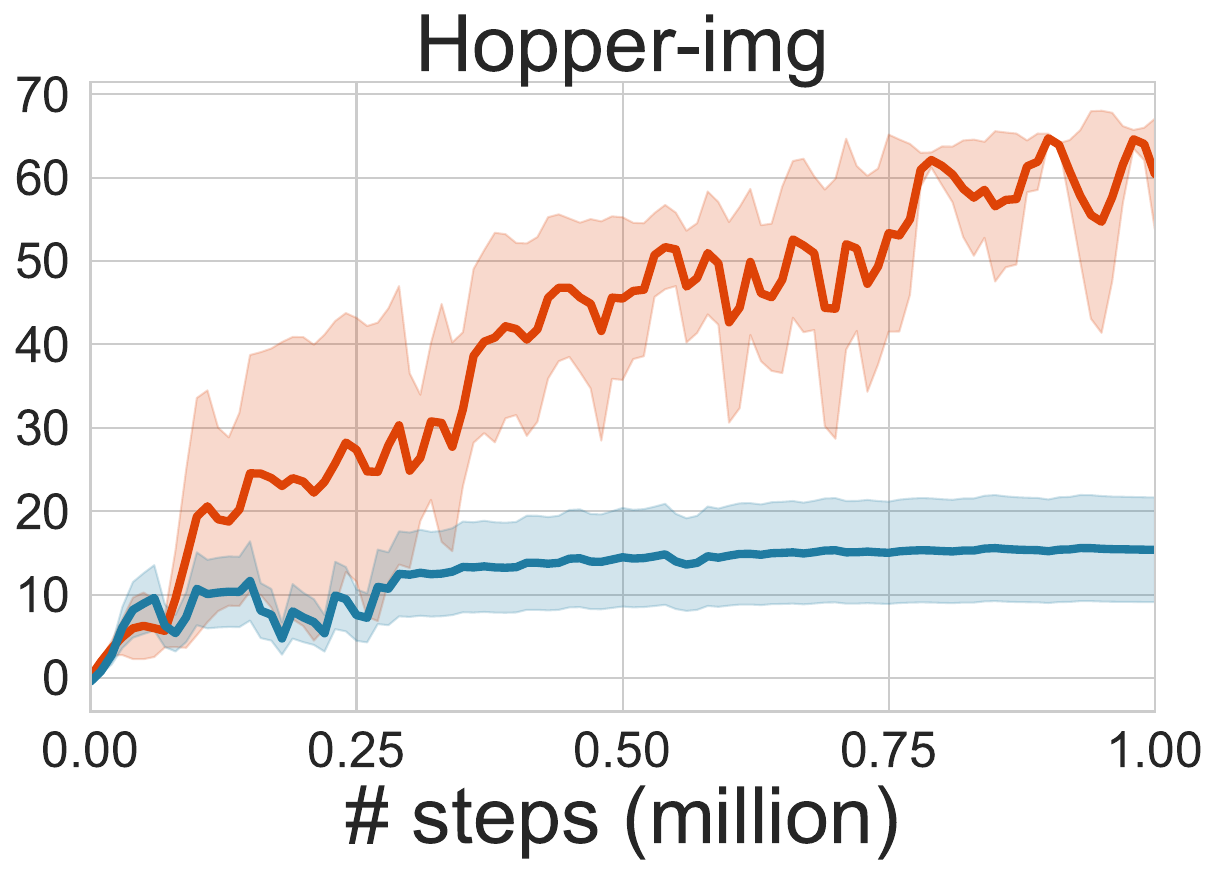}}
    \hspace{-4pt}
    \subfigure{
        \includegraphics[width=0.24\textwidth]{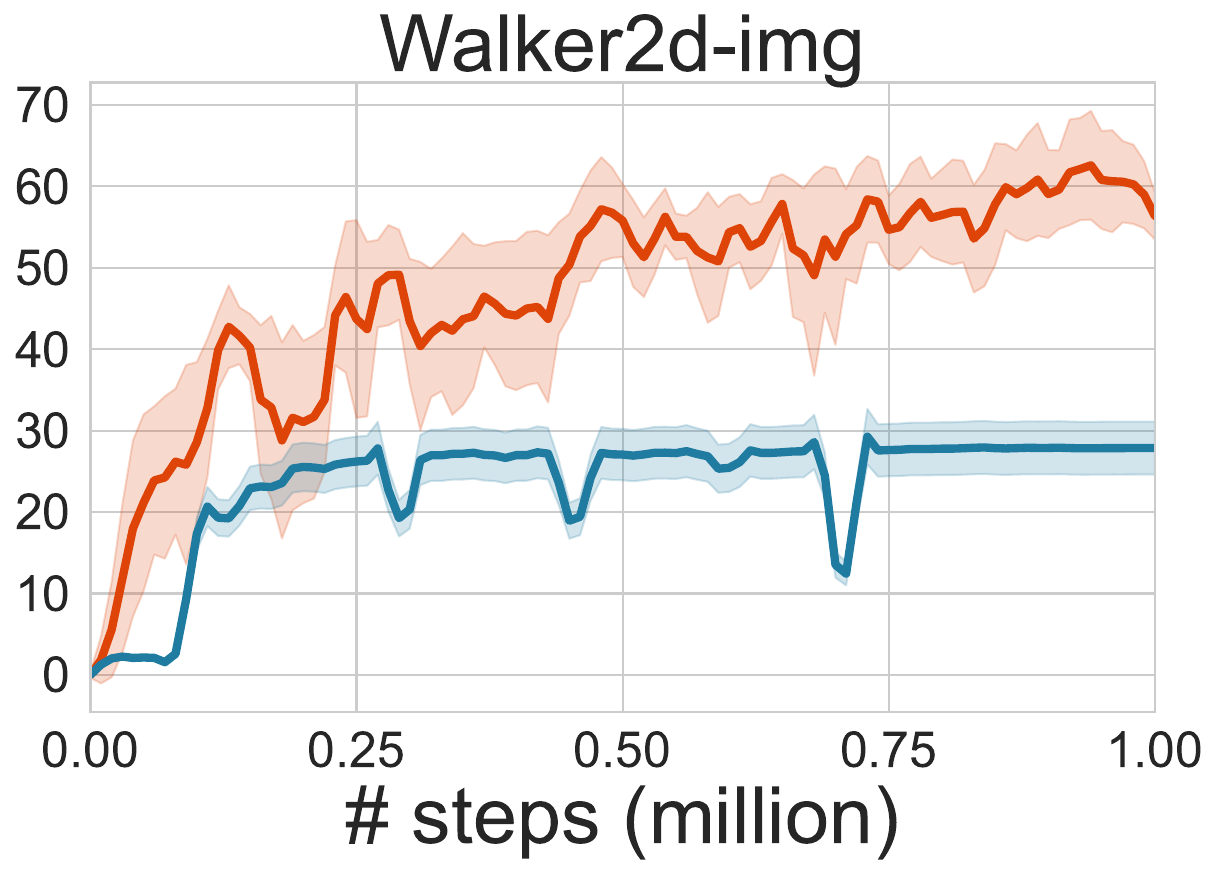}}
    % \vspace{-10pt}
    \caption{Comparison between \texttt{ISWBC} and \texttt{ILID}}
    \label{fig:dynamic_information_all_1}
    % \vspace{-15pt}
\end{figure}

\begin{figure}[!b]
    \centering    
    \subfigure{
        \includegraphics[width=0.24\textwidth]{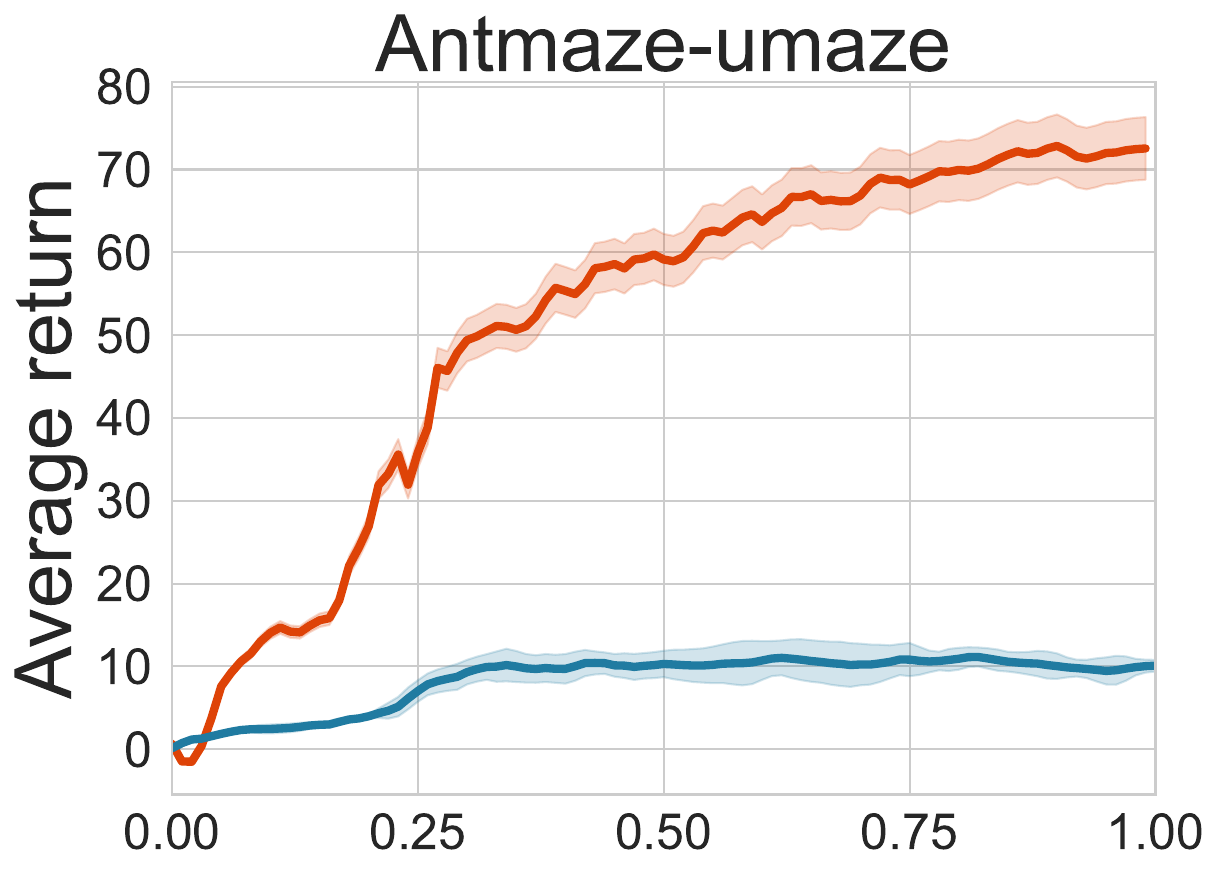}}
    \subfigure{
        \includegraphics[width=0.24\textwidth]{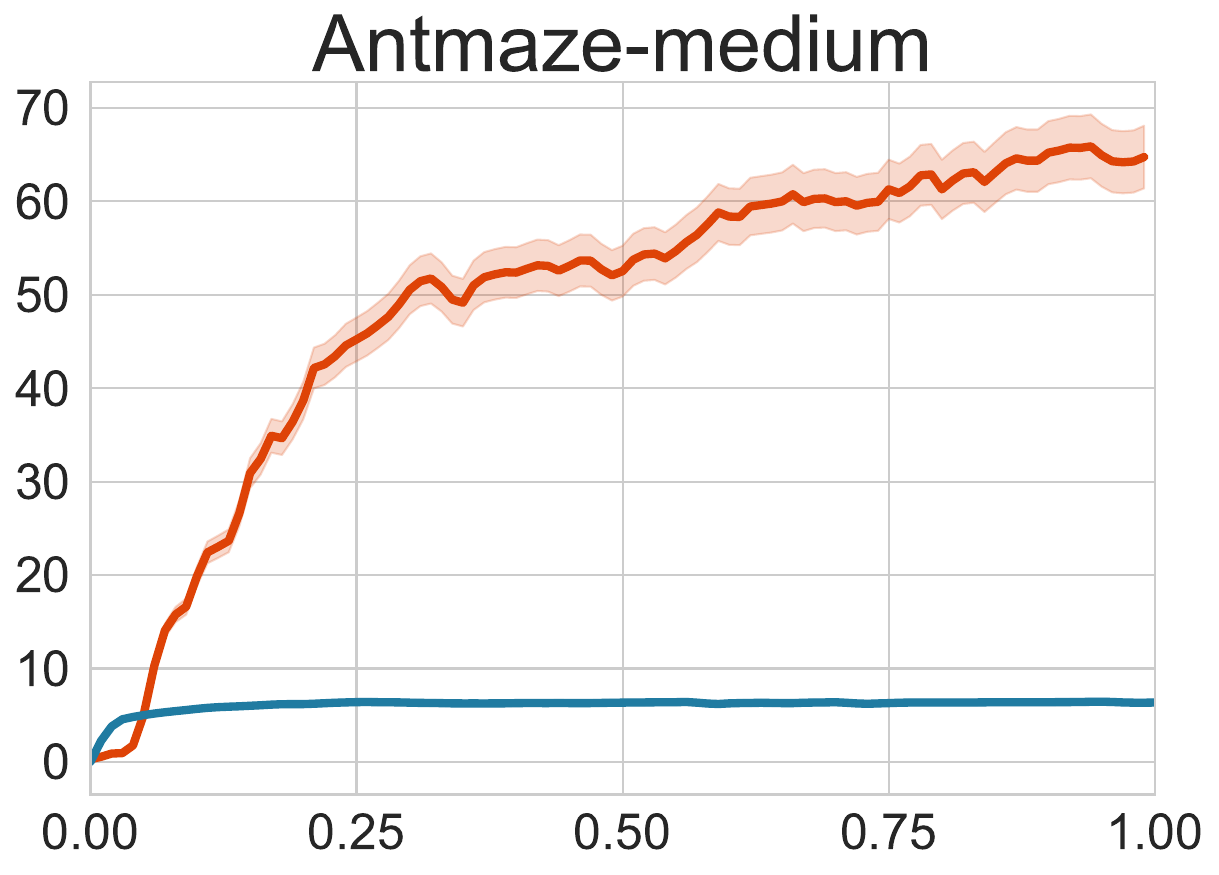}}
    \subfigure{
        \includegraphics[width=0.24\textwidth]{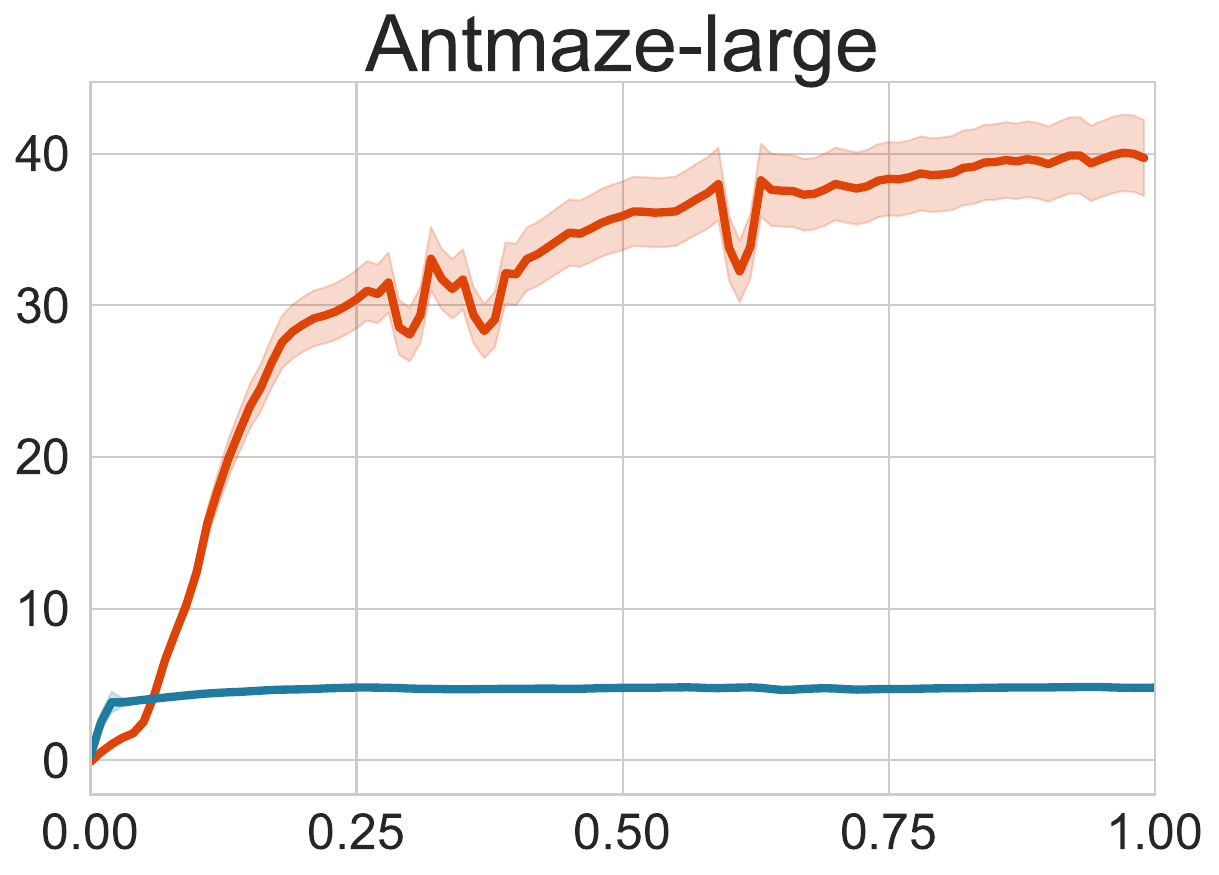}}
    
    \vspace{-10pt}
    \subfigure{
        \includegraphics[width=0.24\textwidth]{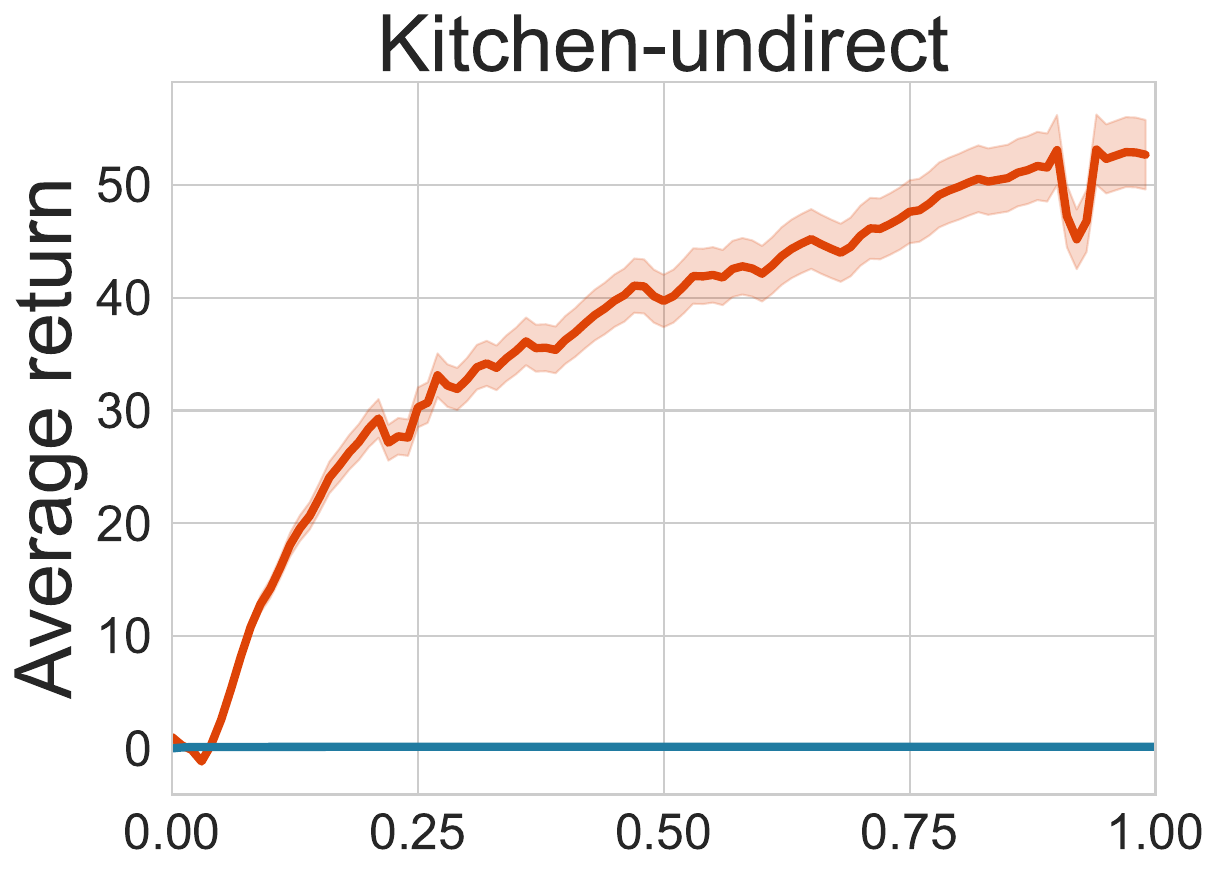}}
    \subfigure{
        \includegraphics[width=0.24\textwidth]{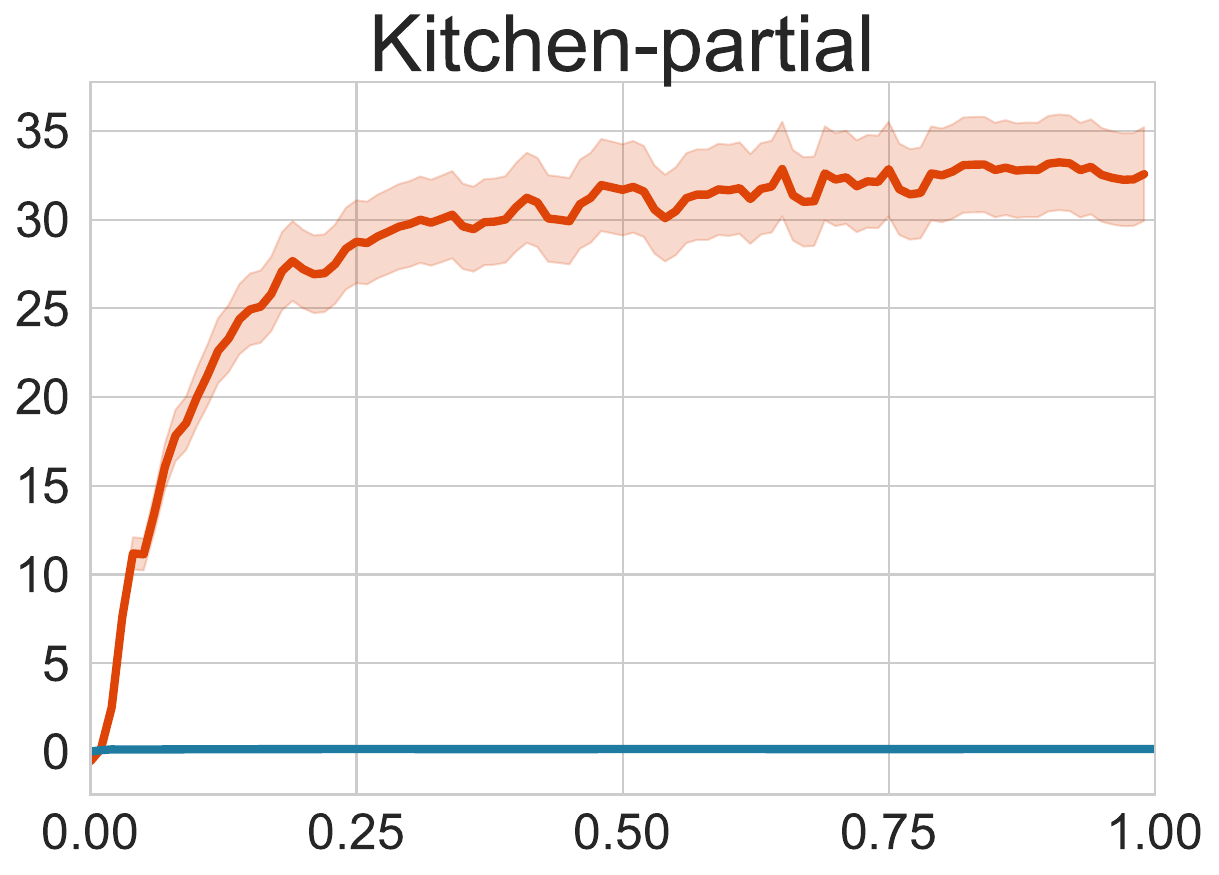}}
    \subfigure{
        \includegraphics[width=0.24\textwidth]{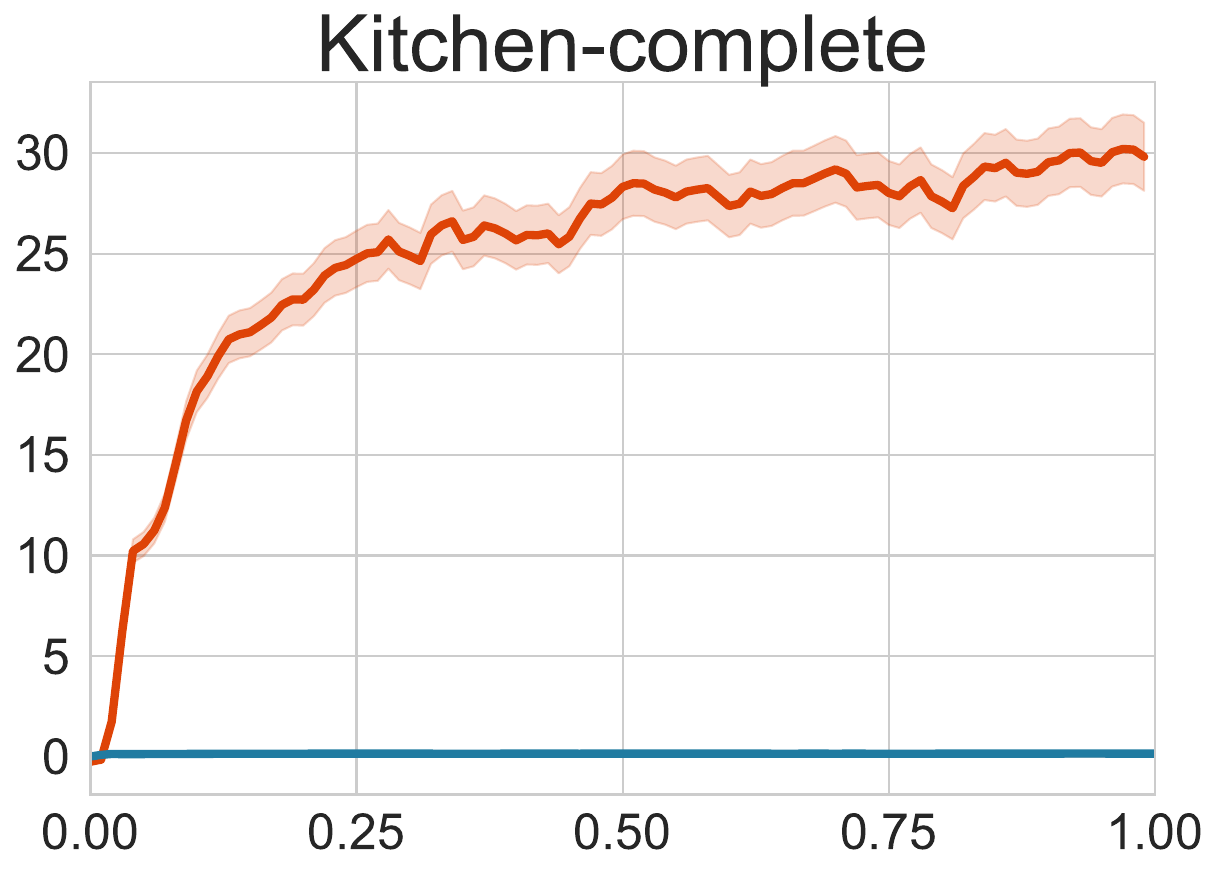}}

    \vspace{-10pt}
    \subfigure{
        \includegraphics[width=0.24\textwidth]{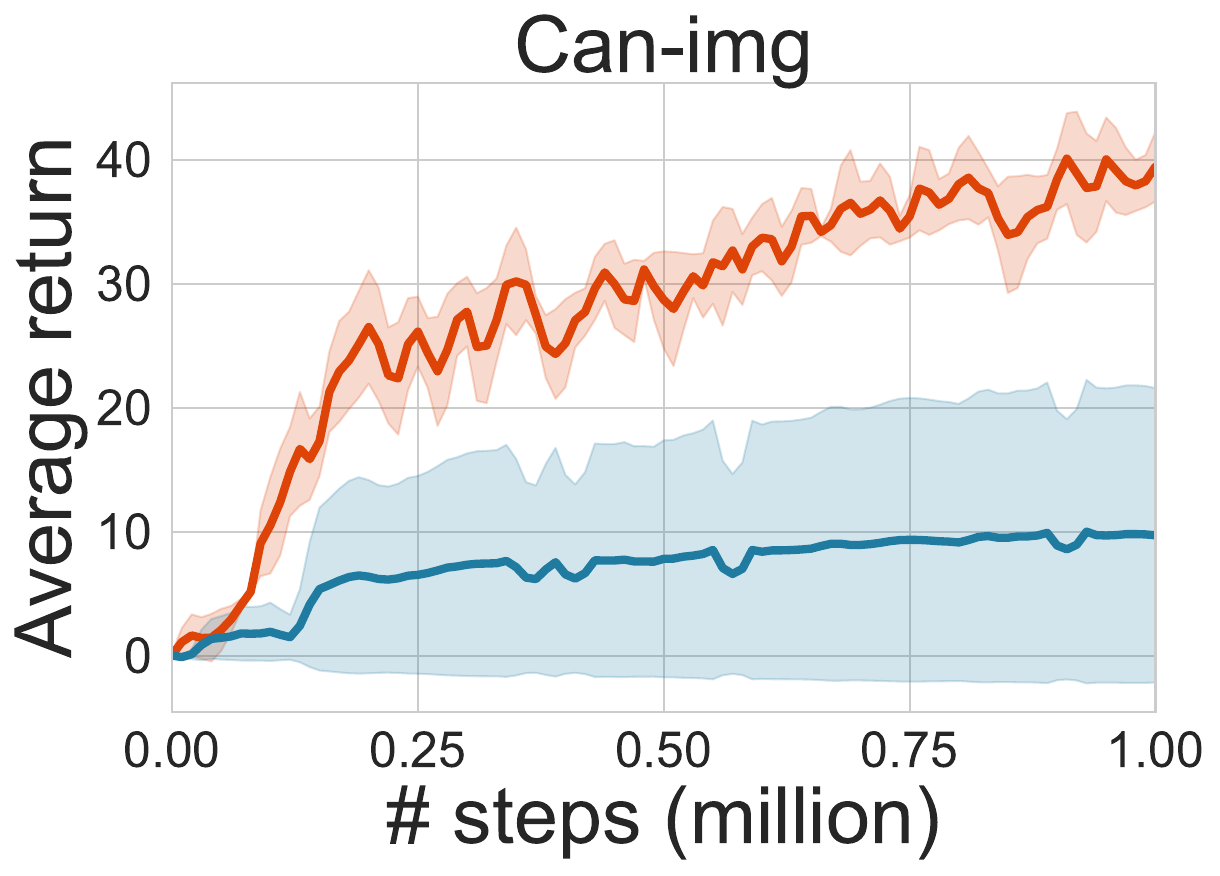}}
    \subfigure{
        \includegraphics[width=0.24\textwidth]{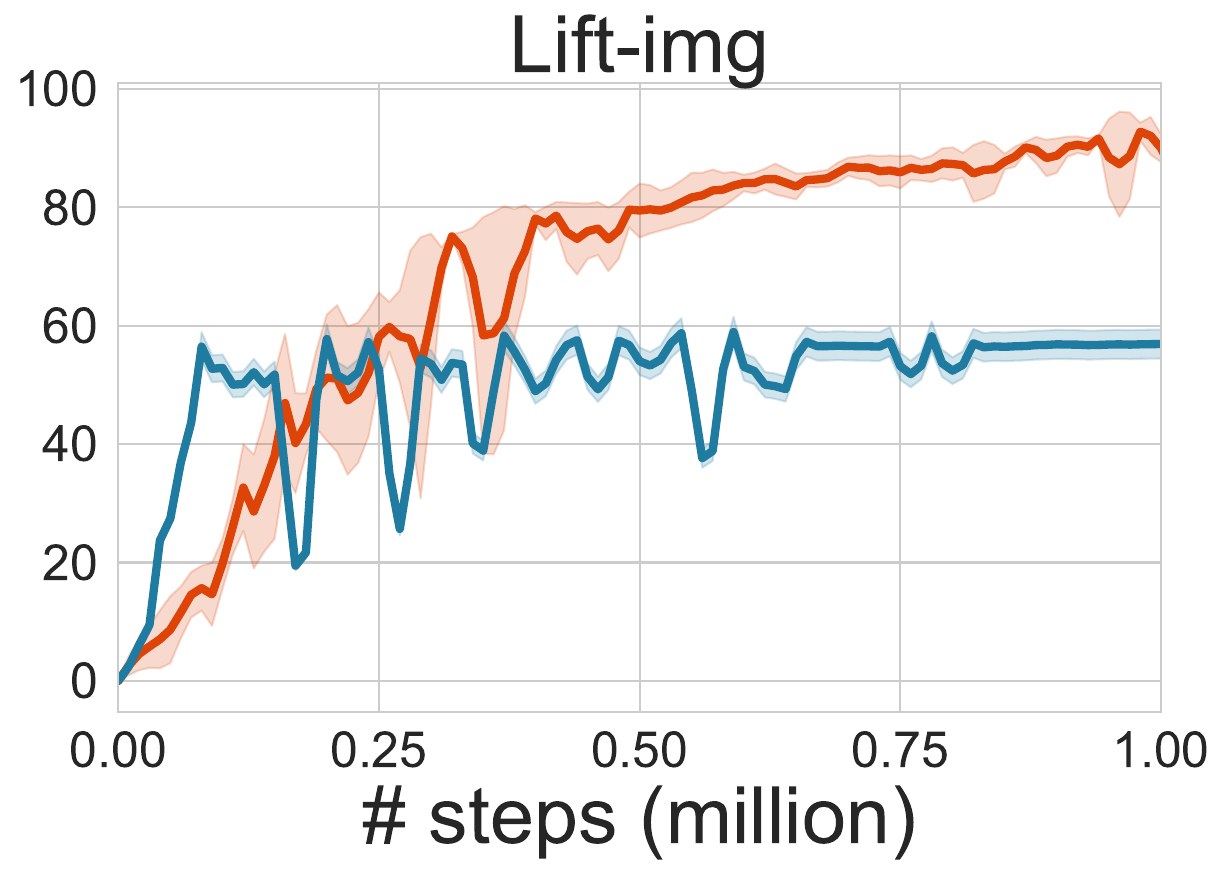}}
    \subfigure{
        \includegraphics[width=0.24\textwidth]{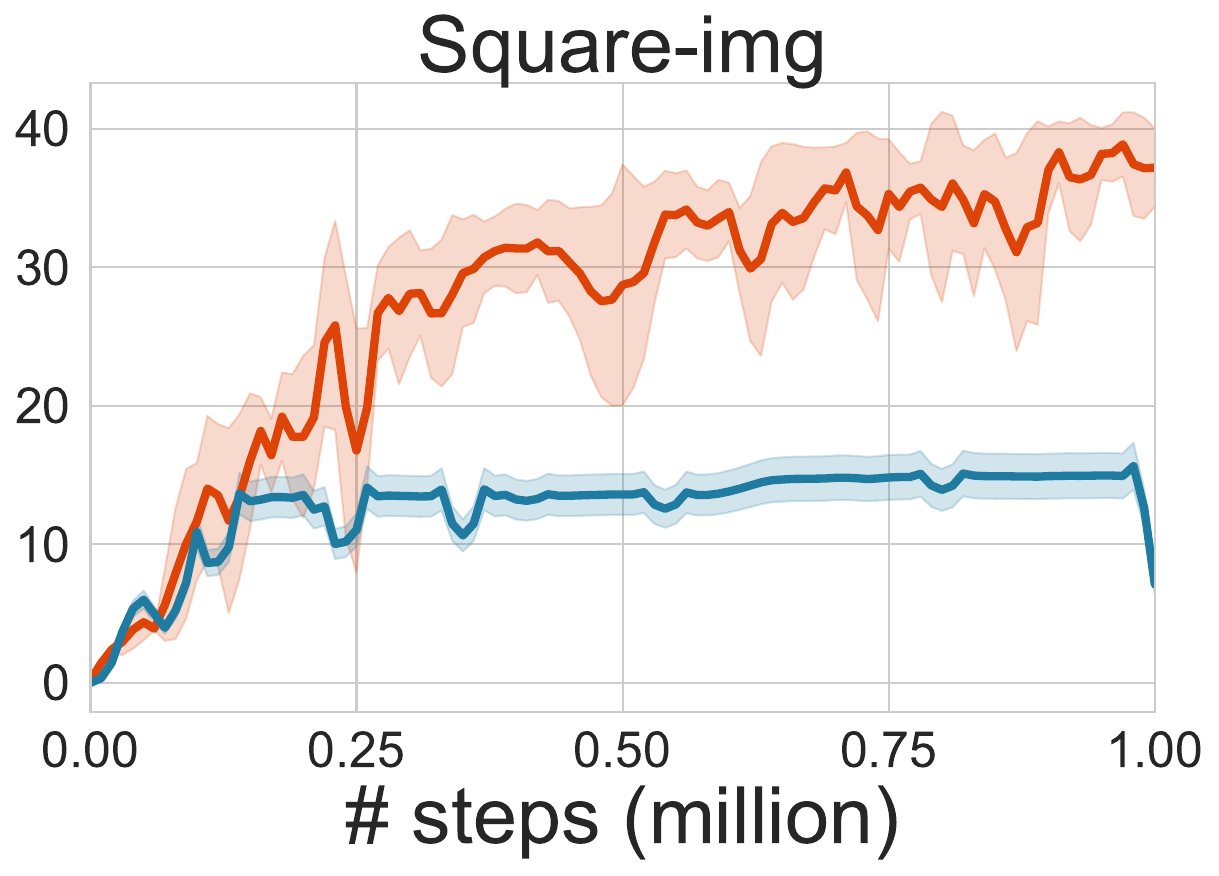}}
    % \vspace{-10pt}
    \caption{Comparison between \texttt{ISWBC} and \texttt{ILID}.}
    \label{fig:dynamic_information_all_2}
    % \vspace{-5pt}
\end{figure}

\clearpage

\subsubsection{Importance of Data Selection} 

In this section, we ablate the data selection and replace $\mathcal{D}_s$ in Problem~(\ref{eq:ilid_objective}) by entire imperfect data of $\mathcal{D}_b$.  \cref{fig:data_selection_all_1,fig:data_selection_all_2} corroborate \cref{hypo:resultant_state} and underscores the importance of our data selection scheme. 

% We ablate the data selection and replace $\mathcal{D}_s$ in Problem~(\ref{eq:ilid_objective}) by entire imperfect data of $\mathcal{D}_b$.  \cref{fig:data_selection_all_1,fig:data_selection_all_2} corroborates \cref{hypo:resultant_state} and underscores the importance of our data selection scheme.

%\textbf{Data selection.} To further answer the above question, we conduct numerous experiments to explore whether the proposed data selection algorithm can select appropriate data and eliminate interference from other poor-quality data. We compare \texttt{ILID} with algorithms that do not undergo data selection. As shown in \cref{fig:data_selection_all}, \texttt{ILID} clearly demonstrates superior performance.

\begin{figure}[ht]
    \centering
    \subfigure{
        \includegraphics[width=0.24\textwidth]{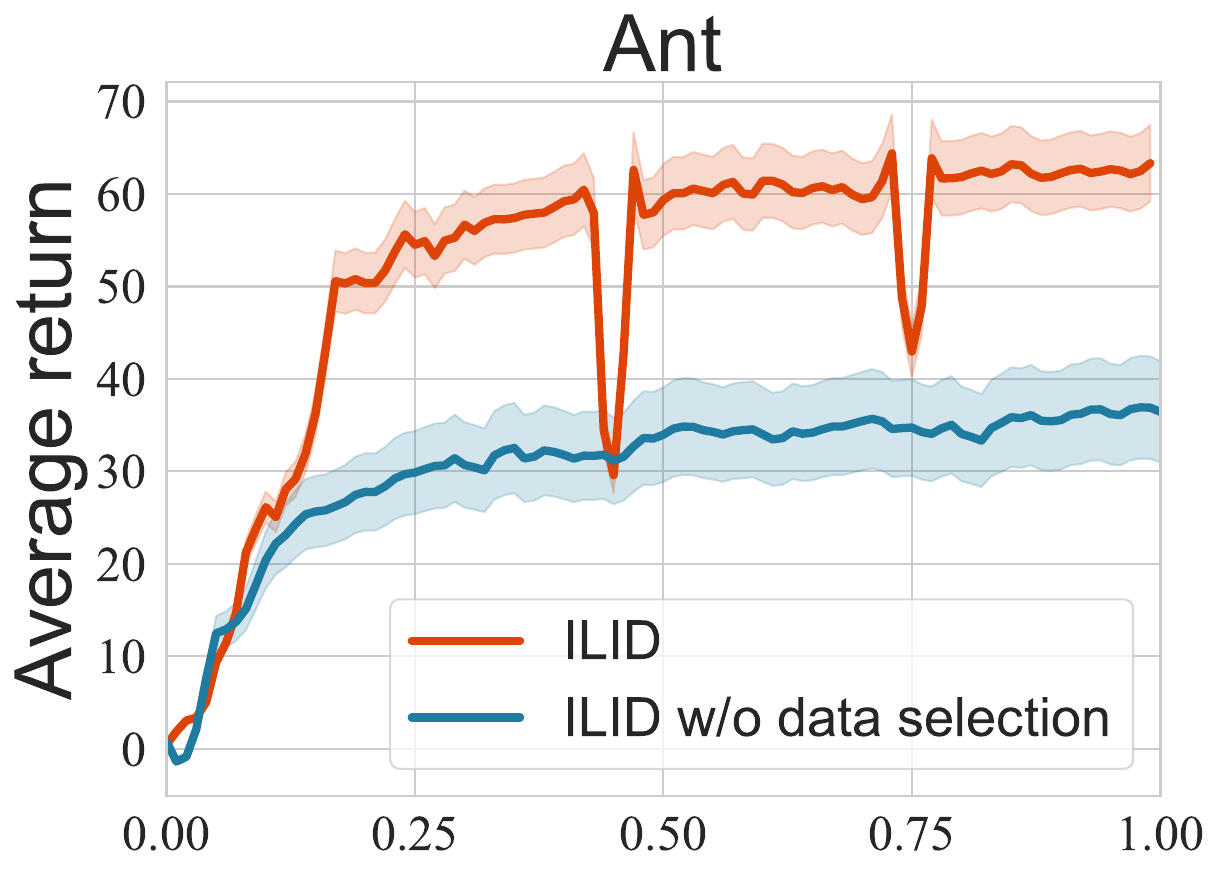}}
    \hspace{-4pt}
    \subfigure{
        \includegraphics[width=0.24\textwidth]{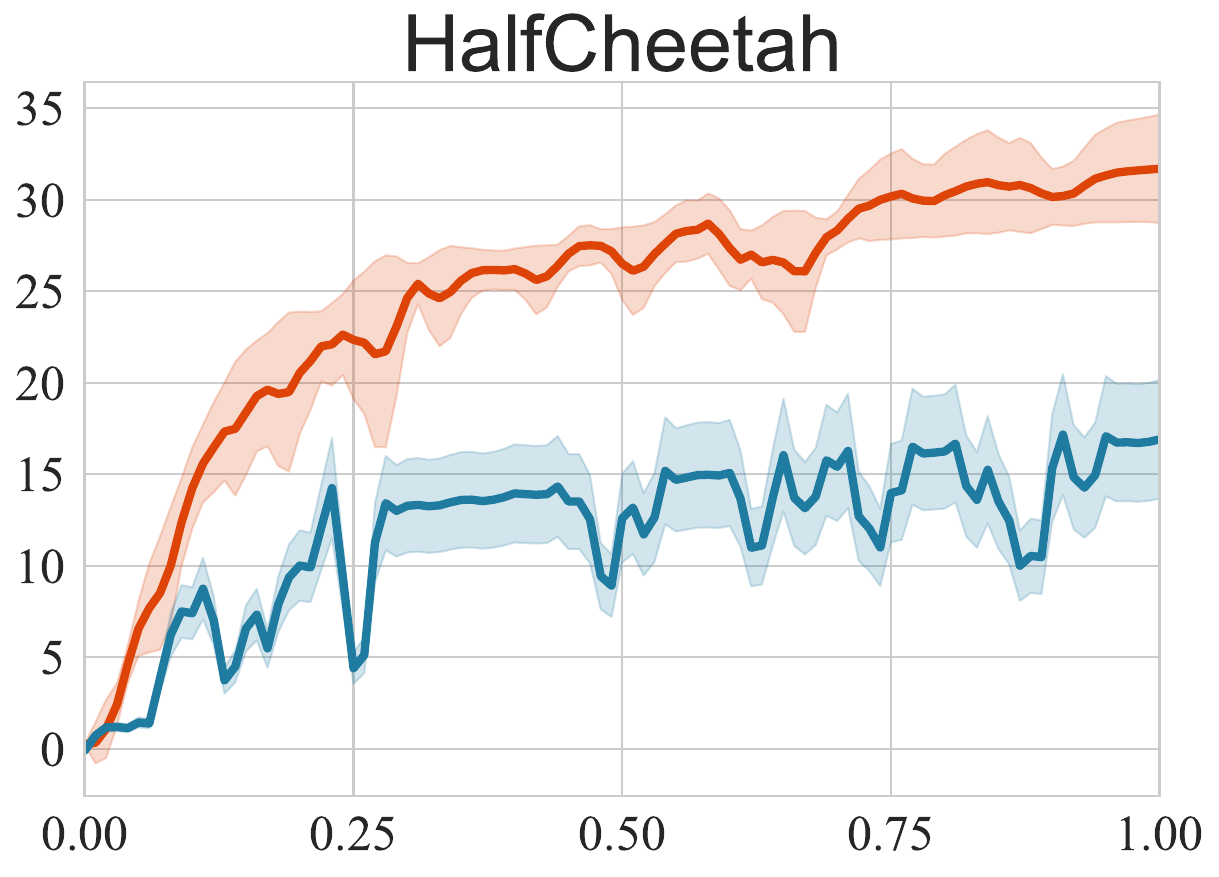}}
    \hspace{-4pt}
    \subfigure{
        \includegraphics[width=0.24\textwidth]{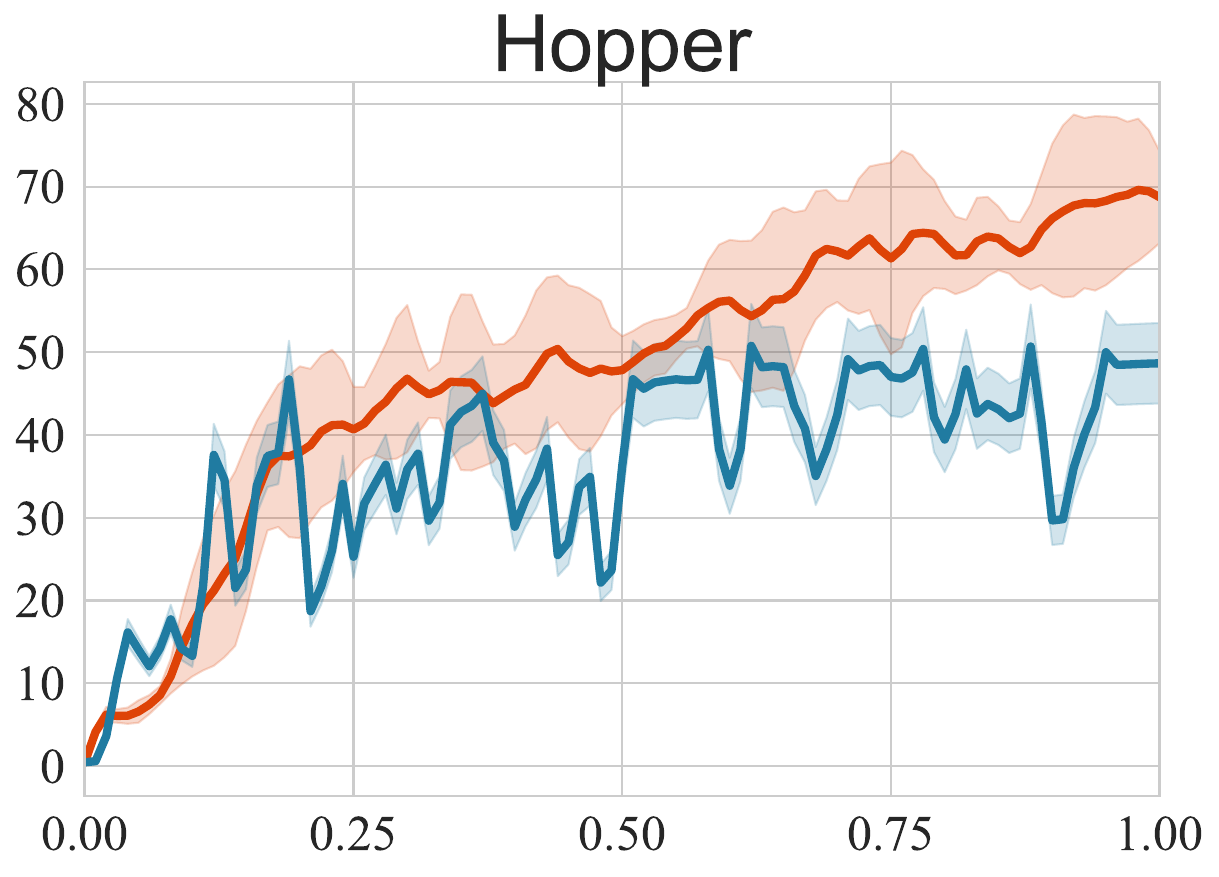}}
    \hspace{-4pt}
    \subfigure{
        \includegraphics[width=0.24\textwidth]{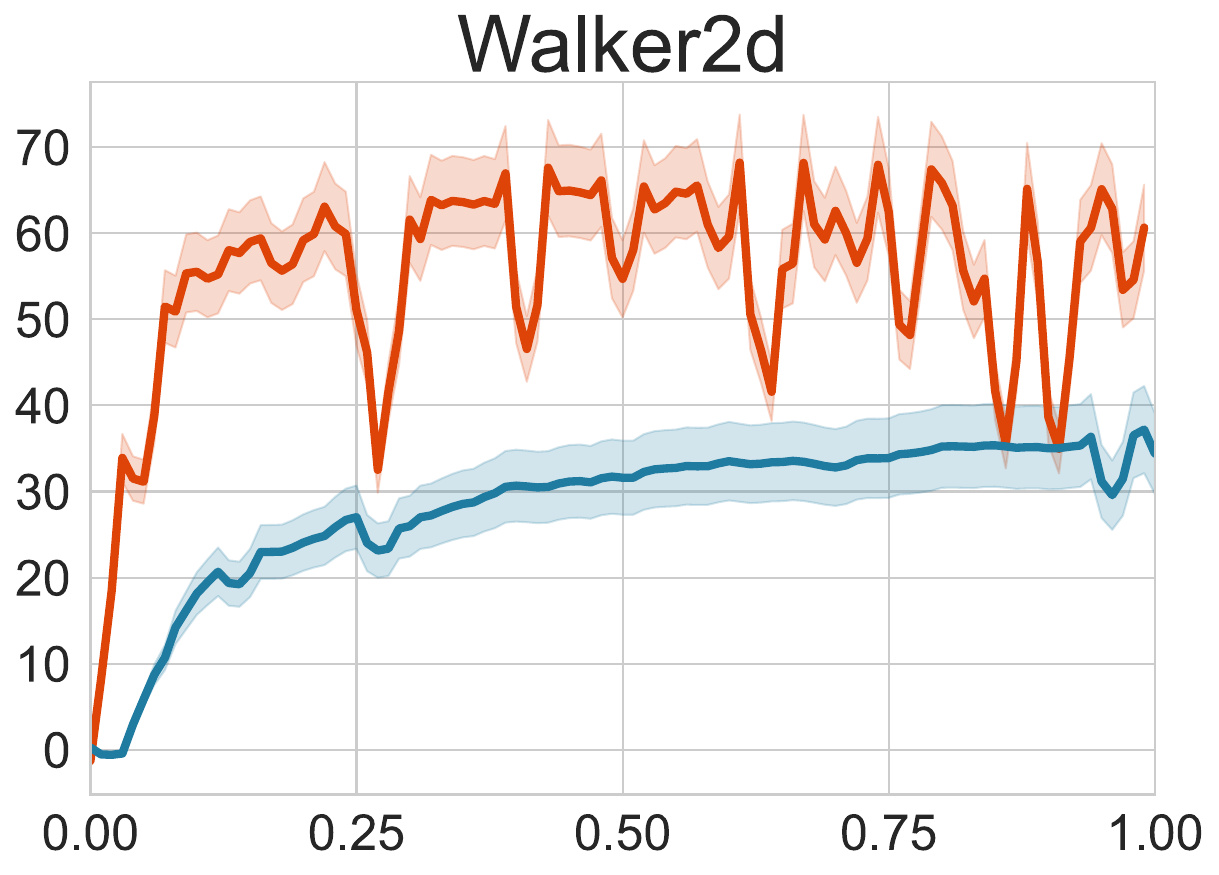}}
    
    \vspace{-10pt}
    \subfigure{
        \includegraphics[width=0.24\textwidth]{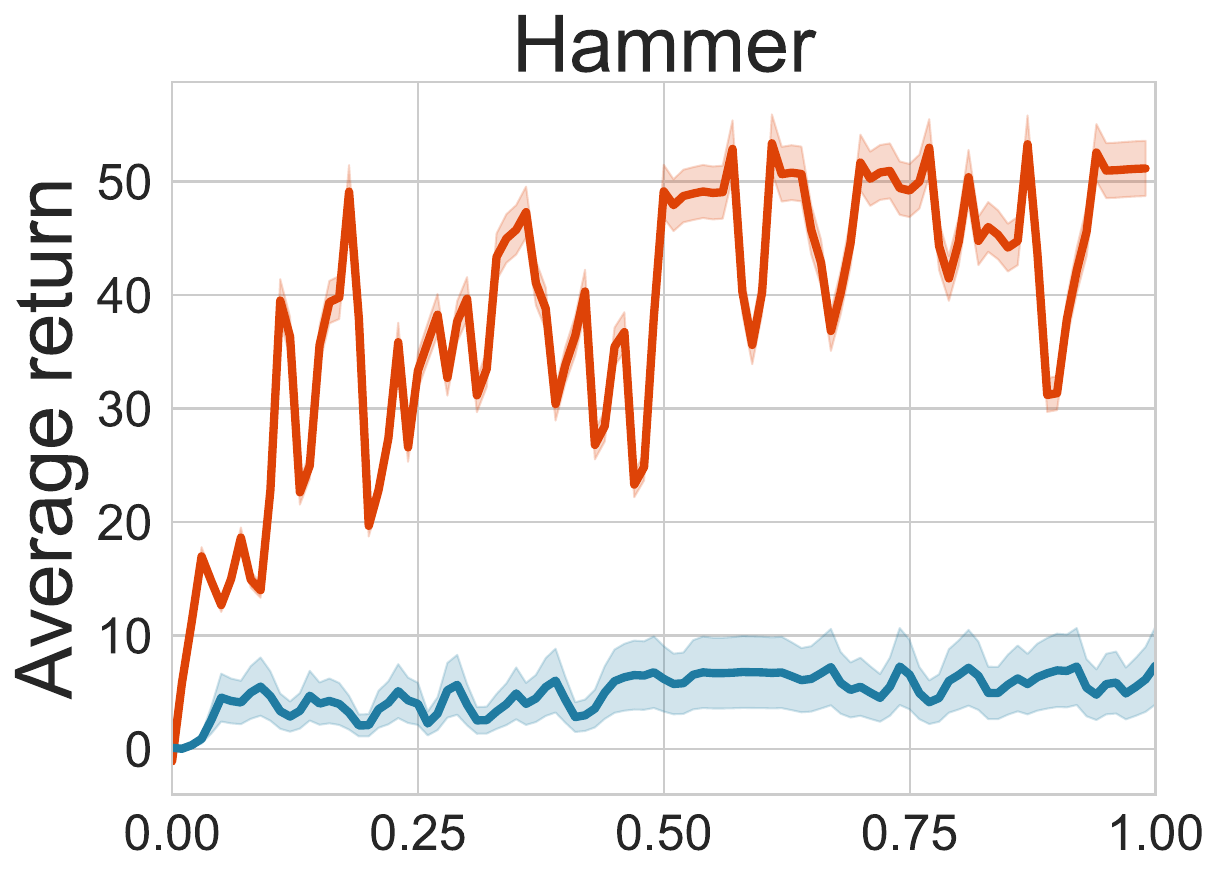}}
    \hspace{-4pt}
    \subfigure{
        \includegraphics[width=0.24\textwidth]{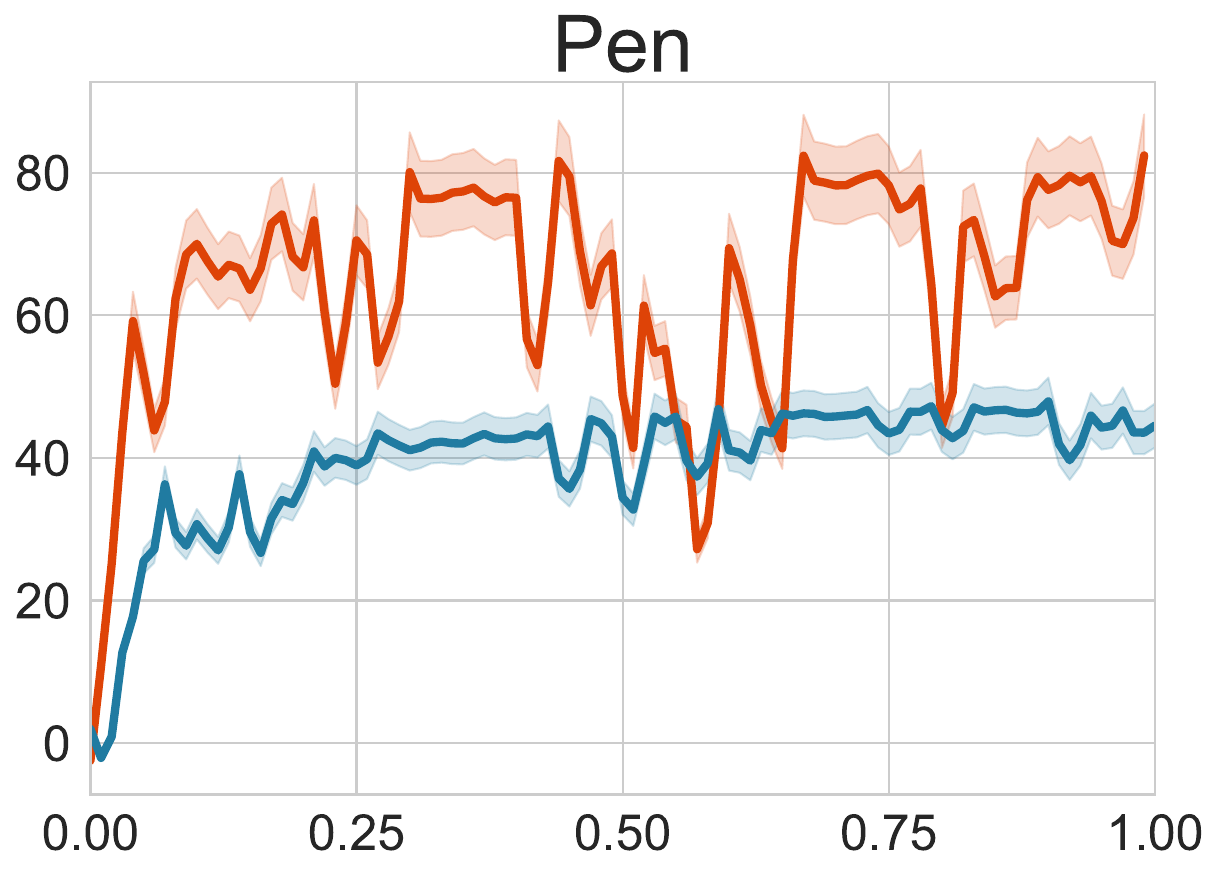}}
    \hspace{-4pt}
    \subfigure{
        \includegraphics[width=0.24\textwidth]{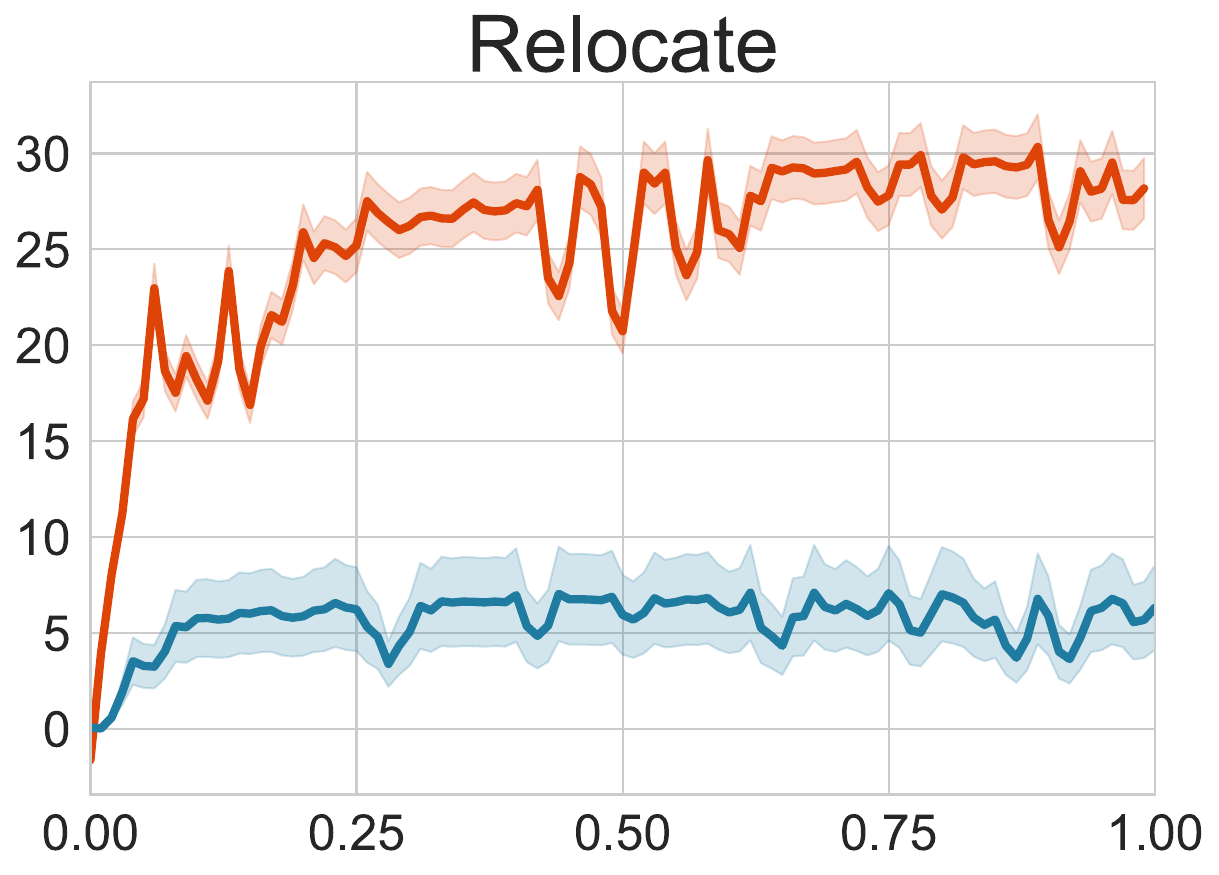}}
    \hspace{-4pt}
    \subfigure{
        \includegraphics[width=0.24\textwidth]{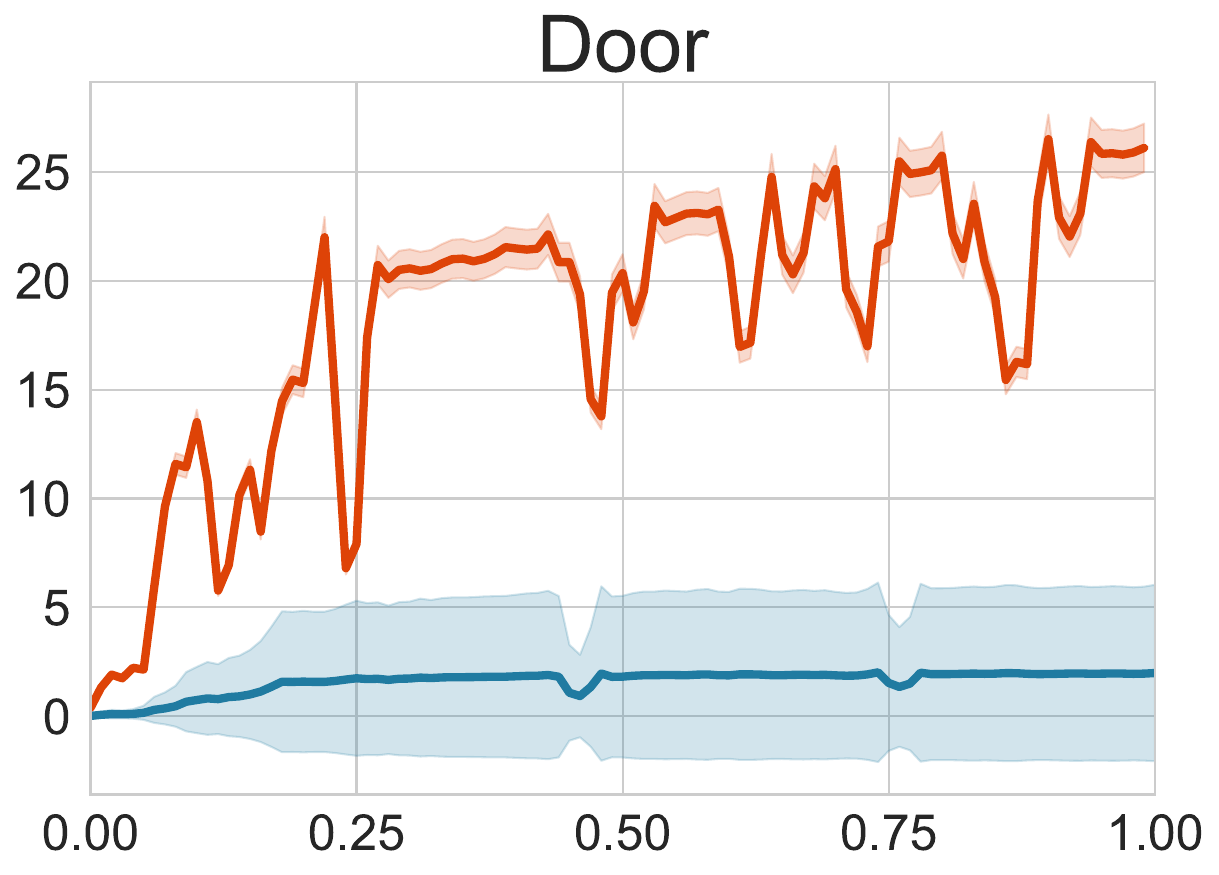}}
    
    \vspace{-10pt}
    \subfigure{
        \includegraphics[width=0.24\textwidth]{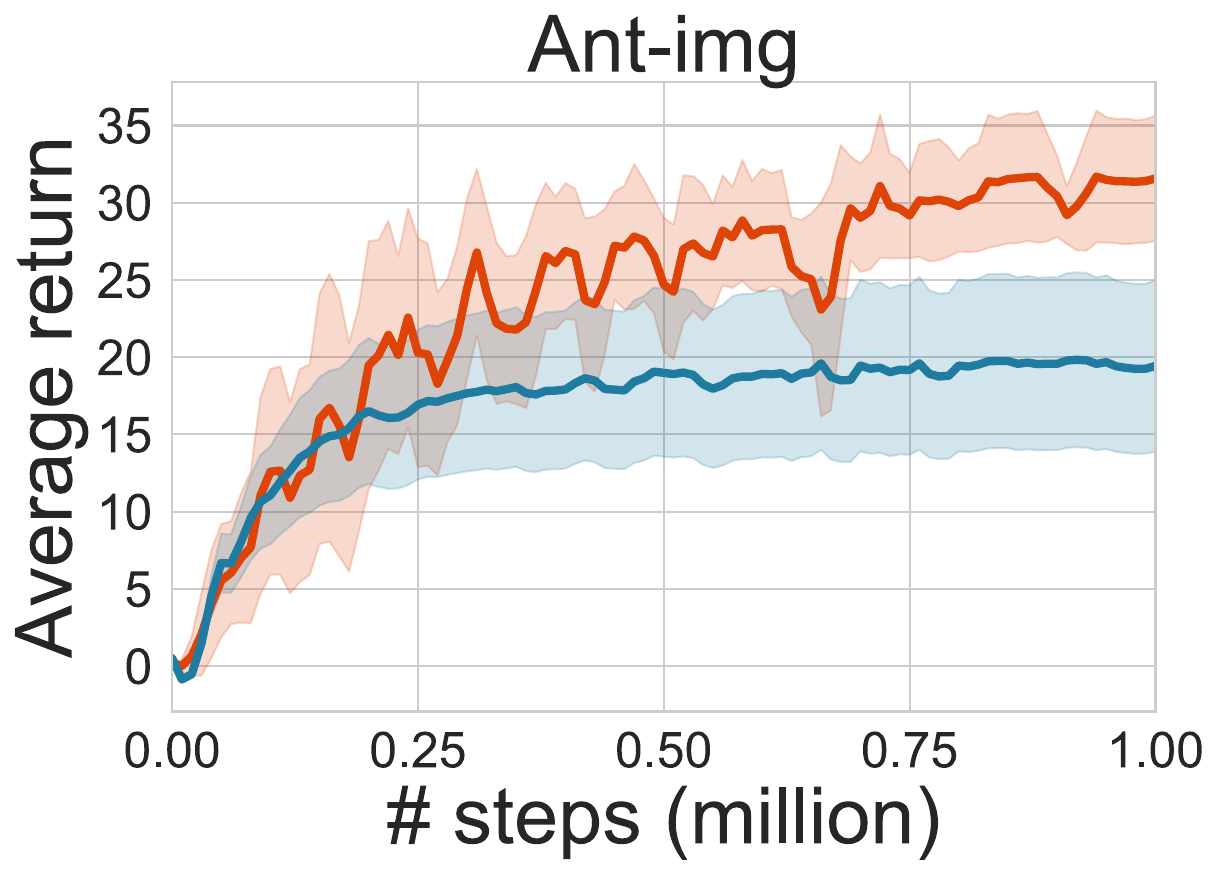}}
    \hspace{-4pt}
    \subfigure{
        \includegraphics[width=0.24\textwidth]{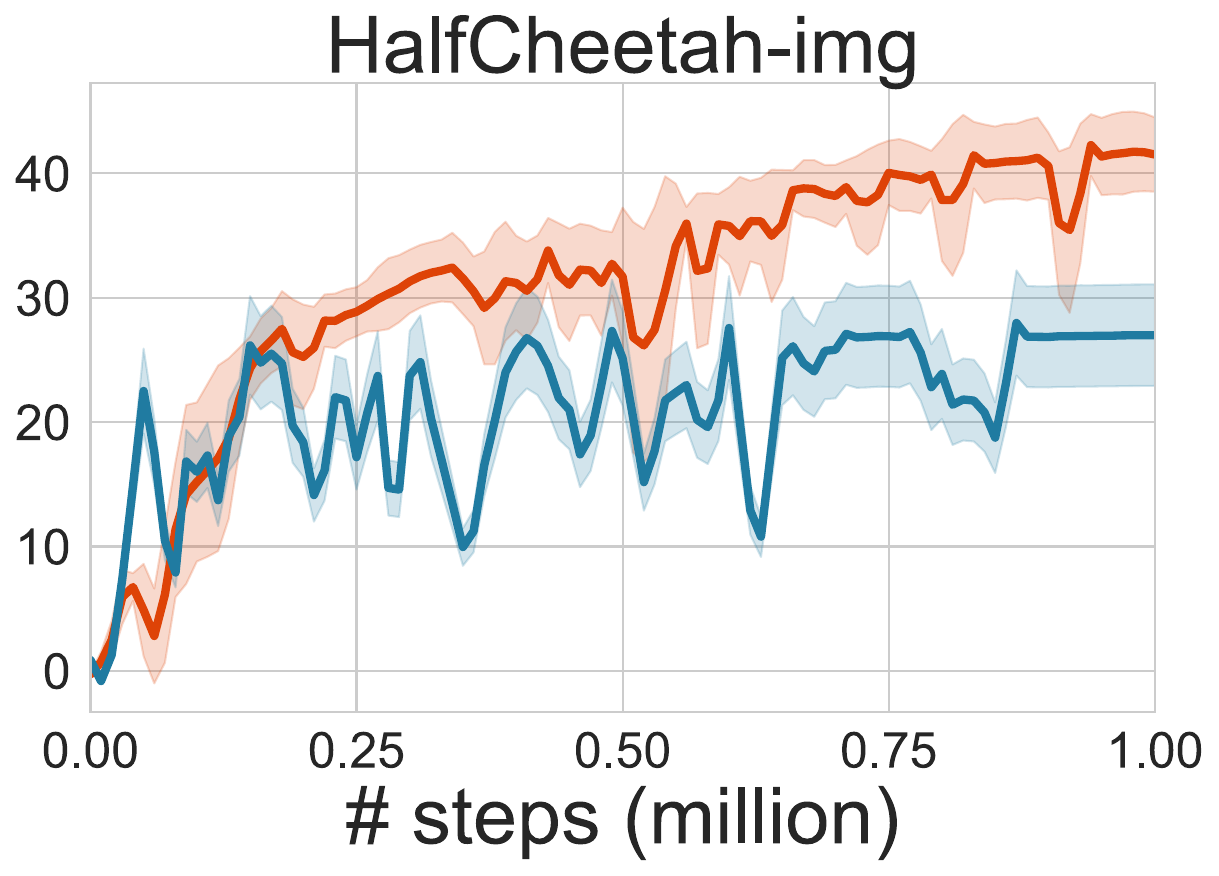}}
    \hspace{-4pt}
    \subfigure{
        \includegraphics[width=0.24\textwidth]{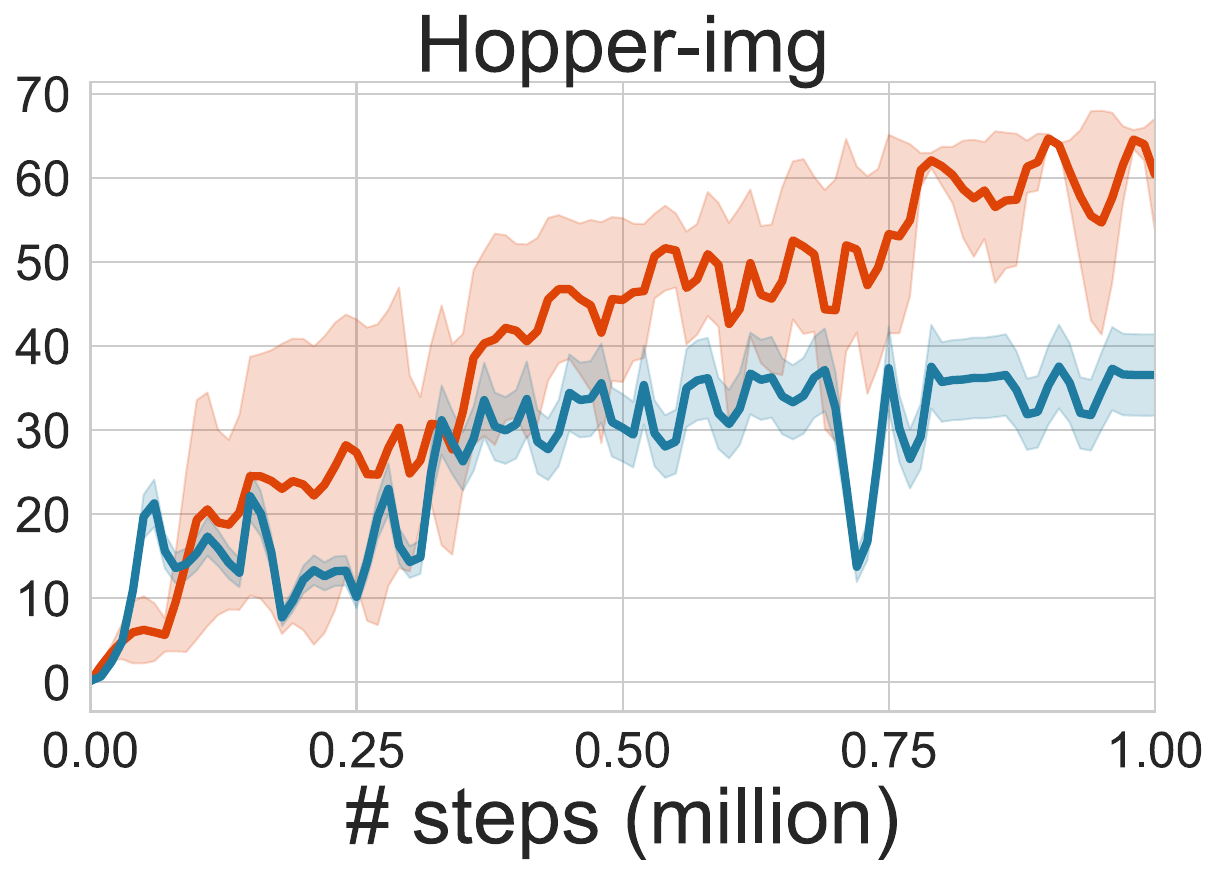}}
    \hspace{-4pt}
    \subfigure{
        \includegraphics[width=0.24\textwidth]{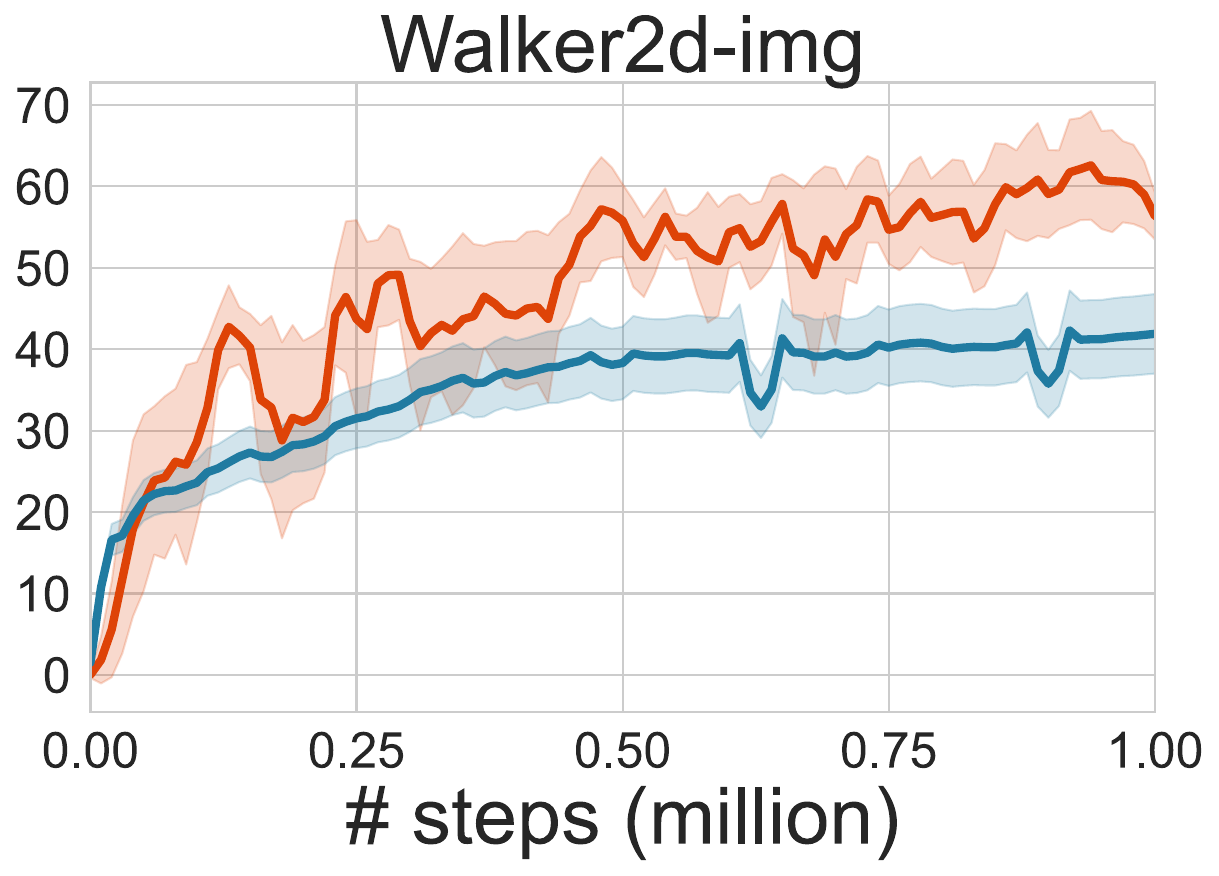}}
    \caption{Effect of data selection.}
    \label{fig:data_selection_all_1}
\end{figure}

\begin{figure}[ht]
    \centering
    \subfigure{
        \includegraphics[width=0.24\textwidth]{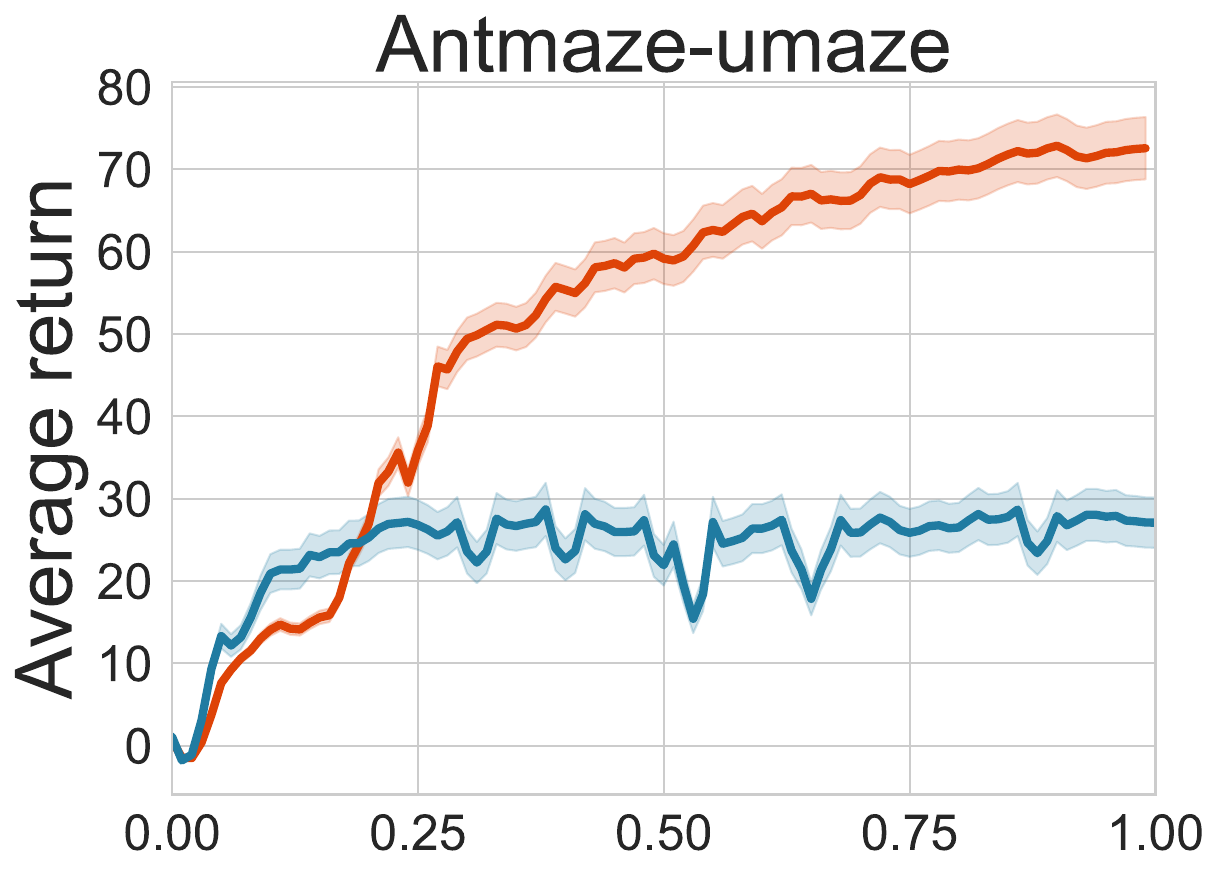}}
    \subfigure{
        \includegraphics[width=0.24\textwidth]{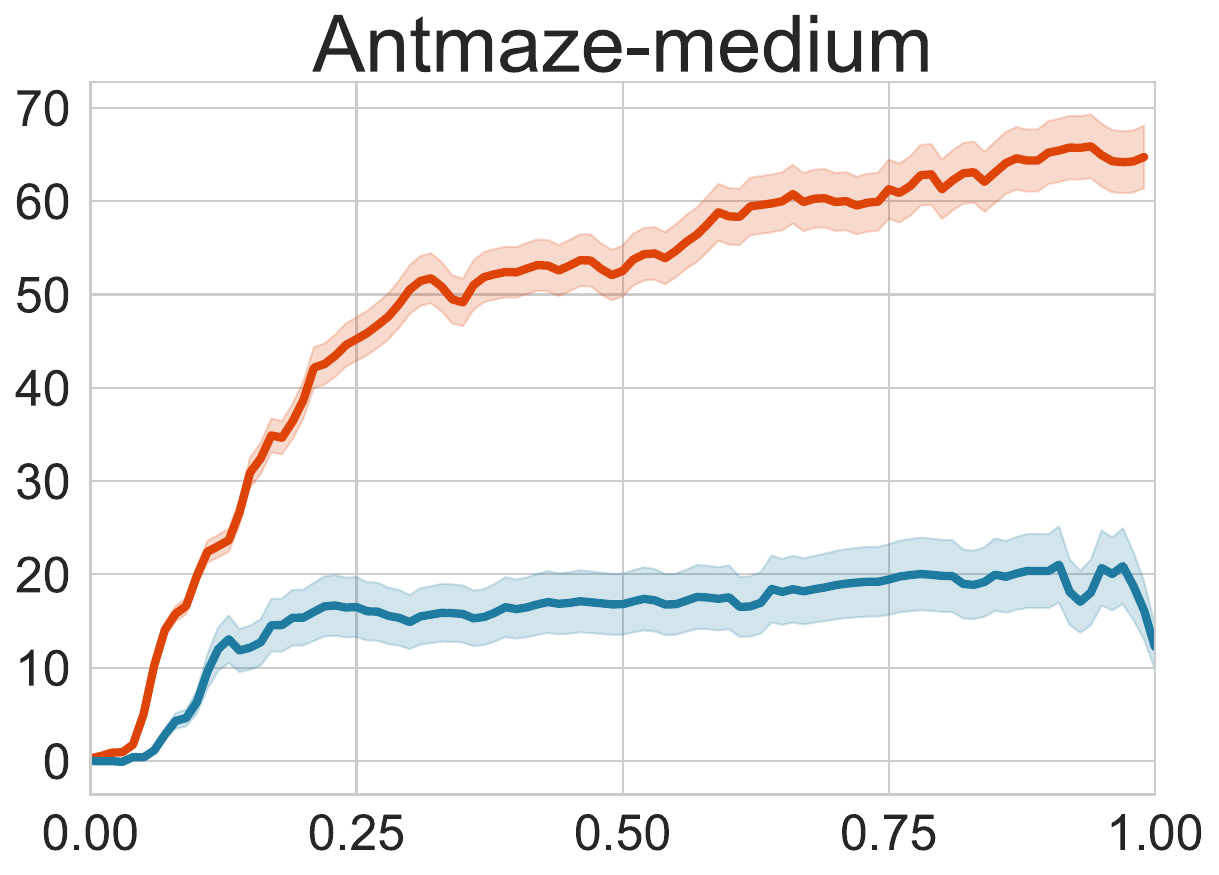}}
    \subfigure{
        \includegraphics[width=0.24\textwidth]{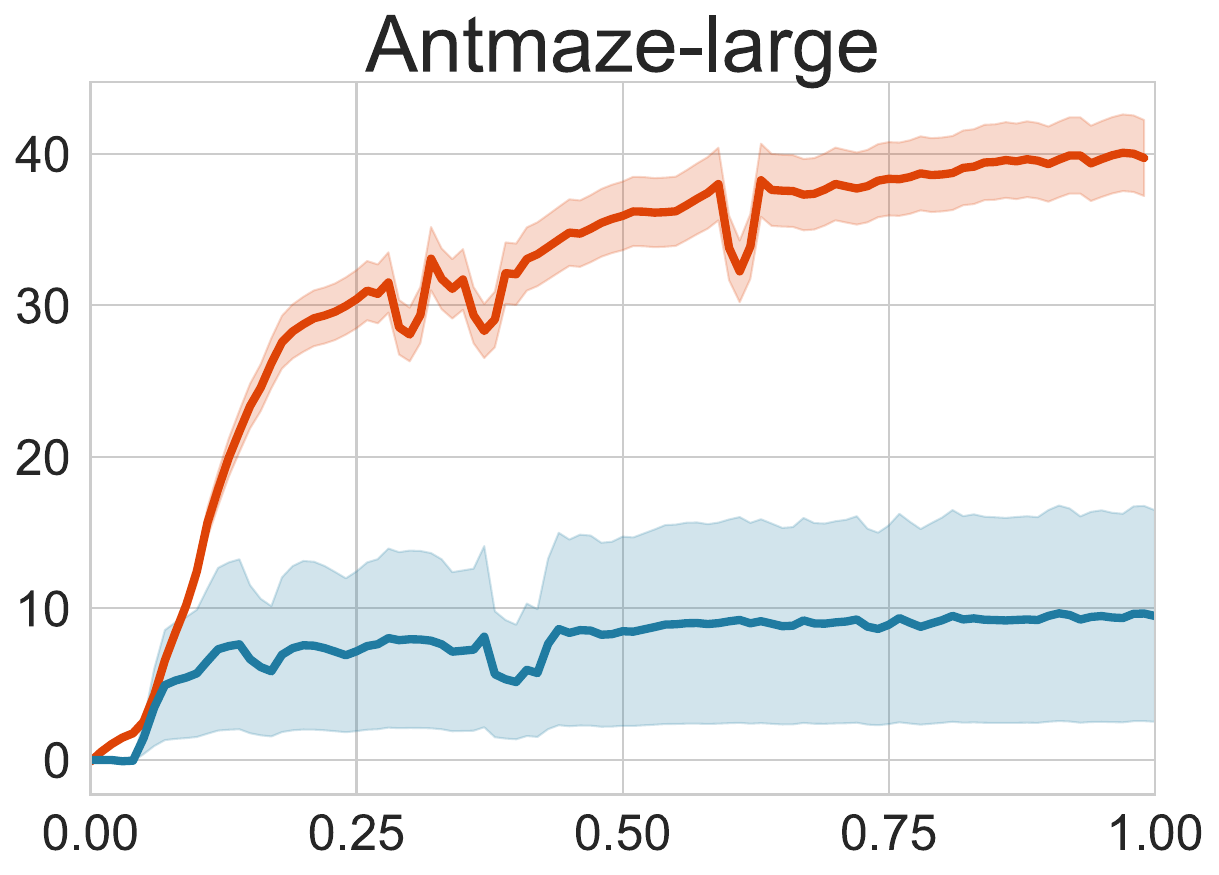}}
    
    \vspace{-10pt}
    \subfigure{
        \includegraphics[width=0.24\textwidth]{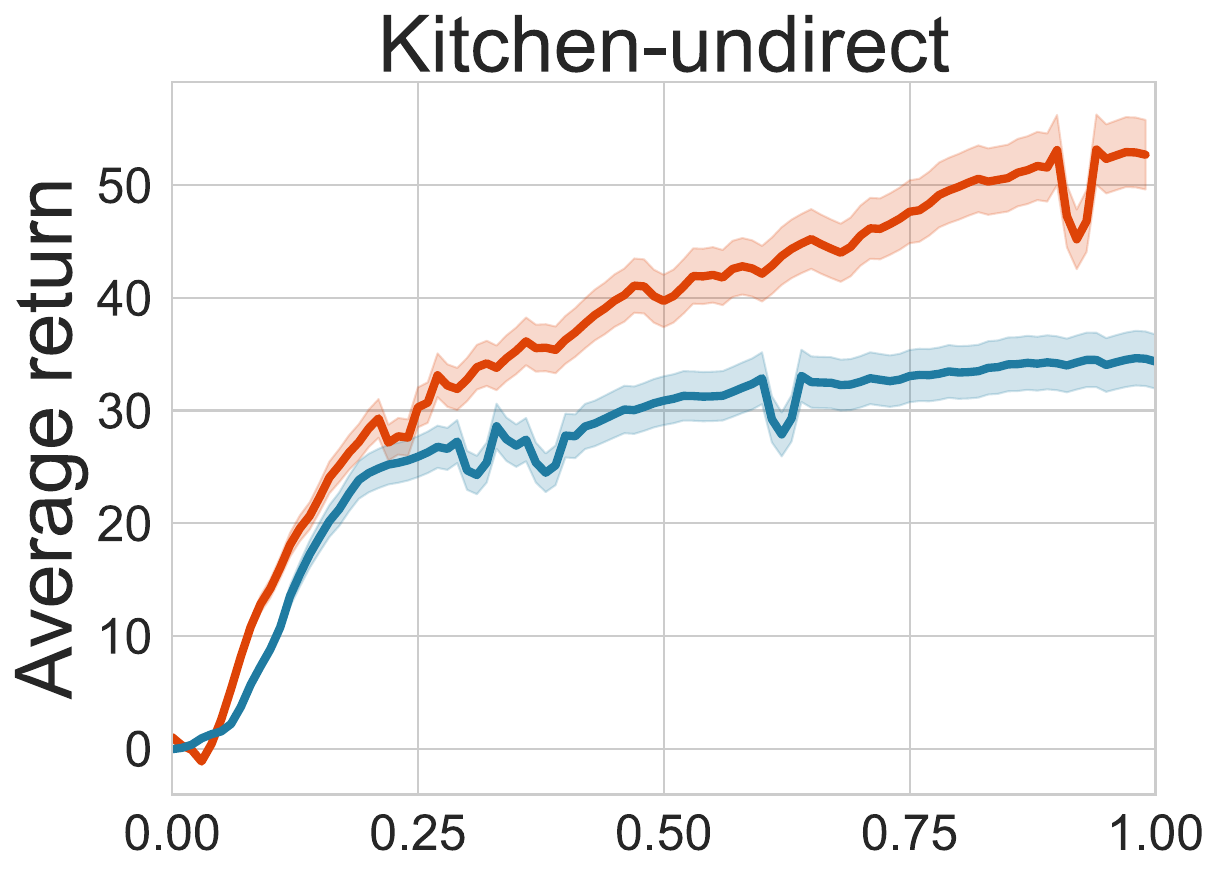}}
    \subfigure{
        \includegraphics[width=0.24\textwidth]{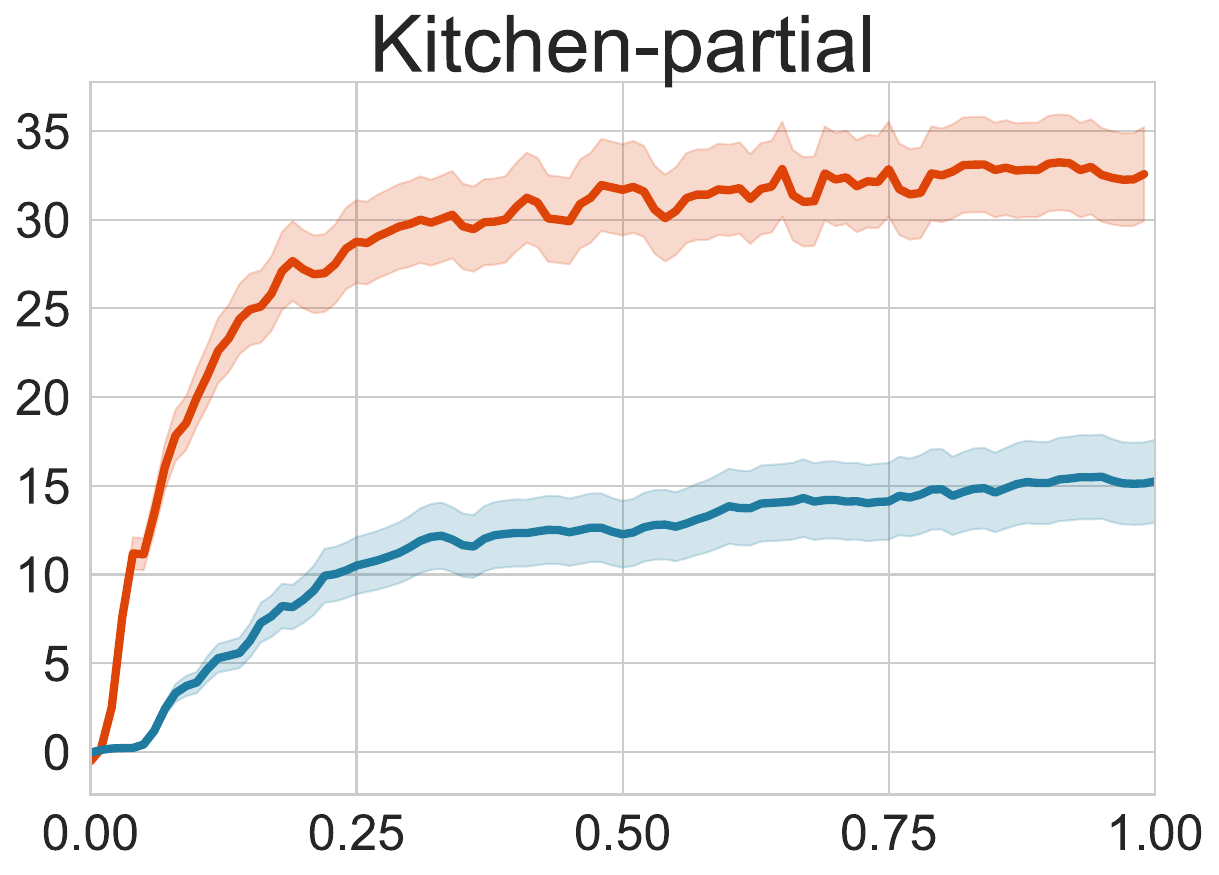}}
    \subfigure{
        \includegraphics[width=0.24\textwidth]{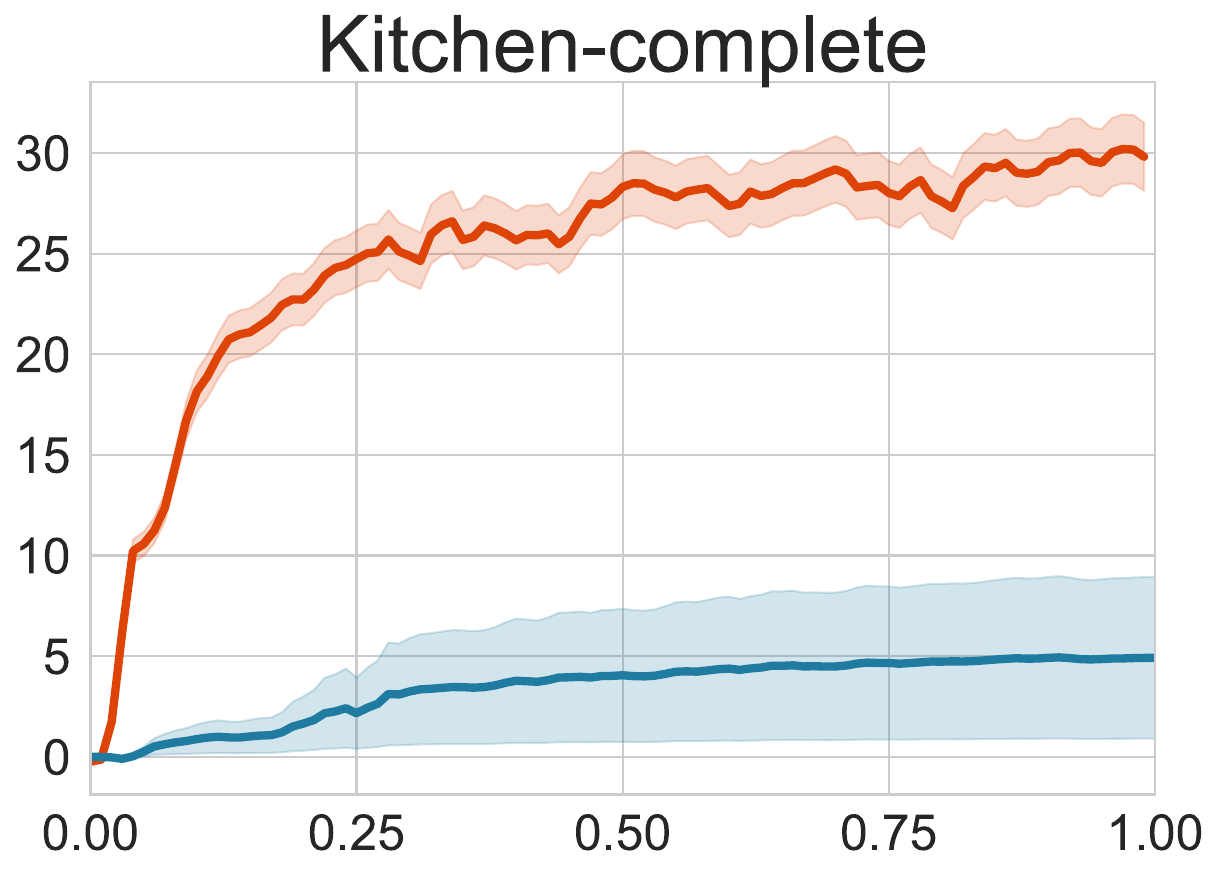}}

    \vspace{-10pt}
    \subfigure{
        \includegraphics[width=0.24\textwidth]{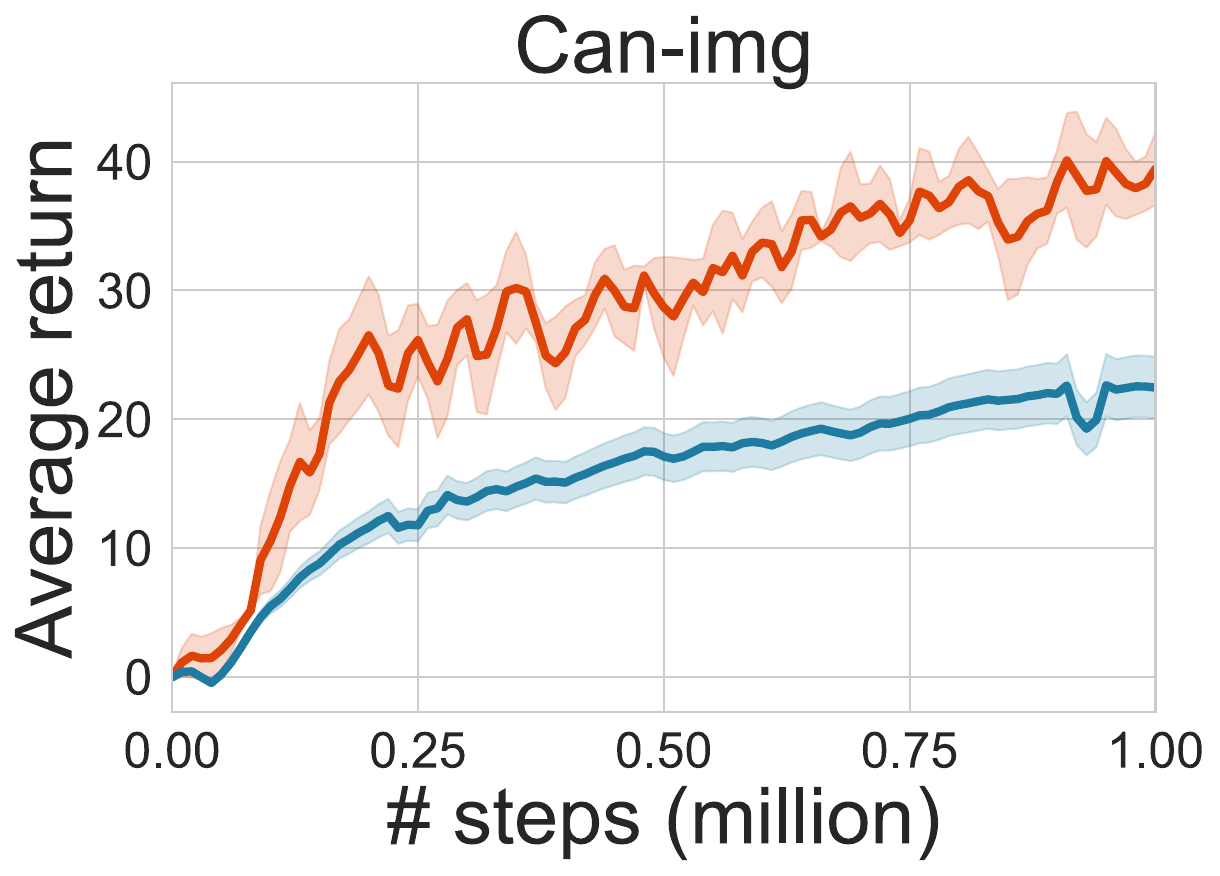}}
    \subfigure{
        \includegraphics[width=0.24\textwidth]{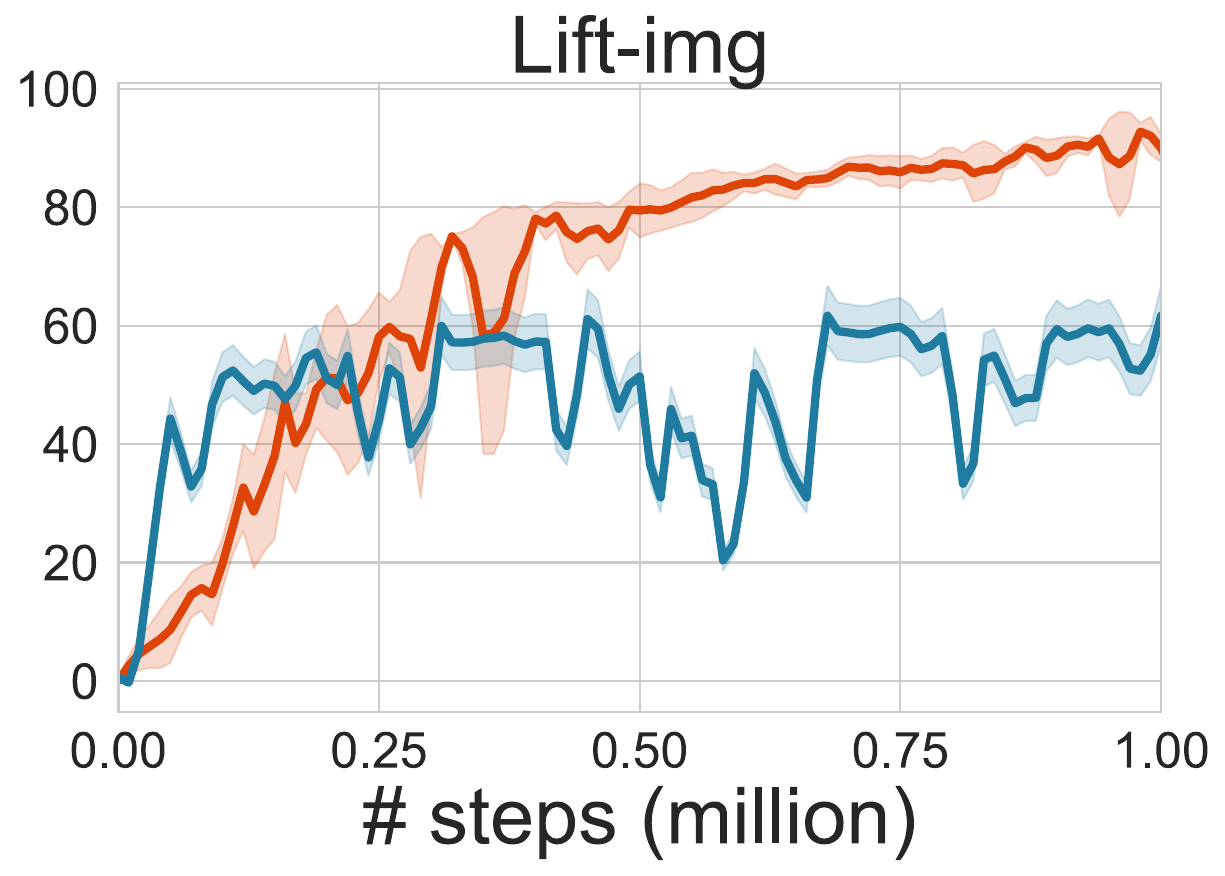}}
    \subfigure{
        \includegraphics[width=0.24\textwidth]{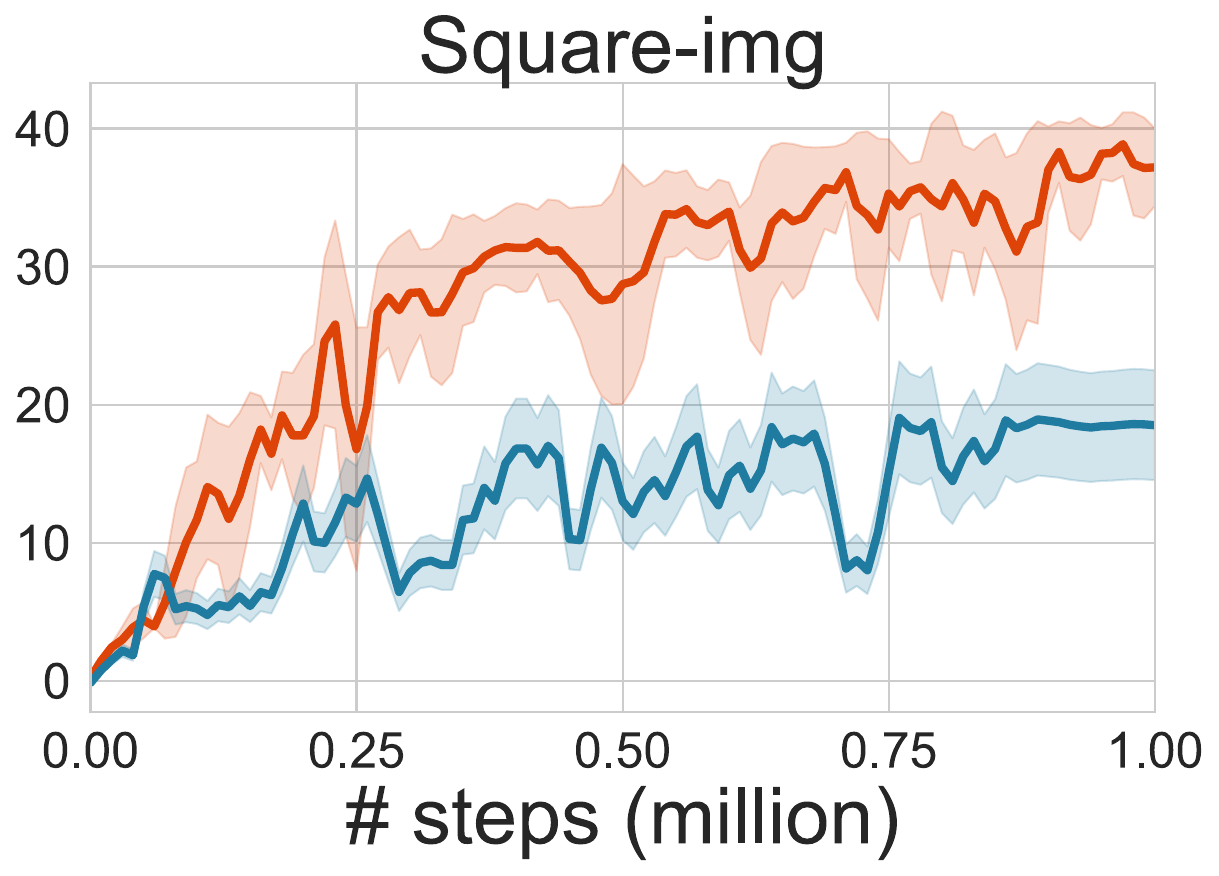}}
    \caption{Effect of data selection.}
    \label{fig:data_selection_all_2}
\end{figure}

\clearpage

\subsubsection{Effect of $\alpha(s,a)$ and $\beta(s,a)$}

In this section, we carry out ablation studies on $\alpha(s,a)$ and $\beta(s,a)$. The observed performance degradation in \cref{fig:weight_all_1,fig:weight_all_2} clearly demonstrates the benefits of $\alpha(s,a)$ and $\beta(s,a)$. The importance-sampling weights can enhance the expert data support for \texttt{BC}, particularly in continuous domains. The absence of $\beta(s,a)$ renders the training becomes ineffective and unstable, due to behavior interference.

%To investigate the necessity of weights in the proposed BC method as outlined in \cref{eq:ilid_objective}, we compared the performance of \texttt{ILID} under various conditions: without using $\alpha$, without using $\beta$, and without using both $\alpha$ and $\beta$, essentially directly imitating the combination of expert and selected data. As shown in \cref{fig:weight_all}, both $\alpha$ and $\beta$ played a crucial role during the learning process, effectively addressing the behavior interference issue caused by the suboptimality of imperfect data. Furthermore, by comparing the scenario without $\alpha$ and $\beta$ to BCU, the importance of data selection is further underscored.

\begin{figure}[ht]
    \centering
    \subfigure{
        \includegraphics[width=0.24\textwidth]{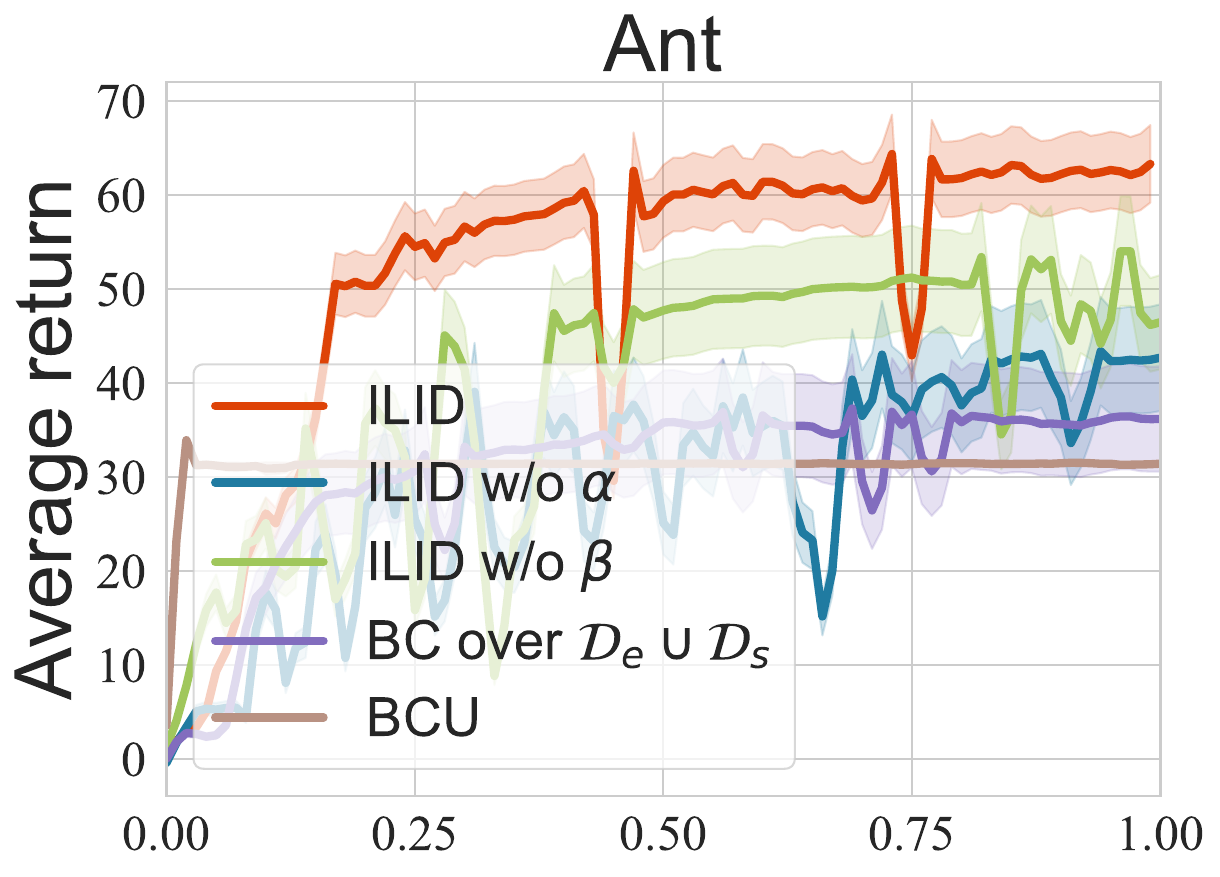}}
    \hspace{-4pt}
    \subfigure{
        \includegraphics[width=0.24\textwidth]{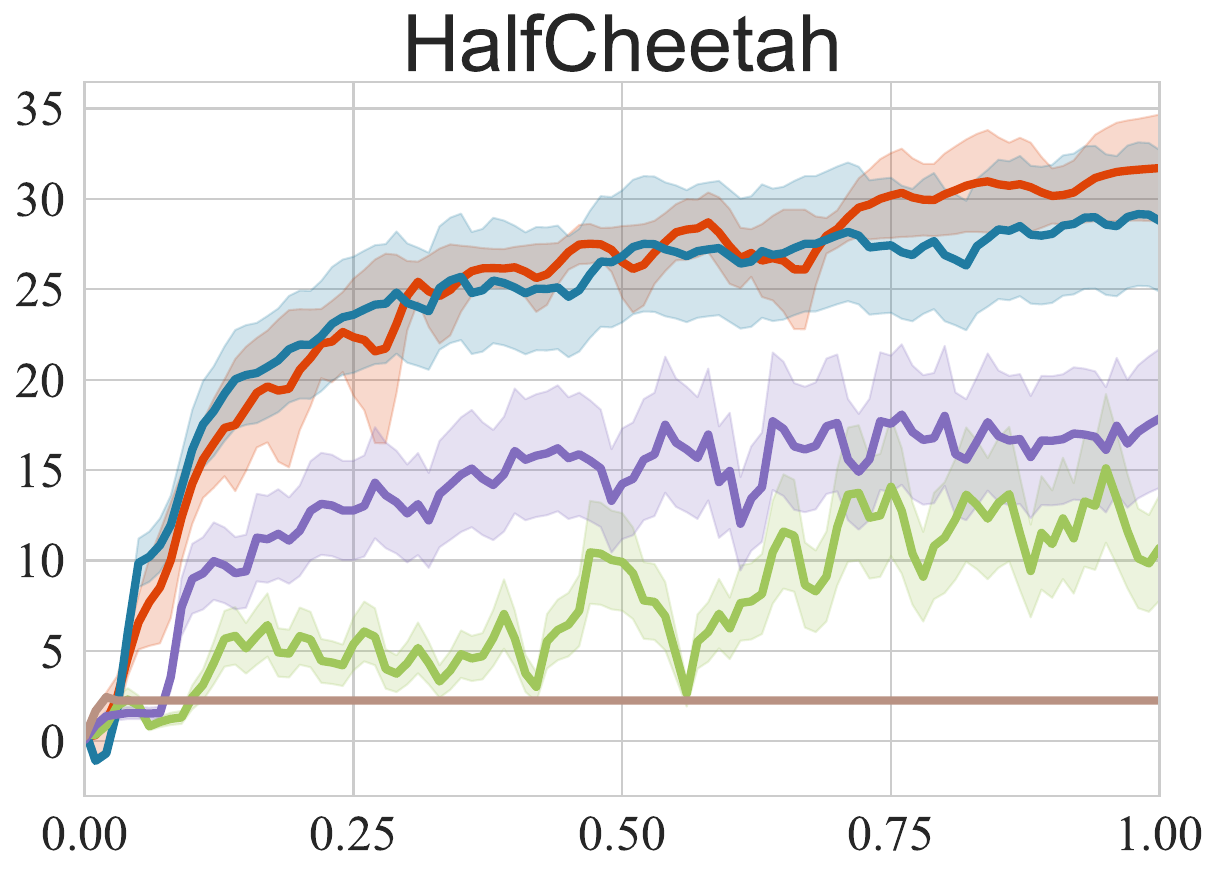}}
    \hspace{-4pt}
    \subfigure{
        \includegraphics[width=0.24\textwidth]{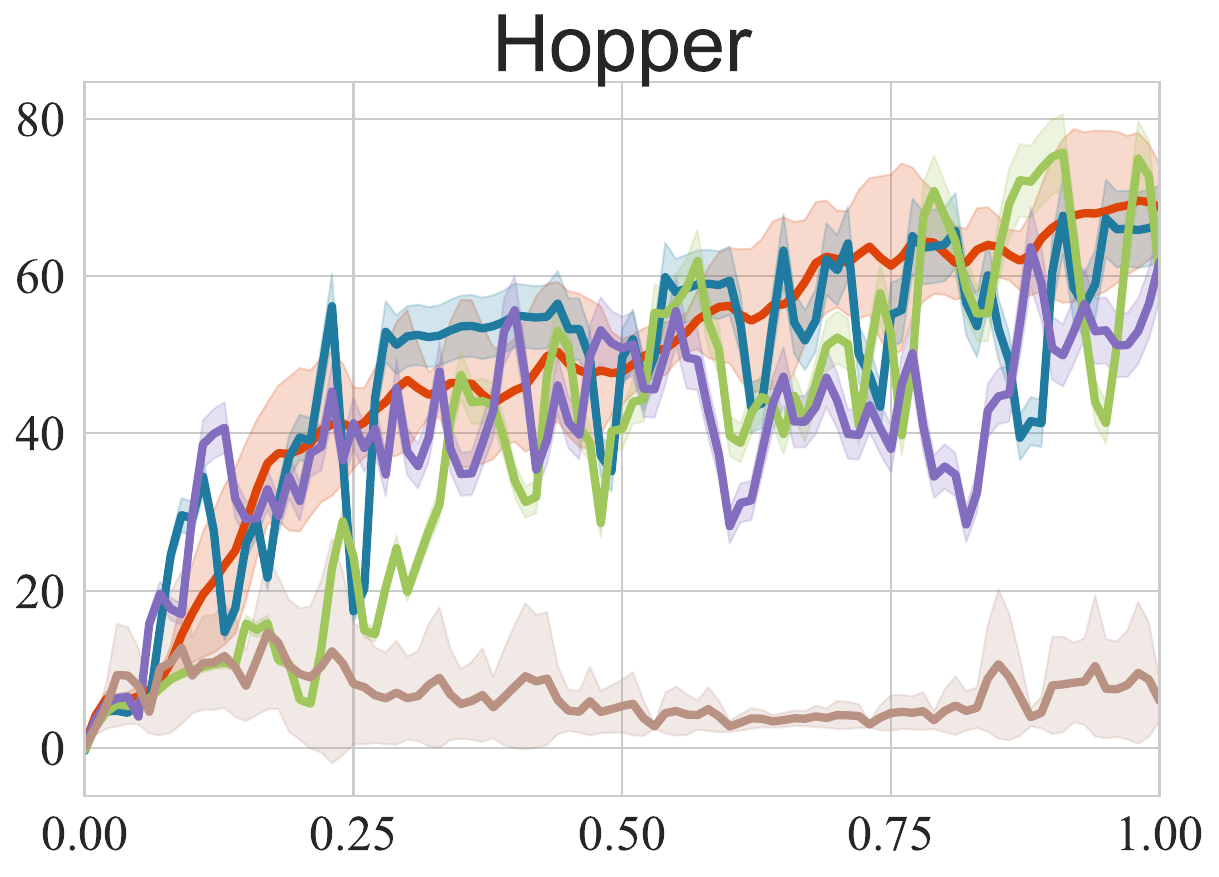}}
    \hspace{-4pt}
    \subfigure{
        \includegraphics[width=0.24\textwidth]{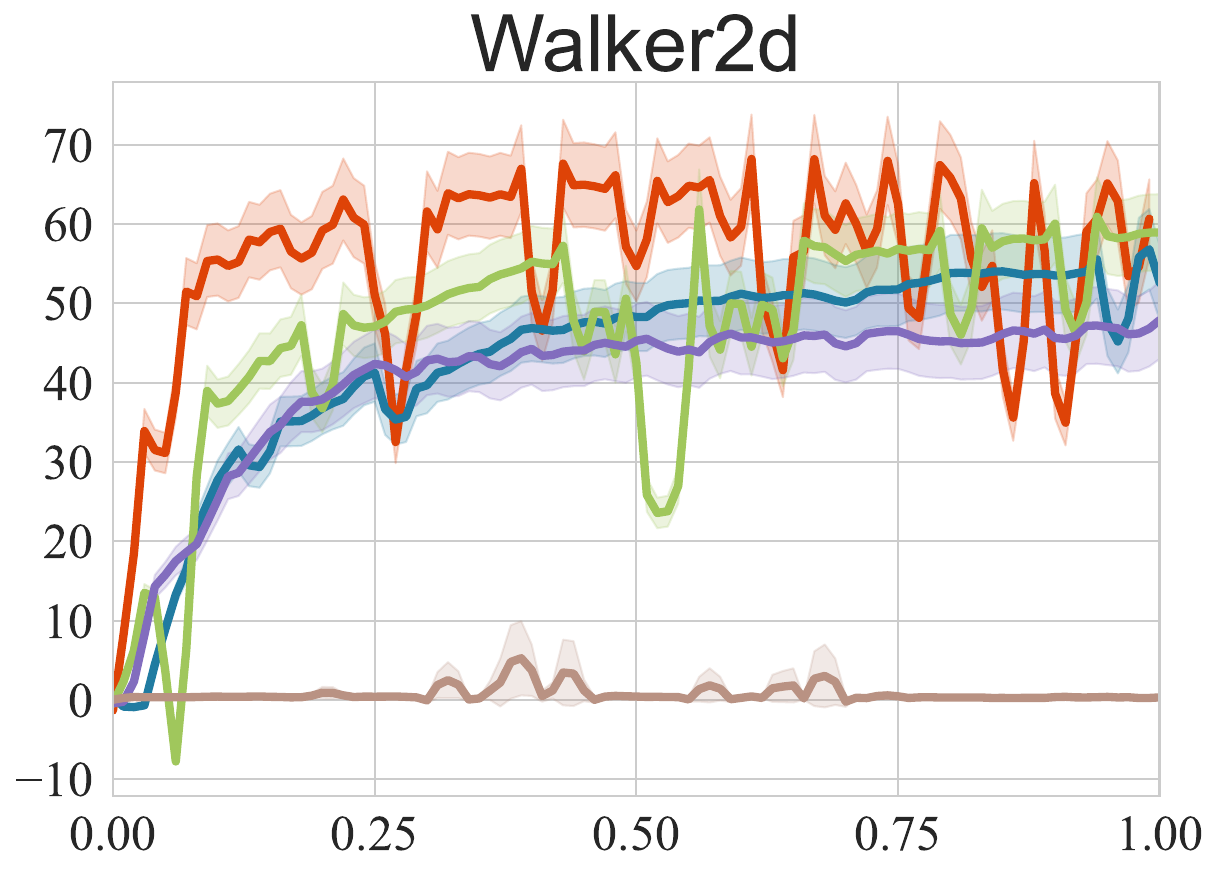}}
    
    \vspace{-10pt}
    \subfigure{
        \includegraphics[width=0.24\textwidth]{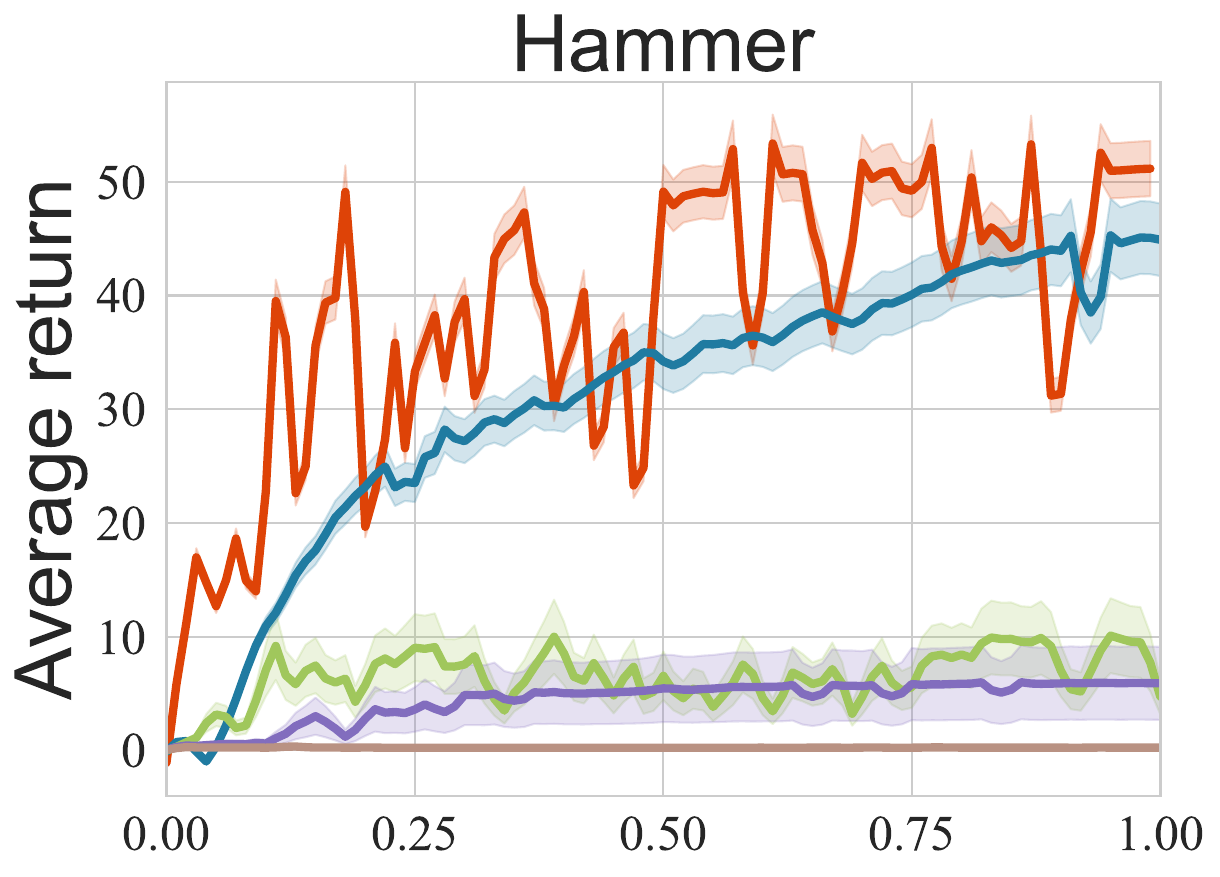}}
    \hspace{-4pt}
    \subfigure{
        \includegraphics[width=0.24\textwidth]{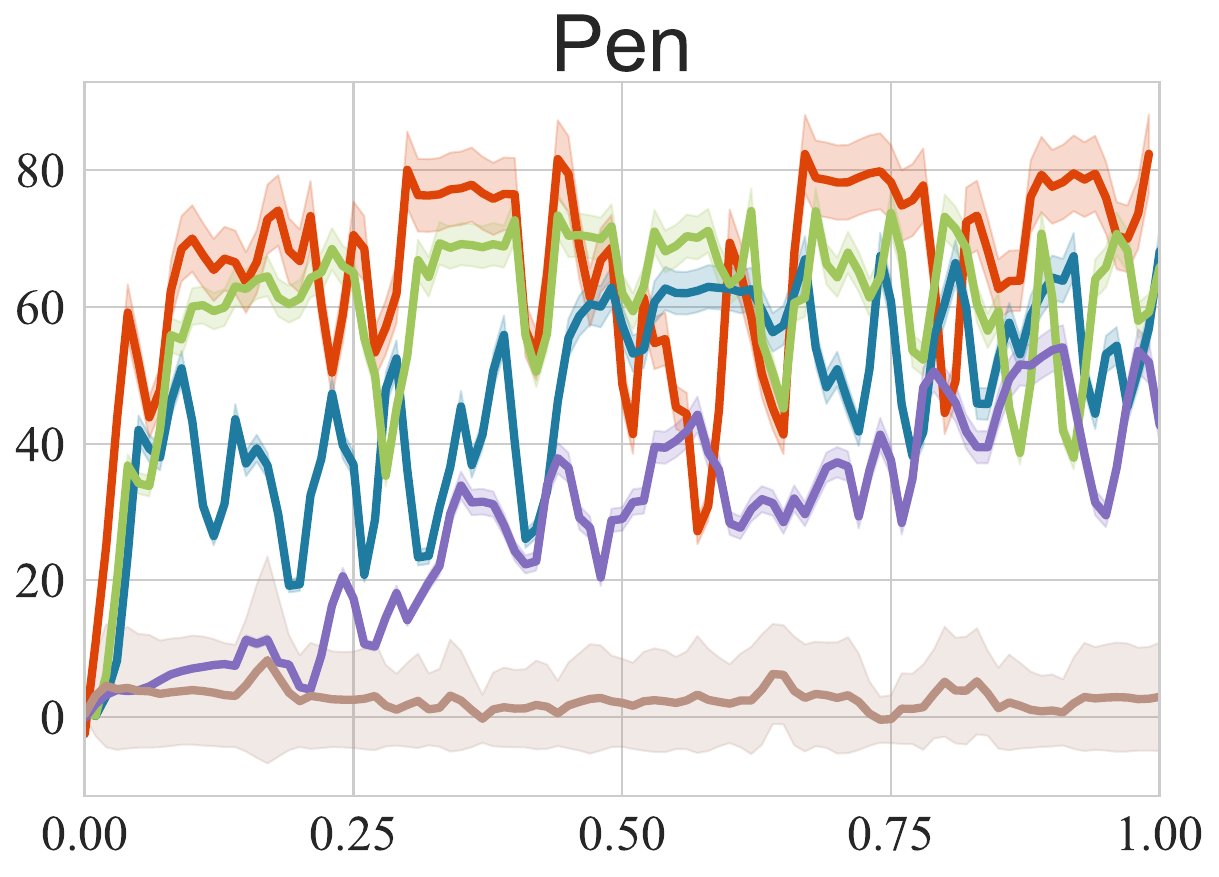}}
    \hspace{-4pt}
    \subfigure{
        \includegraphics[width=0.24\textwidth]{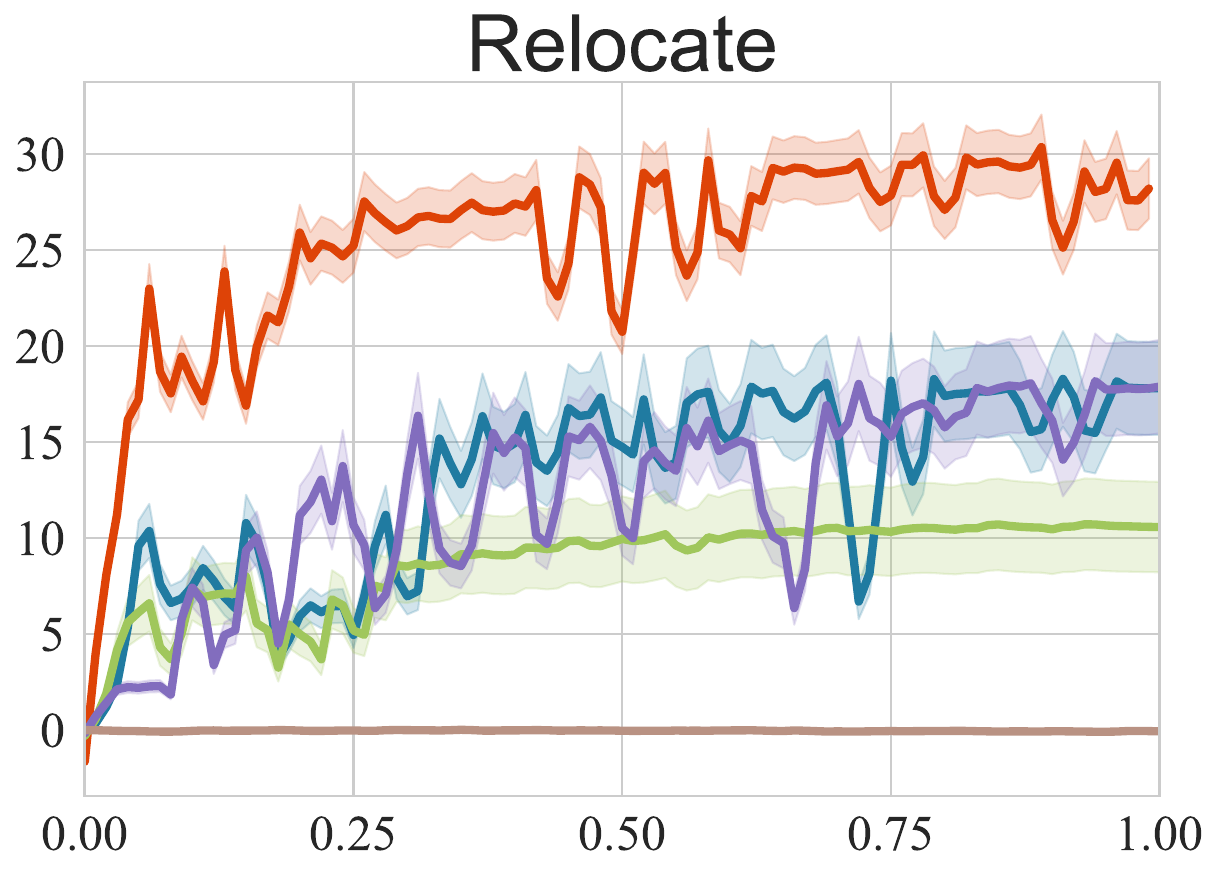}}
    \hspace{-4pt}
    \subfigure{
        \includegraphics[width=0.24\textwidth]{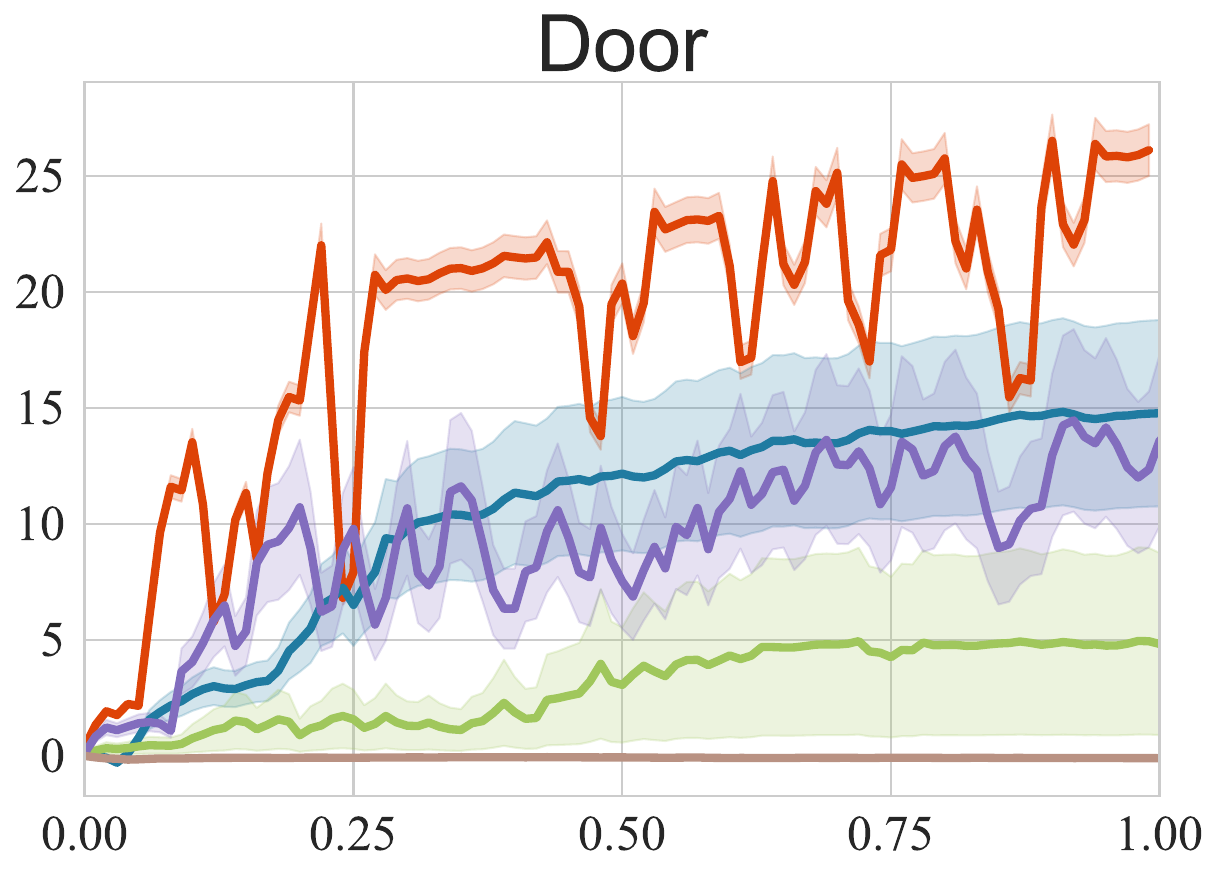}}
    
    \vspace{-10pt}
    \subfigure{
        \includegraphics[width=0.24\textwidth]{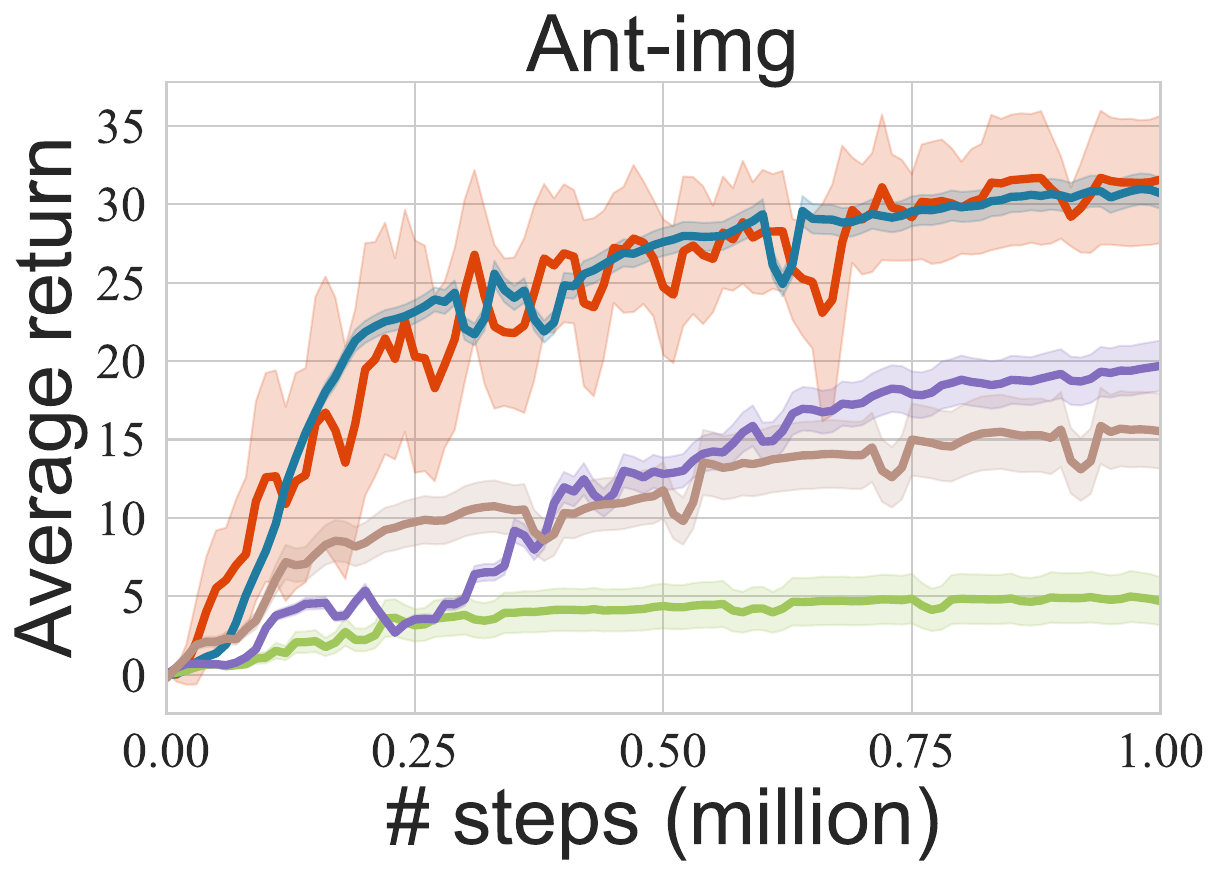}}
    \hspace{-4pt}
    \subfigure{
        \includegraphics[width=0.24\textwidth]{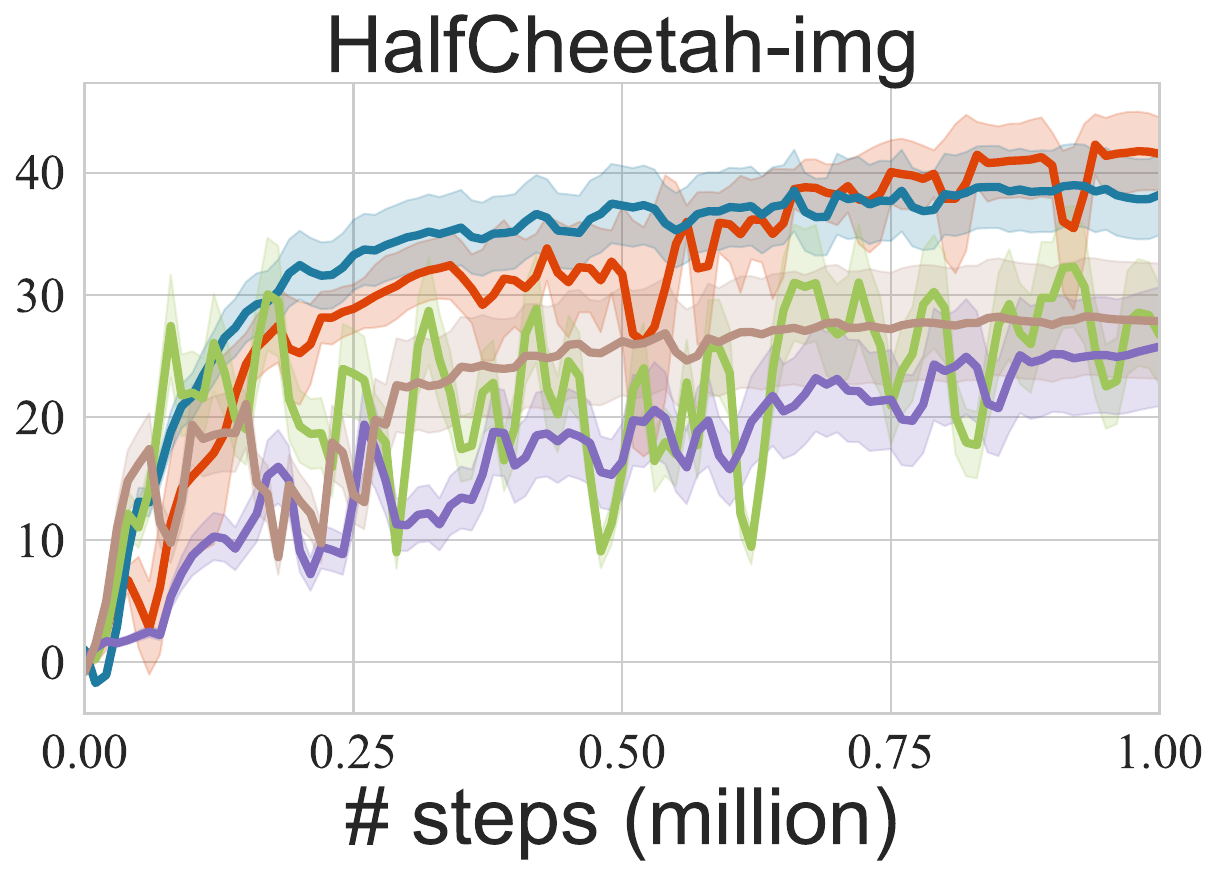}}
    \hspace{-4pt}
    \subfigure{
        \includegraphics[width=0.24\textwidth]{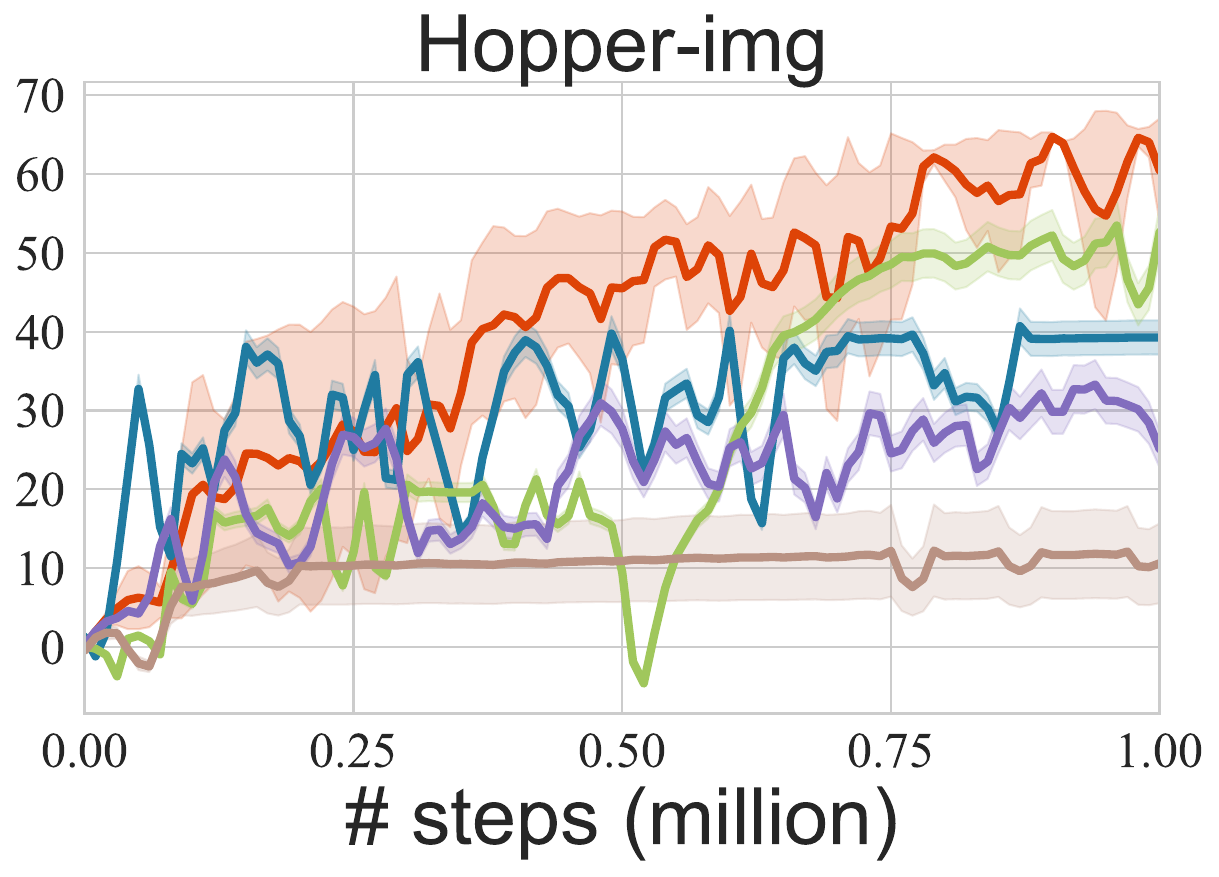}}
    \hspace{-4pt}
    \subfigure{
        \includegraphics[width=0.24\textwidth]{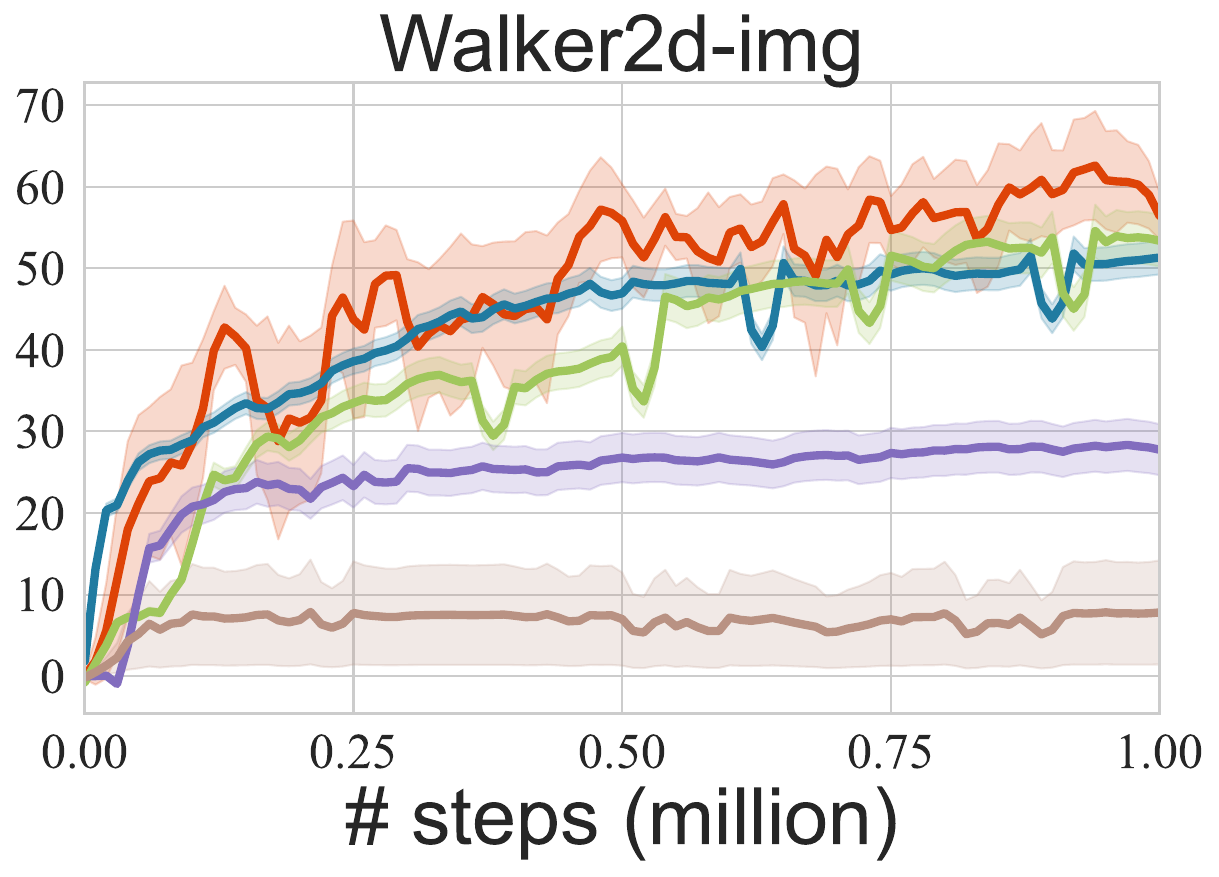}}
    \caption{Importance of $\alpha(s,a)$ and $\beta(s,a)$. `\texttt{ILID} w/o $\alpha$' refers to change the first term in Problem~(\ref{eq:ilid_objective}) to \texttt{BC}.} {`\texttt{ILID} w/o $\beta(s,a)$' refers to setting $\beta(s,a)\equiv1$. `\texttt{BC} over $\mathcal{D}_e\cup\mathcal{D}_s$' refers to running \texttt{BC} on the union of $\mathcal{D}_e$ and $\mathcal{D}_s$.}
    \label{fig:weight_all_1}
\end{figure}

\begin{figure}[ht]
    \centering
    \subfigure{
        \includegraphics[width=0.24\textwidth]{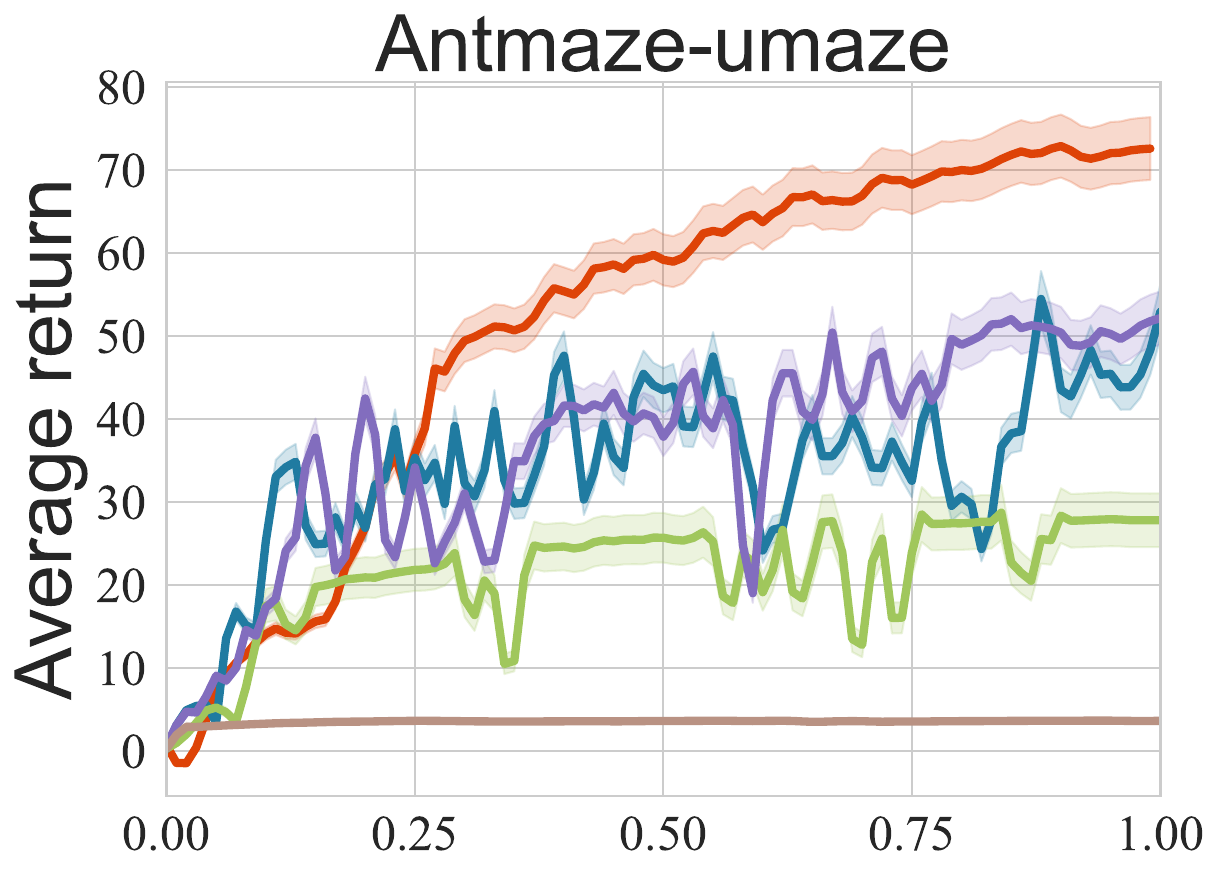}}
    \subfigure{
        \includegraphics[width=0.24\textwidth]{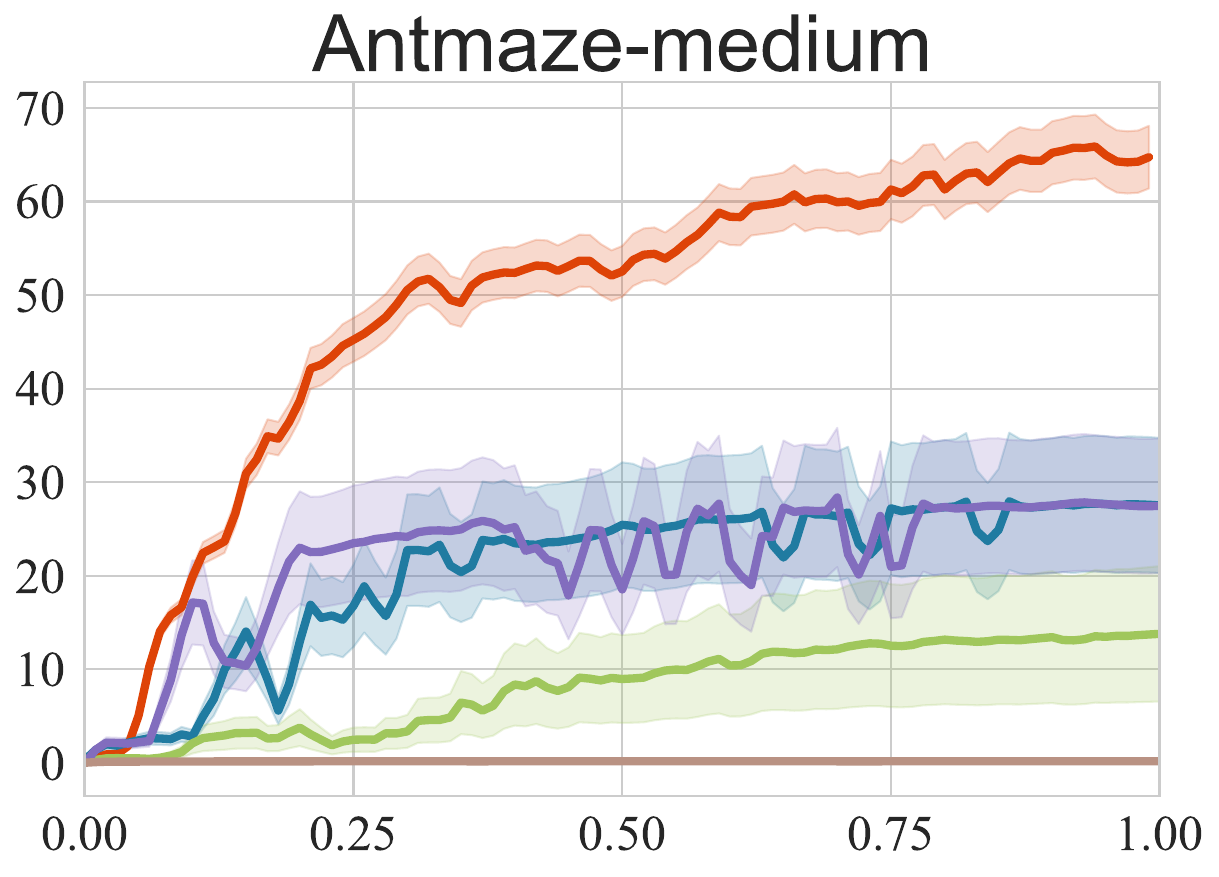}}
    \subfigure{
        \includegraphics[width=0.24\textwidth]{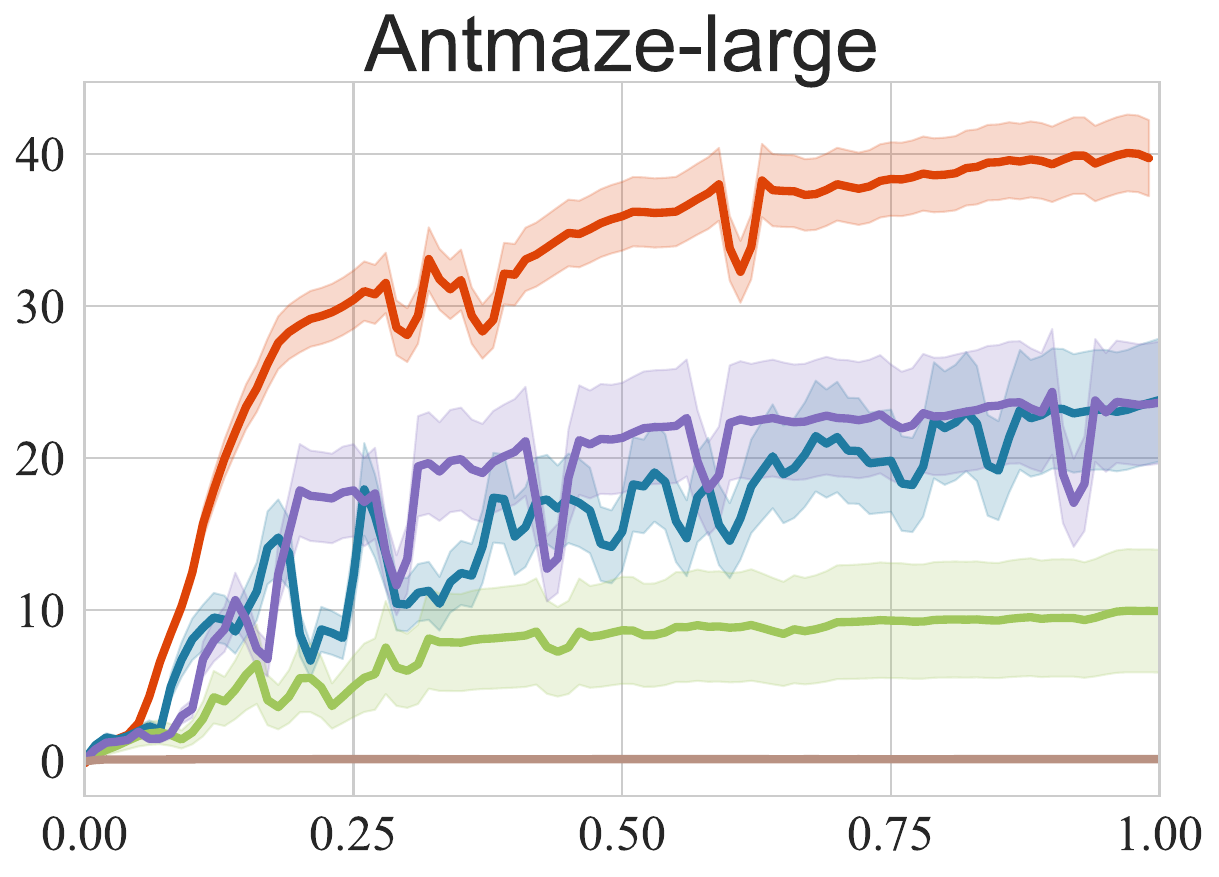}}
    
    \vspace{-10pt}
    \subfigure{
        \includegraphics[width=0.24\textwidth]{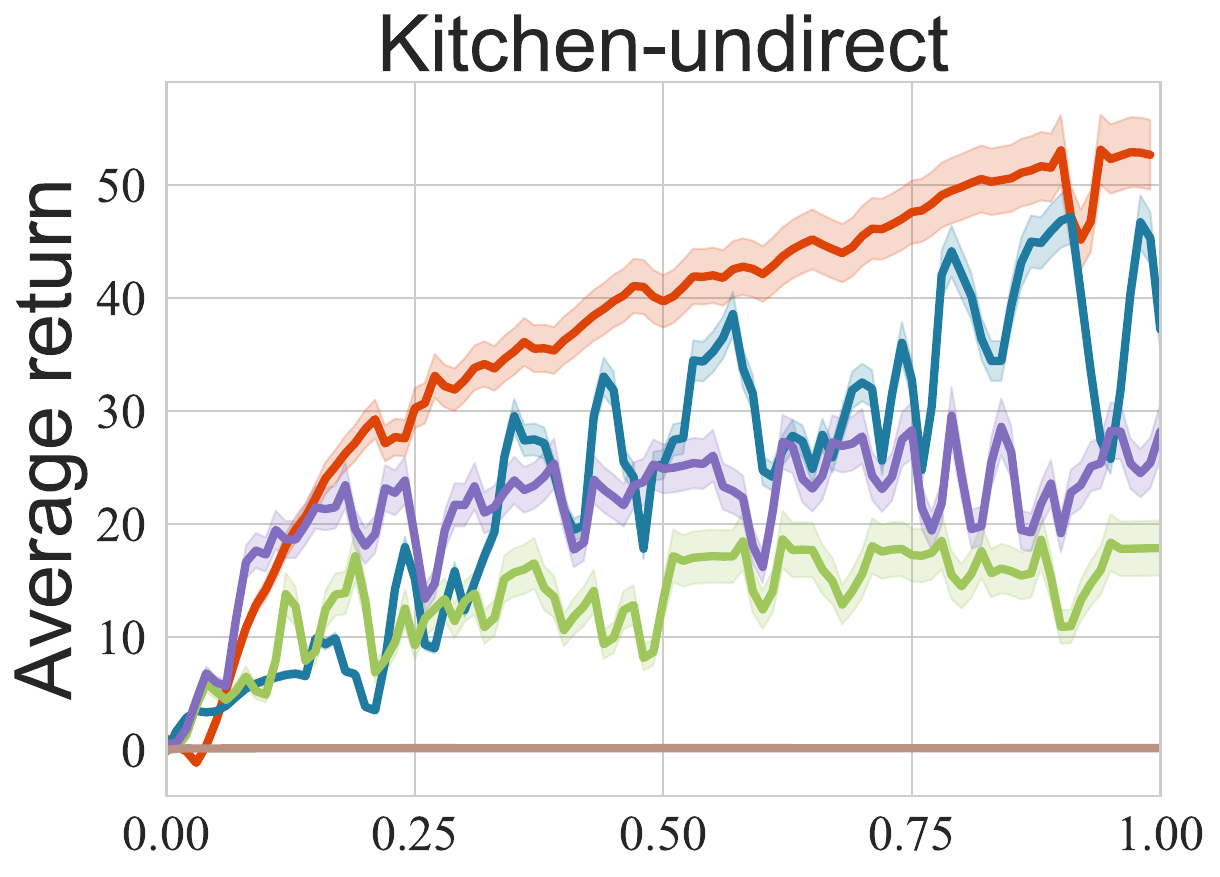}}
    \subfigure{
        \includegraphics[width=0.24\textwidth]{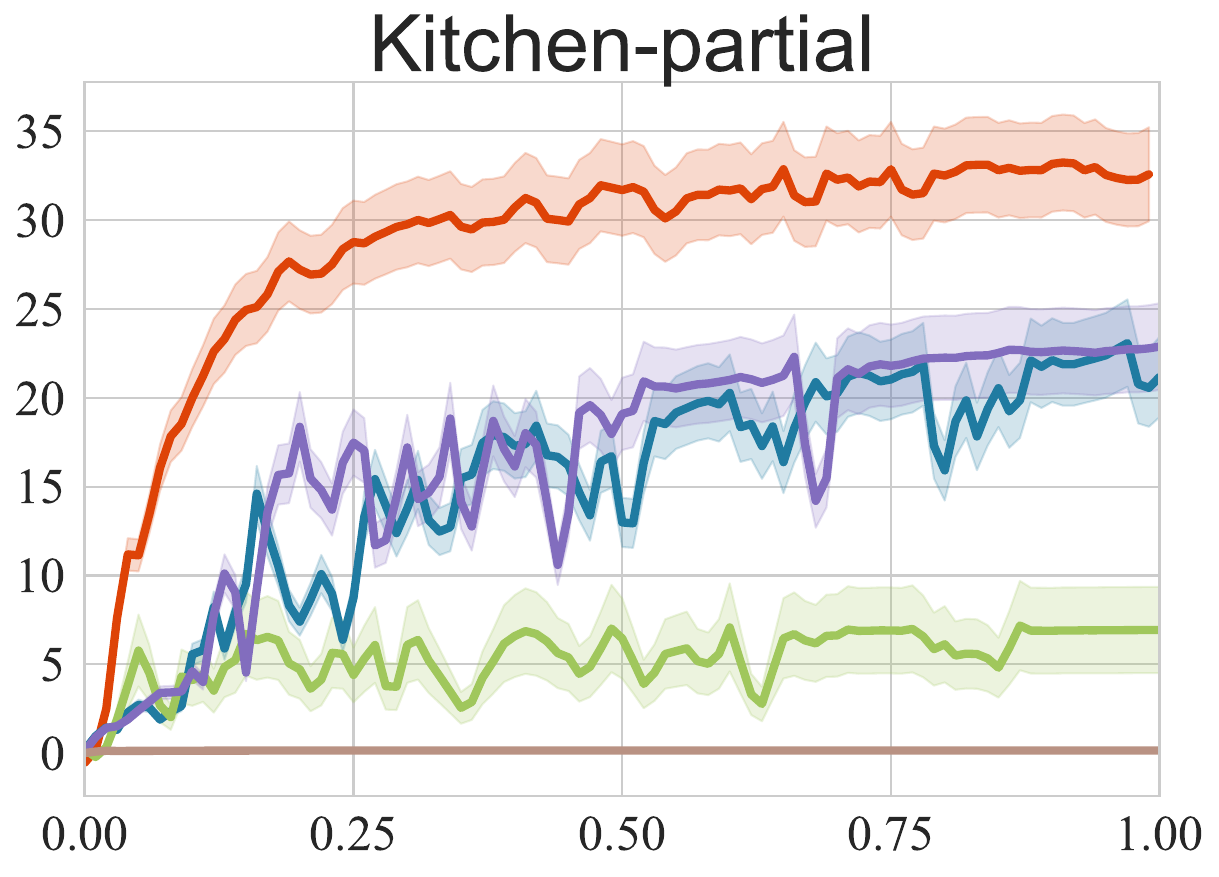}}
    \subfigure{
        \includegraphics[width=0.24\textwidth]{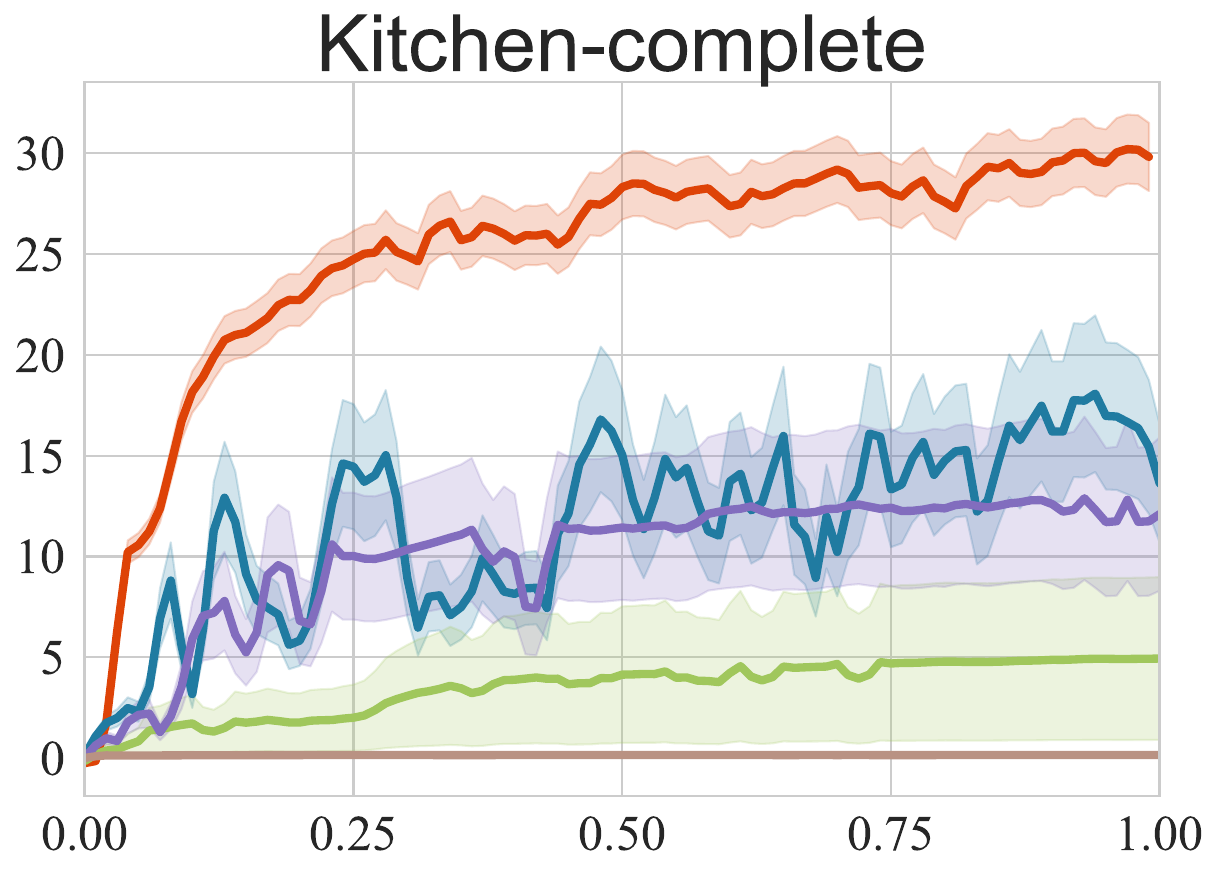}}

    \vspace{-10pt}
    \subfigure{
        \includegraphics[width=0.24\textwidth]{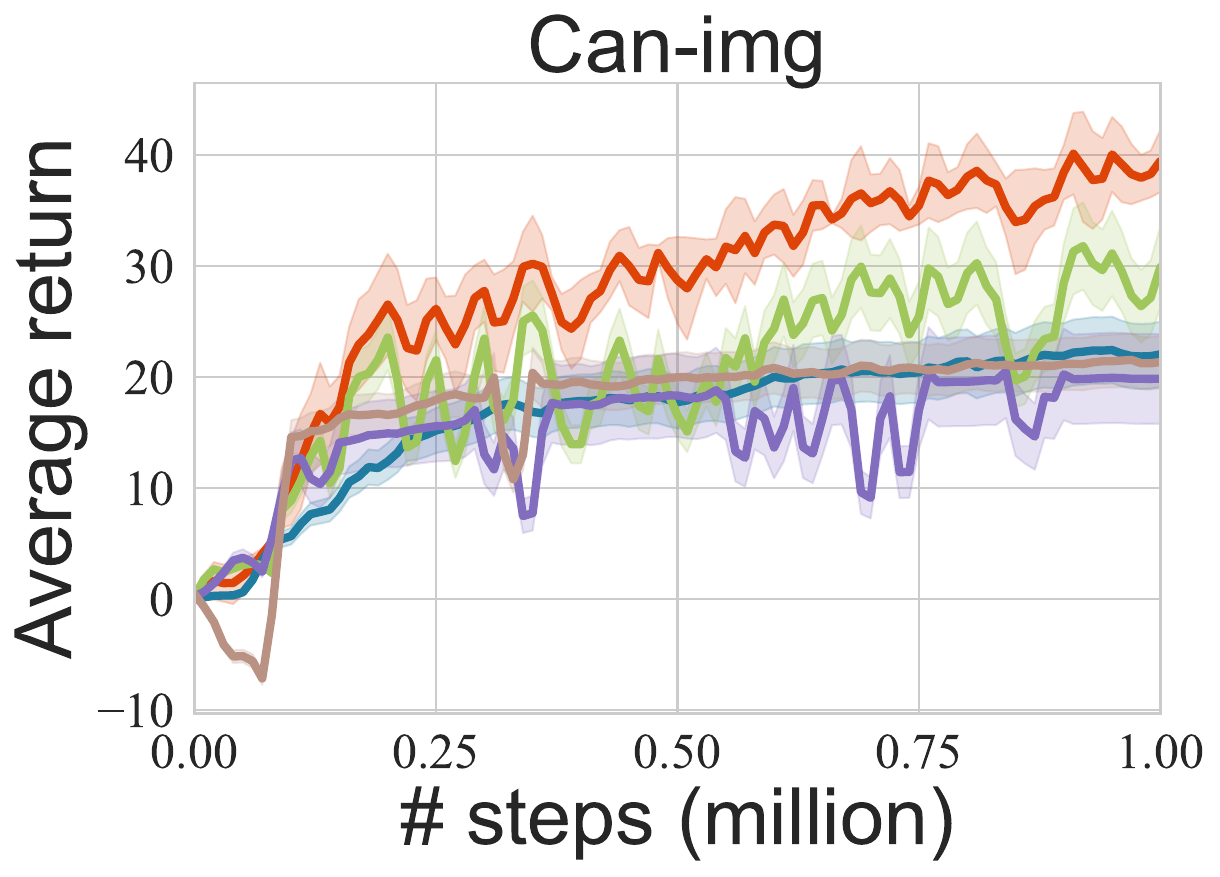}}
    \subfigure{
        \includegraphics[width=0.24\textwidth]{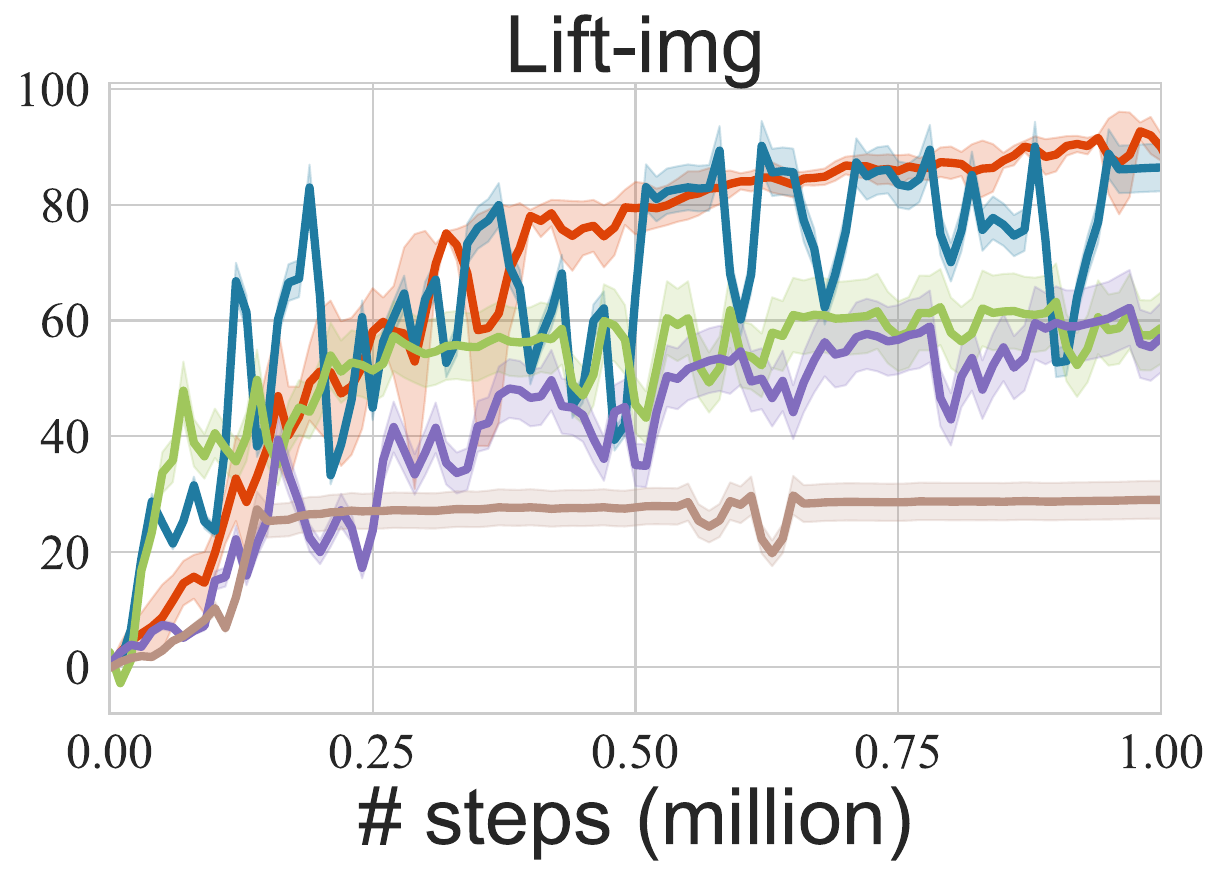}}
    \subfigure{
        \includegraphics[width=0.24\textwidth]{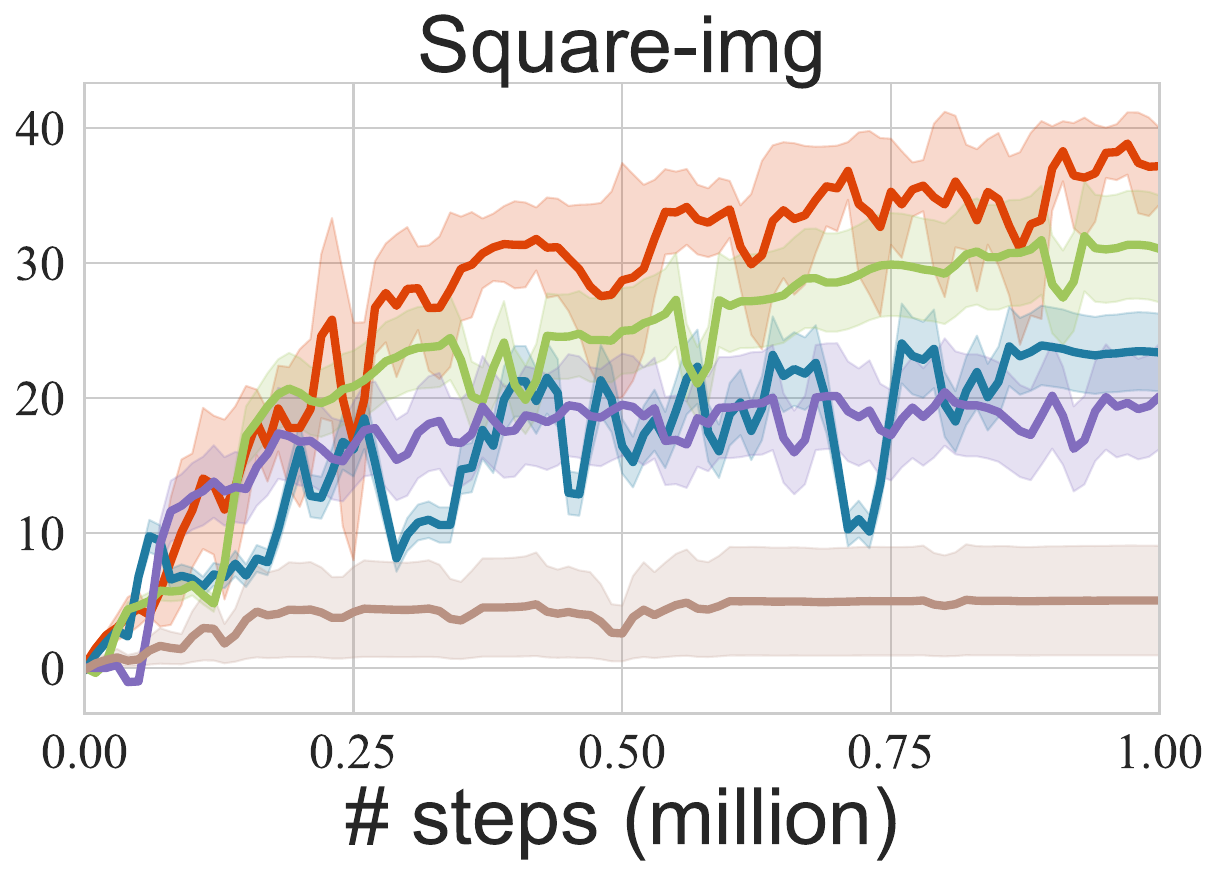}}
    \caption{Importance of $\alpha(s,a)$ and $\beta(s,a)$. `\texttt{ILID} w/o $\alpha$' refers to change the first term} 
    {in Problem~(\ref{eq:ilid_objective}) to \texttt{BC}. `\texttt{ILID} w/o $\beta(s,a)$' refers to setting $\beta(s,a)\equiv1$ in \texttt{ILID}.} 
    
    {`\texttt{BC} over $\mathcal{D}_e\cup\mathcal{D}_s$' refers to running \texttt{BC} on the union of $\mathcal{D}_e$ and $\mathcal{D}_s$.}
    \label{fig:weight_all_2}
\end{figure}

% Due to the cooperation training between the discriminator and policy, \texttt{DWBC} requires additional computation than \texttt{ILID}. CLARE is costly due to the effort to solve an intermediate offline RL problem.

\clearpage
\section{Detailed Proofs}

\subsection{Proof of \cref{thm:det_dyna_gap}}
\label{sec:thm_proof}

In this section, we provide the proof details for \cref{thm:det_dyna_gap}, We use $(s,a,\dots,(s'),(a'))\in\mathcal{D}$ to denote that dataset $\mathcal{D}$ contains sub-trajectory $(s,a,\dots,(s'),(a'))$. When clear from the context, we omit the subscript and use $\mathbb{E}[\cdot]$ instead of $\mathbb{E}_{\mathcal{D}_e,\mathcal{D}_s}[\cdot]$ for conciseness.

First, recalling the definition of $V^\pi$ in \cref{sec:preliminaries}, we can write
\begin{align}
    &\;V^{\pi_e} - V^{\tilde{\pi}}\nonumber\\
    =&\;\mathbb{E}_{s\sim\mu}\big[V^{\pi_e}(s) - V^{\tilde{\pi}}(s)\big]\tag{where $V^\pi(s)=\mathbb{E}_\pi[\sum^H_{h=1}R(s_h,a_h)\mid s_1 = s]$}\\
    =&\;\mathbb{E}_{s\sim\mu}\Big[\1(s\notin\mathcal{S}_1(\mathcal{D}_e))\cdot\big( V^{\pi_e}(s) - V^{\tilde{\pi}}(s)\big)\Big] + \mathbb{E}_{s\sim\mu}\Big[\1(s\in\mathcal{S}_1(\mathcal{D}_e))\cdot\big( V^{\pi_e}(s) - V^{\tilde{\pi}}(s)\big)\Big]\nonumber\\
    =&\;\mathbb{E}_{s\sim\mu}\Big[\1(s\notin\mathcal{S}_1(\mathcal{D}_e))\cdot\big( V^{\pi_e}(s) - V^{\tilde{\pi}}(s)\big)\Big]\tag{due to determinism of expert policy and transition dynamics, detailed below}\\
    =&\;\underbrace{\mathbb{E}_{s\sim\mu}\Big[\1(s\notin\mathcal{S}_1(\mathcal{D}_e))\cdot\1(s\notin\mathcal{S}_1({\mathcal{D}_s}))\cdot\big( V^{\pi_e}(s) - V^{\tilde{\pi}}(s)\big)\Big]}_{(a)}\nonumber\\
    \label{eq:det_val_dif_1}
    &+\underbrace{\mathbb{E}_{s\sim\mu}\Big[\1(s\notin\mathcal{S}_1(\mathcal{D}_e))\cdot\1(s\in\mathcal{S}_1({\mathcal{D}_s}))\cdot\big( V^{\pi_e}(s) - V^{\tilde{\pi}}(s)\big)\Big]}_{(b)}.
\end{align}
More specifically, the third equality holds because: the trajectories, started with the visited initial states, are fully covered in the expert demonstrations; and deterministic dynamics enables $\tilde\pi$ to fully recover the expert trajectories. 

Note that once the policy enters the
states out of training distribution, it may keep making mistakes and remain out-of-distribution for the remainder of the time steps. Hence, we can bound term $\mathbb{E}[(a)]$ as follows:
\begin{align}
    \mathbb{E}[(a)]&=\mathbb{E}\Big[\mathbb{E}_{s\sim\mu}\Big[\1(s\notin\mathcal{S}_1(\mathcal{D}_e))\cdot\1(s\notin\mathcal{S}_1({\mathcal{D}_s}))\cdot\big( V^{\pi_e}(s) - V^{\tilde{\pi}}(s)\big)\Big]\Big]\nonumber\\
    &\le H\mathbb{E}\Big[\mathbb{E}_{s\sim\mu}\Big[\1(s\notin\mathcal{S}_1(\mathcal{D}_e))\cdot\1(s\notin\mathcal{S}_1({\mathcal{D}_s}))\Big]\Big]\tag{due to $V^\pi(s)\le H$}\\
    &=H\mathbb{E}\Big[\mathbb{E}_{s\sim\mu}\Big[\1(s\notin\mathcal{S}_1(\mathcal{D}_e)\cup\mathcal{S}_1({\mathcal{D}_s}))\Big]\Big]\nonumber\\
    \label{eq:term_a_bound}
    &=H\epsilon_o
\end{align} 
where $\epsilon_o=\mathbb{E}[\mathbb{E}_{s_1\sim\mu}
[\1(s_1\notin\mathcal{S}_1(\mathcal{D}_e)\cup\mathcal{S}_1(\mathcal{D}_s))]]$ is the missing mass defined in \cref{thm:det_dyna_gap}.

Regarding term $(b)$, we can write
\begin{align}
	\mathbb{E}[(b)]=\;&\mathbb{E}_{s\sim\mu}\Big[\1(s\notin\mathcal{S}_1(\mathcal{D}_e))\cdot\1(s\in\mathcal{S}_1({\mathcal{D}_s}))\cdot\big( V^{\pi_e}(s) - V^{\tilde{\pi}}(s)\big)\Big]\nonumber\\
	% \le\;&\mathbb{E}_{s\sim\mu}\left[\1(s\in\mathcal{S}_1({\mathcal{D}_s}))\cdot\left( V^{\pi_e}(s) - V^{\tilde{\pi}}(s)\right)\right]\tag{due to $V^{\pi_e}(s) - V^{\tilde{\pi}}(s)\ge0$}\\
	\label{eq:term_b}
	=\;&\mathbb{E}\Big[\mathbb{E}_{s\sim\mu}\Big[\1(s\notin\mathcal{S}_1(\mathcal{D}_e))\cdot\1(s\in\mathcal{S}_1({\mathcal{D}_s}))\cdot V^{\pi_e}(s)\Big]\Big]\nonumber\\
	&-\underbrace{\mathbb{E}\Big[\mathbb{E}_{s\sim\mu}\Big[\1(s\notin\mathcal{S}_1(\mathcal{D}_e))\cdot\1(s\in\mathcal{S}_1({\mathcal{D}_s}))\cdot  V^{\tilde{\pi}}(s)\Big]\Big]}_{(c)}.
\end{align}
For the second term in the last equality of \cref{eq:term_b}, we have
\begin{align}
	(c)&=\mathbb{E}\Big[\mathbb{E}_{s\sim\mu}\Big[\1(s\notin\mathcal{S}_1(\mathcal{D}_e))\cdot\1(s\in\mathcal{S}_1({\mathcal{D}_s}))\cdot  V^{\tilde{\pi}}(s)\Big]\Big]\nonumber\\
	&=\mathbb{E}\Bigg[\mathbb{E}_{s\sim\mu}\Bigg[\1(s\notin\mathcal{S}_1(\mathcal{D}_e))\cdot\1(s\in\mathcal{S}_1({\mathcal{D}_s}))\cdot  \mathbb{E}_{\tilde{\pi}}\Bigg[\sum^H_{h=1}R(s_{h},a_{h})\mid s_1=s\Bigg]\Bigg]\Bigg]\tag{using the definition of $V^{\tilde{\pi}}(s)$ where $s_{h+1}=T(s_{h},a_{h})$ and $a_{h}\sim\tilde{\pi}(\cdot|s_{h})$}\nonumber\\
	&\ge\mathbb{E}\Bigg[\mathbb{E}_{s\sim\mu}\Bigg[\1(s\notin\mathcal{S}_1(\mathcal{D}_e))\cdot\1(s\in\mathcal{S}_1({\mathcal{D}_s}))\cdot  \mathbb{E}_{\tilde{\pi}}\Bigg[\sum^H_{h=2}R(s_{h},a_{h})\mid s_1=s\Bigg]\Bigg]\Bigg]\tag{omitting $R(s_1,a_1)$ and using $R(s_1,a_1)\ge0$}\\
	&=\mathbb{E}\bigg[\mathbb{E}_{s\sim\mu}\bigg[\1(s\notin\mathcal{S}_1(\mathcal{D}_e))\cdot\1(s\in\mathcal{S}_1({\mathcal{D}_s}))\cdot \mathbb{E}_{a\sim\tilde{\pi}(\cdot|s),s'\sim T(s,a)}\big[V'(s')\big]\bigg]\bigg]\tag{denoting $V'(s')\doteq \sum^H_{h=2}r(s_{h},a_{h})$ where $s_2=s'$, $s_{h+1}=T(s_{h},a_{h})$ and $a_{h}\sim\tilde{\pi}(\cdot|s_{h})$}\\
	&=\mathbb{E}\bigg[\mathbb{E}_{s\sim\mu}\bigg[\1(s\notin\mathcal{S}_1(\mathcal{D}_e))\cdot\1(s\in\mathcal{S}_1({\mathcal{D}_s}))\cdot  \mathbb{E}_{s'\sim{\mathcal{D}_s}(\cdot|s)}\big[V'(s')\big]\bigg]\bigg]\label{eq:term_c}
%	&=\mathbb{E}_{s\sim\mu}\bigg[\mathbb{E}_{\mathcal{D}_e}\bigg[\1(s\notin\mathcal{S}_1(\mathcal{D}_e))\cdot\mathbb{E}_{\mathcal{D}_s}\Big[\1(s\in\mathcal{S}_1({\mathcal{D}_s}))\cdot  \mathbb{E}_{s'\sim{\mathcal{D}_s}(\cdot|s)}\big[V'(s')\big]\Big]\bigg]\bigg]\nonumber\\
%	&=\mathbb{E}_{s\sim\mu}\bigg[\mathbb{E}_{\mathcal{D}_e}\bigg[\1(s\notin\mathcal{S}_1(\mathcal{D}_e))\cdot\mathbb{E}_{\mathcal{D}_s}\Big[\1(s\in\mathcal{S}_1({\mathcal{D}_s}))\cdot  \mathbb{E}_{s'\sim\mathcal{S}(\mathcal{D}_e)}\big[V'(s')\big]\Big]\bigg]\bigg]\label{eq:eq2}\\
%	&=\mathbb{E}_{s\sim\mu}\bigg[\mathbb{E}_{{\mathcal{D}_s}}\big[1-\1(s\notin\mathcal{S}_1({\mathcal{D}_s}))\big]\cdot\mathbb{E}_{\mathcal{D}_e}\Big[\1(s\notin\mathcal{S}_1(\mathcal{D}_e))\cdot  \mathbb{E}_{s'\sim\mathcal{S}(\mathcal{D}_e)}\big[V'(s')\big]\Big]\bigg].
	%    =\;&(1-\delta)\cdot\mathbb{E}_{s\sim\mu}\left[\mathbb{E}_{\mathcal{D}_e}\left[\1(s\notin\mathcal{S}_1(\mathcal{D}_e))\cdot\mathbb{E}_{s'\sim\mathcal{S}(\mathcal{D}_e)}V'(s')\right]\right]\nonumber\\
	%    =\;&(1-\delta)\cdot\mathbb{E}_{\mathcal{D}_e}\left[\mathbb{E}_{s'\sim\mathcal{S}(\mathcal{D}_e)}\left[V'(s')\right]\cdot\mathbb{E}_{s\sim\mu}\left[\1(s\notin\mathcal{S}_1(\mathcal{D}_e))\right]\right]\nonumber\\
	%    =\;&\epsilon(1-\delta)\cdot\mathbb{E}_{s'\sim\rho^{\pi_e}}\left[V'(s')\right].
\end{align}
where $\mathcal{D}_s(s'|s) \doteq \sum_a n((s,a,s')\in{\mathcal{D}_s})/n(s\in\mathcal{S}_1({\mathcal{D}_s}))$, and the last equality is obtained by
\begin{align}
	&\1(s\notin\mathcal{S}_1(\mathcal{D}_e))\1(s\in\mathcal{S}_1({\mathcal{D}_s}))\cdot \mathbb{E}_{a\sim\tilde{\pi}(\cdot|s),s'\sim T(s,a)}\big[V'(s')\big]\nonumber\\
	=\;&\1(s\notin\mathcal{S}_1(\mathcal{D}_e))\1(s\in\mathcal{S}_1({\mathcal{D}_s}))\cdot\sum_a  \frac{n((s,a)\in{\mathcal{D}_s})}{n(s\in\mathcal{S}_1({\mathcal{D}_s}))}V'(T(s,a))\tag{from the definition of $\tilde\pi$ in \cref{eq:ideal_policy} and the determinism of $T$}\nonumber\\
	=\;&\1(s\notin\mathcal{S}_1(\mathcal{D}_e))\1(s\in\mathcal{S}_1({\mathcal{D}_s}))\cdot\sum_a  \frac{n((s,a,T(s,a))\in{\mathcal{D}_s})}{n(s\in\mathcal{S}_1({\mathcal{D}_s}))}V'(T(s,a))\tag{due to $n((s,a)\in{\mathcal{D}_s})=n((s,a,T(s,a))\in{\mathcal{D}_s})$}\nonumber\\
	=\;&\1(s\notin\mathcal{S}_1(\mathcal{D}_e))\1(s\in\mathcal{S}_1({\mathcal{D}_s}))\cdot\sum_{s',a} \frac{n((s,a,s')\in{\mathcal{D}_s})}{n(s\in\mathcal{S}_1({\mathcal{D}_s}))}V'(s')\tag{due to the fact that $n((s,a,s')\in{\mathcal{D}_s})=0$ if $s'\neq T(s,a)$}\nonumber\\
%	=\;&\1(s\notin\mathcal{S}_1(\mathcal{D}_e))\1(s\in\mathcal{S}_1({\mathcal{D}_s}))\cdot\sum_{s'} \frac{ n((s,\cdot,s')\in{\mathcal{D}_s})}{n(s\in\mathcal{S}_1({\mathcal{D}_s}))}V'(s')\\
	\label{eq:eq3}
	=\;&\1(s\notin\mathcal{S}_1(\mathcal{D}_e))\1(s\in\mathcal{S}_1({\mathcal{D}_s}))\cdot\mathbb{E}_{s'\sim\mathcal{D}_s(\cdot|s)}\big[V'(s')\big].
\end{align}
Substituting \cref{eq:term_c} to \cref{eq:term_b} yields
\begin{align}
	\mathbb{E}[(b)]&\le \mathbb{E}\Big[\mathbb{E}_{s\sim\mu}\Big[\1(s\notin\mathcal{S}_1(\mathcal{D}_e))\cdot\1(s\in\mathcal{S}_1({\mathcal{D}_s}))\cdot \mathbb{E}_{s'\sim{\mathcal{D}_s}(\cdot|s)}\big[V^{\pi_e}(s)-V'(s')\big]\Big]\Big]\nonumber\\
	&\le (\delta+1)\mathbb{E}\Big[\mathbb{E}_{s\sim\mu}\Big[\1(s\notin\mathcal{S}_1(\mathcal{D}_e))\cdot\1(s\in\mathcal{S}_1({\mathcal{D}_s}))\Big]\Big]\label{eq:eq1}\\
	&= (\delta+1)\mathbb{E}_{s\sim\mu}\Big[\mathbb{E}\Big[\1(s\notin\mathcal{S}_1(\mathcal{D}_e))\cdot\1(s\in\mathcal{S}_1({\mathcal{D}_s}))\Big]\Big]\nonumber\\
	&= (\delta+1)\mathbb{E}_{s\sim\mu}\Big[\mathbb{E}\big[\1(s\notin\mathcal{S}_1(\mathcal{D}_e))\big]\cdot\mathbb{E}\big[\1(s\in\mathcal{S}_1({\mathcal{D}_s}))\big]\Big]\tag{from the independence of $\mathcal{D}_e$ and $\mathcal{D}_s$}\\
	&\le (\delta+1)\sqrt{\mathbb{E}_{s\sim\mu}\big[\mathbb{E}\big[\1(s\notin\mathcal{S}_1(\mathcal{D}_e))\big]^2\big]\cdot\mathbb{E}_{s\sim\mu}\big[\mathbb{E}\big[\1(s\in\mathcal{S}_1(\mathcal{D}_s))\big]^2\big]}\tag{from the Cauchy-Schwarz inequality $\mathbb{E}[XY]\le\sqrt{\mathbb{E}[X^2]\mathbb{E}[Y^2]}$}\\
	&\le (\delta+1)\sqrt{\mathbb{E}_{s\sim\mu}\big[\mathbb{E}\big[\1(s\notin\mathcal{S}_1(\mathcal{D}_e))^2\big]\big]\cdot\mathbb{E}_{s\sim\mu}\big[\mathbb{E}\big[\1(s\in\mathcal{S}_1(\mathcal{D}_s))^2\big]\big]}\tag{from the fact $\mathbb{E}[X]^2\le\mathbb{E}[X^2]$}\\
	&= (\delta+1)\sqrt{\mathbb{E}_{s\sim\mu}\big[\mathbb{E}\big[\1(s\notin\mathcal{S}_1(\mathcal{D}_e))\big]\big]\cdot\mathbb{E}_{s\sim\mu}\big[\mathbb{E}\big[\1(s\in\mathcal{S}_1(\mathcal{D}_s))\big]\big]}\tag{from the fact that $\1(s\notin\mathcal{S}_1(\mathcal{D}))^2 = \1(s\notin\mathcal{S}_1(\mathcal{D}))$}\\
	&= (\delta+1)\sqrt{\mathbb{E}_{s\sim\mu}\big[\mathbb{E}\big[\1(s\notin\mathcal{S}_1(\mathcal{D}_e))\big]\big]\cdot\big(1-\mathbb{E}_{s\sim\mu}\big[\mathbb{E}\big[\1(s\notin\mathcal{S}_1(\mathcal{D}_s))\big]\big]\big)}\nonumber\\
	\label{eq:term_b_bound}
	&=(\delta + 1) \sqrt{\epsilon_e(1-\epsilon_s)}.
\end{align}
Regarding \cref{eq:eq1}, due to $s'\in\mathcal{S}_1(\mathcal{D}_e)$ (see the definition of $\mathcal{D}_s$) and the definition of $\tilde\pi$ (which takes expert actions at given expert states), the sub-trajectory started from $s'$ induced by $\tilde\pi$ follows the corresponding expert trajectory in $\mathcal{D}_e$. Based on the definition of $\delta$, 
\cref{eq:eq1} can be derived by
\begin{align}
	V^{\pi_e}(s) - V'(s') \le V^{\pi_e}(s) - (V^{\pi_e}(s') - 1) \le \delta + 1,
\end{align}
where we use $R(s,a)\le1$ and the definition of $V'(s')$ which sums up over just $H-1$ steps. Combining \cref{eq:term_a_bound,eq:term_b_bound}, we have
\begin{align}
	V^{\pi_e} - \mathbb{E}[V^{\tilde{\pi}}] \le H\epsilon_o + (\delta + 1) \sqrt{\epsilon_e(1-\epsilon_s)},
\end{align}
thereby yielding the result.

\subsection{Proof of \cref{coro:sample_complexity}}
\label{sec:coro_proof}

Analogouly to \cref{eq:det_val_dif_1,eq:term_a_bound}, $V^{\pi_e} - \mathbb{E}[V^{\tilde{\pi}}]$ is also bounded by
\begin{align}
	&\;V^{\pi_e} - \mathbb{E}[V^{\tilde{\pi}}]\nonumber\\
	=&\;\mathbb{E}\big[\mathbb{E}_{s\sim\mu}\big[V^{\pi_e}(s) - V^{\tilde{\pi}}(s)\big]\big]\nonumber\\
	=&\;\mathbb{E}\Big[\mathbb{E}_{s\sim\mu}\Big[\1(s\notin\mathcal{S}_1(\mathcal{D}_e))\cdot\big( V^{\pi_e}(s) - V^{\tilde{\pi}}(s)\big)\Big]\Big] + \mathbb{E}\Big[\mathbb{E}_{s\sim\mu}\Big[\1(s\in\mathcal{S}_1(\mathcal{D}_e))\cdot\big( V^{\pi_e}(s) - V^{\tilde{\pi}}(s)\big)\Big]\Big]\nonumber\\
	=&\;\mathbb{E}\Big[\mathbb{E}_{s\sim\mu}\Big[\1(s\notin\mathcal{S}_1(\mathcal{D}_e))\cdot\big( V^{\pi_e}(s) - V^{\tilde{\pi}}(s)\big)\Big]\Big]\nonumber\\
	\le&\;H\mathbb{E}\big[\mathbb{E}_{s\sim\mu}\big[\1(s\notin\mathcal{S}_1(\mathcal{D}_e))\big]\big].
	\label{eq:eq2}
\end{align}
Invoking \citet[Theorem 2]{xu2021generalization}, we can write
\begin{align}
	\mathbb{E}\big[\mathbb{E}_{s\sim\mu}\big[\1(s\notin\mathcal{S}_1(\mathcal{D}_e))\big]\big]&=\mathbb{E}_{s\sim\mu}\big[\mathbb{E}\big[\1(s\notin\mathcal{S}_1(\mathcal{D}_e))\big]\big]\nonumber\\
	&=\sum\nolimits_s \mu(s)\Pr(\1(s\notin\mathcal{S}_1(\mathcal{D}_e)))\nonumber\\
	&=\mu(s)(1-\mu(s))^{n_e}\nonumber\\
	&\le |\mathcal{S}|\max_{x\in[0,1]} x(1-x)^{n_e}\nonumber\\
	&\le\frac{|\mathcal{S}|}{en_e},
	\label{eq:eq4}
\end{align}
where $e$ is Euler's number, and the last inequality is obtained via solving the maximization. Specifically, denote $f(x)=x(1-x)^{n_e}$ and take its derivative to zero, yielding
\begin{align}
	f'(x) = (1-x)^{n_e-1}(1-(n_e+1)x) = 0\quad\Rightarrow\quad x^* = \frac{1}{n_e+1}.
\end{align}
Therefore, the following holds:
\begin{align}
	\max_{x\in[0,1]} x(1-x)^{n_e}=\frac{1}{n_e+1}\bigg(1-\frac{1}{n_e+1}\bigg)^{n_e}=\frac{1}{n_e}\bigg(1-\frac{1}{n_e+1}\bigg)^{n_e+1}\le\frac{1}{en_e}.
\end{align}
Substituting \cref{eq:eq4} in \cref{eq:eq2}, we obtain
\begin{align}
	V^{\pi_e} - \mathbb{E}[V^{\tilde{\pi}}] \le \frac{|\mathcal{S}|H}{en_e}.
	\label{eq:eq6}
\end{align}
Similarly, from \cref{thm:det_dyna_gap}, we have
\begin{align}
	V^{\pi_e} - \mathbb{E}[V^{\tilde{\pi}}] \le H\epsilon_o +  (\delta + 1) \sqrt{\epsilon_e(1-\epsilon_s)}\le  \frac{\vert\mathcal{S}\vert H}{e(n_e+n_s)} +  (\delta + 1) \sqrt{\frac{|\mathcal{S}|}{en_e}}.
	\label{eq:eq5}
\end{align}
Combining \cref{eq:eq6,eq:eq5}, we can write
\begin{align}
	V^{\pi_e} - \mathbb{E}[V^{\tilde{\pi}}] \le \min\Bigg\{\frac{|\mathcal{S}|H}{en_e},\frac{\vert\mathcal{S}\vert H}{e(n_e+n_s)} +  (\delta + 1) \sqrt{\frac{|\mathcal{S}|}{en_e}}\Bigg\}
\end{align}
Taking $n_s$ to infinity, $V^{\pi_e} - \mathbb{E}[V^{\tilde{\pi}}]\le\min\{(\delta+1)\sqrt{|\mathcal{S}|/(en_e)},|\mathcal{S}|H/(en_e)\}$. Thus, with a sufficiently large $n_s$, to obtain an $\varepsilon$-optimal policy, $\tilde{\pi}$ requires at most $\mathcal{O}(\min\{|\mathcal{S}|/\varepsilon^2,|\mathcal{S}|H/\varepsilon\})$ expert trajectories.

\end{document}

%% file: math_commands.tex
\usepackage{amsmath,amsfonts,bm}

\def\eqref#1{equation~\ref{#1}}
\def\1{\mathbbm{1}}

\DeclareMathAlphabet{\mathsfit}{\encodingdefault}{\sfdefault}{m}{sl}
\SetMathAlphabet{\mathsfit}{bold}{\encodingdefault}{\sfdefault}{bx}{n}

%% file: main.bbl
\begin{thebibliography}{40}
\providecommand{\natexlab}[1]{#1}
\providecommand{\url}[1]{\texttt{#1}}
\expandafter\ifx\csname urlstyle\endcsname\relax
  \providecommand{\doi}[1]{doi: #1}\else
  \providecommand{\doi}{doi: \begingroup \urlstyle{rm}\Url}\fi

\bibitem[Bojarski et~al.(2016)Bojarski, Del~Testa, Dworakowski, Firner, Flepp, Goyal, Jackel, Monfort, Muller, Zhang, et~al.]{bojarski2016end}
Bojarski, M., Del~Testa, D., Dworakowski, D., Firner, B., Flepp, B., Goyal, P., Jackel, L.~D., Monfort, M., Muller, U., Zhang, J., et~al.
\newblock End to end learning for self-driving cars.
\newblock \emph{arXiv preprint arXiv:1604.07316}, 2016.

\bibitem[Chan \& van~der Schaar(2021)Chan and van~der Schaar]{chan2021scalable}
Chan, A.~J. and van~der Schaar, M.
\newblock Scalable bayesian inverse reinforcement learning.
\newblock In \emph{International Conference on Learning Representations}, 2021.

\bibitem[Chang et~al.(2021)Chang, Uehara, Sreenivas, Kidambi, and Sun]{chang2022mitigating}
Chang, J., Uehara, M., Sreenivas, D., Kidambi, R., and Sun, W.
\newblock Mitigating covariate shift in imitation learning via offline data with partial coverage.
\newblock In \emph{Advances in Neural Information Processing Systems}, volume~34, pp.\  965--979. Curran Associates, 2021.

\bibitem[Cideron et~al.(2023)Cideron, Tabanpour, Curi, Girgin, Hussenot, Dulac-Arnold, Geist, Pietquin, and Dadashi]{cideron2023get}
Cideron, G., Tabanpour, B., Curi, S., Girgin, S., Hussenot, L., Dulac-Arnold, G., Geist, M., Pietquin, O., and Dadashi, R.
\newblock Get back here: Robust imitation by return-to-distribution planning.
\newblock \emph{arXiv preprint arXiv:2305.01400}, 2023.

\bibitem[Florence et~al.(2022)Florence, Lynch, Zeng, Ramirez, Wahid, Downs, Wong, Lee, Mordatch, and Tompson]{florence2022implicit}
Florence, P., Lynch, C., Zeng, A., Ramirez, O.~A., Wahid, A., Downs, L., Wong, A., Lee, J., Mordatch, I., and Tompson, J.
\newblock Implicit behavioral cloning.
\newblock In \emph{Proceedings of the 5th Conference on Robot Learning}, volume 164, pp.\  158--168. PMLR, 2022.

\bibitem[Fu et~al.(2020)Fu, Kumar, Nachum, Tucker, and Levine]{fu2020d4rl}
Fu, J., Kumar, A., Nachum, O., Tucker, G., and Levine, S.
\newblock {D4RL}: Datasets for deep data-driven reinforcement learning.
\newblock \emph{arXiv preprint arXiv:2004.07219}, 2020.

\bibitem[Garg et~al.(2021)Garg, Chakraborty, Cundy, Song, and Ermon]{garg2021iq}
Garg, D., Chakraborty, S., Cundy, C., Song, J., and Ermon, S.
\newblock {IQ-Learn}: Inverse soft-q learning for imitation.
\newblock In \emph{Advances in Neural Information Processing Systems}, volume~34, pp.\  4028--4039. Curran Associates, 2021.

\bibitem[Goodfellow et~al.(2014)Goodfellow, Pouget-Abadie, Mirza, Xu, Warde-Farley, Ozair, Courville, and Bengio]{goodfello2016generative}
Goodfellow, I., Pouget-Abadie, J., Mirza, M., Xu, B., Warde-Farley, D., Ozair, S., Courville, A., and Bengio, Y.
\newblock Generative adversarial nets.
\newblock In \emph{Advances in Neural Information Processing Systems}, volume~27, pp.\  2672--2680. Curran Associates, 2014.

\bibitem[Gupta et~al.(2019)Gupta, Kumar, Lynch, Levine, and Hausman]{gupta2019relay}
Gupta, A., Kumar, V., Lynch, C., Levine, S., and Hausman, K.
\newblock Relay policy learning: Solving long-horizon tasks via imitation and reinforcement learning.
\newblock \emph{arXiv preprint arXiv:1910.11956}, 2019.

\bibitem[Haarnoja et~al.(2018)Haarnoja, Zhou, Abbeel, and Levine]{haarnoja18soft}
Haarnoja, T., Zhou, A., Abbeel, P., and Levine, S.
\newblock Soft actor-critic: Off-policy maximum entropy deep reinforcement learning with a stochastic actor.
\newblock In \emph{Proceedings of the 35th International Conference on Machine Learning}, volume~80, pp.\  1861--1870. PMLR, 2018.

\bibitem[Herman et~al.(2016)Herman, Gindele, Wagner, Schmitt, and Burgard]{herman2016inverse}
Herman, M., Gindele, T., Wagner, J., Schmitt, F., and Burgard, W.
\newblock Inverse reinforcement learning with simultaneous estimation of rewards and dynamics.
\newblock In \emph{Proceedings of the 19th International Conference on Artificial Intelligence and Statistics}, volume~51, pp.\  102--110. PMLR, 2016.

\bibitem[Jarrett et~al.(2020)Jarrett, Bica, and van~der Schaar]{jarrett2020strictly}
Jarrett, D., Bica, I., and van~der Schaar, M.
\newblock Strictly batch imitation learning by energy-based distribution matching.
\newblock In \emph{Advances in Neural Information Processing Systems}, volume~33, pp.\  7354--7365. Curran Associates, 2020.

\bibitem[Kim et~al.(2022)Kim, Seo, Lee, Jeon, Hwang, Yang, and Kim]{kim2022demodice}
Kim, G.-H., Seo, S., Lee, J., Jeon, W., Hwang, H., Yang, H., and Kim, K.-E.
\newblock {DemoDICE}: Offline imitation learning with supplementary imperfect demonstrations.
\newblock In \emph{International Conference on Learning Representations}, 2022.

\bibitem[Klein et~al.(2012{\natexlab{a}})Klein, Geist, and Pietquin]{klein2011batch}
Klein, E., Geist, M., and Pietquin, O.
\newblock Batch, off-policy and model-free apprenticeship learning.
\newblock In \emph{Recent Advances in Reinforcement Learning}, pp.\  285--296. Springer Berlin Heidelberg, 2012{\natexlab{a}}.

\bibitem[Klein et~al.(2012{\natexlab{b}})Klein, Geist, Piot, and Pietquin]{klein2012inverse}
Klein, E., Geist, M., Piot, B., and Pietquin, O.
\newblock Inverse reinforcement learning through structured classification.
\newblock In \emph{Advances in Neural Information Processing Systems}, volume~25, pp.\  1007--1015. Curran Associates, 2012{\natexlab{b}}.

\bibitem[Kostrikov et~al.(2020)Kostrikov, Nachum, and Tompson]{kostrikov2020imitation}
Kostrikov, I., Nachum, O., and Tompson, J.
\newblock Imitation learning via off-policy distribution matching.
\newblock In \emph{International Conference on Learning Representations}, 2020.

\bibitem[Lee et~al.(2019)Lee, Srinivasan, and Doshi-Velez]{lee2019truly}
Lee, D., Srinivasan, S., and Doshi-Velez, F.
\newblock Truly batch apprenticeship learning with deep successor features.
\newblock In \emph{Proceedings of the 28th International Joint Conference on Artificial Intelligence}, pp.\  5909--5915, 2019.

\bibitem[Lee et~al.(2021)Lee, Jeon, Lee, Pineau, and Kim]{lee2021optidice}
Lee, J., Jeon, W., Lee, B., Pineau, J., and Kim, K.-E.
\newblock {OptiDICE}: Offline policy optimization via stationary distribution correction estimation.
\newblock In \emph{Proceedings of the 38th International Conference on Machine Learning}, volume 139, pp.\  6120--6130. PMLR, 2021.

\bibitem[Li et~al.(2023)Li, Xu, Qin, Yu, and Luo]{li2023imitation}
Li, Z., Xu, T., Qin, Z., Yu, Y., and Luo, Z.-Q.
\newblock Imitation learning from imperfection: Theoretical justifications and algorithms.
\newblock In \emph{The 37th Conference on Neural Information Processing Systems}, 2023.

\bibitem[Mandlekar et~al.(2021)Mandlekar, Xu, Wong, Nasiriany, Wang, Kulkarni, Fei-Fei, Savarese, Zhu, and Mart\'{i}n-Mart\'{i}n]{robomimic2021}
Mandlekar, A., Xu, D., Wong, J., Nasiriany, S., Wang, C., Kulkarni, R., Fei-Fei, L., Savarese, S., Zhu, Y., and Mart\'{i}n-Mart\'{i}n, R.
\newblock What matters in learning from offline human demonstrations for robot manipulation.
\newblock In \emph{Conference on Robot Learning (CoRL)}, 2021.

\bibitem[Nakamoto et~al.(2023)Nakamoto, Zhai, Singh, Ma, Finn, Kumar, and Levine]{nakamoto2023cal}
Nakamoto, M., Zhai, Y., Singh, A., Ma, Y., Finn, C., Kumar, A., and Levine, S.
\newblock {Cal-QL}: Calibrated offline rl pre-training for efficient online fine-tuning.
\newblock In \emph{The 37th Conference on Neural Information Processing Systems}, 2023.

\bibitem[Piot et~al.(2014)Piot, Geist, and Pietquin]{piot2014boosted}
Piot, B., Geist, M., and Pietquin, O.
\newblock Boosted and reward-regularized classification for apprenticeship learning.
\newblock In \emph{Proceedings of the 13th International Conference on Autonomous Agents and Multi-Agent Systems}, pp.\  1249–1256, 2014.

\bibitem[Pomerleau(1988)]{pomerleau1988alvinn}
Pomerleau, D.~A.
\newblock {ALVINN}: An autonomous land vehicle in a neural network.
\newblock In \emph{Advances in Neural Information Processing Systems}, volume~1, pp.\  305--313. Morgan Kaufmann, 1988.

\bibitem[Rajaraman et~al.(2020)Rajaraman, Yang, Jiao, and Ramchandran]{rajaraman2020toward}
Rajaraman, N., Yang, L., Jiao, J., and Ramchandran, K.
\newblock Toward the fundamental limits of imitation learning.
\newblock In \emph{Advances in Neural Information Processing Systems}, volume~33, pp.\  2914--2924. Curran Associates, 2020.

\bibitem[Rajeswaran et~al.(2017)Rajeswaran, Kumar, Gupta, Vezzani, Schulman, Todorov, and Levine]{rajeswaran2017learning}
Rajeswaran, A., Kumar, V., Gupta, A., Vezzani, G., Schulman, J., Todorov, E., and Levine, S.
\newblock Learning complex dexterous manipulation with deep reinforcement learning and demonstrations.
\newblock \emph{arXiv preprint arXiv:1709.10087}, 2017.

\bibitem[Rashidinejad et~al.(2021)Rashidinejad, Zhu, Ma, Jiao, and Russell]{rashidinejad2021bridging}
Rashidinejad, P., Zhu, B., Ma, C., Jiao, J., and Russell, S.
\newblock Bridging offline reinforcement learning and imitation learning: A tale of pessimism.
\newblock In \emph{Advances in Neural Information Processing Systems}, volume~34, pp.\  11702--11716. Curran Associates, 2021.

\bibitem[Sasaki \& Yamashina(2021)Sasaki and Yamashina]{sasaki2021behavioral}
Sasaki, F. and Yamashina, R.
\newblock Behavioral cloning from noisy demonstrations.
\newblock In \emph{International Conference on Learning Representations}, 2021.

\bibitem[Sutton \& Barto(2018)Sutton and Barto]{sutton2018reinforcement}
Sutton, R.~S. and Barto, A.~G.
\newblock \emph{Reinforcement learning: An introduction}.
\newblock MIT press, 2018.

\bibitem[Sutton et~al.(1999)Sutton, Precup, and Singh]{sutton1999between}
Sutton, R.~S., Precup, D., and Singh, S.
\newblock Between mdps and semi-mdps: A framework for temporal abstraction in reinforcement learning.
\newblock \emph{Artificial intelligence}, 112\penalty0 (1-2):\penalty0 181--211, 1999.

\bibitem[Swamy et~al.(2021)Swamy, Choudhury, Bagnell, and Wu]{swamy2021moments}
Swamy, G., Choudhury, S., Bagnell, J.~A., and Wu, S.
\newblock Of moments and matching: A game-theoretic framework for closing the imitation gap.
\newblock In \emph{Proceedings of the 38th International Conference on Machine Learning}, volume 139, pp.\  10022--10032. PMLR, 2021.

\bibitem[Watson et~al.(2024)Watson, Huang, and Heess]{watson2023coherent}
Watson, J., Huang, S., and Heess, N.
\newblock Coherent soft imitation learning.
\newblock In \emph{Advances in Neural Information Processing Systems}, volume~36, pp.\  14540--14583. Curran Associates, 2024.

\bibitem[Wu et~al.(2019)Wu, Charoenphakdee, Bao, Tangkaratt, and Sugiyama]{wu2019imitation}
Wu, Y.-H., Charoenphakdee, N., Bao, H., Tangkaratt, V., and Sugiyama, M.
\newblock Imitation learning from imperfect demonstration.
\newblock In Chaudhuri, K. and Salakhutdinov, R. (eds.), \emph{Proceedings of the 36th International Conference on Machine Learning}, volume~97, pp.\  6818--6827. PMLR, 2019.

\bibitem[Xu et~al.(2022)Xu, Zhan, Yin, and Qin]{xu2022discriminator}
Xu, H., Zhan, X., Yin, H., and Qin, H.
\newblock Discriminator-weighted offline imitation learning from suboptimal demonstrations.
\newblock In \emph{Proceedings of the 39th International Conference on Machine Learning}, volume 162, pp.\  24725--24742. PMLR, 2022.

\bibitem[Xu et~al.(2021)Xu, Li, Yu, and Luo]{xu2021generalization}
Xu, T., Li, Z., Yu, Y., and Luo, Z.-Q.
\newblock On generalization of adversarial imitation learning and beyond.
\newblock \emph{arXiv preprint arXiv:2106.10424}, 2021.

\bibitem[Yu et~al.(2022)Yu, Kumar, Chebotar, Hausman, Finn, and Levine]{yu2022how}
Yu, T., Kumar, A., Chebotar, Y., Hausman, K., Finn, C., and Levine, S.
\newblock How to leverage unlabeled data in offline reinforcement learning.
\newblock In \emph{Proceedings of the 39th International Conference on Machine Learning}, volume 162, pp.\  25611--25635. PMLR, 2022.

\bibitem[Yue et~al.(2023)Yue, Wang, Shao, Zhang, Lin, Ren, and Zhang]{yue2023clare}
Yue, S., Wang, G., Shao, W., Zhang, Z., Lin, S., Ren, J., and Zhang, J.
\newblock {CLARE}: Conservative model-based reward learning for offline inverse reinforcement learning.
\newblock In \emph{International Conference on Learning Representations}, 2023.

\bibitem[Yue et~al.(2024)Yue, Deng, Wang, Ren, and Zhang]{yue2024federated}
Yue, S., Deng, Y., Wang, G., Ren, J., and Zhang, Y.
\newblock Federated offline reinforcement learning with proximal policy evaluation.
\newblock \emph{Chinese Journal of Electronics}, 33\penalty0 (6):\penalty0 1--13, 2024.

\bibitem[Zeng et~al.(2022)Zeng, Li, Garcia, and Hong]{zeng2022maximum}
Zeng, S., Li, C., Garcia, A., and Hong, M.
\newblock Maximum-likelihood inverse reinforcement learning with finite-time guarantees.
\newblock In \emph{Advances in Neural Information Processing Systems}, volume~35, pp.\  10122--10135. Curran Associates, 2022.

\bibitem[Zeng et~al.(2023)Zeng, Li, Garcia, and Hong]{zeng2023demonstrations}
Zeng, S., Li, C., Garcia, A., and Hong, M.
\newblock When demonstrations meet generative world models: A maximum likelihood framework for offline inverse reinforcement learning.
\newblock In \emph{The 37th Conference on Neural Information Processing Systems}, 2023.

\bibitem[Zolna et~al.(2020)Zolna, Novikov, Konyushkova, Gulcehre, Wang, Aytar, Denil, de~Freitas, and Reed]{zolna2020offline}
Zolna, K., Novikov, A., Konyushkova, K., Gulcehre, C., Wang, Z., Aytar, Y., Denil, M., de~Freitas, N., and Reed, S.
\newblock Offline learning from demonstrations and unlabeled experience.
\newblock In \emph{NeurIPS Workshop on Offline Reinforcement Learning}, 2020.

\end{thebibliography}
